\documentclass{article}
\usepackage{mydef}

\newcommand{\agd}{\textsc{Agd}\xspace}
\newcommand{\mbsgd}{\textsc{Mb-Sgd}\xspace}
\newcommand{\mbacsgd}{\textsc{Mb-Ac-Sgd}\xspace}
\newcommand{\fedac}{\textsc{FedAc}\xspace}
\newcommand{\fedacfull}{Federated Accelerated Stochastic Gradient Descent\xspace}
\newcommand{\fedavg}{\textsc{FedAvg}\xspace}
\newcommand{\fedaci}{\textsc{FedAc-I}\xspace}
\newcommand{\fedacii}{\textsc{FedAc-II}\xspace}

\title{\fedacfull}

\author{Honglin Yuan\thanks{Stanford University, E-mail:
    \texttt{yuanhl@cs.stanford.edu}}~ \and 
Tengyu Ma\thanks{Stanford University, E-mail: \texttt{tengyuma@stanford.edu}}}
\date{}

\begin{document}

\maketitle

\begin{abstract}
    We propose \fedacfull (\fedac), a principled acceleration of Federated Averaging (\fedavg, also known as Local SGD) for distributed optimization.
    \fedac is the first provable acceleration of \fedavg that improves convergence speed and communication efficiency on various types of convex functions.
    For example, for strongly convex and smooth functions, when using $M$ workers, the previous state-of-the-art \fedavg analysis can achieve a linear speedup in $M$ if given $\tildeo(M)$ rounds of synchronization, whereas \fedac only requires $\tildeo(M^{\frac{1}{3}})$ rounds.  
    Moreover, we prove stronger guarantees for \fedac when the objectives are third-order smooth. 
    Our technique is based on a potential-based perturbed iterate analysis, a novel stability analysis of generalized accelerated SGD, and a strategic tradeoff between acceleration and stability.
\end{abstract}
\section{Introduction}
Leveraging distributed computing resources and decentralized data is crucial, if not necessary, for large-scale machine learning applications. 
Communication is usually the major bottleneck for parallelization in both data-center settings and cross-device federated settings \citep{Kairouz.McMahan.ea-arXiv19}.

We study the distributed stochastic optimization $\min_{w \in \reals^d} F(w) :=  \expt_{\xi \sim \mathcal{D}} f(w; \xi) $ where $F$ is convex.
We assume there are $M$ parallel workers and each worker can access $F$ at $w$ via oracle $\nabla f(w; \xi)$ for independent sample $\xi$ drawn from distribution $\mathcal{D}$. 
We assume synchronization (communication) among workers is allowed but limited to $R$ rounds. We denote $T$ as the parallel runtime.

One of the most common and well-studied algorithms for this setting is \emph{Federated Averaging} (\fedavg) \citep{McMahan.Moore.ea-AISTATS17}, also known as Local SGD or Parallel SGD  \citep{Mangasarian-95,Zinkevich.Weimer.ea-NIPS10,Coppola-14,Zhou.Cong-IJCAI18} in the literature.\footnote{In the literature, \fedavg usually runs on a randomly sampled subset of heterogeneous workers for each synchronization round, whereas Local SGD or Parallel SGD usually run on a fixed set of workers. In this paper we do not differentiate the terminology and assumed a fixed set of workers are deployed for simplicity.} 
In \fedavg, each worker runs a local thread of SGD \citep{Robbins.Monro-51}, 
and periodically synchronizes with other workers by collecting the averages and broadcast to all workers.
The analysis of \fedavg \citep{Stich-ICLR19,Stich.Karimireddy-arXiv19,Khaled.Mishchenko.ea-AISTATS20,Woodworth.Patel.ea-ICML20} usually follows the perturbed iterate analysis framework \citep{Mania.Pan.ea-SIOPT17} where the performance of \fedavg is compared with the idealized version with infinite synchronization. The key idea is to control the stability of SGD so that the local  iterates held by parallel workers do not differ much, even with infrequent synchronization.

We study the acceleration of \fedavg and investigate whether it is possible to improve convergence speed and communication efficiency. 
The main challenge for introducing acceleration 
lies in the disaccord of acceleration and stability. 
Stability is essential for analyzing distributed algorithms such as \fedavg, whereas momentum applied for acceleration may amplify the instability of the algorithm. 
Indeed, we show that standard Nesterov accelerated gradient descent algorithm \citep{Nesterov-18} \emph{may not be initial-value stable even for smooth and strongly convex functions}, in the sense that the initial infinitesimal difference may grow exponentially fast (see \cref{instability:sketch}). 
This evidence necessitates a more scrutinized acceleration in distributed settings.

We propose a principled acceleration for \fedavg, namely \emph{\fedacfull} (\fedac), which provably improves convergence rate and communication efficiency. 
Our result extends the results of \citet{Woodworth.Patel.ea-ICML20} on \textsc{Local-Ac-Sa} for quadratic objectives to broader objectives.
To the best of our knowledge, this is the \textbf{first provable acceleration} of \fedavg (and its variants) for general or strongly convex objectives.
\fedac parallelizes a generalized version of Accelerated SGD \citep{Ghadimi.Lan-SIOPT12}, while we carefully 
balance the acceleration-stability tradeoff to accommodate distributed settings.
Under standard assumptions on smoothness, bounded variance, and strong convexity (see \cref{asm1} for details), \fedac converges at rate $\tildeo ( \frac{1}{M T} + \frac{1}{TR^3} )$.\footnote{We hide varaibles other than $T, M, R$ for simplicity. The complete bound can be found in \cref{tab:conv:rate} and the corresponding theorems.}
The bound will be dominated by $\tildeo(\frac{1}{MT})$ for $R$ as low as $\tildeo(M^{\frac{1}{3}})$, which implies the synchronization $R$ required for linear speedup in $M$ is $\tildeo(M^{\frac{1}{3}})$.\footnote{
    ``Synchronization required for linear speedup'' is a simple and common measure of the communication efficiency, which can be derived from the raw convergence rate. It is defined as the minimum number of synchronization $R$, as a function of number of workers $M$ and parallel runtime $T$, required to achieve a linear speed up --- the parallel runtime of $M$ workers is equal to the $\nicefrac{1}{M}$ fraction of a sequential single worker runtime.}
In comparison, the state-of-the-art \fedavg analysis \citet{Khaled.Mishchenko.ea-AISTATS20} showed that \fedavg converges at rate $\tildeo ( \frac{1}{M T} + \frac{1}{TR} )$, which requires $\tildeo(M)$ synchronization for linear speedup.
For general convex objective, \fedac converges at rate $\tildeo(\frac{1}{\sqrt{MT}} + \frac{1}{T^{\frac{1}{3}} R^{\frac{2}{3}}})$, which outperforms both state-of-the-art \fedavg $\tildeo(\frac{1}{\sqrt{MT}} + \frac{1}{T^{\frac{1}{3}} R^{\frac{1}{3}}})$ by \citeauthor{Woodworth.Patel.ea-ICML20} and Minibatch-SGD baseline ${\Theta}(\frac{1}{\sqrt{MT}} + \frac{1}{R})$ \citep{Dekel.Gilad-Bachrach.ea-JMLR12}.\footnote{
    Minibatch-SGD baseline corresponds to running SGD for $R$ steps with batch size $MT/R$, which can be implemented on $M$ parallel workers with $R$ communication and each worker queries $T$ gradients in total.
    }
We summarize communication bounds and convergence rates in \cref{tab:sync:bound,tab:conv:rate} (on the row marked A\ref{asm1}).

\begin{table}
    \caption{\textbf{Summary of results on the synchronization rounds $R$ required for linear speedup in $M$.} All bounds hide multiplicative $\polylog$ factors and variables other than $M$ and $T$ for ease of presentation. Notation: $M$: number of workers; $T$: parallel runtime.
    }
    \label{tab:sync:bound}
    \centering
    \resizebox{\linewidth}{!}{
    \begin{tabular}{lllll}
        \toprule
                    &           & \multicolumn{2}{c}{Synchronization Required for Linear Speedup}
        \\ \cmidrule(r){3-4}
        Assumption        & Algorithm      & Strongly Convex                                   & General Convex                                                                & Reference
        \\
        \midrule
        \cref{asm1} & \fedavg   & $T^{\frac{1}{2}} M^{\frac{1}{2}}$                 & --                                                                            & \citep{Stich-ICLR19}
        \\
                    &           & $T^{\frac{1}{3}} M^{\frac{1}{3}}$                 & --                                                                            & \citep{Haddadpour.Kamani.ea-NeurIPS19}
        \\
                    &           & $M$        & $T^{\frac{1}{2}} M^{\frac{3}{2}}$                                             & \citep{Stich.Karimireddy-arXiv19}
        \\
                    &           & $M$                                               & $T^{\frac{1}{2}} M^{\frac{3}{2}}$                                             & \citep{Khaled.Mishchenko.ea-AISTATS20}
        \\
                    & \fedac  & $\bm{M^{\frac{1}{3}}}$                            & $\min \{ \bm{ T^{\frac{1}{4}} M^{\frac{3}{4}}},\bm{T^{\frac{1}{3}} M^{\frac{2}{3}}}\}$ & 
                    \textbf{\cref{fedac:a1,fedaci:a1:gcvx,fedacii:a1:gcvx}}
        \\
        \midrule
        \cref{asm2} & \fedavg   & $\max\{ \bm{T^{-\frac{1}{2}} M^{\frac{1}{2}}},\bm{1}\}$    & ${T^{\frac{1}{2}} M^{\frac{3}{2}}}$                                           & 
                    \textbf{\cref{fedavg:a2,fedavg:a2:gcvx}}
        \\
                    & \fedac  & $\max\{ \bm{T^{-\frac{1}{6}} M^{\frac{1}{6}}}, \bm{1}\}$    & $\max\{ \bm{T^{\frac{1}{4}} M^{\frac{1}{4}}}, \bm{T^{\frac{1}{6}} M^{\frac{1}{2}}}\}$   & 
                    \textbf{\cref{fedacii:a2,fedacii:a2:gcvx}}
        \\
        \bottomrule
    \end{tabular}
    }
\end{table}

Our results suggest an \textbf{intriguing synergy between acceleration and parallelization}.
In the single-worker sequential setting, the convergence is usually dominated by the term related to stochasticity,
which is in general not possible to be accelerated \citep{Nemirovski.Yudin-83}.
In distributed settings, the communication efficiency is dominated by the overhead caused by infrequent synchronization, which can be accelerated as we show in the convergence rates summary \cref{tab:conv:rate}. 

\begin{table}
    \caption{\textbf{Summary of results on convergence rates.} All bounds omit multiplicative $\polylog$ factors and additive exponential decaying term (for strongly convex objective) for ease of presentation. 
    Notation: $D_0$: $\|w_0 - w^*\|$; $M$: number of workers; $T$: parallel runtime; $R$: synchronization; $\mu$: strong convexity; $L$: smoothness; $Q$: \nth{3}-order-smoothness (in \cref{asm2}).
    }
    \label{tab:conv:rate}
    \centering
    \resizebox{\linewidth}{!}{
    \begin{tabular}{llll}
        \toprule
        Assumption              & Algorithm                     & Convergence Rate $(\expt[F(\cdot)] - F^* \leq \cdots $) & Reference
        \\
        \midrule
        A\ref{asm1}($\mu > 0$) & \fedavg                       &
        exp. decay 
        $+ \frac{\sigma^2}{\mu M T} +  \frac{L \sigma^2}{\mu^2 T R} $
                               & \citep{Woodworth.Patel.ea-ICML20}
        \\
                               & \fedac                        &
        exp. decay $+ \frac{\sigma^2}{\mu M T} +  
        \min \left\{ \frac{L \sigma^2}{\mu^2 T R^2}, \frac{L^2 \sigma^2}{\mu^3 T R^3} \right\}$
                               & \textbf{\cref{fedac:a1}}
        \\
        \midrule
        A\ref{asm2}($\mu > 0$) & \fedavg                       &
        exp. decay $+ \frac{\sigma^2}{\mu M T} + \frac{Q^2 \sigma^4}{\mu^5 T^2 R^2}$  
                               & \textbf{\cref{fedavg:a2}}
        \\
                               & \fedac                        &
        exp. decay $+ \frac{\sigma^2}{\mu MT} 
        + {\frac{Q^2 \sigma^4}{\mu^5 T^2 R^6}}$
                               & \textbf{\cref{fedacii:a2}}
        \\
        \midrule
        A\ref{asm1}($\mu = 0$) & \fedavg                       &
        $\frac{LD_0^2}{T} + \frac{\sigma D_0}{\sqrt{MT}} + \frac{L^{\frac{1}{3}} \sigma^{\frac{2}{3}} D_0^{\frac{4}{3}}}{T^{\frac{1}{3}} R^{\frac{1}{3}}}$
                               & \citep{Woodworth.Patel.ea-ICML20}
        \\
                               & \fedac                        &
        $\frac{LD_0^2}{TR} + \frac{\sigma D_0}{\sqrt{MT}} + \min \left\{\frac{L^{\frac{1}{3}} \sigma^{\frac{2}{3}} D_0^{\frac{4}{3}}}{T^{\frac{1}{3}} R^{\frac{2}{3}}}, \frac{L^{\frac{1}{2}} \sigma^{\frac{1}{2}} D_0^{\frac{3}{2}}}{T^{\frac{1}{4}} R^{\frac{3}{4}}} \right\} $
                               & \textbf{\cref{fedaci:a1:gcvx,fedacii:a1:gcvx}}
        \\
        \midrule
        A\ref{asm2}($\mu = 0$) & \fedavg                       &
        $\frac{LD_0^2}{T} + \frac{\sigma D_0}{\sqrt{MT}} + \frac{Q^{\frac{1}{3}} \sigma^{\frac{2}{3}} D_0^{\frac{5}{3}}}{T^{\frac{1}{3}} R^{\frac{1}{3}}}$
                               & \textbf{\cref{fedavg:a2:gcvx}}
        \\
                               & \fedac                        &
        $\frac{LD_0^2}{TR} + \frac{\sigma D_0}{\sqrt{MT}} + \frac{L^{\frac{1}{3}} \sigma^{\frac{2}{3}} D_0^{\frac{4}{3}}}{M^{\frac{1}{3}} T^{\frac{1}{3}} R^{\frac{2}{3}}} +
            \frac{Q^{\frac{1}{3}} \sigma^{\frac{2}{3}} D_0^{\frac{5}{3}}}{T^{\frac{1}{3}} R}$
                               & \textbf{\cref{fedacii:a2:gcvx}}
        \\
        \bottomrule
    \end{tabular}
    }
\end{table}
We establish \textbf{stronger guarantees for \fedac when objectives are \nth{3}-order-smooth}, or ``close to be quadratic'' intuitively (see \cref{asm2} for details).
For strongly convex objectives, \fedac converges at rate $\tildeo (\frac{1}{MT} + \frac{1}{T^2 R^6})$ (see \cref{fedacii:a2}).
We also prove the convergence rates of \fedavg in this setting for comparison.
We summarize our results in \cref{tab:sync:bound,tab:conv:rate} (on the row marked A\ref{tab:conv:rate}).

We empirically verify the efficiency of \fedac in \cref{sec:experiment}. Numerical results suggest a considerable improvement of \fedac over all three baselines, namely \fedavg, (distributed) Minibatch-SGD, and (distributed) Accelerated Minibatch-SGD \citep{Dekel.Gilad-Bachrach.ea-JMLR12,Cotter.Shamir.ea-NeurIPS11}, especially in the regime of highly infrequent communication and abundant workers.
\subsection{Related work}
The analysis of \fedavg (a.k.a. Local SGD) is an active area of research. 
Early research on \fedavg mostly focused on the particular case of $R = 1$, also known as ``one-shot averaging'', where the iterates are only averaged once at the end of procedure \citep{Mcdonald.Mohri.ea-NIPS09,Zinkevich.Weimer.ea-NIPS10,Zhang.Duchi.ea-13,Shamir.Srebro-Allerton14,Rosenblatt.Nadler-16}.
The first convergence result on \fedavg with general (more than one) synchronization for convex objectives was established by \citet{Stich-ICLR19} under the assumption of uniformly bounded gradients.
\citet{Stich.Karimireddy-arXiv19,Haddadpour.Kamani.ea-NeurIPS19,Dieuleveut.Patel-NeurIPS19,Khaled.Mishchenko.ea-AISTATS20} relaxed this requirement and studied \fedavg under assumptions similar to our \cref{asm1}. 
These works also attained better rates than \citep{Stich-ICLR19} through an improved stability analysis of SGD.
However, recent work \citep{Woodworth.Patel.ea-ICML20} showed that all the above bounds on \fedavg are strictly dominated by minibatch SGD \citep{Dekel.Gilad-Bachrach.ea-JMLR12} baseline. 
\citet{Woodworth.Patel.ea-ICML20} provided the first bound for \fedavg that can improve over minibatch SGD for certain cases.
This is to our knowledge the state-of-the-art bound for \fedavg and its variants.
Our \fedac uniformly dominates this bound on \fedavg.

The specialty of quadratic objectives for better communication efficiency has been studied in an array of contexts \citep{Zhang.Duchi.ea-JMLR15,Jain.Kakade.ea-18}. \citet{Woodworth.Patel.ea-ICML20} studied an acceleration of \fedavg but was limited to quadratic objectives.
More generally, \citet{Dieuleveut.Patel-NeurIPS19} studied the convergence of \fedavg under bounded \nth{3}-derivative, but the bounds are still dominated by minibatch SGD baseline \citep{Woodworth.Patel.ea-ICML20}. 
Recent work by \citet{Godichon-Baggioni.Saadane-20} studied one-shot averaging under similar assumptions. 
Our analysis on \fedavg (\cref{fedavg:a2}) allows for general $R$ and reduces to a comparable bound if $R = 1$, which is further improved by our analysis on \fedac (\cref{fedacii:a2}).

\fedavg has also been studied in other more general settings. A series of recent papers (\eg, \citep{Zhou.Cong-IJCAI18,Haddadpour.Kamani.ea-ICML19,Wang.Joshi-19,Yu.Jin-ICML19,Yu.Jin.ea-ICML19,Yu.Yang.ea-AAAI19}) studied the convergence of \fedavg for non-convex objectives. 
We conjecture that \fedac can be generalized to non-convex objectives to attain better efficiency by combining our result with recent non-convex acceleration algorithms (\emph{e.g.}, \citep{Carmon.Duchi.ea-SIOPT18}).
Numerous recent papers \citep{
    Khaled.Mishchenko.ea-AISTATS20,Li.Huang.ea-ICLR20,Haddadpour.Mahdavi-arXiv19,Koloskova.Loizou.ea-ICML20} studied \fedavg in heterogeneous settings, where each worker has access to stochastic gradient oracles from different distributions.
Other variants of \fedavg have been proposed in the face of heterogeneity \citep{Pathak.Wainwright-NeurIPS20,Li.Sahu.ea-MLSys20,Karimireddy.Kale.ea-ICML20,Wang.Tantia.ea-ICLR20}.
We defer the analysis of \fedac for heterogeneous settings for future work. 
Other techniques, such as quantization, can also reduce communication cost \citep{Alistarh.Grubic.ea-NIPS17,Wen.Xu.ea-NIPS17,Stich.Cordonnier.ea-NeurIPS18,Basu.Data.ea-NeurIPS19,Mishchenko.Gorbunov.ea-arXiv19,Reisizadeh.Mokhtari.ea-AISTATS20}.
We refer readers to \citep{Kairouz.McMahan.ea-arXiv19} for a more comprehensive survey of the recent development of algorithms in Federated Learning.

Stability is one of the major topics in machine learning and has been studied for a variety of purposes \citep{Yu.Kumbier-20}. For example, \citet{Bousquet.Elisseeff-JMLR02,Hardt.Recht.ea-ICML16} showed that algorithmic stability can be used to establish generalization bounds.
\citet{Chen.Jin.ea-arXiv18} provided the stability bound of standard Accelerated Gradient Descent (AGD) for \emph{quadratic objectives}. 
To the best of our knowledge, there is no existing (positive or negative) result on the stability of AGD for general convex or strongly convex objectives. 
This work provides the first (negative) result on the stability of standard deterministic AGD, which suggests that standard AGD may not be initial-value stable even for strongly convex and smooth objectives (\cref{instability:sketch}).\footnote{
We construct the counterexample for initial-value stability for simplicity and clarity. 
We conjecture that our counterexample also extends to other algorithmic stability notions (\eg, uniform stability \citep{Bousquet.Elisseeff-JMLR02}) since initial-value stability is usually milder than the others.
} This result may be of broader interest.
The tradeoff technique of \fedac also provides a possible remedy to mitigate the instability issue, which may be applied to derive better generalization bounds for momentum-based methods.

The stochastic optimization problem $\min_{w \in \reals^d} F(w) :=  \expt_{\xi \sim \mathcal{D}} f(w; \xi)$ we consider in this paper is commonly referred to as the \emph{stochastic approximation} (SA) problem in the literature \citep{Kushner.Yin.ea-03}. Another related question is the \emph{sample average approximation} (SAA), also known as \emph{empirical risk minimization} (ERM) problem \citep{Vapnik-98}. The ERM problem is defined as $\min_{w \in \reals^d} F(w) :=  \frac{1}{N} \sum_{i=1}^N f(w; \xi^{(i)})$, where the sum of a fixed finite set of objectives is to be optimized. 
For strongly-convex ERM, it is possible to leverage variance reduction techniques \citep{Johnson.Zhang-13} to attain linear convergence. For example,  the Distributed Accelerated SVRG (DA-SVRG) \citep{Lee.Lin.ea-JMLR17} can attain expected $\varepsilon$-optimality within $\tilde{\mathcal{O}}(\frac{N}{M} \log(1/\varepsilon) )$ parallel runtime and $\tilde{\mathcal{O}}( \log(1/\varepsilon))$ rounds of communication. 
If we were to apply \fedac for ERM, it can attain expected $\varepsilon$-optimality with 
$\tilde{\mathcal{O}}(\frac{1}{M \varepsilon})$ parallel runtime and $\tilde{\mathcal{O}}(M^{\frac{1}{3}})$ rounds of communication (assuming \cref{asm1} is satisfied).
Therefore one can obtain low accuracy solution with \fedac in a short parallel runtime, whereas DA-SVRG may be preferred if high accuracy is required and $N$ is relatively small. 
It is worth mentioning that \fedac is not designed or proved for the distributed ERM setting, and we include this rough comparison for completeness. 
We conjecture that \fedac can be incorporated with appropriate variance reduction techniques to attain better performance in distributed ERM setting.

\section{Preliminaries}
\subsection{Assumptions}
We conduct our analysis on \fedac in two settings with two sets of assumptions.
The following \cref{asm1} consists of a set of standard assumptions: convexity, smoothness and bounded variance.
Comparable assumptions are assumed in existing studies on \fedavg \citep{Haddadpour.Kamani.ea-NeurIPS19,Stich.Karimireddy-arXiv19,Khaled.Mishchenko.ea-AISTATS20,Woodworth.Patel.ea-ICML20}.\footnote{In fact, \citet{Woodworth.Patel.ea-ICML20} imposes the same assumption in \cref{asm1}; \citet{Khaled.Mishchenko.ea-AISTATS20} assumes $f(w; \xi)$ are convex and smooth for all $\xi$, which is more restricted; \citet{Stich.Karimireddy-arXiv19} assumes quasi-convexity instead of convexity; \citet{Haddadpour.Kamani.ea-NeurIPS19} assumes P-\L\ condition instead of strong convexity. In this work we focus on standard (general or strong) convexity to simplify the analysis.}
\begin{assumption}[$\mu$-strong convexity, $L$-smoothness and $\sigma^2$-uniformly bounded gradient variance] \label{asm1}
    \;
    \begin{enumerate}
        \item[(a)] $F$ is $\mu$-strongly convex, i.e., $F(u) \geq F(w) + \langle \nabla F(w), u - w  \rangle + \frac{1}{2} \mu \|u - w\|^2$ for any $u, w \in \reals^d$. In addition, assume $F$ attains a finite optimum $w^* \in \reals^d$.
        (We will study both the strongly convex case $(\mu > 0)$ and the general convex case $(\mu = 0)$, which will be clarified in the context.)
        \item [(b)]  $F$ is $L$-smooth, i.e., $F(u) \leq F(w) + \langle \nabla F(w), u - w  \rangle + \frac{1}{2} L \|u-w\|^2$ for any $u, w \in \reals^d$.
        \item [(c)]  $\nabla f(w; \xi)$ has $\sigma^2$-bounded variance, i.e., $\sup_{w \in \reals^d}  \expt_{\xi \in \mathcal{D}} \| \nabla f(w; \xi) - \nabla F(w)\|^2 \leq \sigma^2$.
    \end{enumerate}
\end{assumption}
The following \cref{asm2} consists of an additional set of assumptions: \nth{3} order smoothness and bounded $\nth{4}$ central moment.
A similar version 
of \cref{asm2} 
was studied in \citep{Dieuleveut.Patel-NeurIPS19}.\footnote{In fact, \citep{Dieuleveut.Patel-NeurIPS19} assumes bounded \nth{4} central moment at optimum only. We adopt the uniformly bounded \nth{4} central moment for consistency with \cref{asm1}.}
\begin{assumption}\label{asm2}     In addition to \cref{asm1}, assume that
    \begin{enumerate}
        \item[(a)] $F$ is $Q$-\nth{3}-order-smooth, i.e., $F(u) \leq F(w) + \langle \nabla F(w), u-w \rangle +  \frac{1}{2} \langle \nabla^2 F(w) (u-w), (u-w) \rangle + \frac{1}{6} Q \|u-w\|^3$ for any $u, w \in \reals^d$.
        \item[(b)] $\nabla f(w; \xi)$ has $\sigma^4$-bounded \nth{4} central moment, i.e, $\sup_{w \in \reals^d} \expt_{\xi \in \mathcal{D}} \| \nabla f(w; \xi) - \nabla F(w)\|^4 \leq \sigma^4$.
    \end{enumerate}
\end{assumption}

\subsection{Notations}
We use $\|\cdot\|$  to denote the operator norm of a matrix or the $\ell_2$-norm of a vector, $[n]$ to denote the set $\{1, 2, \ldots, n\}$. 
Let $w^{*}$ be the optimum of $F$ and denote $F^*:=F(w^*)$. Let $D_0 := \|w_0 - w^*\|$.
be the Euclidean distance of the initial guess $w_0$ and the optimum $w^*$.
For both \fedac and \fedavg, we use $M$ to denote the number of parallel workers, $R$ to denote synchronization rounds, $K$ to denote the synchronization interval (i.e., the number of local steps per synchronization round), and $T = KR$ to denote the parallel runtime. 
We use the subscript to denote timestep, italicized superscript to denote the index of worker and unitalicized superscript ``md'' or ``ag'' to denote modifier of iterates in \fedac (see definition in \cref{alg:fedac}).
We use overline to denote averaging over all workers, \eg, $\overline{w_t^{\mathrm{ag}}} := \frac{1}{M} \sum_{m=1}^M w_t^{\mathrm{ag}, m}$. We use $\tildeo, \tilde{\Theta}$ to hide multiplicative $\polylog$ factors, which will be clarified in the formal context. 
\section{Main results}
\subsection{Main algorithm: Federated Accelerated Stochastic Gradient Descent (\fedac)} 
We formally introduce our algorithm \fedac in \cref{alg:fedac}. \fedac parallelizes a generalized version of Accelerated SGD by \citet{Ghadimi.Lan-SIOPT12}. 
In \fedac, each worker $m \in [M]$ maintains three intertwined sequences $\{w_t^m, w_t^{\mathrm{ag}, m}, w_t^{\mathrm{md},m}\}$ at each step $t$.
Here $w_t^{\mathrm{ag}, m}$ aggregates the past iterates, $w_t^{\mathrm{md},m}$ is the auxiliary sequence of ``middle points'' on which the gradients are queried, and $w_t^{m}$ is the main sequence of iterates. 
At each step, candidate next iterates $v_{t+1}^{\mathrm{ag}, m}$ and $v_{t+1}^m$ are computed. If this is a local (unsynchronized) step, they will be assigned to the next iterates $w_{t+1}^{\mathrm{ag}, m}$ and $w_{t+1}^{\mathrm{ag}, m}$. Otherwise, they will be collected, averaged, and broadcast to all the workers.
\begin{algorithm}
    \caption{\fedacfull (\fedac)}
    \begin{algorithmic}[1]
        \label{alg:fedac}
        \Procedure{\fedac }{$\alpha, \beta, \eta, \gamma$}  \Comment{See \cref{fedaci,fedacii} for hyperparameter choices}
        \State Initialize $w_0^{\mathrm{ag}, m} = w_0^{m} = w_0$ for all $m \in [M]$
        \For{$t = 0, \ldots, T-1$}
        \For{every worker $m\in [M]$ {\bf in parallel}}
        \State $w_t^{\mathrm{md}, m} \gets \beta^{-1} w_t^m + (1 - \beta^{-1}) w_t^{\mathrm{ag}, m}$
        \Comment{Compute $w_t^{\mathrm{md,m}}$ by coupling}
        \State $g_t^m \gets \nabla f(w_t^{\mathrm{md}, m} ; \xi_t^m)$ 
        \Comment{Query gradient at $w_t^{\mathrm{md}, m}$}
        \State $v_{t+1}^{\mathrm{ag}, m} \gets w_t^{\mathrm{md}, m} - \eta \cdot g_t^m$
        \Comment{Compute next iterate candidate $v_{t+1}^{\mathrm{ag}, m}$}
        \State $v_{t+1}^m \gets (1 - \alpha^{-1}) w_t^m + \alpha^{-1} w_t^{\mathrm{md}, m} - \gamma \cdot g_t^m$
        \Comment{Compute next iterate candidate $v_{t+1}^{m}$}
        \If{ sync (i.e., $t~\textup{mod}~K=-1$)}
        \State $w_{t+1}^m \gets \frac{1}{M} \sum_{m'=1}^M v_{t+1}^{m'}$; \quad $w_{t+1}^{\mathrm{ag},m} \gets \frac{1}{M} \sum_{m'=1}^M v_{t+1}^{\mathrm{ag},m'}$
        \Comment{Average \& broadcast}
        \Else
        \State $w_{t+1}^m \gets  v_{t+1}^m$; \quad $w_{t+1}^{\mathrm{ag},m} \gets v_{t+1}^{\mathrm{ag},m}$
        \Comment{Candidates assigned to be the next iterates }
        \EndIf
        \EndFor
        \EndFor
        \EndProcedure
    \end{algorithmic}
\end{algorithm}
\paragraph{Hyperparameter choice.}
We note that the particular version of Accelerated SGD in \fedac is more flexible than the most standard Nesterov version \citep{Nesterov-18}, as it has four hyperparameters instead of two. 
Our analysis suggests that this flexibility seems crucial for principled acceleration in the distributed setting to allow for acceleration-stability trade-off. 

However, we note that our theoretical analysis gives a very concrete choice of hyperparameter $\alpha, \beta$, and $\gamma$ in terms of $\eta$.
For $\mu$-strongly-convex objectives, we introduce the following two sets of hyperparameter choices, which are referred to as \fedaci and \fedacii, respectively.
As we will see in the \cref{sec:thm:a1}, under \cref{asm1},
\fedaci has a better dependency on condition number $\nicefrac{L}{\mu}$, whereas \fedacii has better communication efficiency.
\begin{alignat}{5}
     & \text{\fedaci}: \quad
     &                        & \eta \in \left(0, \frac{1}{L}\right], \quad &  & \gamma = \max \left\{ \sqrt{\frac{\eta}{\mu K}} , \eta\right\}, \quad &  & \alpha  = \frac{1}{\gamma \mu},  \quad                &  & \beta = \alpha + 1;
    \label{fedaci}
    \\
     & \text{\fedacii}: \quad
     &                        & \eta \in \left(0, \frac{1}{L}\right], \quad &  & \gamma = \max \left\{ \sqrt{\frac{\eta}{\mu K}}, \eta \right\}, \quad &  & \alpha  = \frac{3}{2 \gamma \mu} - \frac{1}{2}, \quad &  & \beta = \frac{2 \alpha^2 - 1}{\alpha - 1}.
    \label{fedacii}
\end{alignat}

Therefore, practically, if the strong convexity estimate $\mu$ is given (which is often taken to be the $\ell_2$ regularization strength), the only hyperparameter to be tuned is $\eta$, whose optimal value depends on the problem parameters.

\subsection{Theorems on the convergence for strongly convex objectives}
Now we present main theorems of \fedac for strongly convex objectives under \cref{asm1} or \ref{asm2}.
\subsubsection{Convergence of \fedac under \cref{asm1}}
\label{sec:thm:a1}

We first introduce the convergence theorem on \fedac under \cref{asm1}. \fedaci and \fedacii lead to slightly different convergence rates. 
\begin{theorem}[Convergence of \fedac]
    \label{fedac:a1}
    Let $F$ be $\mu > 0$-strongly convex, and assume \cref{asm1}.
    \begin{enumerate}
        \item [(a)] (Full version see \cref{fedaci:full}) For $\eta = \min \{ \frac{1}{L}, \tilde{\Theta} (  \frac{1}{\mu TR}) \}$, \fedaci yields
              \begin{equation}
                  \expt \left[ F( {\overline{w_{T}^{\mathrm{ag}}}})  - F^*  \right] \leq
                  \exp \left( \min \left\{ -\frac{\mu T}{L}, - \sqrt{\frac{\mu T R}{L}} \right\} \right)
                  L D_0^2 +  \tildeo\left(\frac{\sigma^2}{\mu M T} + \frac{L \sigma^2} {\mu^2 T R^2}\right).
                  \label{eq:fedac:a1:1}
              \end{equation}
        \item [(b)] (Full version see \cref{fedacii:a1:full}) For $\eta = \min \{ \frac{1}{L}, \tilde{\Theta} (  \frac{1}{\mu T R} ) \}$, \fedacii yields
              \begin{equation}
                  \expt \left[  F( {\overline{w_{T}^{\mathrm{ag}}}})  - F^*  \right] \leq
                  \exp \left( \min \left\{ -\frac{\mu T}{3L}, - \sqrt{\frac{\mu T R}{9L}}  \right\} \right)
                  L D_0^2 +
                  \tildeo \left( \frac{\sigma^2}{\mu M T} +  \frac{L^2 \sigma^2}{\mu^3 T R^3} \right).
                  \label{eq:fedac:a1:2}
              \end{equation}
    \end{enumerate}
\end{theorem}

In comparison, the state-of-the-art \fedavg analysis \citep{Khaled.Mishchenko.ea-AISTATS20,Woodworth.Patel.ea-ICML20} reveals the following result.\footnote{\cref{fedavg:a1} can be (easily) adapted from the Theorem 2 of \citep{Woodworth.Patel.ea-ICML20} which analyzes a decaying learning rate with convergence rate $\mathcal{O}\left( \frac{L^2D_0^2}{\mu T^2}  + \frac{\sigma^2}{\mu M T} \right) +  \tildeo \left( \frac{L \sigma^2}{\mu^2 T R} \right)$. This bound has no $\log$ factor attached to $\frac{\sigma^2}{\mu M T}$ term but worse (polynomial) dependency on initial state $D_0$ than \cref{fedavg:a1}. We present \cref{fedavg:a1} for consistency and the ease of comparison.}
\begin{proposition}[Convergence of \fedavg under \cref{asm1}, adapted from \citeauthor{Woodworth.Patel.ea-ICML20}]
    \label{fedavg:a1}
    In the settings of \cref{fedac:a1}, for $\eta = \min \{ \frac{1}{L}, \tilde{\Theta} ( \frac{1}{\mu T}  ) \}$, for appropriate non-negative $\{\rho_t\}_{t=0}^{T-1}$ with $\sum_{t=0}^{T-1} \rho_t = 1$, $\fedavg$ yields
        \begin{equation}
            \expt \left[ F \left( \sum_{t=0}^{T-1} \rho_t \overline{w_t} \right) - F^* \right]
            \leq \exp \left( - \frac{\mu T}{L} \right) L D_0^2 +
            \tildeo \left(\frac{\sigma^2}{\mu M T} + \frac{L \sigma^2}{\mu^2 T R} \right).
            \label{eq:fedavg:main}
        \end{equation}
\end{proposition}
\begin{remark}
The bound for \fedaci \eqref{eq:fedac:a1:1} \textbf{asymptotically universally outperforms} \fedavg \eqref{eq:fedavg:main}.
The first term in \eqref{eq:fedac:a1:1} cooresponds to the deterministic convergence, which is better than the one for \textsc{FedAvg}. 
The second term corresponds to the stochasticity of the problem which is not improvable. 
The third term corresponds to the overhead of infrequent communication, which is also better than \textsc{FedAvg} due to acceleration.
On the other hand, \fedacii has better communication efficiency since the third term of \eqref{eq:fedac:a1:2} decays at rate $R^{-3}$.
\end{remark}

\subsubsection{Convergence of \fedac under \cref{asm2} --- faster when close to be quadratic}
We establish stronger guarantees for \fedacii \eqref{fedacii} under \cref{asm2}.
\begin{theorem}[Simplified version of \cref{fedacii:a2:full}]
    \label{fedacii:a2}
    Let $F$ be $\mu > 0$-strongly convex, and assume \cref{asm2}, then 
    for $R \geq \sqrt{\frac{L}{\mu}}$,\footnote{The assumption $R \geq \sqrt{{L}/{\mu}}$ is removed in the full version (\cref{fedacii:a2:full}).} 
    for $\eta = \min \{ \frac{1}{L}, \tilde{\Theta} (  \frac{1}{\mu T R} ) \}$, \fedacii yields
        \begin{equation}
            \expt \left[  F( {\overline{w_{T}^{\mathrm{ag}}}})  - F^*  \right]
            \leq  \exp \left( \min \left\{ -\frac{\mu T}{3L}, - \sqrt{\frac{\mu T R}{9L}} \right\} \right)
            2L D_0^2 + \tildeo\left(\frac{\sigma^2}{\mu M T} 
            + \frac{Q^2 \sigma^4}{\mu^5 T^2 R^6}\right).
            \label{eq:fedacii:a2}
        \end{equation}
\end{theorem}
In comparison, we also establish and prove the convergence rate of \fedavg under \cref{asm2}. 
\begin{theorem}[Simplified version of \cref{fedavg:a2:full}]
    \label{fedavg:a2}
    In the settings of \cref{fedacii:a2}, for $\eta = \min \left\{ \frac{1}{4L}, \tilde{\Theta}\left( \frac{1}{\mu T} \right)\right\}$, for appropriate non-negative $\{\rho_t\}_{t=0}^{T-1}$ with $\sum_{t=0}^{T-1} \rho_t = 1$, \fedavg yields
    \begin{small}
    \begin{equation}
        \expt \left[ F \left( \sum_{t=0}^{T-1} {\rho_t} \overline{w_t} \right) - F^* \right] \leq \exp \left( - \frac{\mu T}{8L} \right) 4L D_0^2 + \tildeo \left( \frac{\sigma^2}{\mu M T} + \frac{Q^2 \sigma^4}{\mu^5 T^2 R^2} \right).
        \label{eq:fedavg:a2} 
    \end{equation}
    \end{small}
\end{theorem}
\begin{remark}
Our results give a smooth interpolation of the results of \citep{Woodworth.Patel.ea-ICML20} for quadratic objectives to broader function class --- the third term regarding infrequent communication overhead will vanish when the objective is quadratic since $Q = 0$.
The bound of \fedac \eqref{eq:fedacii:a2} outperforms the bound of \fedavg \eqref{eq:fedavg:a2} as long as $R \geq \sqrt{L/\mu}$ holds. 
Particularly in the case of $T \geq M$, our analysis suggests that only $\tildeo(1)$ synchronization are required for linear speedup in $M$.
We summarize our results on synchronization bounds and convergence rate in \cref{tab:sync:bound,tab:conv:rate}, respectively.
\end{remark}

\subsection{Convergence for general convex objectives}
We also study the convergence of \fedac for general convex objectives ($\mu = 0$). 
The idea is to apply \fedac to $\ell_2$-augmented objective $\tilde{F}_{\lambda}(w) := F(w) + \frac{\lambda}{2}\|w - w_0\|^2$ as a $\lambda$-strongly-convex and $(L + \lambda)$-smooth objective for appropriate $\lambda$, which is similar to the technique of \citep{Woodworth.Patel.ea-ICML20}. 
This augmented technique allows us to reuse most of the analysis for strongly-convex objectives. 
We conjecture that it is possible to construct direct versions of \fedac for general convex objectives that attain the same rates, which we defer for the future work.
We summarize the synchronization bounds in \cref{tab:sync:bound} and the convergence rates in \cref{tab:conv:rate}. We defer the statement of formal theorems to \cref{sec:gcvx} in Appendix.

\section{Proof sketch}
In this section we sketch the proof for two of our main results, namely \cref{fedac:a1}(a) and \ref{fedacii:a2}.
We focus on the proof of \cref{fedac:a1}(a) to outline our proof framework, and then illustrate the difference in the proof of \cref{fedacii:a2}.
\subsection{Proof sketch of \cref{fedac:a1}(a): \fedaci under \cref{asm1} }
\label{sec:proof:sketch:1}
Our proof framework consists of the following four steps.

\paragraph{Step 1: potential-based perturbed iterate analysis.} 
The first step is to study the difference between \fedac and its fully synchronized idealization, namely the case of $K=1$ (recall $K$ denotes the number of local steps).
To this end, we extend the perturbed iterate analysis \citep{Mania.Pan.ea-SIOPT17} to potential-based setting to analyze accelerated convergence.
For \fedaci, we study the \emph{decentralized} potential $\Psi_t := \frac{1}{M} \sum_{m=1}^M F( {w_{t}^{\mathrm{ag}, m}})  - F^* + \frac{1}{2}\mu \|\overline{w_{t}} - w^*\|^2$ and establish the following lemma.
$\Psi_t$ is adapted from the common potential for acceleration analysis \citep{Bansal.Gupta-19}.
\begin{lemma}[Simplified version of \cref{fedaci:conv:main}, Potential-based perturbed iterate analysis for \fedaci]
  \label{fedaci:conv:sketch}
  In the same settings of \cref{fedac:a1}(a), the following inequality holds
    \begin{align}
       \expt \left[\Psi_T \right]    &      \leq
      \exp \left( - \gamma \mu T \right)  \Psi_0 + \frac{\eta^2 L \sigma^2}{2\gamma \mu} + \frac{\gamma\sigma^2}{2M} \tag{\text{Convergence rate in the case of $K=1$}}
      \\
      &+ \underbrace{ L \cdot \max_{0 \leq t < T}  \expt
      \left[\frac{1}{M} \sum_{m=1}^M \left\| \overline{w_t^{\mathrm{md}}} - w_t^{\mathrm{md}, m}  \right\|
      \left\| \frac{1}{1 + \gamma \mu} (\overline{w_t} - w_t^m) + \frac{\gamma \mu}{1 + \gamma \mu}   (\overline{w_t^{\mathrm{ag}}} - w_t^{\mathrm{ag},m})  \right\|  \right]}_{\text{Discrepancy overhead}}.
      \label{eq:fedaci:conv:sketch}
    \end{align}
\end{lemma}
We refer to the last term of \eqref{eq:fedaci:conv:sketch} as ``discrepancy overhead'' since it characterizes the dissimilarities among workers due to infrequent synchronization. The proof of \cref{fedaci:conv:sketch} is deferred to \cref{sec:fedaci:conv:main}.

\paragraph{Step 2: bounding discrepancy overhead.}
The second step is to bound the discrepancy overhead in \eqref{eq:fedaci:conv:sketch} via stability analysis.
Before we look into \fedac, let us first review the intuition for \fedavg. 
There are two forces governing the growth of discrepancy of \fedavg, namely the (negative) gradient and stochasticity. 
Thanks to the convexity, the gradient only makes the discrepancy lower.
The stochasticity incurs $\mathcal{O}(\eta^2 \sigma^2)$ variance per step, so the discrepancy 
$\expt [\frac{1}{M} \sum_{m=1}^M \|\overline{w_t} - w_t^m\|^2 ]$ 
grows at rate $\mathcal{O}(\eta^2 K \sigma^2)$ linear in $K$. 
The detailed proof can be found in \citep{Khaled.Mishchenko.ea-AISTATS20,Woodworth.Patel.ea-ICML20}.

For \fedac, the discrepancy analysis is subtler since acceleration and stability are at odds --- the momentum may amplify the discrepancy accumulated from previous steps.
Indeed, we establish the following \cref{instability:sketch}, which shows that the \emph{standard deterministic} Accelerated GD (\agd) may \emph{not} be initial-value stable even for strongly convex and smooth objectives, in the sense that initial infinitesimal difference may grow exponentially fast. 
We defer the formal setup and the proof of \cref{instability:sketch} to \cref{sec:instability} in Appendix.
\begin{theorem}[Initial-value instability of deterministic standard \agd]
  \label{instability:sketch}
  For any $L, \mu > 0$ such that $\nicefrac{L}{\mu} \geq 25$, and for any $K \geq 1$, there exists a 1D objective $F$ that is $L$-smooth and $\mu$-strongly-convex, and an $\varepsilon_0 > 0$, such that for any positive $\varepsilon < \varepsilon_0$, there exists initialization $w_0, u_0, w_0^{\mathrm{ag}}, u_0^{\mathrm{ag}}$ such that $|w_0 - u_0| \leq \varepsilon$, $|w_0^{\mathrm{ag}} - u_0^{\mathrm{ag}}| \leq \varepsilon$, but
  the trajectories $\{w_t^{\mathrm{ag}}, w_t^{\mathrm{md}}, w_t\}_{t=0}^{3K}$, $\{u_t^{\mathrm{ag}}, u_t^{\mathrm{md}}, u_t\}_{t=0}^{3K}$ generated by applying deterministic \agd with initialization $(w_0, w_0^{\mathrm{ag}})$ and $(u_0, u_0^{\mathrm{ag}})$ satisfies
  \begin{equation}
      |w_{3K} - u_{3K}| \geq \frac{1}{2} \varepsilon (1.02)^K, 
      \qquad
      |w^{\mathrm{ag}}_{3K} - u^{\mathrm{ag}}_{3K}| \geq \varepsilon (1.02)^K.
  \end{equation}
\end{theorem} 
\begin{remark}
  It is worth mentioning that the instability theorem \textbf{does not contradicts the convergence} of \agd \citep{Nesterov-18}. The convergence of \agd suggests that $w_{t}^{\mathrm{ag}}$, $w_t$, $u_t^{\mathrm{ag}}$, and $u_t$ will all converge to the same point $w^*$ as $t \to \infty$, which implies $\lim_{t \to \infty} \|w_t^{\mathrm{ag}} - u_t^{\mathrm{ag}}\| = \|w_t - u_t\| = 0$. However, the convergence theorem does not imply the stability with respect to the initialization --- it does not exclude the possibility that the difference between two instances (possibly with very close initialization) first expand and only shrink until they both approach $w^*$. Our \cref{instability:sketch} suggests this possibility: for any finite steps, no matter how small the (positive) initial difference is, it is possible that the difference will grow exponentially fast. 
This is fundamentally different from the Gradient Descent (for convex objectives), for which the difference between two instances does not expand for standard choice of learning rate $\eta = \frac{1}{L}$ (where $L$ is the smoothness).
\end{remark}

Fortunately, we can show that the discrepancy can grow at a slower exponential rate via less aggressive acceleration, see \cref{fedaci:stab:sketch}. As we will discuss shortly, we adjust $\gamma$ according to $K$ to restrain the growth of discrepancy within the linear regime. The proof of \cref{fedaci:stab:sketch} is deferred to \cref{sec:fedaci:stab:main}.
\begin{lemma}[Simplified version of \cref{fedaci:stab:main}, Discrepancy overhead bounds for \fedaci]
  \label{fedaci:stab:sketch}
  In the same setting of \cref{fedac:a1}(a), the following inequality holds
    \begin{equation}
      \text{``Discrepancy overhead'' in \cref{eq:fedaci:conv:sketch}} \leq \begin{cases}
        7 \eta \gamma L K \sigma^2 \left(1 + \frac{2\gamma^2\mu}{\eta}\right)^{2K}
         & \text{if~} \gamma \in (\eta, \sqrt{\frac{\eta}{\mu}} ],
        \\
        7 \eta^2 L K \sigma^2
         &
        \text{if~} \gamma = \eta.
      \end{cases}
      \label{eq:fedaci:stab:sketch}
    \end{equation}
\end{lemma}
\paragraph{Step 3: trading-off acceleration and discrepancy.} 
Combining \cref{fedaci:conv:sketch,fedaci:stab:sketch} gives 
  \begin{equation}
    \label{fedaci:combine:sketch}
    \expt \left[\Psi_T \right]
    \leq \underbrace{\exp \left( - \gamma \mu T \right)  \Psi_0}_{\mathrm{(I)}}
    + \frac{\eta^2 L \sigma^2}{2\gamma \mu}
    + \frac{\gamma\sigma^2}{2M}
    +
    \underbrace{
      \begin{cases}
        7 \eta \gamma L K \sigma^2 \left(1 + \frac{2\gamma^2\mu}{\eta}\right)^{2K}
         & \text{if~} \gamma \in (\eta, \sqrt{ \frac{\eta}{\mu}} ],
        \\
        7 \eta^2 L K \sigma^2
         &
        \text{if~} \gamma = \eta.
      \end{cases}
    }_{\mathrm{(II)}}
  \end{equation}

The value of $\gamma \in [\eta, \sqrt{\eta/\mu}]$ controls the magnitude of acceleration in (I) and discrepancy growth in (II).
The upper bound choice $\sqrt{\eta/\mu}$ gives full acceleration in (I) but makes (II) grow exponentially in $K$.
On the other hand, the lower bound choice $\eta$ makes (II) linear in $K$ but loses all acceleration.
We wish to attain as much acceleration in (I) as possible while keeping the discrepancy (II) grow moderately. \textbf{Our balanced solution} is to pick $\gamma = \max \{ \sqrt{{\eta}/{(\mu K)}}, \eta\}$. One can verify that the discrepancy grows (at most) linearly in $K$.
Substituting this choice of $\gamma$ to \cref{fedaci:combine:sketch} leads to
  \begin{equation}
    \expt \left[ \Psi_T \right]
    \leq
    \underbrace{\exp \left(  \min \left\{ -\eta \mu T, - \frac{\eta^{\frac{1}{2}} \mu^{\frac{1}{2}} T}{K^{\frac{1}{2}}} \right\} \right) \Psi_0}_{\text{Monotonically decreasing } \varphi_{\downarrow}(\eta)}
    +
    \underbrace{\mathcal{O} \left( \frac{\eta^{\frac{1}{2}} \sigma^2}{ \mu^{\frac{1}{2}} M K^{\frac{1}{2}} }
      + \frac{\eta \sigma^2}{M}
      + \frac{\eta^{\frac{3}{2}} L K^{\frac{1}{2}} \sigma^2} {\mu^{\frac{1}{2}}}
      + \eta^2 L K \sigma^2\right)}_{\text{Monotonically increasing } \varphi_{\uparrow}(\eta)}.
    \label{eq:fedaci:eta:sketch}
  \end{equation}
\paragraph{Step 4: finding $\eta$ to optimize the RHS of \cref{eq:fedaci:eta:sketch}.}
It remains to show that \eqref{eq:fedaci:eta:sketch} gives the desired bound with our choice of $\eta = \min \{ \frac{1}{L}, \tilde{\Theta} (\frac{K}{\mu T^2} )\}$.
The increasing $\varphi_{\uparrow}(\eta)$ in \eqref{eq:fedaci:eta:sketch} is bounded by
\(
\tildeo ( \frac{\sigma^2}{\mu M T} +  \frac{LK^2 \sigma^2}{\mu^2 T^3} ).
\)
The decreasing term $\varphi_{\downarrow}(\eta)$ in \eqref{eq:fedaci:eta:sketch} is bounded by $\varphi_{\downarrow}({\frac{1}{L}}) + \varphi_{\downarrow} (\tilde{\Theta} (  \frac{K}{\mu T^2} ) )$, where $    \varphi_{\downarrow} ({\frac{1}{L}} )
  = \exp ( \min \{ - \frac{\mu T}{L}, - \frac{\mu^{\frac{1}{2}}T}{L^{\frac{1}{2}}K^{\frac{1}{2}} }\} )$,
and
$\varphi_{\downarrow} (\tilde{\Theta} (  \frac{K}{\mu T^2} ) ) 
\leq
\exp \left( - \frac{\mu^{\frac{1}{2}} T}{K^{\frac{1}{2}}} \cdot \sqrt{\tilde{\Theta} \left(  \frac{K}{\mu T^2} \right)} \right)$ 
can be controlled by the bound of $\varphi_{\uparrow}(\eta)$ provided $\tilde{\Theta}$ has appropriate $\polylog$ factors.
Plugging the bounds to \eqref{eq:fedaci:eta:sketch} and 
replacing $K$ with $\nicefrac{T}{R}$ completes the proof of \cref{fedac:a1}(a).
We defer the details to \cref{sec:fedaci}.
\subsection{Proof sketch of \cref{fedacii:a2}: \fedacii under \cref{asm2}}
\label{sec:proof:sketch:2}
In this section, we outline the proof of \cref{fedacii:a2} by explaining the differences with the proof in \cref{sec:proof:sketch:1}.
The first difference is that for \fedacii we study an alternative \emph{centralized potential} $\Phi_t = F(\overline{w_t^{\mathrm{ag}}}) - F^* + \frac{1}{6}\mu \| \overline{w_t} - w^*\|^2$, which leads to an alternative version of \cref{fedaci:conv:sketch} as follows.
  \begin{equation}
    \expt [\Phi_T] \leq \exp \left( - \frac{\gamma \mu T}{3} \right) \Phi_0
    + \frac{3\eta^2 L \sigma^2 }{2 \gamma \mu M}
    + \frac{\gamma \sigma^2}{2 M}
    + \frac{3}{\mu} \max_{0 \leq t < T}   \expt \left\| \frac{1}{M} \sum_{m=1}^M \nabla F (w_t^{\mathrm{md}, m}) - \nabla F(\overline{w_t^{\mathrm{md}}})    \right\|^2 .
    \label{eq:fedacii:conv:sketch}
  \end{equation}
The second difference is that the particular discrepancy in \eqref{eq:fedacii:conv:sketch} can be bounded via \nth{3}-order smoothness $Q$ since
(we omit ``md'' and $t$ to simplify the notations)
\begin{align}
  & \left\|  \frac{1}{M} \sum_{m=1}^M \nabla F(w^m) - \nabla F (\overline{w})    \right\|^2 = \left\|  \frac{1}{M} \sum_{m=1}^M \left( \nabla F(w^m) - \nabla F (\overline{w}) - \nabla^2 F (\overline{w})(w^m - \overline{w_t}) \right)   \right\|^2
\\
\leq & \frac{1}{M} \sum_{m=1}^M\left\|   \nabla F(w^m) - \nabla F (\overline{w}) - \nabla^2 F (\overline{w})(w^m - \overline{w}) \right\|^2   \leq \frac{Q^2}{4M}  \sum_{m=1}^M\left\| w^m - \overline{w} \right\|^4.
\end{align}
The proof then follows by analyzing the \nth{4}-order stability of \fedac. We defer the details to \cref{sec:fedacii}.

\section{Numerical experiments}
\label{sec:experiment}
In this section, we validate our theory and demonstrate the efficiency of \fedac via experiments.\footnote{Code repository link: \url{https://github.com/hongliny/FedAc-NeurIPS20}.}
The performance of \fedac is tested against \fedavg (a.k.a., Local SGD), (distributed) Minibatch-SGD (\mbsgd) and Minibatch-Accelerated-SGD (\mbacsgd) \citep{Dekel.Gilad-Bachrach.ea-JMLR12,Cotter.Shamir.ea-NeurIPS11} on $\ell_2$-regularized logistic regression for UCI \texttt{a9a} dataset \citep{Dua.Graff-17} from \texttt{LibSVM} \citep{Chang.Lin-11}.
The regularization strength is set as $10^{-3}$. 
The hyperparameters $(\gamma, \alpha, \beta)$ of \fedac follows \fedaci where strong-convexity $\mu$ is chosen as regularization strength $10^{-3}$.
We test the settings of $M = 2^2, \ldots, 2^{13}$ workers and $K = 2^0, \ldots, 2^8$ synchronization interval.
For all four algorithms, we tune the learning-rate $\eta$ \emph{only} from the same set of levels within $[10^{-3}, 10]$.
We choose $\eta$ based on the best suboptimality (regularized population loss).
We claim that the best $\eta$ lies in the range $[10^{-3}, 10]$ for all algorithms under all settings. 
We defer the rest of setup details to \cref{additional:expr}.
In \cref{fig:a9a:1e-3:M}, we compare the algorithms by measuring the effect of linear speedup 
under variant $K$.
\begin{figure}[ht]
    \centering
    \includegraphics[width=\textwidth]{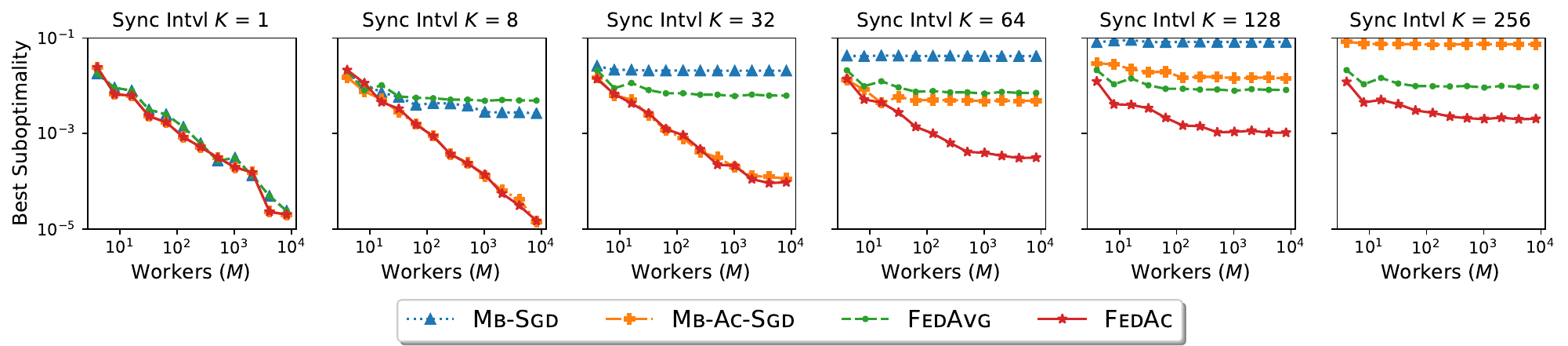}
    \caption{\textbf{Observed linear speedup with respect to the number of workers $M$ under various synchronization intervals $K$.} 
    Our \fedac is tested against three baselines \fedavg, \mbsgd, and \mbacsgd.
    While all four algorithms attain linear speedup for the fully synchronized ($K=1$) setting, \fedavg and \mbsgd lose linear speedup for $K$ as low as 8. 
    \mbacsgd is comparably better than the other two baselines but still deteriorates significantly for $K \geq 64$.
    \fedac is most robust to infrequent synchronization and outperforms the baselines by a margin for $K \geq 64$.
    }
    \label{fig:a9a:1e-3:M}
\end{figure}

In the next experiments, we provide an empirical example to show that the direct parallelization of standard accelerated SGD may indeed suffer from instability. 
This complements our \cref{instability:sketch}) on the initial-value instability of standard AGD. 
Recall that \fedaci \cref{fedaci} and \fedacii \cref{fedacii} adopt an acceleration-stability tradeoff technique that takes $\gamma = \max \{ \sqrt{\frac{\eta}{\mu K}} , \eta\}$. Formally, we denote the following direct acceleration of \fedac without such tradeoff as ``vanilla \fedac'':
$
    \eta \in (0, \frac{1}{L} ], \gamma =  \sqrt{\frac{\eta}{\mu}} , \alpha  = \frac{1}{\gamma \mu},   \beta = \alpha + 1.
$
In \cref{fig:instability}, we compare the vanilla \fedac with the (stable) \fedaci and the baseline \mbacsgd.
\begin{figure}[ht]
    \centering
    \includegraphics[width=\textwidth]{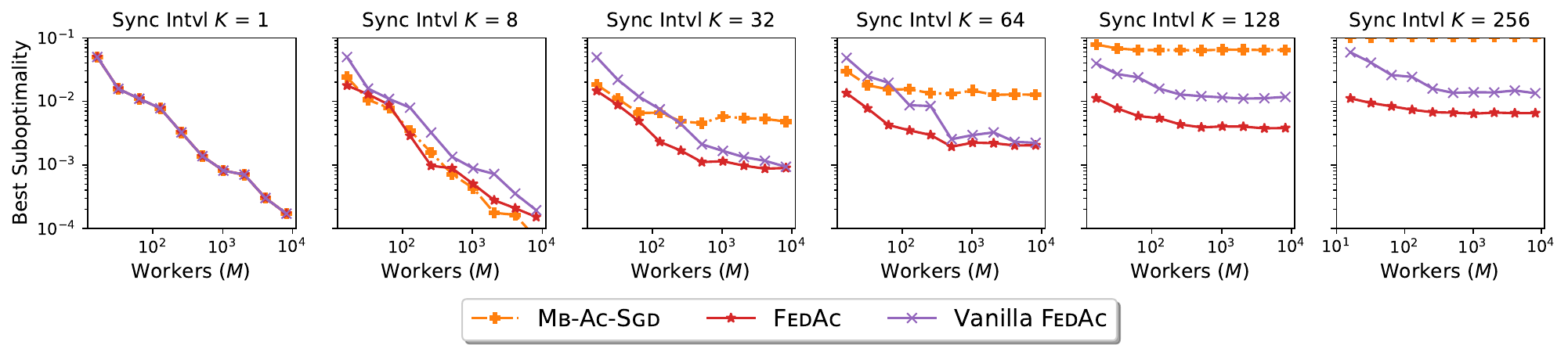}
    \caption{
        \textbf{Vanilla \fedac versus (stable) \fedaci and baseline \mbacsgd on the observed linear speedup w.r.t. $M$ under various synchronization intervals $K$.}
        Observet that Vanilla \fedac is indeed less robust to infrequent synchronization and thus worse than the \fedaci. (dataset: \textsc{a9a}, $\lambda=10^{-4}$)}
    \label{fig:instability}
\end{figure}

We include more experiments on various dataset, and more detailed analysis in \cref{additional:expr}. 

\section{Conclusions}
This work proposes \fedac, a principled acceleration of \fedavg, which provably improves convergence speed and communication efficiency. 
Our theory and experiments suggest that \fedac saves runtime and reduces communication overhead, especially in the setting of abundant workers and infrequent communication. 
We establish stronger guarantees when the objectives are third-order smooth. 
As a by-product, we also study the stability property of accelerated gradient descent, which may be of broader interest.
We expect \fedac could be generalized to broader settings, \eg, non-convex objective and/or heterogenous workers.
\section*{Acknowledgements}
Honglin Yuan would like to thank the support by the Total Innovation Fellowship. 
Tengyu Ma would like to thank the support by the Google Faculty Award. 
The work is also partially supported by SDSI and SAIL. 
We would like to thank Qian Li, Junzi Zhang, and Yining Chen for helpful discussions at various stages of this work. 
We would like to thank the anonymous reviewers for their suggestions and comments.

\clearpage
\begin{small}
\bibliography{FedAc.bib}
\end{small}

\begin{appendices}
The appendices are structured as follows. In \cref{additional:expr}, we include additional experiments with description of setup details. 
In \cref{sec:fedaci,sec:fedacii}, we prove the complete version of \cref{fedac:a1,fedacii:a2} on the convergence of \fedac under \cref{asm1} or \ref{asm2}.
  In \cref{sec:fedavg}, we prove \cref{fedavg:a2} on the convergence of \fedavg under \cref{asm2}. In \cref{sec:gcvx}, we  prove the convergence of \fedac (and \fedavg) for general convex objectives. In \cref{sec:instability}, we prove \cref{instability:sketch}  on the initial-value instability of standard accelerated gradient descent. We include some helper lemmas in \cref{sec:helper}.

\begin{spacing}{1.2}
  \listofappendices
\end{spacing}
\clearpage

\section{Additional experiments and setup details}\label{additional:expr}
\subsection{Additional setup details}
\paragraph{Baselines.}
\fedac is tested against three baselines, namely \fedavg (a.k.a., Local SGD), (distributed) Minibatch-SGD (\mbsgd), and (distributed) Minibatch-Accelerated-SGD (\mbacsgd) \citep{Dekel.Gilad-Bachrach.ea-JMLR12,Cotter.Shamir.ea-NeurIPS11}.
We fix the parallel runtime $T = 4096$, and test variant levels of synchronization interval $K$ and parallel workers $M$.
\mbsgd and \mbacsgd baselines correspond to running SGD or accelerated SGD for $\nicefrac{T}{K}$ steps with batch size $MK$. 
The comparison is fair since all algorithms can be parallelized to $M$ workers with $\nicefrac{T}{K}$ rounds of communication where each worker queries $T$ gradients in total. 
We simulate the parallelization with a \texttt{NumPy} program on a local CPU cluster. 
We start from the same random initialization for all algorithms under all settings. 

\paragraph{Datasets.}
The algorithms are tested on $\ell_2$-regularized logistic regression on the following two binary classification datasets from \texttt{LibSVM}. The preprocessing information and the download links can be found at \url{https://www.csie.ntu.edu.tw/~cjlin/libsvmtools/datasets/binary.html}. 
\begin{enumerate}[leftmargin=*]
    \item The  ``adult'' \texttt{a9a} dataset with 123 features and 32,561 training samples from the UCI Machine Learning Repository \citep{Dua.Graff-17}.
    \item The \texttt{epsilon} dataset with 2,000 features and 400,000 training samples from the PASCAL Challenge 2008 \citep{Sonnenburg.Franc.ea-08}.
\end{enumerate}

\paragraph{Evaluation.}
For all algorithms and all settings, we evaluate the suboptimality (regularized population loss) every $512$ parallel timesteps (gradient queries). 
We compute the suboptimality by comparing with a pre-computed optimum $F^*$.
We record the best suboptimality attained over the evaluations.

\paragraph{Hyperparameter choice.}
For all four algorithms, we tune the ``learning-rate'' hyperparameter $\eta$ only and record the best suboptimality attained. For \mbacsgd, the rest of hyperparameters are determined by the strong-convexity estimate $\mu$ which is taken to be the $\ell_2$-regularization strength $\lambda$. 
For \fedac, the default choice is \fedaci \cref{fedaci},\footnote{\fedacii is qualitatively similar to \fedaci empirically so we show \fedaci only.} where the strong-convexity estimate $\mu$ is also taken to be the $\ell_2$-regularization strength $\lambda$.
\subsection{Results on dataset \texttt{a9a}}
We first test on the \texttt{a9a} dataset with  $\ell_2$-regularization strength $10^{-3}$.
We test the setting of $K = 2^0, \ldots, 2^8$ and $M = 2^2, \ldots, 2^{13}$.
For all algorithms, we tune $\eta$ from the same sets: \{0.001, 0.002, 0.005, 0.01, 0.02, 0.05, 0.1, 0.2, 0.5, 1, 2, 5, 10\}. We claim that the best $\eta$ lies in $[0.001, 10]$ for all algorithms for all settings.\footnote{We search for this range to guarantee that the optimal $\eta$ lies in this range for all algorithms and all settings. One could save effort in tuning if only one algorithm were implemented.}
We plot the observed linear speedup figure in \cref{fig:a9a:1e-3:M} in the main body. To better understand the dependency on synchronization intervals $K$, we plot the following \cref{fig:a9a:1e-3:K}. 
The results suggest that \fedac is more robust to infrequent synchronization and thus more communication-efficient. 
For example, when using 8192 workers, \fedac requries only 32 rounds of communication to attain $10^{-3}$ suboptimality, whereas \mbacsgd, \mbsgd and \fedavg require 128, 1024, 4096 rounds, respectively.
\begin{figure}[!hbp]
    \centering
    \includegraphics[width=\textwidth]{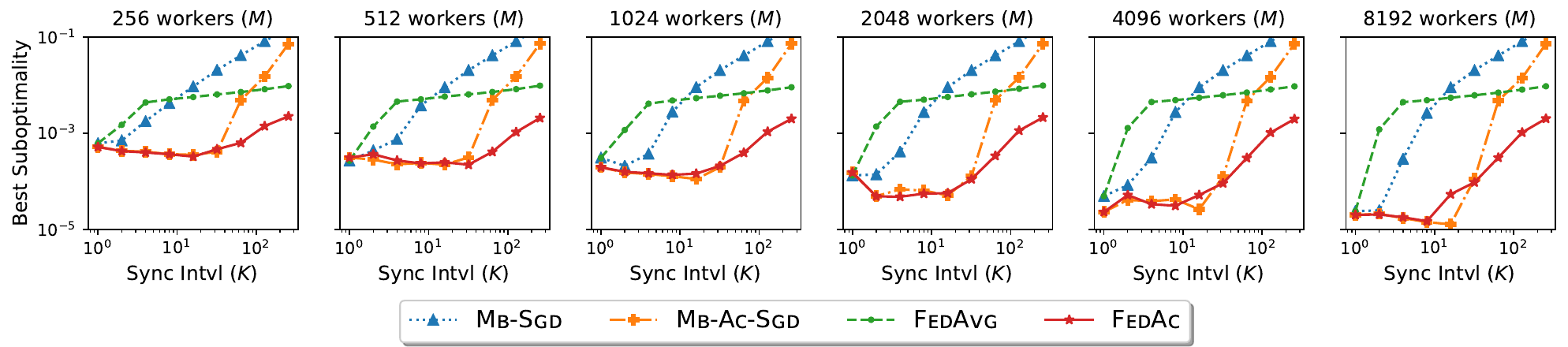}
    \caption{
        \textbf{\fedac versus baselines on the dependency of synchronization interval $K$ under various workers $M$.}
        For all tested $M$, \fedavg and \mbsgd start to deteriorate once $K$ passes $2$;
        \mbacsgd is more robust to moderate $K$ than  \fedavg and \mbsgd but sharply deteriorate once it passes a threshold at around $K=32$. 
        This is because \mbacsgd does not have enough gradient steps for convergence when the communication is too sparse. 
        In comparison, \fedac is more robust to infrequent communication.
        Dataset: \texttt{a9a}, $\ell_2$-regularization strength: $10^{-3}$.}
    \label{fig:a9a:1e-3:K}
\end{figure}

We repeat the experiments with an alternative choice of $\lambda = 10^{-2}$. This problem is relatively ``easier'' in terms of optimization since the condition number $\nicefrac{L}{\mu}$ is lower. We test the same levels of $M$, $K$ and tune the $\eta$ from the same set as above. The results are shown in \cref{fig:a9a:1e-2:M,fig:a9a:1e-2:K}. The results are qualitatively similar to the $\lambda = 10^{-3}$ case. 
For $K \leq 64$, the performance of \fedac and \mbacsgd are similar, which both outperform the other two baselines \fedavg and \mbsgd. For $K \geq 128$, the \mbacsgd drastically worsen because the gradient steps are too few, and \fedac outperforms the other baselines by a margin.

\begin{figure}[!hbp]
    \centering
    \includegraphics[width=\textwidth]{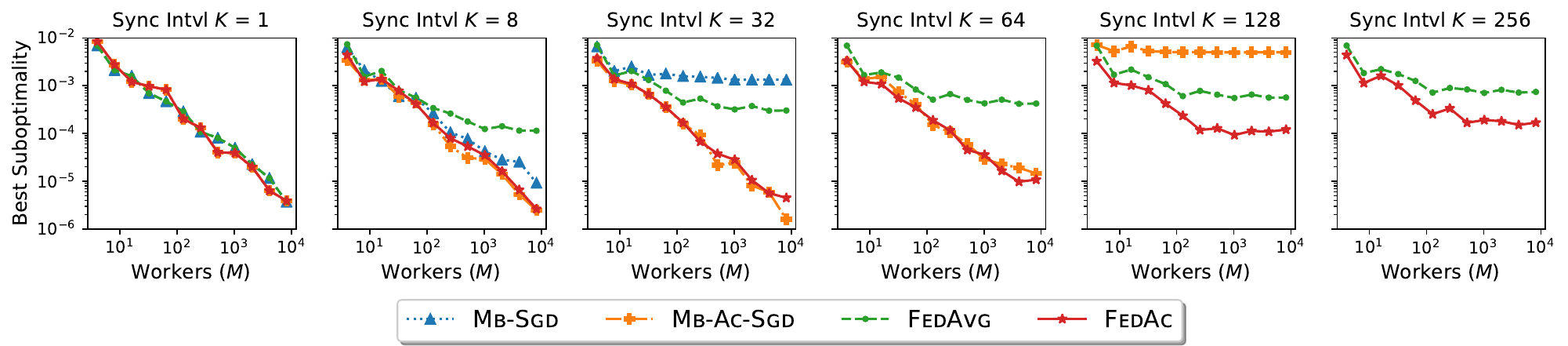}
    \caption{
        \textbf{\fedac versus baselines on the observed linear speedup w.r.t $M$ under various synchronization interval $K$.}
\        The results are qualitatively similar to \cref{fig:a9a:1e-3:M}. 
        Dataset: \texttt{a9a}, $\ell_2$-regularization strength: $10^{-2}$.}
    \label{fig:a9a:1e-2:M}
\end{figure}
\begin{figure}[!hbp]
    \centering
    \includegraphics[width=\textwidth]{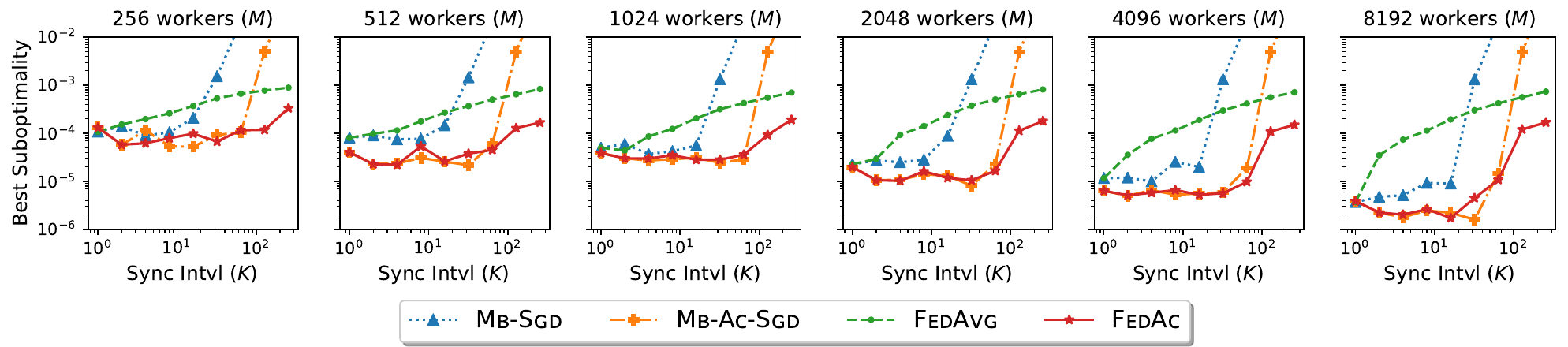}
    \caption{
        \textbf{\fedac versus baselines on the dependency of synchronization interval $K$ under various workers $M$.}
        The results are qualitatively similar to \cref{fig:a9a:1e-3:K}.
        Dataset: \texttt{a9a}, $\ell_2$-regularization strength: $10^{-2}$.}
    \label{fig:a9a:1e-2:K}
\end{figure}

\subsection{Results on dataset \texttt{epsilon}}
In this section we repeat the experiments above on the larger \texttt{epsilon} dataset with $\ell_2$-regularization $\lambda$ taken to be $10^{-4}$. 
 $\eta$ is tuned from $\{0.005, 0.01, 0.02, 0.05, 0.1, 0.2, 0.5, 1, 2, 5, 10, 20, 50\}$.
The optimal $\eta$ lies in the corresponding range for all algorithm under all tested settings. The results are shown in \cref{fig:epsilon:1e-4:M,fig:epsilon:1e-4:K}. 
The results are qualitatively similar to the previous experiments on \texttt{a9a} dataset.
\fedac is more communication-efficient than the baselines. 
For example, when using 2048 workers, \fedac requires only 64 rounds of communication (synchronization) to attain $10^{-4}$ suboptimality, whereas \mbacsgd, \mbsgd and \fedavg require 256, 4096 and 4096 rounds of communication, respectively.
\begin{figure}[!hbp]
    \centering
    \includegraphics[width=\textwidth]{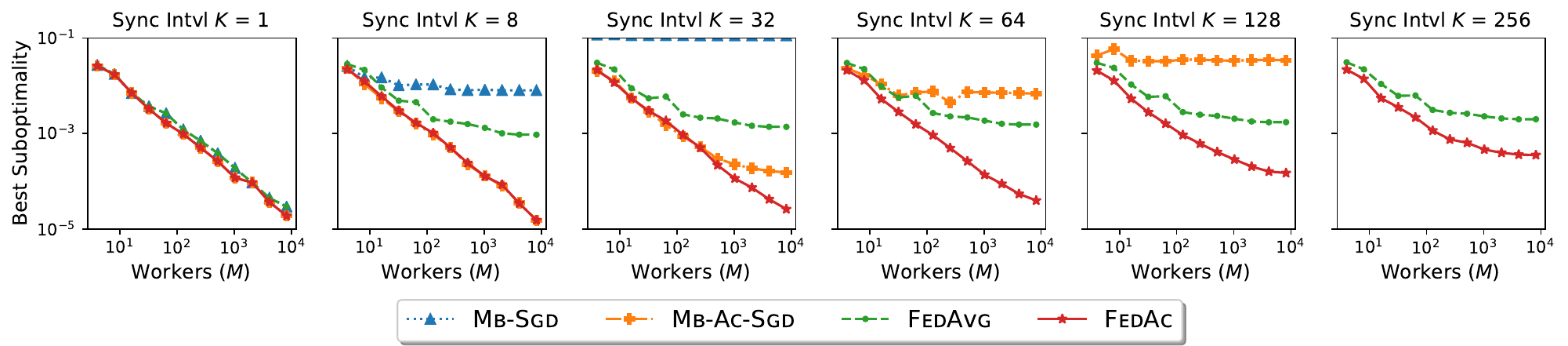}
    \caption{
        \textbf{\fedac versus baselines on the observed linear speedup w.r.t $M$ under various synchronization interval $K$.}
        The results are qualitatively similar to \cref{fig:a9a:1e-3:M}. 
        Dataset: \texttt{epsilon}, $\ell_2$-regularization strength: $10^{-4}$.}
    \label{fig:epsilon:1e-4:M}
\end{figure}
\begin{figure}[!hbp]
    \centering
    \includegraphics[width=\textwidth]{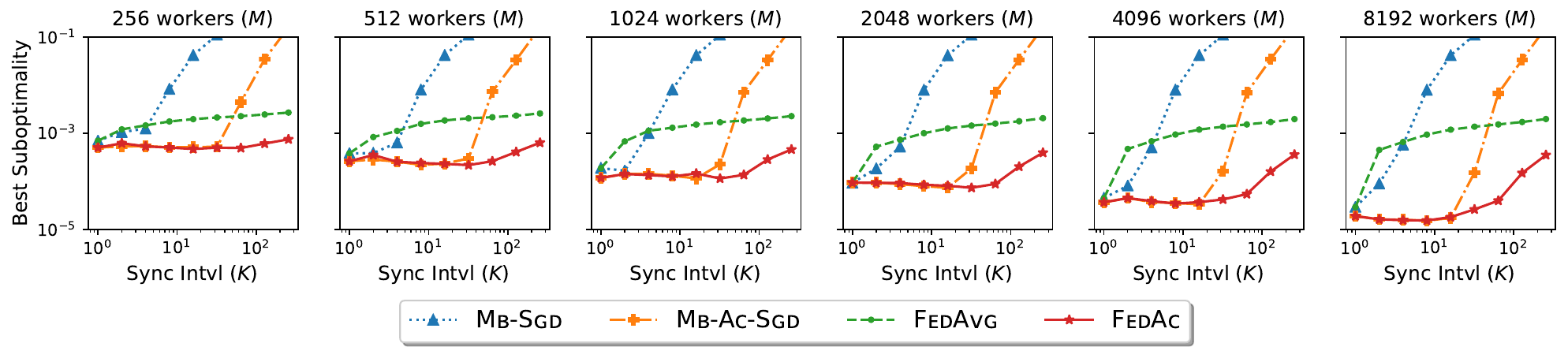}
    \caption{
        \textbf{\fedac versus baselines on the dependency of synchronization interval $K$ under various workers $M$.}
        The results are qualitatively similar to \cref{fig:a9a:1e-3:K}.
        Dataset: \texttt{epsilon}, $\ell_2$-regularization strength: $10^{-4}$.}
    \label{fig:epsilon:1e-4:K}
\end{figure}
\clearpage
\section{Analysis of \fedaci under \cref{asm1}}
\label{sec:fedaci}
In this section we study the convergence of \fedaci. We provide a complete, non-asymptotic version of \cref{fedac:a1}(a) on the convergence of \fedaci under \cref{asm1} and provide the detailed proof, which expands the proof sketch in \cref{sec:proof:sketch:1}. 
Recall that \fedaci is defined as the \fedac (\cref{alg:fedac}) with the following hyperparameters choice
\begin{equation}
  \eta \in \left( 0, \frac{1}{L} \right], \quad \gamma = \max \left\{ \sqrt{\frac{\eta}{\mu K}} , \eta\right\}, \quad   \alpha  = \frac{1}{\gamma \mu},  \quad \beta = \alpha + 1
  \tag{\fedaci}.
\end{equation}
We keep track of the convergence progress of \fedaci via the following decentralized potential $\Psi_t$.
\begin{equation}
  \Psi_t := \frac{1}{M} \sum_{m=1}^M F( {w_{t}^{\mathrm{ag}, m}})  - F^* + \frac{1}{2}\mu \|\overline{w_{t}} - w^*\|^2.
  \label{eq:fedaci:potential}
\end{equation}
Recall $\overline{w_{t}}$ is defined as $\frac{1}{M} \sum_{m=1}^M w_t^m$.
Formally, we use $\mathcal{F}_t$ to denote the $\sigma$-algebra generated by $\{w_\tau^m, w_\tau^{\mathrm{ag, m}}\}_{\tau \leq t, m \in [M]}$. Since \fedac is Markovian, conditioning on $\mathcal{F}_t$ is equivalent to conditioning on $\{w_t^m, w_t^{\mathrm{ag, m}}\}_{m \in [M]}$.

\subsection{Main theorem and lemmas: Complete version of \cref{fedac:a1}(a)}
Now we introduce the main theorem on the convergence of \fedaci.
\footnote{Note that we state our full \cref{fedaci:full} in terms of the synchronization gap $K$ instead of the synchronization round $R$ as in the simplified \cref{fedac:a1}(a). This two quantities are trivially related as $T = KR$.
  In fact, our bound \cref{fedaci:full} in terms of $K$ also holds for irregular synchronization setting as long as the maximum synchronization interval is bounded by $K$.}
\footnote{Throughout this paper we do not optimize the $\polylog$ factors or the constants. We conjecture that certain $\polylog$ factors can be improved or removed via averaging techniques such as \citep{Lacoste-Julien.Schmidt.ea-arXiv12,Stich-arXiv19}.}
\begin{theorem}[Convergence of \fedaci, complete version of \cref{fedac:a1}(a)]
  \label{fedaci:full}
  Let $F$ be $\mu>0$-strongly convex, and assume \cref{asm1}, then for
  \begin{equation}
    \eta = \min \left\{ \frac{1}{L},  \frac{K}{\mu T^2} \log^2 \left( \euler + \min \left\{ \frac{\mu M T \Psi_0}{\sigma^2}, \frac{\mu^2 T^3 \Psi_0}{L K^2 \sigma^2} \right\} \right) \right\},
  \end{equation}
  \fedaci yields
  \begin{align}
    \expt [\Psi_T]
    \leq &
    \min \left\{ \exp \left( - \frac{\mu T}{L} \right), \exp \left( - \frac{\mu^{\frac{1}{2}} T}{L^{\frac{1}{2}} K^{\frac{1}{2}}} \right)\right\}\Psi_0
    \\
    & + \frac{2\sigma^2}{\mu M T} \log^2 \left(  \euler + \frac{\mu M T \Psi_0}{\sigma^2} \right) 
    + \frac{400 LK^2 \sigma^2}{\mu^2 T^3} \log^4 \left( \euler + \frac{\mu^2 T^3 \Psi_0}{LK^2\sigma^2} \right),
  \end{align}
  where $\Psi_t$ is the decentralized potential defined in \cref{eq:fedaci:potential}.
\end{theorem}
\begin{remark}
  The simplified version \cref{fedac:a1}(a) in the main body can be obtained by replacing $K$ with $T/R$ and upper bound $\Psi_0$ by $LD_0^2$.
\end{remark}

The proof of \cref{fedaci:full} is based on the following two lemmas regarding convergence and stability respectively.
To clarify the hyperparameter dependency, we state these lemmas for general $\gamma \in \left[\eta, \sqrt{ \frac{\eta}{\mu}} \right]$, which has one more degree of freedom than \fedaci where $\gamma = \max \left\{ \sqrt{\frac{\eta}{\mu K}}, \eta\right\}$ is fixed.

\begin{lemma}[Potential-based perturbed iterate analysis for \fedaci]
  \label{fedaci:conv:main}
  Let $F$ be $\mu>0$-strongly convex, and assume \cref{asm1}, then for $\alpha = \frac{1}{\gamma \mu}$, $\beta = \alpha + 1$, $\gamma \in \left[\eta, \sqrt{ \frac{\eta}{\mu}} \right]$, $\eta \in \left(0, \frac{1}{L}\right]$, \fedac yields
  \begin{align}
     & \expt \left[\Psi_T \right]          \leq \exp \left( - \gamma \mu T \right)  \Psi_0 + \frac{\eta^2 L \sigma^2}{2\gamma \mu} + \frac{\gamma\sigma^2}{2M}
    \\
     & + L \cdot \max_{0 \leq t < T}  \expt
    \left[\frac{1}{M} \sum_{m=1}^M \left\| \overline{w_t^{\mathrm{md}}} - w_t^{\mathrm{md}, m}  \right\|
    \left\|  \frac{1}{1+\gamma\mu}(\overline{w_t} - w_t^m) + \frac{\gamma \mu}{1+\gamma\mu} (\overline{w_t^{\mathrm{ag}}} - w_t^{\mathrm{ag},m})  \right\|  \right],
  \end{align}
  where $\Psi_t$ is the decentralized potential defined in \cref{eq:fedaci:potential}.
\end{lemma}
The proof of \cref{fedaci:conv:main} is deferred to \cref{sec:fedaci:conv:main}.
\begin{lemma}[Discrepancy overhead bound]
  \label{fedaci:stab:main}
  In the same setting of \cref{fedaci:conv:main}, \fedac satisfies
  \begin{align}
         & \expt
    \left[ \frac{1}{M} \sum_{m=1}^M
    \left\| \overline{w_t^{\mathrm{md}}} - w_t^{\mathrm{md}, m} \right\|
    \left\|  \frac{1}{1+\gamma\mu}(\overline{w_t} - w_t^m) + \frac{\gamma \mu}{1+\gamma\mu} (\overline{w_t^{\mathrm{ag}}} - w_t^{\mathrm{ag},m})  \right\|  \right]
    \\
    \leq &
    \begin{cases}
      7 \eta \gamma K \sigma^2 \left(1 + \frac{2\gamma^2\mu}{\eta}\right)^{2K}
       & \text{if~} \gamma \in \left(\eta, \sqrt{\frac{\eta}{\mu}} \right],
      \\
      7 \eta^2 K \sigma^2
       &
      \text{if~} \gamma = \eta.
    \end{cases}
  \end{align}
\end{lemma}
The proof of \cref{fedaci:stab:main} is deferred to \cref{sec:fedaci:stab:main}.

Now we plug in the choice of $\gamma = \max \left\{ \sqrt{\frac{\eta}{\mu K}}, \eta\right\}$ to \cref{fedaci:conv:main,fedaci:stab:main}, which leads to  the following lemma.
\begin{lemma}[Convergence of \fedaci for general $\eta$]
  \label{fedaci:general:eta}
    Let $F$ be $\mu > 0$-strongly convex, and assume \cref{asm1}, then for any $\eta \in \left(0, \frac{1}{L}\right]$, \fedaci yields
    \begin{equation}
      \expt[\Psi_T]
      \leq \exp \left(  - \max \left\{ \eta \mu, \sqrt{\frac{\eta \mu}{K}}\right\}T \right) \Psi_0
      + 
      \frac{\eta^{\frac{1}{2}} \sigma^2}{2 \mu^{\frac{1}{2}} M K^{\frac{1}{2}} }
      + \frac{\eta \sigma^2}{2M}
      + \frac{390 \eta^{\frac{3}{2}} L K^{\frac{1}{2}} \sigma^2} {\mu^{\frac{1}{2}}}
      + 7 \eta^2 L K \sigma^2,
      \label{eq:fedaci:general:eta}
    \end{equation}
    where $\Psi_t$ is the decentralized potential defined in \cref{eq:fedaci:potential}.
\end{lemma}
\begin{proof}[Proof of \cref{fedaci:general:eta}]
  It is direct to verify that $\gamma = \max \left\{\eta, \sqrt{\frac{\eta}{\mu K}} \right\} \in \left[\eta, \sqrt{\frac{\eta}{\mu}}\right]$ so both \cref{fedaci:conv:main,fedaci:stab:main} are applicable.
  Applying \cref{fedaci:conv:main} yields
    \begin{align}
      \expt[\Psi_T] \leq & \exp\left( - \max \left\{ \eta \mu, \sqrt{\frac{\eta \mu}{K}}\right\}T \right) \Psi_0
      +
      \min\left\{ \frac{\eta L \sigma^2}{2 \mu}, \frac{\eta^{\frac{3}{2}} L K^{\frac{1}{2}} \sigma^2}{2 \mu^{\frac{1}{2}}} \right\}
      + \max\left\{ \frac{\eta \sigma^2}{2M}, \frac{\eta^{\frac{1}{2}} \sigma^2}{2 \mu^{\frac{1}{2}} M K^{\frac{1}{2}} }          \right\}
      \\
                        & + L \cdot \max_{0 \leq t < T}  \expt
      \left[\frac{1}{M} \sum_{m=1}^M \left\| \overline{w_t^{\mathrm{md}}} - w_t^{\mathrm{md}, m}  \right\|
      \left\|  \frac{1}{1+\gamma\mu}(\overline{w_t} - w_t^m) + \frac{\gamma \mu}{1+\gamma\mu} (\overline{w_t^{\mathrm{ag}}} - w_t^{\mathrm{ag},m})  \right\|  \right].
      \label{eq:fedaci:proof:1}
    \end{align}
  We bound $\max\left\{ \frac{\eta \sigma^2}{2M}, \frac{\eta^{\frac{1}{2}} \sigma^2}{2 \mu^{\frac{1}{2}} M K^{\frac{1}{2}} } \right\}$ by $\frac{\eta \sigma^2}{2M}+ \frac{\eta^{\frac{1}{2}} \sigma^2}{2 \mu^{\frac{1}{2}} M K^{\frac{1}{2}} } $, and bound $    \min\left\{ \frac{\eta L \sigma^2}{2 \mu}, \frac{\eta^{\frac{3}{2}} L K^{\frac{1}{2}} \sigma^2}{2 \mu^{\frac{1}{2}}} \right\}$ by    $\frac{\eta^{\frac{3}{2}} L K^{\frac{1}{2}} \sigma^2}{2 \mu^{\frac{1}{2}}}$, which gives
  \begin{equation}
    \min\left\{ \frac{\eta L \sigma^2}{2 \mu}, \frac{\eta^{\frac{3}{2}} L K^{\frac{1}{2}} \sigma^2}{2 \mu^{\frac{1}{2}}} \right\}
    + \max\left\{ \frac{\eta \sigma^2}{2M}, \frac{\eta^{\frac{1}{2}} \sigma^2}{2 \mu^{\frac{1}{2}} M K^{\frac{1}{2}} }          \right\}
    \leq
    \frac{\eta^{\frac{3}{2}} L K^{\frac{1}{2}} \sigma^2}{2 \mu^{\frac{1}{2}}}  +  \frac{\eta \sigma^2}{2M} + \frac{\eta^{\frac{1}{2}} \sigma^2}{2 \mu^{\frac{1}{2}} M K^{\frac{1}{2}} }.
    \label{eq:fedaci:proof:1.1}
  \end{equation}
  Applying \cref{fedaci:stab:main} with $\gamma = \max\left\{ \eta, \sqrt{\frac{\eta}{\mu K}}\right\}$ gives
  \begin{align}
         & \expt
    \left[ \frac{1}{M} \sum_{m=1}^M
    \left\| \overline{w_t^{\mathrm{md}}} - w_t^{\mathrm{md}, m} \right\|
    \left\|  \frac{1}{1+\gamma\mu}(\overline{w_t} - w_t^m) + \frac{\gamma \mu}{1+\gamma\mu} (\overline{w_t^{\mathrm{ag}}} - w_t^{\mathrm{ag},m})  \right\|  \right]
    \\
    \leq &
    \begin{cases}
      7  \eta \sqrt{\frac{\eta}{\mu K}} K \sigma^2 \left(1 + \frac{2}{K}\right)^{2K}
       & \text{if~} \gamma = \sqrt{\frac{\eta}{\mu K}}
      \\
      7  \eta^2 K \sigma^2
       &
      \text{if~} \gamma = \eta
    \end{cases}
    \\
    \leq & \frac{7  \euler^4 \eta^{\frac{3}{2}} K^{\frac{1}{2}} \sigma^2} {\mu^{\frac{1}{2}}} + 7 \eta^2 K \sigma^2.
    \label{eq:fedaci:proof:2}
  \end{align}
  Combining \cref{eq:fedaci:proof:1,eq:fedaci:proof:1.1,eq:fedaci:proof:2} yields
    \begin{equation}
      \expt[\Psi_T]
      \leq \exp \left(  - \max \left\{ \eta \mu, \sqrt{\frac{\eta \mu}{K}}\right\}T \right) \Psi_0
      + 
      \frac{\eta^{\frac{1}{2}} \sigma^2}{2 \mu^{\frac{1}{2}} M K^{\frac{1}{2}} }
      + \frac{\eta \sigma^2}{2M}
      + \frac{(7 \euler^4 + \frac{1}{2}) \eta^{\frac{3}{2}} L K^{\frac{1}{2}} \sigma^2} {\mu^{\frac{1}{2}}}
      + 7 \eta^2 L K \sigma^2.
    \end{equation}
  The lemma then follows by leveraging the estimate $7 \euler^4 + \frac{1}{2} < 390$ for the coefficient of $\frac{\eta^{\frac{3}{2}} L K^{\frac{1}{2}} \sigma^2} {\mu^{\frac{1}{2}}}$.
\end{proof}
The main \cref{fedaci:full} then follows by plugging an appropriate $\eta$ to \cref{fedaci:general:eta}.
\begin{proof}[Proof of \cref{fedaci:full}]
  To simplify the notation, we denote the decreasing term in \cref{eq:fedaci:general:eta} as $\varphi_{\downarrow}(\eta)$ and the increasing term as $\varphi_{\uparrow}(\eta)$, namely
  \begin{align}
    \varphi_{\downarrow}(\eta) := \exp \left(  - \max \left\{ \eta \mu, \sqrt{\frac{\eta \mu}{K}}\right\}T \right) \Psi_0,
    \quad
    \varphi_{\uparrow}(\eta) :=  \frac{\eta^{\frac{1}{2}} \sigma^2}{2 \mu^{\frac{1}{2}} M K^{\frac{1}{2}} }
    + \frac{\eta \sigma^2}{2M}
    + \frac{390 \eta^{\frac{3}{2}} L K^{\frac{1}{2}} \sigma^2} {\mu^{\frac{1}{2}}}
    + 7 \eta^2 L K \sigma^2.
  \end{align}
  Now let 
  \begin{equation}
    \eta_0 := \frac{K}{\mu T^2} \log^2 \left( \euler + \min \left\{ \frac{\mu M T \Psi_0}{\sigma^2}, \frac{\mu^2 T^3 \Psi_0}{L K^2 \sigma^2} \right\} \right),
  \end{equation}
  and then $ \eta = \min \left\{ \frac{1}{L}, \eta_0  \right\}$. 
  Therefore, the decreasing term $\varphi_{\downarrow}(\eta)$  is upper bounded by $\varphi_{\downarrow}(\frac{1}{L}) + \varphi_{\downarrow}(\eta_0)$, where
  \begin{equation}
    \varphi_{\downarrow} \left(\frac{1}{L} \right)
    =
    \min \left\{ \exp \left( - \frac{\mu T}{L} \right), \exp \left( - \frac{\mu^{\frac{1}{2}} T}{L^{\frac{1}{2}} K^{\frac{1}{2}}} \right)\right\}\Psi_0,
    \label{eq:fedaci:proof:4}
  \end{equation}
  and
  \begin{equation}
    \varphi_{\downarrow}(\eta_0) \leq \exp \left(  - \sqrt{\frac{\eta_0 \mu}{K}}T \right) \Psi_0
    =
    \left( \euler + \min \left\{ \frac{\mu M T \Psi_0}{\sigma^2}, \frac{\mu^2 T^3 \Psi_0}{L K^2 \sigma^2} \right\} \right)^{-1} \Psi_0
    \leq
    \frac{\sigma^2}{\mu M T} + \frac{LK^2 \sigma^2}{ \mu^2 T^3}.
    \label{eq:fedaci:proof:5}
  \end{equation}
  On the other hand
  \begin{align}
    \varphi_{\uparrow}(\eta)
    \leq
    \varphi_{\uparrow}(\eta_0)
    \leq
    & 
    \frac{\sigma^2}{2\mu MT} \log \left(  \euler + \frac{\mu M T \Psi_0}{\sigma^2} \right) 
    +
    \frac{K \sigma^2}{2 \mu M T^2} \log^2 \left(  \euler + \frac{\mu M T \Psi_0}{\sigma^2} \right) 
    \\
    & + \frac{390  LK^2 \sigma^2}{\mu^2 T^3} \log^3 \left( \euler + \frac{\mu^2 T^3 \Psi_0}{LK^2\sigma^2} \right)
    + \frac{7 LK^3 \sigma^2}{\mu^2 T^4} \log^4 \left( \euler + \frac{\mu^2 T^3 \Psi_0}{LK^2\sigma^2} \right)
    \\
    \leq & \frac{\sigma^2}{\mu M T} \log^2 \left(  \euler + \frac{\mu M T \Psi_0}{\sigma^2} \right) 
    + \frac{397 LK^2 \sigma^2}{\mu^2 T^3} \log^4 \left( \euler + \frac{\mu^2 T^3 \Psi_0}{LK^2\sigma^2} \right),
    \label{eq:fedaci:proof:6}
  \end{align}
  where the last inequality is due to $\frac{K\sigma^2}{2\mu MT} \leq \frac{\sigma^2}{\mu MT}$ and $\frac{7LK^3 \sigma^2}{\mu^2T^4} \leq \frac{7LK^2 \sigma^2}{\mu^2T^3}$ since $K \leq T$.

  Combining \cref{fedaci:general:eta,eq:fedaci:proof:4,eq:fedaci:proof:5,eq:fedaci:proof:6} gives
  \begin{align}
    & \expt [\Psi_T]
    \leq 
    \varphi_{\downarrow} \left( \frac{1}{L} \right) + \varphi_{\downarrow}(\eta_0) + \varphi_{\uparrow}(\eta) 
    \\
    \leq & 
    \min \left\{ \exp \left( - \frac{\mu T}{L} \right), \exp \left( - \frac{\mu^{\frac{1}{2}} T}{L^{\frac{1}{2}} K^{\frac{1}{2}}} \right)\right\}\Psi_0
    + \frac{2\sigma^2}{\mu M T} \log^2 \left(  \euler + \frac{\mu M T \Psi_0}{\sigma^2} \right) 
    + \frac{400 LK^2 \sigma^2}{\mu^2 T^3} \log^4 \left( \euler + \frac{\mu^2 T^3 \Psi_0}{LK^2\sigma^2} \right),
  \end{align}
  completing the proof of main \cref{fedaci:full}.
\end{proof}
\subsection{Perturbed iterate analysis for \fedaci: Proof of \cref{fedaci:conv:main}}
\label{sec:fedaci:conv:main}
In this section we will prove \cref{fedaci:conv:main}.
We start by the one-step analysis of the decentralized potential $\Psi_t$ defined in \cref{eq:fedaci:potential}. The following two propositions establish the one-step analysis of the two quantities in $\Psi_t$, namely $\|\overline{w_t} - w^*\|^2$ and $\frac{1}{M} \sum_{m=1}^M F(w_t^{\mathrm{ag},m}) - F^*$. 
We only require minimal hyperparameter assumptions, namely $\alpha \geq 1, \beta \geq 1, \eta \leq \frac{1}{L}$, for these two propositions.
We will then show how the choice of $\alpha, \beta$ is determined towards the proof of \cref{fedaci:conv:main} in order to couple the two quantities into potential $\Psi_t$.
\begin{proposition}
  \label{fedaci:conv:1}
  Let $F$ be $\mu>0$-strongly convex, and assume \cref{asm1}, then for \fedac with hyperparameters assumptions $\alpha \geq 1$, $\beta \geq 1$, $\eta \leq \frac{1}{L}$, the following inequality holds
  \begin{align}
         & \expt [ \|\overline{w_{t+1}} - w^*\|^2 |\mathcal{F}_t]
    \\
    \leq & (1 - \alpha^{-1}) \| \overline{w_t}  - w^*\|^2  + \alpha^{-1} \| \overline{w_t^{\mathrm{md}}} - w^*\|^2 + \gamma^2 \left\|  \frac{1}{M} \sum_{m=1}^M \nabla F (w_t^{\mathrm{md},m}) \right\|^2 + \frac{1}{M}\gamma^2 \sigma^2
    \\
         & - 2 \gamma \cdot \frac{1}{M} \sum_{m=1}^M \left\langle \nabla F (w_t^{\mathrm{md}, m}), (1 - \alpha^{-1}(1 - \beta^{-1})) {w_t^m} + \alpha^{-1} (1 - \beta^{-1}) {w_t^{\mathrm{ag},m}} - w^* \right\rangle
    \\
         & + 2 \gamma L \frac{1}{M} \sum_{m=1}^M
    \left\| \overline{w_t^{\mathrm{md}}} - w_t^{\mathrm{md}, m} \right\|
    \left\| (1-\alpha^{-1}(1-\beta^{-1})) (\overline{w_t} - w_t^m) + \alpha^{-1} (1 - \beta^ {-1}) (\overline{w_t^{\mathrm{ag}}} - w_t^{\mathrm{ag}, m})   \right\|.
  \end{align}
\end{proposition}
\begin{proposition}
  \label{fedaci:conv:2}
  In the same setting of \cref{fedaci:conv:1}, the following inequality holds
  \begin{align}
         & \expt \left[ \frac{1}{M} \sum_{m=1}^M F( {w_{t+1}^{\mathrm{ag}, m}}) - F^* \middle| \mathcal{F}_t \right]
    \\
    \leq & (1 - \alpha^{-1}) \left( \frac{1}{M} \sum_{m=1}^M F( {w_{t}^{\mathrm{ag}, m}}) - F^* \right)
    - \frac{1}{2} \eta \left\| \frac{1}{M} \sum_{m=1}^M  \nabla F (w_t^{\mathrm{md}, m})  \right\|^2
    +  \frac{1}{2} \eta^2 L \sigma^2
    \\
         & + \alpha^{-1} \frac{1}{M} \sum_{m=1}^M \left\langle \nabla F(w_t^{\mathrm{md}, m}), \alpha \beta^{-1} w_t^m + (1 - \alpha \beta^{-1}) w_t^{\mathrm{ag},m} - w^*\right\rangle
    - \frac{1}{2} \mu \alpha^{-1} \| \overline{w_t^{\mathrm{md}}} - w^*\|^2.
  \end{align}
\end{proposition}
We defer the proofs of \cref{fedaci:conv:1,fedaci:conv:2} to \cref{sec:proof:fedaci:conv:1,sec:proof:fedaci:conv:2}, respectively.

With \cref{fedaci:conv:1,fedaci:conv:2} at hand we are ready to prove \cref{fedaci:conv:main}.
\begin{proof}[Proof of \cref{fedaci:conv:main}]
  Applying \cref{fedaci:conv:1} with the specified $\alpha = \frac{1}{\gamma \mu}, \beta = \alpha + 1$ yields (for any $t$)
  \begin{align}
         & \expt [ \|\overline{w_{t+1}} - w^*\|^2 |\mathcal{F}_t]
    \\
    \leq & (1 - \gamma \mu) \| \overline{w_t}  - w^*\|^2  + \gamma\mu \| \overline{w_t^{\mathrm{md}}} - w^*\|^2 + \gamma^2 \left\|  \frac{1}{M} \sum_{m=1}^M \nabla F (w_t^{\mathrm{md},m}) \right\|^2 + \frac{1}{M}\gamma^2 \sigma^2
    \\
         & - 2 \gamma \cdot \frac{1}{M} \sum_{m=1}^M \left\langle \nabla F (w_t^{\mathrm{md}, m}), \frac{1}{1 + \gamma\mu} {w_t^m} + \frac{\gamma \mu}{1 + \gamma\mu} {w_t^{\mathrm{ag},m}} - w^* \right\rangle
    \\
         & + 2 \gamma L \cdot \frac{1}{M} \sum_{m=1}^M \left\| \overline{w_t^{\mathrm{md}}} - w_t^{\mathrm{md}, m}  \right\|
    \left\|  \frac{1}{1+\gamma\mu}(\overline{w_t} - w_t^m) + \frac{\gamma \mu}{1+\gamma\mu} (\overline{w_t^{\mathrm{ag}}} - w_t^{\mathrm{ag},m})  \right\|.
    \label{eq:fedaci:conv:main:1}
  \end{align}
  Applying \cref{fedaci:conv:2} with the specified $\alpha = \frac{1}{\gamma \mu}, \beta = \alpha + 1$ yields (for any $t$)
  \begin{align}
         & \expt \left[ \frac{1}{M} \sum_{m=1}^M F( {w_{t+1}^{\mathrm{ag}, m}}) - F^* \middle| \mathcal{F}_t \right]
    \\
    \leq & (1 - \gamma \mu) \left( \frac{1}{M} \sum_{m=1}^M F( {w_{t}^{\mathrm{ag}, m}}) - F^* \right)
    - \frac{1}{2} \eta \left\| \frac{1}{M} \sum_{m=1}^M  \nabla F (w_t^{\mathrm{md}, m})  \right\|^2
    +  \frac{1}{2} \eta^2 L \sigma^2
    \\
         & + \gamma \mu \cdot \frac{1}{M} \sum_{m=1}^M \left\langle \nabla F(w_t^{\mathrm{md}, m}), \frac{1}{1+\gamma\mu} w_t^{m} + \frac{\gamma \mu}{1 + \gamma \mu} w_t^{\mathrm{ag}, m} - w^* \right\rangle
    - \frac{1}{2} \gamma \mu^2 \| \overline{w_t^{\mathrm{md}}} - w^*\|^2.
    \label{eq:fedaci:conv:main:2}
  \end{align}
  Adding \cref{eq:fedaci:conv:main:2} with $\frac{1}{2}\mu$ times of \cref{eq:fedaci:conv:main:1} yields
  \begin{align}
     & \expt [\Psi_{t+1} |\mathcal{F}_t] \leq (1 - \gamma \mu) \Psi_t + \frac{1}{2} \left( \eta^2 L + \frac{1}{M} \gamma^2 \mu \right) \sigma^2 +  \frac{1}{2}\left( \gamma^2 \mu - \eta \right) \left\| \frac{1}{M} \sum_{m=1}^M  \nabla F (w_t^{\mathrm{md}, m})  \right\|^2
    \\
     & \quad + \gamma \mu L \cdot \frac{1}{M} \sum_{m=1}^M \left\| \overline{w_t^{\mathrm{md}}} - w_t^{\mathrm{md}, m}  \right\|
    \left\|  \frac{1}{1+\gamma\mu}(\overline{w_t} - w_t^m) + \frac{\gamma \mu}{1+\gamma\mu} (\overline{w_t^{\mathrm{ag}}} - w_t^{\mathrm{ag},m})  \right\|.
  \end{align}
  Since $\gamma^2 \mu \leq \eta$, the coefficient of $\left\| \frac{1}{M} \sum_{m=1}^M  \nabla F (w_t^{\mathrm{md}, m})  \right\|^2$ is non-positive. Thus
  \begin{align}
     & \expt [\Psi_{t+1} |\mathcal{F}_t] \leq (1 - \gamma \mu) \Psi_t + \frac{1}{2} \left( \eta^2 L + \frac{1}{M} \gamma^2 \mu \right) \sigma^2
    \\
     & \quad + \gamma \mu L \cdot \frac{1}{M} \sum_{m=1}^M \left\| \overline{w_t^{\mathrm{md}}} - w_t^{\mathrm{md}, m}  \right\|
    \left\|  \frac{1}{1+\gamma\mu}(\overline{w_t} - w_t^m) + \frac{\gamma \mu}{1+\gamma\mu} (\overline{w_t^{\mathrm{ag}}} - w_t^{\mathrm{ag},m})  \right\|.
  \end{align}
  Telescoping the above inequality up to timestep $T$ yields
  \begin{align}
         & \expt \left[\Psi_T \right] \leq  \left( 1 - \gamma \mu \right)^T \Psi_0 +
    \left( \sum_{t=0}^{T-1} \left( 1 - \gamma \mu \right)^t \right) \cdot \frac{1}{2} \left( \eta^2 L + \frac{1}{M} \gamma^2 \mu \right) \sigma^2
    \\
         & +
    \gamma \mu L  \cdot \sum_{t=0}^{T-1} \left\{   \left( 1 - \gamma\mu \right)^{T-t-1} \cdot  \expt
    \left[ \frac{1}{M} \sum_{m=1}^M \left\| \overline{w_t^{\mathrm{md}}} - w_t^{\mathrm{md}, m}  \right\|
    \left\|  \frac{1}{1+\gamma\mu}(\overline{w_t} - w_t^m) + \frac{\gamma \mu}{1+\gamma\mu} (\overline{w_t^{\mathrm{ag}}} - w_t^{\mathrm{ag},m})  \right\|  \right] \right\}
    \\
    \leq & \exp \left( - \gamma \mu T \right) \Psi_0 + \frac{\eta^2 L \sigma^2}{2\gamma \mu} + \frac{\gamma\sigma^2}{2M}
    \\
         & + L \cdot \max_{0 \leq t < T}  \expt
    \left[\frac{1}{M} \sum_{m=1}^M \left\| \overline{w_t^{\mathrm{md}}} - w_t^{\mathrm{md}, m}  \right\|
    \left\|  \frac{1}{1+\gamma\mu}(\overline{w_t} - w_t^m) + \frac{\gamma \mu}{1+\gamma\mu} (\overline{w_t^{\mathrm{ag}}} - w_t^{\mathrm{ag},m})  \right\|  \right],
  \end{align}
  where in the last inequality we used the fact that $(1 - \gamma \mu)^T \leq \exp(-\gamma \mu T)$ and $\sum_{t=0}^{T-1} \left( 1 - \gamma \mu \right)^t  \leq \frac{1}{\gamma \mu}$.
\end{proof}

\subsubsection{Proof of \cref{fedaci:conv:1}}
\label{sec:proof:fedaci:conv:1}
\begin{proof}[Proof of \cref{fedaci:conv:1}]
  By definition of the \fedac procedure (\cref{alg:fedac}), for all $m \in [M]$ (recall $v_{t+1}^m$ is the candidate for next step),
  \begin{equation}
    v_{t+1}^m = (1 - \alpha^{-1})w_t^m + \alpha^{-1} w_t^{\mathrm{md},m} - \gamma \cdot \nabla f(w_t^{\mathrm{md},m};\xi_t^m).
  \end{equation}
  Taking average over $m = 1, \ldots, M$ gives
  \begin{equation}
    \overline{w_{t+1}} - w^* = (1 - \alpha^{-1}) \overline{w_t} + \alpha^{-1} \overline{w_t^{\mathrm{md}}} - \gamma \cdot \frac{1}{M} \sum_{m=1}^M \nabla f(w_t^{\mathrm{md},m}; \xi_t^m) - w^*.
  \end{equation}
  Taking conditional expectation gives
  \begin{align}
    & \expt [ \|\overline{w_{t+1}} - w^*\|^2 |\mathcal{F}_t]
    \\
    =    & \left\| (1 - \alpha^{-1}) \overline{w_t} + \alpha^{-1} \overline{w_t^{\mathrm{md}}} - \gamma \cdot \frac{1}{M} \sum_{m=1}^M \nabla F (w_t^{\mathrm{md},m})- w^* \right\|^2
    \\
         & \quad +  \expt \left[  \left\| \frac{1}{M} \sum_{m=1}^M  \left(  \nabla f(w_t^{\mathrm{md},m}; \xi_t^m) - \nabla F(w_t^{\mathrm{md};m}) \right) \right\|^2  \middle| \mathcal{F}_t \right]
    \tag{independence}
    \\
    \leq & \left\| (1 - \alpha^{-1}) \overline{w_t} + \alpha^{-1} \overline{w_t^{\mathrm{md}}} - \gamma \cdot \frac{1}{M} \sum_{m=1}^M \nabla F (w_t^{\mathrm{md},m})- w^* \right\|^2
    + \frac{1}{M}\gamma^2 \sigma^2,
    \label{eq:fedaci:conv:1:0}
  \end{align}
  where the last inequality of \cref{eq:fedaci:conv:1:0} is due to the bounded variance assumption (\cref{asm1}(c)) and independence.  Expanding the squared norm term of \cref{eq:fedaci:conv:1:0} and applying Jensen's inequality,
  \begin{align}
    & \left\| (1 - \alpha^{-1}) \overline{w_t} + \alpha^{-1} \overline{w_t^{\mathrm{md}}} - \gamma \cdot \frac{1}{M} \sum_{m=1}^M \nabla F (w_t^{\mathrm{md},m})- w^* \right\|^2
    \\
    =    & \left\| (1 - \alpha^{-1}) \overline{w_t} + \alpha^{-1} \overline{w_t^{\mathrm{md}}} - w^* \right\|^2 + \gamma^2 \left\|  \frac{1}{M} \sum_{m=1}^M \nabla F (w_t^{\mathrm{md},m}) \right\|^2
    \\
         & - 2 \gamma \cdot \frac{1}{M} \sum_{m=1}^M \left\langle \nabla F (w_t^{\mathrm{md},m}), (1 - \alpha^{-1}) \overline{w_t} + \alpha^{-1} \overline{w_t^{\mathrm{md}}} - w^* \right\rangle
    \tag{expansion of squared norm}
    \\
    \leq & (1 - \alpha^{-1}) \| \overline{w_t}  - w^*\|^2 + \alpha^{-1} \|
    \overline{w_t^{\mathrm{md}}} - w^*\|^2  + \gamma^2 \left\|  \frac{1}{M} \sum_{m=1}^M \nabla F (w_t^{\mathrm{md},m}) \right\|^2
    \\
         & - 2 \gamma \cdot \frac{1}{M} \sum_{m=1}^M  \left\langle\nabla F (w_t^{\mathrm{md},m}), (1 - \alpha^{-1}) \overline{w_t} + \alpha^{-1} \overline{w_t^{\mathrm{md}}} - w^* \right\rangle,
    \label{eq:fedaci:conv:1:1}
  \end{align}
  It remains to analyze the inner product term of \cref{eq:fedaci:conv:1:1}. Note that
  \begin{align}
         & -\frac{1}{M} \sum_{m=1}^M \left\langle  \nabla F (w_t^{\mathrm{md}, m}), (1 - \alpha^{-1}) \overline{w_t} + \alpha^{-1} \overline{w_t^{\mathrm{md}}} - w^* \right\rangle
    \\
    =    & -\frac{1}{M} \sum_{m=1}^M \left\langle \nabla F (w_t^{\mathrm{md}, m}) , (1-\alpha^{-1}(1-\beta^{-1})) \overline{w_t} + \alpha^{-1} (1 - \beta^ {-1}) \overline{w_t^{\mathrm{ag}}} - w^*\right\rangle
    \tag{definition of $\overline{w_t^{\mathrm{md}}}$}
    \\
    =    & -\frac{1}{M} \sum_{m=1}^M \left\langle \nabla F (w_t^{\mathrm{md}, m}), (1-\alpha^{-1}(1-\beta^{-1})) (\overline{w_t} - w_t^m) + \alpha^{-1} (1 - \beta^ {-1}) (\overline{w_t^{\mathrm{ag}}} - w_t^{\mathrm{ag}, m})\right\rangle
    \\
         & -\frac{1}{M} \sum_{m=1}^M \left\langle \nabla F (w_t^{\mathrm{md}, m}), (1 - \alpha^{-1}(1 - \beta^{-1})) {w_t^m} + \alpha^{-1} (1 - \beta^{-1}) {w_t^{\mathrm{ag},m}} - w^* \right\rangle
    \\
    =    & \frac{1}{M} \sum_{m=1}^M \left\langle \nabla F(\overline{w_t^{\mathrm{md}}}) - \nabla F (w_t^{\mathrm{md}, m}), (1-\alpha^{-1}(1-\beta^{-1})) (\overline{w_t} - w_t^m) + \alpha^{-1} (1 - \beta^ {-1}) (\overline{w_t^{\mathrm{ag}}} - w_t^{\mathrm{ag}, m})  \right\rangle
    \\
         & - \frac{1}{M} \sum_{m=1}^M \left\langle \nabla F (w_t^{\mathrm{md}, m}), (1 - \alpha^{-1}(1 - \beta^{-1})) {w_t^m} + \alpha^{-1} (1 - \beta^{-1}) {w_t^{\mathrm{ag},m}} - w^* \right\rangle
    \\
    \leq & L \cdot \frac{1}{M} \sum_{m=1}^M
    \left\| \overline{w_t^{\mathrm{md}}} - w_t^{\mathrm{md}, m} \right\|
    \left\| (1-\alpha^{-1}(1-\beta^{-1})) (\overline{w_t} - w_t^m) + \alpha^{-1} (1 - \beta^ {-1}) (\overline{w_t^{\mathrm{ag}}} - w_t^{\mathrm{ag}, m})   \right\|
    \\
         & - \frac{1}{M} \sum_{m=1}^M \left\langle \nabla F (w_t^{\mathrm{md}, m}), (1 - \alpha^{-1}(1 - \beta^{-1})) {w_t^m} + \alpha^{-1} (1 - \beta^{-1}) {w_t^{\mathrm{ag},m}} - w^* \right\rangle,
    \label{eq:fedaci:conv:1:2}
  \end{align}
  where the last equality is due to the $L$-smoothness (\cref{asm1}(b)).
  Combining \cref{eq:fedaci:conv:1:0,eq:fedaci:conv:1:1,eq:fedaci:conv:1:2} completes the proof of \cref{fedaci:conv:1}.
\end{proof}

\subsubsection{Proof of \cref{fedaci:conv:2}}
\label{sec:proof:fedaci:conv:2}
Before stating the proof of \cref{fedaci:conv:2}, we first introduce and prove the following claim for a single worker $m \in [M]$.
\begin{claim}
  \label{fedaci:conv:2:claim}
  Under the same assumptions of \cref{fedaci:conv:2}, for any $m \in [M]$, the following inequality holds (recall that $v_{t+1}^{\mathrm{ag}, m}$ is defined as the candidate next update (see \cref{alg:fedac}) before possible synchronization)
  \begin{align}
    & \expt \left[ F( {v_{t+1}^{\mathrm{ag}, m}}) - F^* |\mathcal{F}_t \right]
    \leq  ~
    (1 - \alpha^{-1}) \left( F( {w_{t}^{\mathrm{ag}, m}}) - F^* \right)
    - \frac{1}{2} \eta \left\| \nabla F (w_t^{\mathrm{md}, m})  \right\|^2 +  \frac{1}{2} \eta^2 L \sigma^2
    \\
    & 
    - \frac{1}{2} \mu \alpha^{-1} \|w_t^{\mathrm{md}, m} - w^*\|^2
        + \alpha^{-1} \left\langle \nabla F(w_t^{\mathrm{md}, m}) ,  \alpha \beta^{-1} w_t^m + (1 - \alpha \beta^{-1}) w_t^{\mathrm{ag},m} - w^* \right\rangle.  
  \end{align}
\end{claim}
\begin{proof}[Proof of \cref{fedaci:conv:2:claim}]
  By definition of \fedac (\cref{alg:fedac}), ${v_{t+1}^{\mathrm{ag}, m}} = {w_t^{\mathrm{md}, m}} - \eta \cdot \nabla f(w_t^{\mathrm{md},m}; \xi_t^m)$. Thus, by $L$-smoothness (\cref{asm1}(b)),
  \begin{equation}
    F( {v_{t+1}^{\mathrm{ag}, m}})
    \leq
    F( {w_t^{\mathrm{md}, m}}) - \eta \left\langle \nabla F({w_t^{\mathrm{md}, m}}), \nabla f(w_t^{\mathrm{md},m}; \xi_t^m) \right\rangle + \frac{1}{2}  \eta^2 L \left\| \nabla f(w_t^{\mathrm{md},m}; \xi_t^m)  \right\|^2.
  \end{equation}
  Taking conditional expectation gives
  \begin{align}
    \expt \left[ F( {v_{t+1}^{\mathrm{ag}, m}}) |\mathcal{F}_t \right]
     & \leq
    F( {w_t^{\mathrm{md}, m}}) - \eta \left\| \nabla F (w_t^{\mathrm{md}, m})  \right\|^2 + \frac{1}{2} \eta^2 L \left\| \nabla F (w_t^{\mathrm{md}, m})  \right\|^2 + \frac{1}{2} \eta^2 L \sigma^2
    \\
     & = F( {w_t^{\mathrm{md}, m}}) - \eta \left( 1 - \frac{1}{2} \eta L \right) \left\| \nabla F (w_t^{\mathrm{md}, m})  \right\|^2 + \frac{1}{2} \eta^2 L \sigma^2.
  \end{align}
  Since $\eta \leq \frac{1}{L}$ we have $1 - \frac{1}{2} \eta L \geq \frac{1}{2}$. Thus
  \begin{equation}
    \expt \left[ F( {v_{t+1}^{\mathrm{ag}, m}}) \middle| \mathcal{F}_t \right]
    \leq
    F( {w_t^{\mathrm{md}, m}}) - \frac{1}{2} \eta \left\| \nabla F (w_t^{\mathrm{md}, m})  \right\|^2 + \frac{1}{2} \eta^2 L \sigma^2.
    \label{eq:fedaci:conv:2:1}
  \end{equation}
  Now we connect $F(w_t^{\mathrm{md},m})$ with $F(w_t^{\mathrm{ag},m})$ as follows.
  \begin{align}
         & F(w_t^{\mathrm{md},m}) - F^*
    \\
    =    & (1 - \alpha^{-1}) \left( F(w_t^{\mathrm{ag}, m}) - F^* \right)
    + \alpha^{-1} \left( F(w_t^{\mathrm{md}, m}) - F^* \right)
    + (1 - \alpha^{-1}) \left( F(w_t^{\mathrm{md}, m}) - F(w_t^{\mathrm{ag},m}) \right)
    \\
    \leq & (1 - \alpha^{-1}) \left( F(w_t^{\mathrm{ag}, m}) - F^* \right) - \frac{1}{2} \mu \alpha^{-1} \|w_t^{\mathrm{md}, m} - w^*\|^2
    + \alpha^{-1} \left\langle \nabla F(w_t^{\mathrm{md}, m}), w_t^{\mathrm{md}, m} - w^*  \right\rangle
    \\
         & \quad + (1 - \alpha^{-1}) \left\langle \nabla F(w_t^{\mathrm{md}, m}),  w_t^{\mathrm{md}, m} - w_t^{\mathrm{ag}, m}  \right\rangle 
    \tag{$\mu$-strong-convexity}
    \\
    =    & (1 - \alpha^{-1}) \left( F(w_t^{\mathrm{ag}, m}) - F^* \right)  - \frac{1}{2}\mu  \alpha^{-1} \|w_t^{\mathrm{md}, m} - w^*\|^2
    \\
         & \quad + \alpha^{-1} \left\langle \nabla F(w_t^{\mathrm{md}, m}) ,  \alpha \beta^{-1} w_t^m + (1 - \alpha \beta^{-1}) w_t^{\mathrm{ag},m} - w^* \right\rangle,
    \label{eq:fedaci:conv:2:2}
  \end{align}
  where the last equality is due to the definition of $w_t^{\mathrm{md}, m}$. Plugging \cref{eq:fedaci:conv:2:2} to \cref{eq:fedaci:conv:2:1} completes the proof of \cref{fedaci:conv:2:claim}.
\end{proof}
Now we complete the proof of \cref{fedaci:conv:2} by assembling the bound for all workers in \cref{fedaci:conv:2:claim}.
\begin{proof}[Proof of \cref{fedaci:conv:2}]
  If $t+1$ is a synchronized step, then $w_{t+1}^{\mathrm{ag}, m} = \overline{v_{t+1}^{\mathrm{ag}}}$ for all $m$. Then
  by convexity,
  \begin{equation}
    \frac{1}{M} \sum_{m=1}^M F(w_{t+1}^{\mathrm{ag,m}})
    = \frac{1}{M} \cdot M \cdot F \left( \overline{v_{t+1}^{\mathrm{ag}}} \right)
    = F \left( \overline{v_{t+1}^{\mathrm{ag}}} \right)
    \leq \frac{1}{M} \sum_{m=1}^M F (v_{t+1}^{\mathrm{ag,m}}).
  \end{equation}
  If $t+1$ is not a synchronized step, then trivially $ \frac{1}{M} \sum_{m=1}^M F(w_{t+1}^{\mathrm{ag,m}}) = \frac{1}{M} \sum_{m=1}^M F (v_{t+1}^{\mathrm{ag,m}})$. 
  
  Hence in either case
  \begin{equation}
    \frac{1}{M} \sum_{m=1}^M F(w_{t+1}^{\mathrm{ag,m}}) \leq \frac{1}{M} \sum_{m=1}^M F (v_{t+1}^{\mathrm{ag,m}}).
  \end{equation}
  Now we average the bounds of \cref{fedaci:conv:2:claim} for $m = 1,\ldots,M$, which gives
  \begin{align}
         & \expt \left[ \frac{1}{M} \sum_{m=1}^M F( {w_{t+1}^{\mathrm{ag}, m}}) - F^* \middle|\mathcal{F}_t \right]
    \leq \expt \left[ \frac{1}{M} \sum_{m=1}^M F( {v_{t+1}^{\mathrm{ag}, m}}) - F^* \middle|\mathcal{F}_t \right]
    \\
    \leq & (1 - \alpha^{-1}) \left( \frac{1}{M} \sum_{m=1}^M F( {w_{t}^{\mathrm{ag}, m}}) - F^* \right)
    - \frac{1}{2} \eta  \cdot \frac{1}{M} \sum_{m=1}^M   \left\|\nabla F (w_t^{\mathrm{md}, m})  \right\|^2
    +  \frac{1}{2} \eta^2 L \sigma^2
    \\
         & + \alpha^{-1} \frac{1}{M} \sum_{m=1}^M \left\langle \nabla F(w_t^{\mathrm{md}, m}), \alpha \beta^{-1} w_t^m + (1 - \alpha \beta^{-1}) w_t^{\mathrm{ag},m} - w^* \right\rangle
    - \frac{1}{2} \mu \alpha^{-1} \frac{1}{M} \sum_{m=1}^M   \| {w_t^{\mathrm{md},m}} - w^*\|^2
    \\
    \leq & (1 - \alpha^{-1}) \left( \frac{1}{M} \sum_{m=1}^M F( {w_{t}^{\mathrm{ag}, m}}) - F^* \right)
    - \frac{1}{2} \eta \left\| \frac{1}{M} \sum_{m=1}^M  \nabla F (w_t^{\mathrm{md}, m})  \right\|^2
    +  \frac{1}{2} \eta^2 L \sigma^2
    \\
         & + \alpha^{-1} \frac{1}{M} \sum_{m=1}^M \left\langle \nabla F(w_t^{\mathrm{md}, m}), \alpha \beta^{-1} w_t^m + (1 - \alpha \beta^{-1}) w_t^{\mathrm{ag},m} - w^*\right\rangle
      - \frac{1}{2} \mu \alpha^{-1} \| \overline{w_t^{\mathrm{md}}} - w^*\|^2,
  \end{align}
  where the last inequality is due to Jensen's inequality on the convex function $\|\cdot\|^2$.
\end{proof}
\subsection{Discrepancy overhead bound for \fedaci: Proof of \cref{fedaci:stab:main}}
\label{sec:fedaci:stab:main}
In this subsection we prove \cref{fedaci:stab:main} regarding the growth of discrepancy overhead introduced in \cref{fedaci:conv:main}.

We first introduce a few more notations to simplify the discussions throughout this subsection. Let $m_1, m_2 \in [M]$ be two arbitrary distinct workers. For any timestep $t$, denote $\Delta_{t} := w_{t}^{m_1} - w_t^{m_2}$,  $\Delta_t^{\mathrm{ag}} := w_{t}^{\mathrm{ag}, m_1} - w_t^{\mathrm{ag}, m_2}$ and $\Delta_t^{\mathrm{md}} := w_{t}^{\mathrm{md}, m_1} - w_t^{\mathrm{md}, m_2}$ be the corresponding vector differences. Let $\Delta_t^{\varepsilon} = \varepsilon_t^{m_1} - \varepsilon_t^{m_2}$, where $\varepsilon_t^m := \nabla f(w_t^{\mathrm{md},m}; \xi_t^m) - \nabla F(w_t^{\mathrm{md}, m})$ be the noise of the stochastic gradient oracle of the $m$-th worker evaluated at $w_t^{\mathrm{md}}$.

The proof of \cref{fedaci:stab:main} is based on the following propositions.

The following \cref{fedaci:stab:1} studies the growth of $\begin{bmatrix} \Delta_{t}^{\mathrm{ag}} \\ \Delta_{t} \end{bmatrix}$ at each step. The proof of \cref{fedaci:stab:1} is deferred to \cref{sec:fedaci:stab:1}.
\begin{proposition}
  \label{fedaci:stab:1}
  In the same setting of \cref{fedaci:stab:main},
  suppose $t+1$ is not a synchronized step,
  then there exists a matrix $H_t$ such that $\mu I \preceq H_t \preceq LI$ satisfying
  \begin{equation}
    \begin{bmatrix}
      \Delta_{t+1}^{\mathrm{ag}} \\ \Delta_{t+1}
    \end{bmatrix}
    =
    \mathcal{A} (\mu, \gamma, \eta, H_t)
    \begin{bmatrix} \Delta_{t}^{\mathrm{ag}} \\ \Delta_{t} \end{bmatrix}
    -
    \begin{bmatrix} \eta I \\ \gamma I \end{bmatrix}
    \Delta_t^{\varepsilon},
  \end{equation}
  where $\mathcal{A}(\mu, \gamma, \eta, H)$ is a matrix-valued function defined as
  \begin{equation}
    \mathcal{A}(\mu, \gamma, \eta, H) = \frac{1}{1 + \gamma\mu}
    \begin{bmatrix}
      I- \eta H            & \gamma \mu (I- \eta H)
      \\
      - \gamma (H - \mu I) & I - \gamma^2 \mu H
    \end{bmatrix}.
    \label{eq:fedaci:A:def}
  \end{equation}
\end{proposition}
Let us pause for a moment and discuss the intuition of the next steps of our plan. Our goal is to bound the product of several $\mathcal{A}(\mu, \gamma, \eta, H_i)$ where the $H_i$ matrix may be different. The natural idea is to bound the uniform norm bound of $\mathcal{A}$ for some norm $\| \cdot \|_{\star}$: $\sup_{\mu I \preceq H \preceq LI} \|\mathcal{A}\|_{\star}$. It is worth noticing that the matrix operator norm will not give the desired bound --- $\sup_{\mu I \preceq H \preceq LI} \|\mathcal{A}\|_2$ is not sufficiently small for our purpose. Our approach is to leverage the ``transformed'' norm \citep{Golub.VanLoan-13} $\|\mathcal{A}\|_{\mathcal{X}} := \| \mathcal{X}^{-1} \mathcal{A} \mathcal{X} \|_2$ for certain non-singular $\mathcal{X}$ and analyze the uniform norm bound for $\sup_{\mu I \preceq H \preceq LI} \| \mathcal{X}^{-1} \mathcal{A} \mathcal{X} \|_2$. 

Formally, the following \cref{fedaci:stab:bound} studies the uniform norm bound of $\mathcal{A}$ under the proposed transformation $\mathcal{X}$. The proof of \cref{fedaci:stab:bound} is deferred to \cref{sec:fedaci:stab:bound}.
\begin{proposition}[Uniform norm bound of $\mathcal{A}$ under transformation $\mathcal{X}$]
  \label{fedaci:stab:bound}
  Let $\mathcal{A}(\mu, \gamma, \eta, H)$ be defined in \cref{eq:fedaci:A:def}.
  and assume $\mu > 0$, $\gamma \in [\eta, \sqrt{\frac{\eta}{\mu}}]$, $\eta \in (0,\frac{1}{L}]$.
  Then the following uniform norm bound holds
  \begin{equation}
    \sup_{\mu I \preceq H \preceq LI}
    \left\| \mathcal{X}(\gamma, \eta)^{-1} \mathcal{A}(\mu, \gamma, \eta, H) \mathcal{X}(\gamma, \eta) \right\| \leq
    \begin{cases}
      1 + \frac{2\gamma^2 \mu}{\eta} & \text{if~} \gamma \in \left(\eta, \sqrt{\frac{\eta}{\mu}}\right], \\
      1                              & \text{if~} \gamma =  \eta,
    \end{cases}
  \end{equation}
  where $\mathcal{X} (\gamma, \eta)$ is a matrix-valued function defined as
  \begin{equation}
    \mathcal{X}(\gamma, \eta) :=
    \begin{bmatrix}
      \frac{\eta}{\gamma} I & 0
      \\
      I                     & I
    \end{bmatrix}.
    \label{eq:fedaci:X:def}
  \end{equation}
\end{proposition}

\cref{fedaci:stab:1,fedaci:stab:bound} suggest the one step growth of $ \left\| \mathcal{X}(\gamma, \eta)^{-1} \begin{bmatrix}
    \Delta_{t}^{\mathrm{ag}}
    \\
    \Delta_{t}
  \end{bmatrix}  \right\|^2 $ as follows.
\begin{proposition}
  \label{fedaci:stab:2}
  In the same setting of \cref{fedaci:stab:main}, the following inequality holds (for all possible $t$)
  \begin{equation}
    \expt \left[ \left\| \mathcal{X}(\gamma, \eta)^{-1} \begin{bmatrix}
        \Delta_{t+1}^{\mathrm{ag}}
        \\
        \Delta_{t+1}
      \end{bmatrix}  \right\|^2 \middle| \mathcal{F}_t \right]
    \leq
    2\gamma^2 \sigma^2 +
    \left\| \mathcal{X}(\gamma, \eta)^{-1} \begin{bmatrix}
      \Delta_{t}^{\mathrm{ag}}
      \\
      \Delta_{t}
    \end{bmatrix}  \right\|^2
    \cdot
    \begin{cases}
      \left(1 + \frac{2\gamma^2 \mu}{\eta} \right)^2 & \text{if~} \gamma \in \left(\eta, \sqrt{\frac{\eta}{\mu}}\right], \\
      1                                              & \text{if~} \gamma =  \eta,
    \end{cases}
  \end{equation}
  where $\mathcal{X}$ is the matrix-valued function defined in \cref{eq:fedaci:X:def}.
\end{proposition}
The proof of \cref{fedaci:stab:2} is deferred to \cref{sec:fedaci:stab:2}.

The following \cref{fedaci:stab:3} relates the discrepancy overhead we wish to bound for \cref{fedaci:stab:main} with the quantity analyzed in \cref{fedaci:stab:2}.
The proof of \cref{fedaci:stab:3} is deferred to \cref{sec:fedaci:stab:3}.
\begin{proposition}
  \label{fedaci:stab:3}
  In the same setting of \cref{fedaci:stab:main}, the following inequality holds (for all $t$)
    \begin{equation}
      \frac{1}{M} \sum_{m=1}^M
      \left\| \overline{w_t^{\mathrm{md}}} - w_t^{\mathrm{md}, m}  \right\|
      \left\|  \frac{1}{1+\gamma\mu}(\overline{w_t} - w_t^m) + \frac{\gamma \mu}{1+\gamma\mu} (\overline{w_t^{\mathrm{ag}}} - w_t^{\mathrm{ag},m})  \right\|
      \leq \frac{\sqrt{10} \eta}{\gamma}
      \left\| \mathcal{X}(\gamma, \eta)^{-1} \begin{bmatrix} \Delta_t^{\mathrm{ag}} \\ \Delta_t \end{bmatrix} \right\|^2,
    \end{equation}
  where $\mathcal{X}$ is the matrix-valued function defined in \cref{eq:fedaci:X:def}.
\end{proposition}

We are ready to finish the proof of \cref{fedaci:stab:main}.
\begin{proof}[Proof of \cref{fedaci:stab:main}]
  Let $t_0$ be the latest synchronized step prior to $t$ (note that the initial state $t = 0$ is always synchronized so $t_0$ is well-defined),
  then telescoping \cref{fedaci:stab:2} from $t_0$ to $t$ gives (note that $\Delta_{t_0}^{\mathrm{ag}} = \Delta_{t_0} = 0$ due to synchronization)
  \begin{align}
    \expt \left[ \left\| \mathcal{X}(\gamma, \eta)^{-1} \begin{bmatrix}
      \Delta_{t}^{\mathrm{ag}}
     \\
      \Delta_{t}
    \end{bmatrix}  \right\|^2 \middle| \mathcal{F}_{t_0} \right]
  & \leq
  2 \gamma^2 \sigma^2 (t - t_0) 
  \cdot
  \begin{cases}
    \left( 1 + \frac{2\gamma^2 \mu}{\eta} \right)^{2(t-t_0)} & \text{if~} \gamma \in \left(\eta, \sqrt{\frac{\eta}{\mu}} \right], \\
    1                              & \text{if~} \gamma =  \eta
  \end{cases}
  \\
  & \leq 
  2 \gamma^2 \sigma^2 K 
  \cdot
  \begin{cases}
    \left( 1 + \frac{2\gamma^2 \mu}{\eta} \right)^{2K} & \text{if~} \gamma \in \left(\eta, \sqrt{\frac{\eta}{\mu}} \right], \\
    1                              & \text{if~} \gamma =  \eta,
  \end{cases}
  \end{align}
  where the last inequality is due to $t - t_0 \leq K$ since $K$ is the synchronization interval.
  
  Consequently, by \cref{fedaci:stab:3} we have
  \begin{align}
    & \frac{1}{M} \sum_{m=1}^M
    \expt \left[ \left\| \overline{w_t^{\mathrm{md}}} - w_t^{\mathrm{md}, m}  \right\|
    \left\|  \frac{1}{1+\gamma\mu}(\overline{w_t} - w_t^m) + \frac{\gamma \mu}{1+\gamma\mu} (\overline{w_t^{\mathrm{ag}}} - w_t^{\mathrm{ag},m})  \right\| \middle| \mathcal{F}_{t_0} \right]
    \\
    \leq & \frac{\sqrt{10} \eta}{\gamma} \expt \left[ \left\| \mathcal{X}(\gamma, \eta)^{-1} \begin{bmatrix}
      \Delta_{t}^{\mathrm{ag}}
     \\
      \Delta_{t}
    \end{bmatrix}  \right\|^2 \middle| \mathcal{F}_{t_0} \right]
    \leq 
    \begin{cases}
      7 \eta \gamma K \sigma^2 \left(1 + \frac{2\gamma^2\mu}{\eta}\right)^{2K}
       & \text{if~} \gamma \in \left(\eta, \sqrt{\frac{\eta}{\mu}}\right],
      \\
      7 \eta^2 K \sigma^2
       &
      \text{if~} \gamma = \eta,
    \end{cases}
   \end{align}
  where in the last inequality we used the estimate that $2\sqrt{10} < 7$. 
\end{proof}

\subsubsection{Proof of \cref{fedaci:stab:1}}
\label{sec:fedaci:stab:1}
In this section we will prove \cref{fedaci:stab:1}.
Let us first state and prove a more general version of \cref{fedaci:stab:1} regarding \fedac with general hyperparameter assumptions $\alpha \geq 1$, $\beta \geq 1$ .
\begin{claim}
  \label{fedac:general:stab}
  Assume \cref{asm1} and assume $F$ to be $\mu > 0$-strongly convex. 
  Suppose $t+1$ is not a synchronized step, then there exists a matrix $H_t$ such that $\mu I \preceq H_t \preceq LI$ satisfying
  \begin{align}
    \begin{bmatrix} \Delta_{t+1}^{\mathrm{ag}} \\ \Delta_{t+1} \end{bmatrix}
     & =
    \begin{bmatrix}
      (1 - \beta^{-1}) (I - \eta H_t)
       &
      \beta^{-1} (I - \eta H_t)
      \\
      (1 - \beta^{-1}) (\alpha^{-1} - \gamma H_t)
       &
      \beta^{-1} (\alpha^{-1} I - \gamma H_t) + (1 - \alpha^{-1}) I
    \end{bmatrix}
    \begin{bmatrix} \Delta_{t}^{\mathrm{ag}} \\ \Delta_{t} \end{bmatrix}
    -
    \begin{bmatrix} \eta I \\  \gamma I \end{bmatrix} \Delta_t^{\varepsilon}.
    \label{eq:fedac:general:stab}
  \end{align}
\end{claim}
\begin{proof}[Proof of \cref{fedac:general:stab}]
  First note that \fedac can be written as the following two-point recursions.
  \begin{align}
    w_{t+1}^{\mathrm{ag}, m} & =  (1 - \beta^{-1}) w_t^{\mathrm{ag}, m} + \beta^{-1} w_t^m  - \eta \cdot \nabla F ( w_t^{\mathrm{md}, m} ) - \eta \varepsilon_t^m;
    \\
    w_{t+1}^m                & = \alpha^{-1} w_t^{\mathrm{md}, m} + (1 - \alpha^{-1}) w_t^m - \gamma \cdot \nabla F(w_t^{\mathrm{md}, m}) - \gamma \varepsilon_t^m
    \\
                             & =  \alpha^{-1} (1 - \beta^{-1}) w_t^{\mathrm{ag}, m} + (1 - \alpha^{-1} + \alpha^{-1} \beta^{-1}) w_t^m - \gamma \cdot \nabla F(w_t^{\mathrm{md}, m}) - \gamma \varepsilon_t^m.
  \end{align}
  Taking difference gives
  \begin{align}
    \Delta_{t+1}^{\mathrm{ag}} & = (1 - \beta^{-1}) \Delta_t^{\mathrm{ag}}  + \beta^{-1} \Delta_t
    - \eta \left( \nabla F ( w_t^{\mathrm{md}, m_1} )  -  \nabla F ( w_t^{\mathrm{md}, m_2} ) \right)
    - \eta \Delta_t^{\varepsilon};
    \\
    \Delta_{t+1}               & = \alpha^{-1} (1 - \beta^{-1}) \Delta_t^{\mathrm{ag}} + (1 - \alpha^{-1} + \alpha^{-1} \beta^{-1}) \Delta_t - \gamma \left( \nabla F ( w_t^{\mathrm{md}, m_1} )  -  \nabla F ( w_t^{\mathrm{md}, m_2} ) \right)
    - \gamma  \Delta_t^{\varepsilon}.
  \end{align}
  By mean-value theorem, there exists a symmetric positive-definite matrix $H_t$ such that $\mu I \preceq H_t \preceq L I$ satisfying
  \begin{equation}
    \nabla F(w_t^{\mathrm{md}, {m_1}}) -  \nabla F(w_t^{\mathrm{md}, {m_2}})
    =
    H_t \Delta_t^{\mathrm{md}}
    =
    H_t \left( (1 - \beta^{-1}) \Delta_t^{\mathrm{ag}} +  \beta^{-1} \Delta_t \right).
  \end{equation}
  Thus
  \begin{align}
    \Delta_{t+1}^{\mathrm{ag}} & = (1 - \beta^{-1}) \Delta_t^{\mathrm{ag}}  + \beta^{-1} \Delta_t
    - \eta H_t \left( (1 - \beta^{-1}) \Delta_t^{\mathrm{ag}} +  \beta^{-1} \Delta_t \right)
    - \eta \Delta_t^{\varepsilon}
    \\
    \Delta_{t+1}               & = \alpha^{-1} (1 - \beta^{-1}) \Delta_t^{\mathrm{ag}} + (1 - \alpha^{-1} + \alpha^{-1} \beta^{-1}) \Delta_t
    - \gamma H_t \left( (1 - \beta^{-1}) \Delta_t^{\mathrm{ag}} +  \beta^{-1} \Delta_t \right)
    - \gamma \Delta_t^{\varepsilon}
  \end{align}
  Rearranging into matrix form completes the proof of  \cref{fedac:general:stab}.
\end{proof}
\cref{fedaci:stab:1} is a special case of \cref{fedac:general:stab}.
\begin{proof}[Proof of \cref{fedaci:stab:1}]
  The proof follows instantly by applying \cref{fedac:general:stab} with particular choice $\alpha = \frac{1}{\gamma \mu}$ and $\beta = \alpha + 1 = \frac{1 + \gamma \mu}{\gamma \mu}$.
\end{proof}

\subsubsection{Proof of \cref{fedaci:stab:bound}: uniform norm bound}
\label{sec:fedaci:stab:bound}
\begin{proof}[Proof of \cref{fedaci:stab:bound}]
  Define another matrix-valued function $\mathcal{B}$ as
  \begin{equation}
    \mathcal{B}(\mu, \gamma, \eta, H) := \mathcal{X}(\gamma, \eta)^{-1} \mathcal{A}(\mu, \gamma, \eta, H) \mathcal{X}(\gamma, \eta).
  \end{equation}
  Since $ \mathcal{X}(\gamma, \eta)^{-1} =
  \begin{bmatrix}
    \frac{\gamma}{\eta}I  & 0
    \\
    -\frac{\gamma}{\eta}I & I
  \end{bmatrix}$ we can compute that
  \begin{equation}
    \mathcal{B}(\mu, \gamma, \eta, H) = \frac{1}{(1 + \gamma\mu)\eta}
    \begin{bmatrix}
      (\eta + \gamma^2 \mu) (I - \eta H) & \gamma^2 \mu (I - \eta H)
      \\
      - \mu (\gamma^2 - \eta^2) I        & \eta - \gamma^2 \mu
    \end{bmatrix}.
  \end{equation}
  Define the four blocks of $\mathcal{B}(\mu, \gamma, \eta, H)$ as $\mathcal{B}_{11}(\mu, \gamma, \eta, H)$, $\mathcal{B}_{12}(\mu, \gamma, \eta, H)$, $\mathcal{B}_{21}(\mu, \gamma, \eta)$, $\mathcal{B}_{22}(\mu, \gamma, \eta)$ (note that the lower two blocks do not involve $H$), \ie,
  \begin{align}
     & \mathcal{B}_{11}(\mu, \gamma, \eta, H) = \frac{\eta + \gamma^2 \mu}{(1 + \gamma \mu)\eta}(I - \eta H),
     &
     & \mathcal{B}_{12}(\mu, \gamma, \eta, H)  = \frac{\gamma^2 \mu}{(1 + \gamma \mu) \eta}(I - \eta H),
    \\
     & \mathcal{B}_{21}(\mu, \gamma, \eta)    = -\frac{\mu (\gamma^2 - \eta^2)}{(1 + \gamma \mu) \eta}I,
     &
     & \mathcal{B}_{22}(\mu, \gamma, \eta)    = \frac{\eta - \gamma^2 \mu}{(1 + \gamma \mu) \eta}I.
  \end{align}
  \paragraph{Case I: $\eta < \gamma \leq \sqrt{\frac{\eta}{\mu}}$.} In this case we have
  \begin{align}
    \|\mathcal{B}_{11}(\mu, \gamma, \eta, H)\| &
    \leq \frac{\eta + \gamma^2 \mu}{(1 + \gamma\mu)\eta}(1 - \eta \mu)
    \leq \frac{\eta + \gamma^2 \mu}{\eta}
    = 1 + \frac{\gamma^2 \mu}{\eta},
    \tag{since $\eta \mu \leq 1$}
    \\
    \|\mathcal{B}_{12}(\mu, \gamma, \eta, H)\| &
    \leq \frac{\gamma^2 \mu}{(1 + \gamma\mu)\eta}(1 - \eta \mu)
    \leq \frac{\gamma^2 \mu}{\eta},
    \tag{since $\eta \mu \leq 1$}
    \\
    \|\mathcal{B}_{21}(\mu, \gamma, \eta)\|    &
    = \frac{\mu (\gamma^2 - \eta^2)}{(1 + \gamma \mu) \eta}
    \leq \frac{\gamma^2 \mu}{\eta},
    \tag{since $\eta < \gamma \leq \sqrt{\frac{\eta}{\mu}}$}
    \\
    \|\mathcal{B}_{22}(\mu, \gamma, \eta)\|    & = \frac{\eta - \gamma^2\mu}{(1 + \gamma \mu)\eta} \leq \frac{1}{1 + \gamma \mu}\leq 1.
    \tag{since $\gamma \leq \sqrt{\frac{\eta}{\mu}}$}
  \end{align}
  The operator norm of $\mathcal{B}$ can be bounded via its blocks via helper \cref{helper:blocknorm} as
  \begin{align}
         & \mathcal{B}(\mu, \gamma, \eta, H)
    \\
    \leq & \max \left\{\| \mathcal{B}_{11}(\mu, \gamma, \eta, H) \|, \| \mathcal{B}_{22}(\mu, \gamma, \eta) \| \right\}
    +
    \max \left\{\| \mathcal{B}_{12}(\mu, \gamma, \eta, H) \|, \| \mathcal{B}_{21}(\mu, \gamma, \eta) \| \right\} \tag{\cref{helper:blocknorm}}
    \\
    \leq & \max \left\{ 1 + \frac{\gamma^2 \mu}{\eta}, 1 \right\} + \max \left\{ \frac{\gamma^2 \mu}{\eta}, \frac{\gamma^2 \mu}{\eta} \right\} = 1 + \frac{2\gamma^2\mu}{\eta}.
  \end{align}
  \paragraph{Case II: $\gamma = \eta$.} In this case we have
  \begin{align}
    \|\mathcal{B}_{11}(\mu, \gamma, \eta, H)\| &
    \leq \frac{\eta + \eta^2 \mu}{(1 + \eta\mu)\eta}(1 - \eta \mu)
    = 1 - \eta \mu,
    \\
    \|\mathcal{B}_{12}(\mu, \gamma, \eta, H)\| &
    \leq \frac{\eta^2 \mu}{(1 + \eta\mu)\eta}(1 - \eta \mu)
    = \frac{(1 - \eta \mu)\eta \mu}{1 + \eta \mu},
    \\
    \|\mathcal{B}_{21}(\mu, \gamma, \eta)\|    &
    = 0,
    \\
    \|\mathcal{B}_{22}(\mu, \gamma, \eta)\|    & = \frac{\eta - \eta^2\mu}{(1 + \eta \mu)\eta}
    = \frac{1 - \eta \mu}{1 + \eta \mu}.
  \end{align}
   Similarly the operator norm of block matrix $\mathcal{B}$ can be bounded via its blocks via helper \cref{helper:blocknorm} as
  \begin{align}
         & \mathcal{B}(\mu, \gamma, \eta, H)
    \\
    \leq & \max \left\{\| \mathcal{B}_{11}(\mu, \gamma, \eta, H) \|, \| \mathcal{B}_{22}(\mu, \gamma, \eta) \| \right\}
    +
    \max \left\{\| \mathcal{B}_{12}(\mu, \gamma, \eta, H) \|, \| \mathcal{B}_{21}(\mu, \gamma, \eta) \| \right\} \tag{\cref{helper:blocknorm}}
    \\
    \leq & \max \left\{ 1 - \eta \mu, \frac{1 - \eta \mu}{1 + \eta \mu} \right\} + \frac{\eta \mu ( 1- \eta \mu)}{1 + \eta \mu} =  1 - \eta \mu + \frac{\eta \mu ( 1- \eta \mu)}{1 + \eta \mu}
    =
    \frac{1 + \eta \mu - 2 \eta^2 \mu^2}{1 + \eta \mu} \leq 1.
  \end{align}
  Summarizing the above two cases completes the proof of \cref{fedaci:stab:bound}.
\end{proof}

\subsubsection{Proof of \cref{fedaci:stab:2}}
\label{sec:fedaci:stab:2}
In this section we apply \cref{fedaci:stab:1,fedaci:stab:bound} to establish \cref{fedaci:stab:2}.
\begin{proof}[Proof of \cref{fedaci:stab:2}]
  If $t+1$ is a synchronized step, then the bound trivially holds since $\Delta_{t+1}^{\mathrm{ag}} = \Delta_{t+1} = 0$ due to synchronization.

  Now assume $t+1$ is not a synchronized step, for which \cref{fedaci:stab:1} is applicable. Multiplying $\mathcal{X}(\gamma, \eta)^{-1}$ to the left on both sides of \cref{fedaci:stab:1} gives
  \begin{align}
    \mathcal{X}(\gamma, \eta)^{-1} \begin{bmatrix}
      \Delta_{t+1}^{\mathrm{ag}}
      \\
      \Delta_{t+1}
    \end{bmatrix}
     & =
    \mathcal{X}(\gamma, \eta)^{-1}  \mathcal{A}(\mu, \gamma, \eta, H)
    \begin{bmatrix}
      \Delta_{t}^{\mathrm{ag}}
      \\
      \Delta_{t}
    \end{bmatrix}
    -
    \mathcal{X}(\gamma, \eta)^{-1}  \begin{bmatrix} \eta I \\ \gamma I \end{bmatrix} \Delta_t^{\varepsilon}
    \\
     & =
    \mathcal{X}(\gamma, \eta)^{-1} \mathcal{A}(\mu, \gamma, \eta, H_t) \mathcal{X}(\gamma, \eta)^{-1} \left( \mathcal{X}(\gamma, \eta) \begin{bmatrix}
        \Delta_{t}^{\mathrm{ag}}
        \\
        \Delta_{t}
      \end{bmatrix}  \right)
    -
    \begin{bmatrix}   \gamma I \\ 0 \end{bmatrix} \Delta_t^{\varepsilon},
  \end{align}
  where the last equality is due to
  \begin{equation}
    \mathcal{X}(\gamma, \eta)^{-1} =
    \begin{bmatrix}
      \frac{\gamma}{\eta}I  & 0
      \\
      -\frac{\gamma}{\eta}I & I
    \end{bmatrix},\qquad
    \mathcal{X}(\gamma, \eta)^{-1}
    \begin{bmatrix} \eta I \\ \gamma I \end{bmatrix}
    =
    \begin{bmatrix}
      \gamma I \\ 0
    \end{bmatrix}.
  \end{equation}
  Taking conditional expectation,
  \begin{align}
         & \expt \left[ \left\| \mathcal{X}(\gamma, \eta)^{-1} \begin{bmatrix}
        \Delta_{t+1}^{\mathrm{ag}}
        \\
        \Delta_{t+1}
      \end{bmatrix}  \right\|^2 \middle| \mathcal{F}_t \right]
    \\
    =    & \left\| \mathcal{X}^{-1} \mathcal{A} \mathcal{X}
    \left( \mathcal{X}^{-1}
    \begin{bmatrix}
      \Delta_{t}^{\mathrm{ag}} \\ \Delta_{t}
    \end{bmatrix} \right) \right\|^2
    +
    \expt \left[\left\| \begin{bmatrix}
        \gamma I \\ 0
      \end{bmatrix}
      \Delta_t^{\varepsilon} \right\|^2 \middle| \mathcal{F}_t  \right]
    \tag{independence}
    \\
    \leq & \| \mathcal{X}^{-1} \mathcal{A} \mathcal{X} \|^2
    \left\|  \mathcal{X}^{-1}
    \begin{bmatrix}
      \Delta_{t}^{\mathrm{ag}} \\ \Delta_{t}
    \end{bmatrix}\right\|^2
    + 2\gamma^2 \sigma^2
    \tag{bounded variance, sub-multiplicativity}
    \\
    \leq & 2\gamma^2 \sigma^2 +
    \left\| \mathcal{X}(\gamma, \eta)^{-1} \begin{bmatrix}
      \Delta_{t}^{\mathrm{ag}}
      \\
      \Delta_{t}
    \end{bmatrix}  \right\|^2
    \cdot
    \begin{cases}
      \left(1 + \frac{2\gamma^2 \mu}{\eta} \right)^2 & \text{if~} \gamma \in \left(\eta, \sqrt{\frac{\eta}{\mu}}\right], \\
      1                                              & \text{if~} \gamma =  \eta.
    \end{cases}
    \tag{by \cref{fedaci:stab:bound}}
  \end{align}
\end{proof}

\subsubsection{Proof of \cref{fedaci:stab:3}}
\label{sec:fedaci:stab:3}
In this section we will prove \cref{fedaci:stab:3} in three steps via the following three claims. For all the three claims $\mathcal{X}$ stands for the matrix-valued functions defined in \cref{eq:fedaci:X:def}.
\begin{claim}
  \label{fedaci:stab:3:1}
  In the same setting of  \cref{fedaci:stab:3}, 
  \begin{align}
        &  \frac{1}{M} \sum_{m=1}^M
         \left\| \overline{w_t^{\mathrm{md}}} - w_t^{\mathrm{md}, m}  \right\|
         \left\|  \frac{1}{1+\gamma\mu}(\overline{w_t} - w_t^m) + \frac{\gamma \mu}{1+\gamma\mu} (\overline{w_t^{\mathrm{ag}}} - w_t^{\mathrm{ag},m})  \right\|
    \\
    \leq &  \left\| \begin{bmatrix} \frac{1}{1 + \gamma \mu} I \\ \frac{\gamma \mu}{1 + \gamma \mu} I \end{bmatrix}^\intercal \mathcal{X}(\gamma, \eta) \right\| \cdot
    \left\| \begin{bmatrix} \frac{\gamma \mu}{1 + \gamma \mu} I \\ \frac{1}{1 + \gamma \mu} I \end{bmatrix}^\intercal \mathcal{X}(\gamma, \eta) \right\| \cdot
    \left\| \mathcal{X}(\gamma, \eta)^{-1}  \begin{bmatrix} \Delta_{t}^{\mathrm{ag}} \\ \Delta_{t} \end{bmatrix}
    \right\|^2.
  \end{align}
\end{claim}
\begin{claim}
  \label{fedaci:stab:left1}
  Assume $\mu > 0$, $\gamma \in [\eta, \sqrt{\frac{\eta}{\mu}}]$, $\eta \in (0,\frac{1}{L}]$, then
  \(
    \left\| \mathcal{X}(\gamma, \eta)^\intercal
  \begin{bmatrix} \frac{1}{1 + \gamma \mu} I \\ \frac{\gamma \mu}{1 + \gamma \mu} I \end{bmatrix} \right\| \leq \frac{\sqrt{5} \eta}{\gamma}.
  \)
\end{claim}
\begin{claim}
  \label{fedaci:stab:left2}
  Assume $\mu > 0$, $\gamma \in [\eta, \sqrt{\frac{\eta}{\mu}}]$, $\eta \in (0,\frac{1}{L}]$, then
  \(
    \left\| \mathcal{X}(\gamma, \eta)^\intercal
    \begin{bmatrix} \frac{\gamma \mu}{1 + \gamma \mu} I \\ \frac{1}{1 + \gamma \mu} I \end{bmatrix}
    \right\| \leq \sqrt{2}.
  \)
\end{claim}
\cref{fedaci:stab:3} follows immediately once we have \cref{fedaci:stab:3:1,fedaci:stab:left1,fedaci:stab:left2}.
\begin{proof}[Proof of \cref{fedaci:stab:3}]
  Follows trivially with \cref{fedaci:stab:3:1,fedaci:stab:left1,fedaci:stab:left2}.
  \begin{equation}
    \frac{1}{M} \sum_{m=1}^M
    \left\| \overline{w_t^{\mathrm{md}}} - w_t^{\mathrm{md}, m}  \right\|
    \left\|  \frac{1}{1+\gamma\mu}(\overline{w_t} - w_t^m) + \frac{\gamma \mu}{1+\gamma\mu} (\overline{w_t^{\mathrm{ag}}} - w_t^{\mathrm{ag},m})  \right\|
    \leq \frac{\sqrt{10} \eta}{\gamma}
    \left\| \mathcal{X}(\gamma, \eta)^{-1} \begin{bmatrix} \Delta_t^{\mathrm{ag}} \\ \Delta_t \end{bmatrix} \right\|^2.
  \end{equation}
\end{proof}
Now we finish the proof of the three claims.
\begin{proof}[Proof of \cref{fedaci:stab:3:1}]
  Note that 
  \begin{align}
         & \frac{1}{M} \sum_{m=1}^M  \left\| \overline{w_t^\mathrm{md}} - w_t^{\mathrm{md},m } \right\|^2
          \leq \|\Delta_t^{\mathrm{md}}\|^2
    \tag{convexity of $\|\cdot\|^2$}
    \\
    =    & \left\|
    \begin{bmatrix} (1 - \beta^{-1}) I \\ \beta^{-1} I \end{bmatrix}^\intercal
    \begin{bmatrix} \Delta_{t}^{\mathrm{ag}} \\ \Delta_{t} \end{bmatrix}
    \right\|^2
        = 
      \left\|
    \begin{bmatrix} \frac{1}{1 + \gamma \mu} I \\ \frac{\gamma \mu}{1 + \gamma \mu} I \end{bmatrix}^\intercal
    \begin{bmatrix} \Delta_{t}^{\mathrm{ag}} \\ \Delta_{t} \end{bmatrix}
    \right\|^2
        \tag{definition of ``md''}
    \\
    \leq & \left\| \begin{bmatrix} \frac{1}{1 + \gamma \mu} I \\ \frac{\gamma \mu}{1 + \gamma \mu} I \end{bmatrix}^\intercal \mathcal{X}(\gamma, \eta) \right\|^2
    \left\|  \mathcal{X}(\gamma, \eta)^{-1}
    \begin{bmatrix} \Delta_{t}^{\mathrm{ag}} \\ \Delta_{t} \end{bmatrix}
    \right\|^2 ,
    \tag{sub-multiplicativity}
  \end{align}
  and similarly
  \begin{align}
         & \frac{1}{M} \sum_{m=1}^M \left\|  \frac{1}{1+\gamma\mu}(\overline{w_t} - w_t^m) + \frac{\gamma \mu}{1+\gamma\mu} (\overline{w_t^{\mathrm{ag}}} - w_t^{\mathrm{ag},m})  \right\|^2
    \\
    \leq & \left\|
      \begin{bmatrix} \frac{\gamma \mu}{1 + \gamma \mu} I \\ \frac{1}{1 + \gamma \mu} I \end{bmatrix}^\intercal
      \begin{bmatrix} \Delta_{t}^{\mathrm{ag}} \\ \Delta_{t} \end{bmatrix}
      \right\|^2 
    \tag{convexity of $\|\cdot\|^2$}
    \\
    \leq & \left\| 
    \begin{bmatrix} \frac{\gamma \mu}{1 + \gamma \mu} I \\ \frac{1}{1 + \gamma \mu} I \end{bmatrix}^\intercal \mathcal{X}(\gamma, \eta) \right\|^2
      \left\| \mathcal{X}(\gamma, \eta)^{-1}
      \begin{bmatrix} \Delta_{t}^{\mathrm{ag}} \\ \Delta_{t} \end{bmatrix}
      \right\|^2 .
      \tag{sub-multiplicativity}
  \end{align}
  Thus, by Cauchy-Schwarz inequality,
  \begin{align}
         & \frac{1}{M} \sum_{m=1}^M
         \left\| \overline{w_t^{\mathrm{md}}} - w_t^{\mathrm{md}, m}  \right\|
         \left\|  \frac{1}{1+\gamma\mu}(\overline{w_t} - w_t^m) + \frac{\gamma \mu}{1+\gamma\mu} (\overline{w_t^{\mathrm{ag}}} - w_t^{\mathrm{ag},m})  \right\|
    \\
    \leq &
    \left(\frac{1}{M} \sum_{m=1}^M  \left\| \overline{w_t^\mathrm{md}} - w_t^{\mathrm{md},m } \right\|^2 \right)^{\frac{1}{2}}
    \left(\frac{1}{M} \sum_{m=1}^M \left\|  \frac{1}{1+\gamma\mu}(\overline{w_t} - w_t^m) + \frac{\gamma \mu}{1+\gamma\mu} (\overline{w_t^{\mathrm{ag}}} - w_t^{\mathrm{ag},m})  \right\|^2 \right)^{\frac{1}{2}}
    \tag{Cauchy-Schwarz}
    \\
    \leq & \left\| \begin{bmatrix} \frac{1}{1 + \gamma \mu} I \\ \frac{\gamma \mu}{1 + \gamma \mu} I \end{bmatrix}^\intercal \mathcal{X}(\gamma, \eta) \right\| \cdot
    \left\| \begin{bmatrix} \frac{\gamma \mu}{1 + \gamma \mu} I \\ \frac{1}{1 + \gamma \mu} I \end{bmatrix}^\intercal \mathcal{X}(\gamma, \eta) \right\| \cdot
    \left\| \mathcal{X}(\gamma, \eta)^{-1}  \begin{bmatrix} \Delta_{t}^{\mathrm{ag}} \\ \Delta_{t} \end{bmatrix}
    \right\|^2,
  \end{align}
  completing the proof of \cref{fedaci:stab:3:1}.
\end{proof}

\begin{proof}[Proof of \cref{fedaci:stab:left1}]
  Direct calculation shows that
  \begin{equation}
    \mathcal{X}(\gamma, \eta)^\intercal
    \begin{bmatrix} \frac{1}{1 + \gamma\mu}  I \\ \frac{\gamma \mu}{1 + \gamma \mu} I \end{bmatrix}
    =
    \begin{bmatrix}
      \frac{\eta}{\gamma} I & I
      \\
      0                     & I
    \end{bmatrix}
    \begin{bmatrix} \frac{1}{1 + \gamma\mu}  I \\ \frac{\gamma \mu}{1 + \gamma \mu} I \end{bmatrix}
    =
    \frac{1}{1 + \gamma \mu}
    \begin{bmatrix} (\frac{\eta}{\gamma} + \gamma \mu) I \\ \gamma \mu I \end{bmatrix}.
  \end{equation}
  Since
  \begin{equation}
    \left\| \begin{bmatrix} (\frac{\eta}{\gamma} + \gamma \mu)  I \\ \gamma \mu I \end{bmatrix} \right\|
    =
    \sqrt{\left(\frac{\eta}{\gamma} + \gamma \mu\right)^2  + (\gamma \mu)^2}
    \leq
    \sqrt{\left(\frac{2\eta}{\gamma} \right)^2  + \left(\frac{\eta}{\gamma} \right)^2}
    = \frac{\sqrt{5}\eta}{\gamma}.
    \tag{since $\gamma \mu \leq \frac{\eta}{\gamma}$}
  \end{equation}
  We conclude that
  \begin{equation}
    \left\| \mathcal{X}(\gamma, \eta)^\intercal
  \begin{bmatrix} \frac{1}{1 + \gamma\mu}  I \\ \frac{\gamma \mu}{1 + \gamma \mu} I \end{bmatrix}
    \right\|
    \leq \frac{1}{1 + \gamma \mu} \cdot \frac{\sqrt{5}\eta}{\gamma} \leq \frac{\sqrt{5} \eta}{\gamma}.
  \end{equation}
\end{proof}

\begin{proof}[Proof of \cref{fedaci:stab:left2}]
  Direct calculation shows that
  \begin{equation}
    \mathcal{X}(\gamma, \eta)^\intercal
    \begin{bmatrix} \frac{\gamma \mu}{1 + \gamma \mu} I \\ \frac{1}{1 + \gamma \mu} I \end{bmatrix}
    =
    \begin{bmatrix}
      \frac{\eta}{\gamma} I & I
      \\
      0                     & I
    \end{bmatrix}
    \begin{bmatrix} \frac{\gamma \mu}{1 + \gamma \mu} I \\ \frac{1}{1 + \gamma \mu} I \end{bmatrix}
    =
    \begin{bmatrix} \frac{1 + \eta \mu}{1 + \gamma \mu} I \\  \frac{1}{1 + \gamma \mu} I \end{bmatrix},
  \end{equation}
  and
  \begin{equation}
    \left\| \begin{bmatrix} \frac{1 + \eta \mu}{1 + \gamma \mu} I \\  \frac{1}{1 + \gamma \mu} I \end{bmatrix} \right\|
    =
    \sqrt{ \left( \frac{1 + \eta \mu}{1 + \gamma \mu} \right)^2 + \left( \frac{1}{1 + \gamma \mu} \right)^2 }
    \leq \sqrt{2}
    \tag{since $\eta \leq  \gamma$},
  \end{equation}
  completing the proof of \cref{fedaci:stab:left2}.
\end{proof}
\section{Analysis of \fedacii under \cref{asm1} or \ref{asm2}}
\label{sec:fedacii}
In this section we study the convergence of \fedacii.
We provide a complete, non-asymptotic version of \cref{fedacii:a2} on the convergence of \fedacii under \cref{asm2} and provide the detailed proof, which expands the proof sketch in \cref{sec:proof:sketch:2}. 
We also study the convergence of \fedacii under \cref{asm1}, which we defer to the end of this section (see \cref{sec:fedacii:a1}) since the analysis is mostly shared.

Recall that \fedacii is defined as the \fedac algorithm with the following hyperparameter choice:
\begin{equation}
  \eta \in \left(0, \frac{1}{L}\right], \quad  \gamma = \max \left\{ \sqrt{\frac{\eta}{\mu K}}, \eta \right\}, \quad \alpha  = \frac{3}{2 \gamma \mu} - \frac{1}{2}, \quad \beta = \frac{2 \alpha^2 - 1}{\alpha - 1}
\tag{\fedacii}.
\end{equation}
As we discussed in the proof sketch \cref{sec:proof:sketch:2}, for \fedacii, we keep track of the convergence via the ``centralized'' potential $\Phi_t$.
\begin{equation}
  \Phi_t := F( \overline{w_{t}^{\mathrm{ag}}})  - F^* + \frac{1}{6}\mu \|\overline{w_{t}} - w^*\|^2.
  \label{eq:centralied:potential}
\end{equation}
Recall $\overline{w_t}$ is defined as $\frac{1}{M} \sum_{m=1}^M w_t^m$ and $\overline{w_t^{\mathrm{ag}}}$ is defined as $\frac{1}{M} \sum_{m=1}^M w_t^{\mathrm{ag},m}$. We use $\mathcal{F}_t$ to denote the $\sigma$-algebra generated by $\{w_\tau^m, w_\tau^{\mathrm{ag, m}}\}_{\tau \leq t, m \in [M]}$. Since \fedac is Markovian, conditioning on $\mathcal{F}_t$ is equivalent to conditioning on $\{w_t^m, w_t^{\mathrm{ag, m}}\}_{m \in [M]}$.

\subsection{Main theorem and lemmas: Complete version of \cref{fedacii:a2}}
Now we introduce the main theorem on the convergence of \fedacii under \cref{asm2}.
\begin{theorem}[Convergence of \fedacii under \cref{asm2}, complete version of \cref{fedacii:a2}]
  \label{fedacii:a2:full}
  Let $F$ be $\mu > 0$ strongly convex, and assume \cref{asm2}, then for 
  \begin{equation}
    \eta := \min \left\{ \frac{1}{L}, \frac{9K}{\mu T^2} \log^2 \left( \euler
     + \min \left\{ \frac{\mu M T \Phi_0}{\sigma^2} + \frac{\mu^2 MT^3 \Phi_0}{LK^2 \sigma^2}, \frac{\mu^5 T^8 \Phi_0}{Q^2 K^6 \sigma^4} \right\} \right) \right\},
  \end{equation}
  \fedacii yields
  \begin{align}
    \expt [\Phi_T] \leq & \min \left\{ \exp \left( - \frac{\mu T}{3 L} \right), \exp \left( - \frac{\mu^{\frac{1}{2}} T}{3 L^{\frac{1}{2}} K^{\frac{1}{2}}} \right)\right\} \Phi_0 
    + \frac{4 \sigma^2}{\mu M T} \log \left( \euler + \frac{\mu MT \Phi_0}{\sigma^2} \right)
    \\
    & 
    \quad +
    \frac{55 LK^2 \sigma^2}{\mu^2 MT^3} \log^3 \left( \euler + \frac{\mu^2 MT^3 \Phi_0}{LK^2\sigma^2} \right)
    + \frac{\euler^{18} Q^2 K^6 \sigma^4}{\mu^5 T^8} \log^8\left( \euler
    + \frac{\mu^5 T^8 \Phi_0}{Q^2 K^6 \sigma^4} \right),
  \end{align}
  where $\Phi_t$ is the ``centralized'' potential defined in \cref{eq:centralied:potential}.
\end{theorem}
\begin{remark}
  The simplified version \cref{fedacii:a2} in main body can be obtained by replacing $K$ with $T/R$ and upper bound $\Phi_0$ by $LD_0^2$.
\end{remark}

The proof of \cref{fedacii:a2:full} is based on the following two lemmas regarding convergence and stability respectively.
To clarify the hyperparameter dependency, we state our lemma for general $\gamma \in \left[\eta, \sqrt{ \frac{\eta}{\mu}} \right]$, which has one more degree of freedom than \fedacii where $\gamma = \max \left\{ \sqrt{\frac{\eta}{\mu K}}, \eta\right\}$ is fixed.

\begin{lemma}[Potential-based perturbed iterate analysis for \fedacii]
  \label{fedacii:conv:main}
  Let $F$ be $\mu>0$-strongly convex, and assume \cref{asm1}, then for $\alpha = \frac{3}{2 \gamma \mu} - \frac{1}{2}$, $\beta =\frac{2 \alpha^2 - 1}{\alpha - 1}$,
  $\gamma \in \left[\eta, \sqrt{ \frac{\eta}{\mu}}\right]$, $\eta \in (0, \frac{1}{L}]$, \fedac yields
  
  \begin{equation}
    \expt [\Phi_T] \leq \exp \left( - \frac{1}{3}\gamma \mu T \right) \Phi_0
    + \frac{3\eta^2 L \sigma^2 }{2 \gamma \mu M}
    + \frac{\gamma \sigma^2}{2 M}
    + \frac{3}{\mu} \max_{0 \leq t < T}   \expt \left[ \left\|  \nabla F(\overline{w_t^{\mathrm{md}}}) -  \frac{1}{M} \sum_{m=1}^M \nabla F (w_t^{\mathrm{md}, m})  \right\|^2 \right],
  \end{equation}

where $\Phi_t$ is the decentralized potential defined in \cref{eq:centralied:potential}.
\end{lemma}
The proof of \cref{fedacii:conv:main} is deferred to \cref{sec:fedacii:conv:main}. Note that \cref{fedacii:conv:main} only requires \cref{asm1} (recall that \cref{asm1} is strictly weaker than \cref{asm2}), which enables us to recycle this Lemma towards the convergence proof of \fedacii under \cref{asm1} (see \cref{sec:fedacii:a1}).

The following lemma studies the discrepancy overhead by \nth{4}-th order stability, which requires \cref{asm2}.
\begin{lemma}[Discrepancy overhead bounds]
  \label{fedacii:stab:main}
  Let $F$ be $\mu>0$-strongly convex, and assume \cref{asm2}, then for the same hyperparameter choice as in \cref{fedacii:conv:main}, \fedac satisfies (for all $t$)
  \begin{align}
     & \expt \left[ \left\|  \nabla F(\overline{w_t^{\mathrm{md}}}) -  \frac{1}{M} \sum_{m=1}^M \nabla F (w_t^{\mathrm{md}, m})  \right\|^2 \right]
    \leq  \begin{cases}
      44 \eta^4 Q^2 K^2 \sigma^4 \left(1 + \frac{\gamma^2\mu}{\eta}\right)^{4K}
       & \text{if~} \gamma \in \left(\eta, \sqrt{\frac{\eta}{\mu}} \right],
      \\
      44 \eta^4 Q^2 K^2 \sigma^4
       &
      \text{if~} \gamma = \eta.
    \end{cases}
  \end{align}
\end{lemma}
The proof of \cref{fedacii:stab:main} is deferred to \cref{sec:fedacii:stab:main}.

Now we plug in the choice of $\gamma = \max \left\{ \sqrt{\frac{\eta}{\mu K}}, \eta\right\}$ to \cref{fedacii:conv:main,fedacii:stab:main}, which leads to  the following lemma.
\begin{lemma}[Convergence of \fedacii for general $\eta$]
  \label{fedacii:a2:general:eta}
    Let $F$ be $\mu > 0$-strongly convex, and assume \cref{asm2}, then for any $\eta \in (0, \frac{1}{L}]$, \fedacii yields
    \begin{align}
      \expt[\Phi_T]
      \leq & \exp \left(  - \frac{1}{3} \max \left\{ \eta \mu, \sqrt{\frac{\eta \mu}{K}}\right\}T \right) \Phi_0
      + \frac{\eta^{\frac{1}{2}} \sigma^2}{\mu^{\frac{1}{2}} M K^{\frac{1}{2}} }
      + \frac{2 \eta^{\frac{3}{2}} L K^{\frac{1}{2}} \sigma^2}{\mu^{\frac{1}{2}} M}
      + \frac{\euler^9 \eta^4 Q^2 K^2 \sigma^4}{\mu},
      \label{eq:fedacii:a2:general:eta}
    \end{align}
    where $\Phi_t$ is the decentralized potential defined in \cref{eq:centralied:potential}.
\end{lemma}
\begin{proof}[Proof of \cref{fedacii:a2:general:eta}]
  It is direct to verify that $\gamma = \max \left\{\eta, \sqrt{\frac{\eta}{\mu K}} \right\} \in \left[\eta, \sqrt{\frac{\eta}{\mu}}\right]$ so both \cref{fedacii:conv:main,fedacii:stab:main} are applicable.
  Applying \cref{fedacii:conv:main} yields
  \begin{align}
    \expt[\Phi_T] \leq & \exp\left( - \frac{1}{3} \max \left\{ \eta \mu, \sqrt{\frac{\eta \mu}{K}}\right\}T \right) \Phi_0
    +
    \min\left\{ \frac{3\eta L \sigma^2}{2 \mu M}, \frac{3 \eta^{\frac{3}{2}} L K^{\frac{1}{2}} \sigma^2}{2 \mu^{\frac{1}{2}} M} \right\}
    \\
                       & + \max\left\{ \frac{\eta \sigma^2}{2M}, \frac{\eta^{\frac{1}{2}} \sigma^2}{2 \mu^{\frac{1}{2}} M K^{\frac{1}{2}} }          \right\}
    + \frac{3}{\mu} \max_{0 \leq t < T}  \expt \left[ \left\|  \nabla F(\overline{w_t^{\mathrm{md}}}) -  \frac{1}{M} \sum_{m=1}^M \nabla F (w_t^{\mathrm{md}, m})  \right\|^2 \right].
    \label{eq:fedacii:proof:1}
  \end{align}
  We bound $\min\left\{ \frac{3\eta L \sigma^2}{2 \mu M}, \frac{3 \eta^{\frac{3}{2}} L K^{\frac{1}{2}} \sigma^2}{2 \mu^{\frac{1}{2}} M} \right\}$ with $\frac{3 \eta^{\frac{3}{2}} L K^{\frac{1}{2}} \sigma^2}{2 \mu^{\frac{1}{2}} M}$, and bound $\max\left\{ \frac{\eta \sigma^2}{2M}, \frac{\eta^{\frac{1}{2}} \sigma^2}{2 \mu^{\frac{1}{2}} M K^{\frac{1}{2}} }          \right\}$ with $\frac{\eta \sigma^2}{2M} + \frac{\eta^{\frac{1}{2}} \sigma^2}{2 \mu^{\frac{1}{2}} M K^{\frac{1}{2}} }$.
  By AM-GM inequality and $\mu \leq L$, we have 
  \begin{equation}
    \frac{\eta \sigma^2}{2M} 
    \leq 
    \frac{ \eta^{\frac{3}{2}} \mu^{\frac{1}{2}} K^{\frac{1}{2}} \sigma^2}{4 M} + \frac{ \eta^{\frac{1}{2}} \sigma^2}{4 \mu^{\frac{1}{2}} M K^{\frac{1}{2}} }
    \leq
    \frac{ \eta^{\frac{3}{2}} L K^{\frac{1}{2}} \sigma^2}{4\mu^{\frac{1}{2}} M} + \frac{ \eta^{\frac{1}{2}} \sigma^2}{4 \mu^{\frac{1}{2}} M K^{\frac{1}{2}} }
  \end{equation}
  Thus
  \begin{align}
         & \min\left\{ \frac{3\eta L \sigma^2}{2 \mu M}, \frac{3 \eta^{\frac{3}{2}} L K^{\frac{1}{2}} \sigma^2}{2 \mu^{\frac{1}{2}} M} \right\}
    + \max\left\{ \frac{\eta \sigma^2}{2M}, \frac{\eta^{\frac{1}{2}} \sigma^2}{2 \mu^{\frac{1}{2}} M K^{\frac{1}{2}} }          \right\}
    \\
    \leq & \frac{3 \eta^{\frac{3}{2}} L K^{\frac{1}{2}} \sigma^2}{2 \mu^{\frac{1}{2}} M} + \frac{\eta \sigma^2}{2M} + \frac{\eta^{\frac{1}{2}} \sigma^2}{2 \mu^{\frac{1}{2}} M K^{\frac{1}{2}} }
    \leq
    \frac{7 \eta^{\frac{3}{2}} L K^{\frac{1}{2}} \sigma^2}{4 \mu^{\frac{1}{2}} M} + \frac{3 \eta^{\frac{1}{2}} \sigma^2}{4 \mu^{\frac{1}{2}} M K^{\frac{1}{2}} },
    \label{eq:fedacii:proof:1:1}
  \end{align}

  Applying \cref{fedacii:stab:main} yields (for all $t$)
  \begin{align}
         & \frac{3}{\mu} \expt \left[ \left\|  \nabla F(\overline{w_t^{\mathrm{md}}}) -  \frac{1}{M} \sum_{m=1}^M \nabla F (w_t^{\mathrm{md}, m})  \right\|^2 \right]
    \leq
    \begin{cases}
      \frac{132}{\mu} \eta^4 Q^2 K^2 \sigma^4  \left(1 + \frac{1}{K}\right)^{4K}
       & \text{if~} \gamma = \sqrt{\frac{\eta}{\mu K}}
      \\
      \frac{132}{\mu} \eta^4 Q^2 K^2 \sigma^4,
       &
      \text{if~} \gamma = \eta
    \end{cases}
    \\
    \leq & 132 \euler^{4} \mu^{-1} \eta^4 Q^2 K^2 \sigma^4 \leq \euler^9 \mu^{-1} \eta^4 Q^2 K^2 \sigma^4,
    \label{eq:fedacii:proof:2}
  \end{align}
  where in the last inequality we used the estimation that $132 \euler^{4} < \euler^9$.

  Combining \cref{eq:fedacii:proof:1,eq:fedacii:proof:1:1,eq:fedacii:proof:2} yields
  \begin{align}
    \expt[\Phi_T]
    \leq & \exp \left(  - \frac{1}{3} \max \left\{ \eta \mu, \sqrt{\frac{\eta \mu}{K}}\right\}T \right) \Phi_0
    + \frac{\eta^{\frac{1}{2}} \sigma^2}{\mu^{\frac{1}{2}} M K^{\frac{1}{2}} }
    + \frac{2 \eta^{\frac{3}{2}} L K^{\frac{1}{2}} \sigma^2}{\mu^{\frac{1}{2}} M}
    + \frac{\euler^9 \eta^4 Q^2 K^2 \sigma^4}{\mu}.
  \end{align}
\end{proof}

The main \cref{fedacii:a2:full} then follows by plugging the appropriate $\eta$ to \cref{fedacii:a2:general:eta}. 
\begin{proof}[Proof of \cref{fedacii:a2:full}]
  To simplify the notation, we denote the decreasing term in \cref{eq:fedacii:a2:general:eta} in \cref{fedacii:a2:general:eta} as $\varphi_{\downarrow}(\eta)$ and the increasing term as $\varphi_{\uparrow}(\eta)$, namely
  \begin{align}
    \varphi_{\downarrow}(\eta) := \exp \left(  - \frac{1}{3} \max \left\{ \eta \mu, \sqrt{\frac{\eta \mu}{K}}\right\}T \right) \Phi_0,
    \quad
    \varphi_{\uparrow}(\eta) := \frac{\eta^{\frac{1}{2}} \sigma^2}{\mu^{\frac{1}{2}} M K^{\frac{1}{2}} }
    + \frac{2 \eta^{\frac{3}{2}} L K^{\frac{1}{2}} \sigma^2}{\mu^{\frac{1}{2}} M}
    + \frac{\euler^9 \eta^4 Q^2 K^2 \sigma^4}{\mu}.
  \end{align}
  Now let
  \begin{equation}
    \eta_0 := \frac{9K}{\mu T^2} \log^2 \left( \euler
     + \min \left\{ \frac{\mu M T \Phi_0}{\sigma^2} + \frac{\mu^2 MT^3 \Phi_0}{LK^2 \sigma^2}, \frac{\mu^5 T^8 \Phi_0}{Q^2 K^6 \sigma^4} \right\} \right)
  \end{equation}
  then $\eta := \min \left\{ \frac{1}{L}, \eta_0 \right\}$. Therefore, the decreasing term $\varphi_{\downarrow}(\eta)$ is upper bounded by $\varphi_{\downarrow}(\frac{1}{L}) + \varphi_{\downarrow}(\eta_0)$, where
  \begin{equation}
    \varphi_{\downarrow} \left(\frac{1}{L} \right)
    \leq
    \min \left\{ \exp \left( - \frac{\mu T}{3 L} \right), \exp \left( - \frac{\mu^{\frac{1}{2}} T}{3 L^{\frac{1}{2}} K^{\frac{1}{2}}} \right)\right\} \Phi_0,
    \label{eq:fedacii:a2:1}
  \end{equation}
  and
  \begin{align}
    \varphi_{\downarrow}(\eta_0) 
    \leq & \exp \left(  - \frac{1}{3} \sqrt{\frac{\eta_0 \mu}{K}} T \right) \Phi_0
    = 
    \left( \euler
     + \min \left\{ \frac{\mu M T \Phi_0}{\sigma^2} + \frac{\mu^2 MT^3 \Phi_0}{LK^2 \sigma^2}, \frac{\mu^5 T^8 \Phi_0}{Q^2 K^6 \sigma^4} \right\} \right)^{-1} \Phi_0
    \\
    \leq & \frac{\sigma^2}{\mu MT} + \frac{LK^2 \sigma^2}{\mu^2 MT^3} + \frac{Q^2 K^6 \sigma^4}{\mu^5 T^8}.
    \label{eq:fedacii:a2:2}
  \end{align}
  On the other hand
  \begin{align}
    \varphi_{\uparrow}(\eta)  \leq \varphi_{\uparrow}(\eta_0) 
    \leq
    & 
    \frac{3 \sigma^2}{\mu M T} \log \left( \euler + \frac{\mu MT \Phi_0}{\sigma^2} \right)
    +
    \frac{54 LK^2 \sigma^2}{\mu^2 MT^3} \log^3 \left( \euler + \frac{\mu^2 MT^3 \Phi_0}{LK^2\sigma^2} \right)
    \\
    & + \frac{9^4 \euler^9 Q^2 K^6 \sigma^4}{\mu^5 T^8} \log^8\left( \euler
    + \frac{\mu^5 T^8 \Phi_0}{Q^2 K^6 \sigma^4} \right).
    \label{eq:fedacii:a2:3}
  \end{align}
  Combining \cref{fedacii:a2:general:eta,eq:fedacii:a2:1,eq:fedacii:a2:2,eq:fedacii:a2:3} gives
  \begin{align}
    & \expt[\Phi_T] \leq \varphi_{\downarrow} \left(\frac{1}{L} \right) + \varphi_{\downarrow}(\eta_0) + \varphi_{\uparrow}(\eta_0) 
    \\
    \leq & \min \left\{ \exp \left( - \frac{\mu T}{3 L} \right), \exp \left( - \frac{\mu^{\frac{1}{2}} T}{3 L^{\frac{1}{2}} K^{\frac{1}{2}}} \right)\right\} \Phi_0 
    + \frac{4 \sigma^2}{\mu M T} \log \left( \euler + \frac{\mu MT \Phi_0}{\sigma^2} \right)
    \\
    & 
    +
    \frac{55 LK^2 \sigma^2}{\mu^2 MT^3} \log^3 \left( \euler + \frac{\mu^2 MT^3 \Phi_0}{LK^2\sigma^2} \right)
    + \frac{\euler^{18} Q^2 K^6 \sigma^4}{\mu^5 T^8} \log^8\left( \euler
    + \frac{\mu^5 T^8 \Phi_0}{Q^2 K^6 \sigma^4} \right),
  \end{align}
  where in the last inequality we used the estimate $9^4 \euler^9 + 1 < \euler^{18}$.
\end{proof}

\subsection{Perturbed iterate analysis for \fedacii: Proof of \cref{fedacii:conv:main}}
\label{sec:fedacii:conv:main}
In this subsection we will prove \cref{fedacii:conv:main}.
We start by the one-step analysis of the centralized potential defined in \cref{eq:centralied:potential}. 
The following two propositions establish the one-step analysis of the two quantities in $\Phi_t$, namely $\|\overline{w_{t}} - w^*\|^2$ and $F( \overline{w_{t}^{\mathrm{ag}}}) - F^*$. 
We only require minimal hyperparameter assumptions, namely $\alpha \geq 1, \beta \geq 1, \eta \leq \frac{1}{L}$ for these two propositions.
We will then show how the choice of $\alpha, \beta$ are determined towards the proof of \cref{fedacii:conv:main} in order to couple the two quantities into potential $\Phi_t$.

\begin{proposition}
  \label{fedacii:conv:1}
  Let $F$ be $\mu>0$-strongly convex, and assume \cref{asm1}, then for \fedac with hyperparameters assumptions $\alpha \geq 1$, $\beta \geq 1$, $\eta \leq \frac{1}{L}$, the following inequality holds
  \begin{align}
         & \expt [ \|\overline{w_{t+1}} - w^*\|^2 |\mathcal{F}_t]
    \\
    \leq & \left( 1 - \frac{1}{2} \alpha^{-1} \right)\|\overline{w_t} - w^*\|^2 + \frac{3}{2} \alpha^{-1} \| \overline{w_t^{\mathrm{md}}} - w^*\|^2 + \frac{3}{2} \gamma^2  \left\| \nabla F (\overline{w_t^{\mathrm{md}}}) \right\|^2
    \\
         & - 2 \gamma  \left( 1 + \frac{1}{2} \alpha^{-1} \right) \left\langle \nabla F (\overline{w_t^{\mathrm{md}}}) , (1-\alpha^{-1}(1-\beta^{-1})) \overline{w_t} + \alpha^{-1} (1 - \beta^ {-1}) \overline{w_t^{\mathrm{ag}}} - w^*\right\rangle
    \\
         &
    + \gamma^2 \left( 1 + 2 \alpha \right)  \left\| \nabla F (\overline{w_t^{\mathrm{md}}}) - \frac{1}{M} \sum_{m=1}^M \nabla F (w_t^{\mathrm{md}, m}) \right\|^2
    +  \frac{\gamma^2 \sigma^2}{M}.
    \label{eq:fedacii:conv:1}
  \end{align}
\end{proposition}

\begin{proposition}
  \label{fedacii:conv:2}
  In the same setting of \cref{fedacii:conv:1}, the following inequality holds
  \begin{align}
         & \expt \left[F( \overline{w_{t+1}^{\mathrm{ag}}}) - F^*| \mathcal{F}_t \right]
    \\
    \leq & \left( 1 - \frac{1}{2} \alpha^{-1} \right) \left(F(\overline{w_{t}^{\mathrm{ag}}}) - F^* \right) - \frac{1}{4} \mu \alpha^{-1} \left\|  \overline{w_{t}^{\mathrm{md}}} - w^*  \right\|^2 - \frac{1}{2} \eta \left\| \nabla F(\overline{w_t^{\mathrm{md}}}) \right\|^2
    \\
         & + {\frac{1}{2} \alpha^{-1} \left\langle \nabla F(\overline{w_{t}^{\mathrm{md}}}), 2 \alpha \beta^{-1} \overline{w_t} + (1 - 2 \alpha \beta^{-1}) \overline{w_t^{\mathrm{ag}}} -  w^* \right\rangle}
    \\
         & + \frac{1}{2}  \eta \left\|  \nabla F(\overline{w_t^{\mathrm{md}}}) -  \frac{1}{M} \sum_{m=1}^M \nabla F (w_t^{\mathrm{md}, m})  \right\|^2 +  \frac{\eta^2 L \sigma^2 }{2M}.
    \label{eq:fedacii:conv:2}
  \end{align}
\end{proposition}
We defer the proofs of \cref{fedacii:conv:1,fedacii:conv:2} to \cref{sec:proof:fedacii:conv:1,sec:proof:fedacii:conv:2}, respectively.

Now we are ready to prove \cref{fedacii:conv:main}.
\begin{proof}[Proof of \cref{fedacii:conv:main}]
  Since $\gamma \leq \sqrt{\frac{\eta}{\mu}} \leq \sqrt{\frac{1}{\mu L}} \leq \frac{1}{\mu}$, we have $\alpha = \frac{3}{2 \gamma \mu} - \frac{1}{2} \geq 1$, and therefore $\beta = \frac{2 \alpha^2 - 1}{\alpha - 1} \geq 1$. Hence both \cref{fedacii:conv:1,fedacii:conv:2} are applicable.

  Adding \cref{eq:fedacii:conv:2} with $\frac{1}{6} \mu$ times of \cref{eq:fedacii:conv:1} gives (note that the $\|\overline{w_t^{\mathrm{md}}} - w^*\|^2$ term is cancelled because $\frac{1}{4}\mu \alpha^{-1} = \frac{1}{6} \mu \cdot \frac{3}{2}\alpha^{-1}$)
  \begin{align}
     & \expt \left[ \Phi_{t+1} |\mathcal{F}_t \right] 
     \leq 
    \underbrace{\left( 1 - \frac{1}{2} \alpha^{-1} \right) \Phi_t}_{\text{(I)}}
    + 
    \underbrace{\left( \frac{1}{4} \gamma^2 \mu  - \frac{1}{2} \eta  \right) \left\| \nabla F(\overline{w_t^{\mathrm{md}}}) \right\|^2}_{\text{(II)}}
    \\
     & \quad + 
    \underbrace{\frac{1}{2} \alpha^{-1} \left\langle \nabla F(\overline{w_{t}^{\mathrm{md}}}), 2 \alpha \beta^{-1} \overline{w_t} + (1 - 2 \alpha \beta^{-1}) \overline{w_t^{\mathrm{ag}}} -  w^* \right\rangle}_{\text{(III)}}
    \\
     & \quad - 
    \underbrace{\frac{1}{3} \gamma \mu \left( 1 + \frac{1}{2} \alpha^{-1} \right) \left\langle \nabla F (\overline{w_t^{\mathrm{md}}}) , (1-\alpha^{-1}(1-\beta^{-1})) \overline{w_t} + \alpha^{-1} (1 - \beta^ {-1}) \overline{w_t^{\mathrm{ag}}} - w^*\right\rangle}_{\text{(IV)}}
    \\
     & \quad + 
    \underbrace{\left( \frac{1}{2}  \eta  + \frac{1}{6} \gamma^2 \mu (1 + 2 \alpha)  \right) \left\|  \nabla F(\overline{w_t^{\mathrm{md}}}) -  \frac{1}{M} \sum_{m=1}^M \nabla F (w_t^{\mathrm{md}, m})  \right\|^2}_{\text{(V)}}
    +  \frac{\eta^2 L \sigma^2 }{2M} + \frac{\gamma^2 \mu \sigma^2}{6 M}.
    \label{eq:fedacii:conv:main:1}
  \end{align}

  Now we analyze the RHS of \cref{eq:fedacii:conv:main:1} term by term.
  \paragraph{Term (I) of \cref{eq:fedacii:conv:main:1}}
  Note that $\alpha^{-1} = \frac{2 \gamma \mu}{3 - \gamma \mu} \geq \frac{2}{3}\gamma \mu$, we have
  \begin{equation}
    \left( 1 - \frac{1}{2} \alpha^{-1} \right) \Phi_t \leq \left( 1 - \frac{1}{3} \gamma \mu \right) \Phi_t.
    \label{eq:fedacii:conv:main:2}
  \end{equation}

  \paragraph{Term (II) of \cref{eq:fedacii:conv:main:1}}
  Since $\gamma^2 \mu \leq \eta$ we have
  \begin{equation}
    \left( \frac{1}{4} \gamma^2 \mu  - \frac{1}{2} \eta  \right) \left\| \nabla F(\overline{w_t^{\mathrm{md}}}) \right\|^2 \leq 0.
    \label{eq:fedacii:conv:main:3}
  \end{equation}

  \paragraph{Term (III) and (IV) of \cref{eq:fedacii:conv:main:1}}
  Since $\beta =\frac{2 \alpha^2 - 1}{\alpha - 1}$, we have $2 \alpha \beta^{-1} = \frac{2 \alpha (\alpha - 1)}{2 \alpha^2 - 1} = (1 - \alpha^{-1} (1 - \beta^{-1}))$, and $1 - 2\alpha\beta^{-1} = \frac{2 \alpha - 1}{2 \alpha^2 - 1} = \alpha^{-1}(1 - \beta^{-1})$. Therefore, the two inner-product terms are cancelled:
  \begin{align}
      & \frac{1}{2} \alpha^{-1} \left\langle \nabla F(\overline{w_{t}^{\mathrm{md}}}), 2 \alpha \beta^{-1} \overline{w_t} + (1 - 2 \alpha \beta^{-1}) \overline{w_t^{\mathrm{ag}}} -  w^* \right\rangle
    \\
      & \qquad - \frac{1}{3} \gamma \mu \left( 1 + \frac{1}{2} \alpha^{-1} \right) \left\langle \nabla F (\overline{w_t^{\mathrm{md}}}) , (1-\alpha^{-1}(1-\beta^{-1})) \overline{w_t} + \alpha^{-1} (1 - \beta^ {-1}) \overline{w_t^{\mathrm{ag}}} - w^*\right\rangle
    \\
    = & \left( \frac{1}{2} \alpha^{-1} - \frac{1}{3} \gamma \mu  \left( 1 + \frac{1}{2} \alpha^{-1} \right) \right)   \left\langle \nabla F (\overline{w_t^{\mathrm{md}}}) , \frac{2 \alpha - 1}{2 \alpha^2 - 1} \overline{w_t^{\mathrm{ag}}} + \left(\frac{2 \alpha^2 - 2 \alpha}{2 \alpha^2 - 1}  \right) \overline{w_t}  - w^*\right\rangle
    \\
    = & \left( \frac{\gamma \mu}{3 - \gamma \mu} - \frac{1}{3} \gamma \mu  \left( 1 + \frac{\gamma \mu}{3 - \gamma \mu} \right) \right)   \left\langle \nabla F (\overline{w_t^{\mathrm{md}}}) , \frac{2 \alpha - 1}{2 \alpha^2 - 1} \overline{w_t^{\mathrm{ag}}} + \left(\frac{2 \alpha^2 - 2 \alpha}{2 \alpha^2 - 1}  \right) \overline{w_t}  - w^*\right\rangle
    \tag{since $\alpha^{-1} = \frac{2 \gamma \mu}{3 - \gamma \mu}$}
    \\
    = & 0.
    \label{eq:fedacii:conv:main:4}
  \end{align}

  \paragraph{Term (V) of \cref{eq:fedacii:conv:main:1}} Since $\alpha = \frac{3 - \gamma \mu}{2 \gamma \mu}$ and $\gamma \geq \eta$ we have
  \begin{equation}
    \left( \frac{1}{2}  \eta  + \frac{1}{6} \gamma^2 \mu (1 + 2 \alpha)  \right)
    = \frac{1}{2} \eta + \frac{1}{6} \gamma^2 \mu \left( \frac{6}{2 \gamma \mu}  \right)
    = \frac{1}{2}(\eta + \gamma) \leq \gamma.
    \label{eq:fedacii:conv:main:5}
  \end{equation}

  Plugging \cref{eq:fedacii:conv:main:2,eq:fedacii:conv:main:3,eq:fedacii:conv:main:4,eq:fedacii:conv:main:5} to \cref{eq:fedacii:conv:main:1} gives
  \begin{equation}
    \expt \left[\Phi_{t+1} \middle| \mathcal{F}_t \right] \leq \left( 1 - \frac{1}{3}\gamma\mu \right) \Phi_{t} +  \frac{\eta^2 L \sigma^2 }{2M} + \frac{\gamma^2 \mu \sigma^2}{6 M}
    + \gamma  \left\|  \nabla F(\overline{w_t^{\mathrm{md}}}) -  \frac{1}{M} \sum_{m=1}^M \nabla F (w_t^{\mathrm{md}, m})  \right\|^2.
    \label{eq:fedacii:conv:main:6}
  \end{equation}
  Telescoping the above inequality up to timestep $T$ yields
  \begin{align}
    \expt \left[\Phi_T \right]
    \leq & \left( 1 - \frac{1}{3} \gamma \mu \right)^T \Phi_0 +
    \left( \sum_{t=0}^{T-1} \left( 1 - \frac{1}{3} \gamma \mu \right)^t \right) \cdot \left( \frac{\eta^2 L \sigma^2 }{2M} + \frac{\gamma^2 \mu \sigma^2}{6 M} \right)
    \\
         & + \gamma  \sum_{t=0}^{T-1}  \left( 1 - \frac{1}{3}\gamma\mu \right)^{T-t-1}  \expt \left[ \left\|  \nabla F(\overline{w_t^{\mathrm{md}}}) -  \frac{1}{M} \sum_{m=1}^M \nabla F (w_t^{\mathrm{md}, m})  \right\|^2 \right]
    \\
    \leq & \exp \left( - \frac{1}{3} \gamma \mu T \right) \Phi_0 +  \left( \frac{3\eta^2 L \sigma^2 }{2 \gamma \mu M} + \frac{\gamma \sigma^2}{2 M} \right)
    + \frac{3}{\mu}  \cdot \max_{0 \leq t < T}   \expt \left[ \left\|  \nabla F(\overline{w_t^{\mathrm{md}}}) -  \frac{1}{M} \sum_{m=1}^M \nabla F (w_t^{\mathrm{md}, m})  \right\|^2 \right],
  \end{align}
  where in the last inequality we used the fact that $(1 - \frac{1}{3}\gamma \mu)^T \leq \exp(- \frac{1}{3}\gamma \mu T)$ and $\sum_{t=0}^{T-1} \left( 1 - \frac{1}{3}\gamma \mu \right)^t  \leq \sum_{t=0}^{T-1} \left( 1 - \frac{1}{3}\gamma \mu \right)^{\infty} = \frac{3}{\gamma \mu}$.
\end{proof}

\subsubsection{Proof of \cref{fedacii:conv:1}}
\label{sec:proof:fedacii:conv:1}
\begin{proof}[Proof of \cref{fedacii:conv:1}]
  By definition of the \fedac procedure (\cref{alg:fedac}),
  \begin{equation}
    \overline{w_{t+1}} - w^* = (1 - \alpha^{-1}) \overline{w_t} + \alpha^{-1} \overline{w_t^{\mathrm{md}}} - \gamma \cdot \frac{1}{M} \sum_{m=1}^M \nabla f(w_t^{\mathrm{md},m}; \xi_t^m) - w^*.
  \end{equation}
  Taking conditional expectation gives
  \begin{equation}
    \expt \left[ \|\overline{w_{t+1}} - w^*\|^2 \middle| \mathcal{F}_t \right]
    \leq
    \left\| (1 - \alpha^{-1}) \overline{w_t} + \alpha^{-1} \overline{w_t^{\mathrm{md}}} - \gamma \cdot \frac{1}{M} \sum_{m=1}^M \nabla F (w_t^{\mathrm{md},m})- w^* \right\|^2
    + \frac{1}{M}\gamma^2 \sigma^2.
    \label{eq:fedacii:conv:1:0}
  \end{equation}
  The squared norm in \cref{eq:fedacii:conv:1:0} is bounded as
  \begin{align}
         & \left\| (1 - \alpha^{-1}) \overline{w_t} + \alpha^{-1} \overline{w_t^{\mathrm{md}}} - \gamma \cdot \frac{1}{M} \sum_{m=1}^M \nabla F (w_t^{\mathrm{md},m})- w^* \right\|^2
    \\
    =    & \left\| (1 - \alpha^{-1}) \overline{w_t} + \alpha^{-1} \overline{w_t^{\mathrm{md}}} - \gamma \nabla F (\overline{w_t^{\mathrm{md}}}) - w^*
    + \gamma \left(  \nabla F (\overline{w_t^{\mathrm{md}}}) - \frac{1}{M} \sum_{m=1}^M \nabla F (w_t^{\mathrm{md}}) \right)\right\|^2 
    \\
    \leq & {\left( 1 + \frac{1}{2} \alpha^{-1} \right)} \left\|  (1 - \alpha^{-1}) \overline{w_t} + \alpha^{-1} \overline{w_t^{\mathrm{md}}} - w^* - \gamma \nabla F (\overline{w_t^{\mathrm{md}}}) \right\|^2
    \\
         & + \gamma^2  {\left( 1 + 2 \alpha \right)} \left\| \nabla F (\overline{w_t^{\mathrm{md}}}) - \frac{1}{M} \sum_{m=1}^M \nabla F (w_t^{\mathrm{md}, m}) \right\|^2
    \tag{apply helper \cref{helper:unbalanced:ineq} with $\zeta = \frac{1}{2}\alpha^{-1}$}
    \\
    =    & \underbrace{\left( 1 + \frac{1}{2} \alpha^{-1} \right) \left\| (1 - \alpha^{-1}) \overline{w_t} + \alpha^{-1} \overline{w_t^{\mathrm{md}}} - w^* \right\|^2}_{\text{(I)}}
    + \underbrace{\gamma^2  \left( 1 + \frac{1}{2} \alpha^{-1} \right) \left\| \nabla F (\overline{w_t^{\mathrm{md}}}) \right\|^2}_{\text{(II)}}
    \\
         & \underbrace{- 2 \gamma  \left( 1 + \frac{1}{2} \alpha^{-1} \right) \left\langle \nabla F (\overline{w_t^{\mathrm{md}}}), (1 - \alpha^{-1}) \overline{w_t} + \alpha^{-1} \overline{w_t^{\mathrm{md}}} - w^*\right\rangle}_{\text{(III)}}
    \\
         & + \gamma^2 \left( 1 + 2 \alpha \right)  \left\| \nabla F (\overline{w_t^{\mathrm{md}}}) - \frac{1}{M} \sum_{m=1}^M \nabla F (w_t^{\mathrm{md}, m}) \right\|^2.
    \label{eq:fedacii:conv:1:1}
  \end{align}

  The first term (I) of \cref{eq:fedacii:conv:1:1} is bounded via Jensen's inequality as follows:
  \begin{align}
         & \left( 1 + \frac{1}{2} \alpha^{-1} \right) \left\| (1 - \alpha^{-1}) \overline{w_t} + \alpha^{-1} \overline{w_t^{\mathrm{md}}} - w^* \right\|^2
    \\
    \leq & \left( 1 + \frac{1}{2} \alpha^{-1} \right) \left( (1 - \alpha^{-1}) \|\overline{w_t} - w^*\|^2 + \alpha^{-1} \| \overline{w_t^{\mathrm{md}}} - w^*\|^2 \right) \tag{Jensen's inequality}\
    \\
    \leq & \left( 1 - \frac{1}{2} \alpha^{-1} \right)\|\overline{w_t} - w^*\|^2 + \frac{3}{2} \alpha^{-1} \| \overline{w_t^{\mathrm{md}}} - w^*\|^2.
    \label{eq:fedacii:conv:1:2}
  \end{align}
  where in the last inequality of \cref{eq:fedacii:conv:1:2} we used the fact that $( 1 + \frac{1}{2} \alpha^{-1})(1 - \alpha^{-1}) = 1 - \frac{1}{2} \alpha^{-1} - \frac{1}{2} \alpha^{-2} < 1 - \frac{1}{2} \alpha^{-1}$, and $( 1 + \frac{1}{2} \alpha^{-1}) \alpha^{-1} \leq \frac{3}{2}\alpha^{-1}$ as $\alpha \geq 1$.

  The second term (II) of \cref{eq:fedacii:conv:1:1} is bounded as (since $\alpha \geq 1$)
  \begin{equation}
    \gamma^2  \left( 1 + \frac{1}{2} \alpha^{-1} \right) \left\| \nabla F (\overline{w_t^{\mathrm{md}}}) \right\|^2 \leq \frac{3}{2} \gamma^2  \left\| \nabla F (\overline{w_t^{\mathrm{md}}}) \right\|^2.
    \label{eq:fedacii:conv:1:3}
  \end{equation}

  To analyze the third term (III) of \cref{eq:fedacii:conv:1:1}, we note that by definition of $\overline{w_t^{\mathrm{md}}}$,
  \begin{align}
      & - 2 \gamma  \left( 1 + \frac{1}{2} \alpha^{-1} \right) \left\langle \nabla F (\overline{w_t^{\mathrm{md}}}), (1 - \alpha^{-1}) \overline{w_t} + \alpha^{-1} \overline{w_t^{\mathrm{md}}} - w^*\right\rangle
    \\
    = & - 2 \gamma  \left( 1 + \frac{1}{2} \alpha^{-1} \right) \left\langle \nabla F (\overline{w_t^{\mathrm{md}}}) , (1-\alpha^{-1}(1-\beta^{-1})) \overline{w_t} + \alpha^{-1} (1 - \beta^ {-1}) \overline{w_t^{\mathrm{ag}}} - w^*\right\rangle.
    \label{eq:fedacii:conv:1:4}
  \end{align}
  Plugging \cref{eq:fedacii:conv:1:1,eq:fedacii:conv:1:2,eq:fedacii:conv:1:3,eq:fedacii:conv:1:4} back to \cref{eq:fedacii:conv:1:0} yields
  \begin{align}
         & \expt [ \|\overline{w_{t+1}} - w^*\|^2 |\mathcal{F}_t]
    \\
    \leq & \left( 1 - \frac{1}{2} \alpha^{-1} \right)\|\overline{w_t} - w^*\|^2 + \frac{3}{2} \alpha^{-1} \| \overline{w_t^{\mathrm{md}}} - w^*\|^2 + \frac{3}{2}\gamma^2  \left\| \nabla F (\overline{w_t^{\mathrm{md}}}) \right\|^2
    \\
         & - 2 \gamma  \left( 1 + \frac{1}{2} \alpha^{-1} \right) \left\langle \nabla F (\overline{w_t^{\mathrm{md}}}) , (1-\alpha^{-1}(1-\beta^{-1})) \overline{w_t} + \alpha^{-1} (1 - \beta^ {-1}) \overline{w_t^{\mathrm{ag}}} - w^*\right\rangle
    \\
         &
    + \gamma^2 \left( 1 + 2 \alpha \right)  \left\| \nabla F (\overline{w_t^{\mathrm{md}}}) - \frac{1}{M} \sum_{m=1}^M \nabla F (w_t^{\mathrm{md}, m}) \right\|^2
    +  \frac{\gamma^2 \sigma^2}{M},
  \end{align}
  completing the proof of \cref{fedacii:conv:1}.
\end{proof}

\subsubsection{Proof of \cref{fedacii:conv:2}}
\label{sec:proof:fedacii:conv:2}
\begin{proof}[Proof of \cref{fedacii:conv:2}]
  By definition of the \fedac procedure we have
  \begin{equation}
    \overline{w_{t+1}^{\mathrm{ag}}}
    =
    \overline{w_t^{\mathrm{md}}} - \eta \cdot \frac{1}{M} \sum_{m=1}^M \nabla f(w_t^{\mathrm{md},m}; \xi_t^m),
  \end{equation}
  and thus by $L$-smoothness (\cref{asm1}(b)) we obtain
  \begin{equation}
    F( \overline{w_{t+1}^{\mathrm{ag}}})
    \leq
    F(\overline{w_t^{\mathrm{md}}}) - \eta \left\langle \nabla F(\overline{w_t^{\mathrm{md}}}), \frac{1}{M} \sum_{m=1}^M \nabla f(w_t^{\mathrm{md},m}; \xi_t^m) \right\rangle + \frac{\eta^2 L}{2} \left\| \frac{1}{M} \sum_{m=1}^M \nabla f(w_t^{\mathrm{md},m}; \xi_t^m)  \right\|^2.
  \end{equation}
  Taking conditional expectation, and by bounded variance (\cref{asm1}(c))
  \begin{equation}
    \expt \left[F( \overline{w_{t+1}^{\mathrm{ag}}}) |\mathcal{F}_t \right]
    \leq  F(\overline{w_t^{\mathrm{md}}})
    - \eta \left\langle \nabla F(\overline{w_t^{\mathrm{md}}}), \frac{1}{M} \sum_{m=1}^M \nabla F (w_t^{\mathrm{md}, m}) \right\rangle + \frac{\eta^2 L}{2} \left\| \frac{1}{M} \sum_{m=1}^M \nabla F (w_t^{\mathrm{md}, m})  \right\|^2 + \frac{\eta^2 L \sigma^2 }{2M}.
    \label{eq:fedacii:conv:2:0:1}
  \end{equation}
  By polarization identity we have
  \begin{align}
    & \left\langle \nabla F(\overline{w_t^{\mathrm{md}}}), \frac{1}{M} \sum_{m=1}^M \nabla F (w_t^{\mathrm{md}, m}) \right\rangle
    \\
    = &
    \frac{1}{2}
    \left( \left\| \nabla F(\overline{w_t^{\mathrm{md}}}) \right\|^2 + \left\| \frac{1}{M} \sum_{m=1}^M \nabla F (w_t^{\mathrm{md}, m}) \right\|^2 - \left\|  \nabla F(\overline{w_t^{\mathrm{md}}}) -  \frac{1}{M} \sum_{m=1}^M \nabla F (w_t^{\mathrm{md}, m})  \right\|^2 \right).
    \label{eq:fedacii:conv:2:0:2}
  \end{align}
  Combining \cref{eq:fedacii:conv:2:0:1,eq:fedacii:conv:2:0:2} gives
  \begin{align}
         & \expt \left[F( \overline{w_{t+1}^{\mathrm{ag}}}) |\mathcal{F}_t \right]
    \\
    =    & F(\overline{w_t^{\mathrm{md}}}) - \frac{1}{2} \eta \left\| \nabla F(\overline{w_t^{\mathrm{md}}}) \right\|^2
    + \frac{1}{2} \eta \left\|  \nabla F(\overline{w_t^{\mathrm{md}}}) -  \frac{1}{M} \sum_{m=1}^M \nabla F (w_t^{\mathrm{md}, m})  \right\|^2
    \\
         & - \frac{1}{2} \eta (1 - \eta L) \left\| \frac{1}{M} \sum_{m=1}^M \nabla F (w_t^{\mathrm{md}, m})  \right\|^2 +  \frac{\eta^2 L \sigma^2 }{2M}
    \\
    \leq & F(\overline{w_t^{\mathrm{md}}}) - \frac{1}{2} \eta \left\| \nabla F(\overline{w_t^{\mathrm{md}}}) \right\|^2 \
    + \frac{1}{2}  \eta \left\|  \nabla F(\overline{w_t^{\mathrm{md}}}) -  \frac{1}{M} \sum_{m=1}^M \nabla F (w_t^{\mathrm{md}, m})  \right\|^2 +  \frac{\eta^2 L \sigma^2 }{2M},
    \label{eq:fedacii:conv:2:1}
  \end{align}
  where the last inequality is due to the assumption that $\eta \leq \frac{1}{L}$.

  Now we relate $F( \overline{w_{t}^{\mathrm{md}}})$ and $F( \overline{w_{t}^{\mathrm{ag}}})$ as follows
  \begin{align}
         & F(\overline{w_t^{\mathrm{md}}})  - F^*
    \\
    =    & \left( 1 - \frac{1}{2} \alpha^{-1} \right) \left(F(\overline{w_{t}^{\mathrm{ag}}}) - F^* \right)
    + \left( 1 - \frac{1}{2} \alpha^{-1} \right) \left(F(\overline{w_{t}^{\mathrm{md}}}) - F(\overline{w_{t}^{\mathrm{ag}}}) \right)
    + \frac{1}{2} \alpha^{-1} \left(F(\overline{w_{t}^{\mathrm{md}}}) - F^* \right)
    \\
    \leq & \left( 1 - \frac{1}{2} \alpha^{-1} \right) \left(F(\overline{w_{t}^{\mathrm{ag}}}) - F^* \right)
    + \left( 1 - \frac{1}{2} \alpha^{-1} \right) \left\langle \nabla F(\overline{w_{t}^{\mathrm{md}}}), \overline{w_{t}^{\mathrm{md}}} - \overline{w_{t}^{\mathrm{ag}}} \right\rangle
    \\
         & + \frac{1}{2} \alpha^{-1}
    \left( \left\langle \nabla F(\overline{w_{t}^{\mathrm{md}}}), \overline{w_{t}^{\mathrm{md}}} - w^* \right\rangle  - \frac{\mu}{2}  \left\|  \overline{w_{t}^{\mathrm{md}}} - w^*  \right\|^2 \right)
    \tag{$\mu$-strong convexity}
    \\
    =    & \left( 1 - \frac{1}{2} \alpha^{-1} \right) \left(F(\overline{w_{t}^{\mathrm{ag}}}) - F^* \right) - \frac{1}{4} \mu \alpha^{-1} \left\|  \overline{w_{t}^{\mathrm{md}}} - w^*  \right\|^2
    \\
         &
    + {\frac{1}{2} \alpha^{-1} \left\langle \nabla F(\overline{w_{t}^{\mathrm{md}}}), 2 \alpha \overline{w_t^{\mathrm{md}}} - (2 \alpha - 1) \overline{w_t^{\mathrm{ag}}} -  w^* \right\rangle}
    \tag{rearranging}
    \\
    =    & \left( 1 - \frac{1}{2} \alpha^{-1} \right) \left(F(\overline{w_{t}^{\mathrm{ag}}}) - F^* \right) - \frac{1}{4} \mu \alpha^{-1} \left\|  \overline{w_{t}^{\mathrm{md}}} - w^*  \right\|^2
    \\
         &
    + {\frac{1}{2} \alpha^{-1} \left\langle \nabla F(\overline{w_{t}^{\mathrm{md}}}), 2 \alpha \beta^{-1} \overline{w_t} + (1 - 2 \alpha \beta^{-1}) \overline{w_t^{\mathrm{ag}}} -  w^* \right\rangle},
    \label{eq:fedacii:conv:2:2}
  \end{align}
  where the last equality is due to the definition of $\overline{w_t^{\mathrm{md}}}$.

  Plugging \cref{eq:fedacii:conv:2:2} back to \cref{eq:fedacii:conv:2:1} yields
  \begin{align}
         & \expt \left[F( \overline{w_{t+1}^{\mathrm{ag}}})  - F^* \middle|\mathcal{F}_t \right]
    \\
    \leq & \left( 1 - \frac{1}{2} \alpha^{-1} \right) \left(F(\overline{w_{t}^{\mathrm{ag}}}) - F^* \right) - \frac{1}{4} \mu \alpha^{-1} \left\|  \overline{w_{t}^{\mathrm{md}}} - w^*  \right\|^2 - \frac{1}{2} \eta \left\| \nabla F(\overline{w_t^{\mathrm{md}}}) \right\|^2
    \\
         & + {\frac{1}{2} \alpha^{-1} \left\langle \nabla F(\overline{w_{t}^{\mathrm{md}}}), 2 \alpha \beta^{-1} \overline{w_t} + (1 - 2 \alpha \beta^{-1}) \overline{w_t^{\mathrm{ag}}} -  w^* \right\rangle}
    + \frac{1}{2}  \eta \left\|  \nabla F(\overline{w_t^{\mathrm{md}}}) -  \frac{1}{M} \sum_{m=1}^M \nabla F (w_t^{\mathrm{md}, m})  \right\|^2 +  \frac{\eta^2 L \sigma^2 }{2M},
  \end{align}
  completing the proof of \cref{fedacii:conv:2}.
\end{proof}

\subsection{Discrepancy overhead bound for \fedacii: Proof of \cref{fedacii:stab:main}}
\label{sec:fedacii:stab:main}
In this subsection we prove \cref{fedacii:stab:main} regarding the regarding the growth of discrepancy overhead introduced in \cref{fedacii:conv:main}. The core of the proof is the \nth{4}-order stability of \fedacii.
Note that most of the analysis in this subsection follows closely with the analysis on \fedaci (see \cref{sec:fedaci:stab:main}), but the analysis is technically more complicated.

We will reuse a set of notations defined in \cref{sec:fedaci:stab:main}, which we restate here for clearance. Let $m_1, m_2 \in [M]$ be two arbitrary distinct machines. For any timestep $t$, denote $\Delta_{t} := w_{t}^{m_1} - w_t^{m_2}$,  $\Delta_t^{\mathrm{ag}} := w_{t}^{\mathrm{ag}, m_1} - w_t^{\mathrm{ag}, m_2}$ and $\Delta_t^{\mathrm{md}} := w_{t}^{\mathrm{md}, m_1} - w_t^{\mathrm{md}, m_2}$ be the corresponding vector differences. Let $\Delta_t^{\varepsilon} = \varepsilon_t^{m_1} - \varepsilon_t^{m_2}$, where $\varepsilon_t^m := \nabla f(w_t^{\mathrm{md},m}; \xi_t^m) - \nabla F(w_t^{\mathrm{md}, m})$ be the bias of the gradient oracle of the $m$-th worker evaluated at $w_t^{\mathrm{md}}$.

The proof of \cref{fedacii:stab:main} is based on the following propositions.

The following \cref{fedacii:stab:1} studies the growth of $\begin{bmatrix} \Delta_{t}^{\mathrm{ag}} \\ \Delta_{t} \end{bmatrix}$ at each step. 
\cref{fedacii:stab:1} is analogous to \cref{fedaci:stab:1}, but the $\mathcal{A}$ is different. Note that \cref{fedacii:stab:1} requires only \cref{asm1}.
\begin{proposition}
  \label{fedacii:stab:1}
  Let $F$ be $\mu>0$-strongly convex, assume \cref{asm1} and assume the same hyperparameter choice is taken as in  \cref{fedacii:stab:main} (namely $\alpha = \frac{3}{2 \gamma \mu} - \frac{1}{2}$, $\beta =\frac{2 \alpha^2 - 1}{\alpha - 1}$, $\gamma \in [\eta, \sqrt{ \frac{\eta}{\mu}}]$, $\eta \in (0, \frac{1}{L}]$).
  Suppose $t+1$ is not a synchronization gap, then there exists a matrix $H_t$ such that $\mu I \preceq H_t \preceq LI$ satisfying
  \begin{equation}
    \begin{bmatrix}
      \Delta_{t+1}^{\mathrm{ag}} \\ \Delta_{t+1}
    \end{bmatrix}
    =
    \mathcal{A} (\mu, \gamma, \eta, H_t)
    \begin{bmatrix} \Delta_{t}^{\mathrm{ag}} \\ \Delta_{t} \end{bmatrix}
    -
    \begin{bmatrix} \eta I \\ \gamma I \end{bmatrix}
    \Delta_t^{\varepsilon},
  \end{equation}
  where $\mathcal{A}(\mu, \gamma, \eta, H)$ is a matrix-valued function defined as
  \begin{equation}
    \mathcal{A}(\mu, \gamma, \eta, H) = \frac{1}{9 - \gamma \mu (6 + \gamma \mu)}
    \begin{bmatrix}
      (3 - \gamma \mu)(3 - 2\gamma \mu)(I - \eta H)
       & 3 \gamma \mu (1 - \gamma \mu) (I - \eta H)
      \\
      (3 - 2 \gamma \mu) (2 \gamma \mu - (3 - \gamma \mu) \gamma H)
       &
      3 (1 - \gamma \mu) ((3 - \gamma \mu)I - \gamma^2 \mu H)
    \end{bmatrix}.
    \label{eq:fedacii:A:def}
  \end{equation}
\end{proposition}
The proof of \cref{fedacii:stab:1} is almost identical with \cref{fedaci:stab:1} except the choice of $\alpha$ and $\beta$ are different. We include this proof in \cref{sec:fedacii:stab:1} for completeness.

The following \cref{fedacii:stab:bound} studies the uniform norm bound of $\mathcal{A}$ under the proposed transformation $\mathcal{X}$. The transformation $\mathcal{X}$ is the same as the one studied in \fedaci, which we restate here for the ease of reference.
The bound is also similar to the corresponding bound for on \fedaci as shown in \cref{fedaci:stab:bound}, though the proof is technically more complicated due to the complexity of $\mathcal{A}$.
We defer the proof of  \cref{fedacii:stab:bound} to \cref{sec:fedacii:stab:bound}.
\begin{proposition}[Uniform norm bound of $\mathcal{A}$ under transformation $\mathcal{X}$]
  \label{fedacii:stab:bound}
  Let $\mathcal{A}(\mu, \gamma, \eta, H)$ be defined as in \cref{eq:fedacii:A:def}.
  and assume $\mu > 0$, $\gamma \in [\eta, \sqrt{\frac{\eta}{\mu}}]$, $\eta \in (0,\frac{1}{L}]$.
  Then the following uniform norm bound holds
  \begin{equation}
    \sup_{\mu I \preceq H \preceq LI}
    \left\| \mathcal{X}(\gamma, \eta)^{-1} \mathcal{A}(\mu, \gamma, \eta, H) \mathcal{X}(\gamma, \eta) \right\| \leq
    \begin{cases}
      1 + \frac{\gamma^2 \mu}{\eta} & \text{if~} \gamma \in \left(\eta, \sqrt{\frac{\eta}{\mu}}\right], \\
      1                              & \text{if~} \gamma =  \eta,
    \end{cases}
  \end{equation}
  where $\mathcal{X} (\gamma, \eta)$ is a matrix-valued function defined as
  \begin{equation}
    \mathcal{X}(\gamma, \eta) :=
    \begin{bmatrix}
      \frac{\eta}{\gamma} I & 0
      \\
      I                     & I
    \end{bmatrix}.
    \label{eq:fedacii:X:def}
  \end{equation}
\end{proposition}

\cref{fedacii:stab:1,fedacii:stab:bound} suggest the one-step growth of $ \left\| \mathcal{X}(\gamma, \eta)^{-1} \begin{bmatrix}
  \Delta_{t}^{\mathrm{ag}}
  \\
  \Delta_{t}
\end{bmatrix}  \right\|^4$ as follows.
\begin{proposition}
  \label{fedacii:stab:2}
  In the same setting of \cref{fedacii:stab:main}, the following inequality holds (for all possible $t$)
  \begin{equation}
    \sqrt{
      \expt \left[ \left\| \mathcal{X}(\gamma, \eta)^{-1} \begin{bmatrix}
        \Delta_{t+1}^{\mathrm{ag}}
        \\
        \Delta_{t+1}
      \end{bmatrix}  \right\|^4 \middle| \mathcal{F}_t \right]
    }
    \leq
    7 \gamma^2 \sigma^2 +
    \left\| \mathcal{X}(\gamma, \eta)^{-1} \begin{bmatrix}
      \Delta_{t}^{\mathrm{ag}}
      \\
      \Delta_{t}
    \end{bmatrix}  \right\|^2
    \cdot
    \begin{cases}
      \left(1 + \frac{\gamma^2 \mu}{\eta} \right)^2 & \text{if~} \gamma \in \left(\eta, \sqrt{\frac{\eta}{\mu}}\right], \\
      1                                              & \text{if~} \gamma =  \eta,
    \end{cases}
  \end{equation}
  where $\mathcal{X}$ is the matrix-valued function defined in \cref{eq:fedacii:X:def}.
\end{proposition}
We defer the proof of \cref{fedacii:stab:2} to \cref{sec:fedacii:stab:2}.

The following \cref{fedacii:stab:3} links the discrepancy overhead we wish to bound for \cref{fedacii:stab:main} with the quantity analyzed in \cref{fedacii:stab:2} via \nth{3}-order-smoothness (\cref{asm2}(a)). 
The proof of \cref{fedacii:stab:3} is deferred to \cref{sec:fedacii:stab:3}.
\begin{proposition} 
  \label{fedacii:stab:3}
  In the same setting of \cref{fedacii:stab:main}, the following inequality holds (for all possible $t$)
  \begin{equation}
    \left\|  \nabla F(\overline{w_t^{\mathrm{md}}}) -  \frac{1}{M} \sum_{m=1}^M \nabla F (w_t^{\mathrm{md}, m})  \right\|^2 
    \leq
    \frac{289 \eta^4 Q^2}{324 \gamma^4} \left\| \mathcal{X}(\gamma, \eta)^{-1} \begin{bmatrix}
      \Delta_{t}^{\mathrm{ag}}
      \\
      \Delta_{t}
    \end{bmatrix}  \right\|^4,
  \end{equation}
  where $\mathcal{X}$ is the matrix-valued function defined in \cref{eq:fedacii:X:def}.
\end{proposition}

We are ready to complete the proof of \cref{fedacii:stab:main}.
\begin{proof}[Proof of \cref{fedacii:stab:main}]
  Let $t_0$ be the latest synchronized step prior to $t$. Applying \cref{fedacii:stab:2} gives
  \begin{align}
    & \sqrt{
      \expt \left[ \left\| \mathcal{X}(\gamma, \eta)^{-1} \begin{bmatrix}
        \Delta_{t+1}^{\mathrm{ag}}
        \\
        \Delta_{t+1}
      \end{bmatrix}  \right\|^4 \middle| \mathcal{F}_{t_0} \right]
    }
    \\
    \leq & 7 \gamma^2 \sigma^2 +
    \sqrt{
    \expt \left[
    \left\| \mathcal{X}(\gamma, \eta)^{-1} \begin{bmatrix}
      \Delta_{t}^{\mathrm{ag}}
      \\
      \Delta_{t}
    \end{bmatrix}  \right\|^2 \middle| \mathcal{F}_{t_0}\right]}
    \cdot
    \begin{cases}
      \left(1 + \frac{\gamma^2 \mu}{\eta} \right)^2 & \text{if~} \gamma \in \left(\eta, \sqrt{\frac{\eta}{\mu}}\right], \\
      1                                              & \text{if~} \gamma =  \eta.
    \end{cases}
  \end{align}
  Telescoping from $t_0$ to $t$ gives (note that $\Delta_{t_0}^{\mathrm{ag}} = \Delta_{t_0} = 0$)
  \begin{align}
    \expt \left[ \left\| \mathcal{X}(\gamma, \eta)^{-1} \begin{bmatrix}
      \Delta_{t}^{\mathrm{ag}}
     \\
      \Delta_{t}
    \end{bmatrix}  \right\|^4 \middle| \mathcal{F}_{t_0} \right]
  & \leq
  49 \gamma^4 \sigma^4 (t - t_0)^2
  \cdot
  \begin{cases}
    \left( 1 + \frac{\gamma^2 \mu}{\eta} \right)^{4(t-t_0)} & \text{if~} \gamma \in \left(\eta, \sqrt{\frac{\eta}{\mu}} \right], \\
    1                              & \text{if~} \gamma =  \eta
  \end{cases}
  \\
  & \leq 
  49 \gamma^4 \sigma^4 K^2
  \cdot
  \begin{cases}
    \left( 1 + \frac{\gamma^2 \mu}{\eta} \right)^{4K} & \text{if~} \gamma \in \left(\eta, \sqrt{\frac{\eta}{\mu}} \right], \\
    1                              & \text{if~} \gamma =  \eta,
  \end{cases}
  \end{align}
  where the last inequality is due to $t - t_0 \leq K$ since $K$ is the maximum synchronization interval.
  
  Consequently, by \cref{fedacii:stab:3} we have
  \begin{align}
    & 
    \expt \left[  \left\|  \nabla F(\overline{w_t^{\mathrm{md}}}) -  \frac{1}{M} \sum_{m=1}^M \nabla F (w_t^{\mathrm{md}, m})  \right\|^2 \middle| \mathcal{F}_{t_0}  \right]
    \leq \frac{289 \eta^4 Q^2}{324 \gamma^4} \expt \left[ \left\| \mathcal{X}(\gamma, \eta)^{-1} \begin{bmatrix}
      \Delta_{t}^{\mathrm{ag}}
     \\
      \Delta_{t}
    \end{bmatrix}  \right\|^4 \middle| \mathcal{F}_{t_0} \right]
    \\ 
    \leq & 
    \begin{cases}
      44 \eta^4 Q^2 K^2 \sigma^4 \left(1 + \frac{\gamma^2\mu}{\eta}\right)^{4K}
       & \text{if~} \gamma \in \left(\eta, \sqrt{\frac{\eta}{\mu}} \right],
      \\
      44 \eta^4 Q^2 K^2 \sigma^4
       &
      \text{if~} \gamma = \eta,
    \end{cases}
  \end{align}
  where in the last inequality we used the estimate that $\frac{289}{324} \cdot 49 < 44$. 
\end{proof}

\subsubsection{Proof of \cref{fedacii:stab:1}}

\label{sec:fedacii:stab:1}
\begin{proof}[Proof of \cref{fedacii:stab:1}]
  The proof of  \cref{fedacii:stab:1} follows instantly by plugging $\alpha = \frac{3}{2\gamma \mu} - \frac{1}{2}$, $\beta = \frac{2 \alpha^2 - 1}{\alpha - 1} = \frac{9 - \gamma \mu (6 + \gamma \mu)}{3 \gamma \mu (1 - \gamma\mu)}$ to the general claim on \fedac \cref{fedac:general:stab}:
  \begin{align}
    & \begin{bmatrix}
      (1 - \beta^{-1}) (I - \eta H)
       &
      \beta^{-1} (I - \eta H)
      \\
      (1 - \beta^{-1}) (\alpha^{-1} - \gamma H)
       &
      \beta^{-1} (\alpha^{-1} I - \gamma H) + (1 - \alpha^{-1}) I
    \end{bmatrix}
    \\
    = & \frac{1}{9 - \gamma \mu (6 + \gamma \mu)}
    \begin{bmatrix}
      (3 - \gamma \mu)(3 - 2\gamma \mu)(I - \eta H)
       & 3 \gamma \mu (1 - \gamma \mu) (I - \eta H)
      \\
      (3 - 2 \gamma \mu) (2 \gamma \mu - (3 - \gamma \mu) \gamma H)
       &
      3 (1 - \gamma \mu) ((3 - \gamma \mu)I - \gamma^2 \mu H)
    \end{bmatrix}.
  \end{align}
\end{proof}

\subsubsection{Proof of \cref{fedacii:stab:bound}: uniform norm bound}
\label{sec:fedacii:stab:bound}
The proof idea of this proposition is very similar to \cref{fedaci:stab:bound}, though more complicated technically.
\begin{proof}
  Define another matrix-valued function $\mathcal{B}$ as
  \begin{equation}
    \mathcal{B}(\mu, \gamma, \eta, H) := \mathcal{X}(\gamma, \eta)^{-1} \mathcal{A}(\mu, \gamma, \eta, H) \mathcal{X}(\gamma, \eta).
  \end{equation}
  Since $ \mathcal{X}(\gamma, \eta)^{-1} =
  \begin{bmatrix}
    \frac{\gamma}{\eta}I  & 0
    \\
    -\frac{\gamma}{\eta}I & I
  \end{bmatrix}$ 
  we have 
  \begin{align}
      & \mathcal{B}(\mu, \gamma, \eta, H)
    \\
    = & \frac{1}{(9 - (6 + \gamma \mu)\gamma \mu)\eta}
    \begin{bmatrix}
      \left( 3 \gamma^2 \mu ( 1 - \gamma \mu) + \eta (3 - \gamma \mu)(3 - 2 \gamma \mu) \right) (I - \eta H) &
      3\gamma^2 \mu(1 - \gamma \mu)(I - \eta H)
      \\
      - (\gamma - \eta)\left( 3 \gamma + 6 \eta - \gamma \mu (3 \gamma + 4\eta) \right) I                    &
      3 (1 - \gamma \mu)\left( 3 \eta - \gamma \mu (\gamma + \eta)  \right) I
    \end{bmatrix}.
  \end{align}
  Define the four blocks of $\mathcal{B}(\mu, \gamma, \eta, H)$ as $\mathcal{B}_{11}(\mu, \gamma, \eta, H)$, $\mathcal{B}_{12}(\mu, \gamma, \eta, H)$, $\mathcal{B}_{21}(\mu, \gamma, \eta)$, $\mathcal{B}_{22}(\mu, \gamma, \eta)$ (note that the lower two blocks do not involve $H$), namely
  \begin{align}
    \mathcal{B}_{11}(\mu, \gamma, \eta, H) & = \frac{3 \gamma^2 \mu ( 1 - \gamma \mu) + \eta (3 - \gamma \mu)(3 - 2 \gamma \mu)}{(9 - (6 + \gamma \mu)\gamma \mu) \eta } (I - \eta H),
    \\
    \mathcal{B}_{12}(\mu, \gamma, \eta, H) & = \frac{3\gamma^2 \mu(1 - \gamma \mu)}{(9 - (6 + \gamma \mu)\gamma \mu) \eta }(I - \eta H),
    \\
    \mathcal{B}_{21}(\mu, \gamma, \eta)    & = -\frac{(\gamma - \eta) \mu \left( 3 \gamma + 6 \eta - \gamma \mu (3 \gamma + 4\eta) \right)}{(9 - (6 + \gamma \mu)\gamma \mu) \eta }I,
    \\
    \mathcal{B}_{22}(\mu, \gamma, \eta)    & = \frac{3 (1 - \gamma \mu)\left( 3 \eta - \gamma \mu (\gamma + \eta)  \right)}{(9 - (6 + \gamma \mu)\gamma \mu) \eta}I.
  \end{align}
  \paragraph{Case I: $\eta < \gamma \leq \sqrt{\frac{\eta}{\mu}}$.}
  Since $\gamma \mu \leq 1$, we know that the common denominator
  \begin{equation}
    (9 - (6 + \gamma \mu)\gamma \mu) \eta \geq 2 \eta > 0.
    \label{eq:fedacii:B:denom}
  \end{equation}
  Now we bound the operator norm of each block as follows.

  \paragraph{Bound for $\|\mathcal{B}_{11}\|$.} Since $3 \gamma^2 \mu ( 1 - \gamma \mu) + \eta (3 - \gamma \mu)(3 - 2 \gamma \mu) \geq 0$, we have $\mathcal{B}_{11} \succeq 0$, and therefore
  \begin{align}
         & \|\mathcal{B}_{11}(\mu, \gamma, \eta, H)\|
    \\
    \leq & \frac{3 \gamma^2 \mu ( 1 - \gamma \mu) + \eta (3 - \gamma \mu)(3 - 2 \gamma \mu)}{(9 - (6 + \gamma \mu)\gamma \mu) \eta } (1 - \eta \mu)
    \\
    \leq & \frac{3 \gamma^2 \mu ( 1 - \gamma \mu) + \eta (3 - \gamma \mu)(3 - 2 \gamma \mu)}{(9 - (6 + \gamma \mu)\gamma \mu) \eta }
    \\
    =    &  1 + \frac{3 (\gamma - \eta) \gamma \mu  (1 - \gamma \mu)}{(9 - (6 + \gamma \mu)\gamma \mu) \eta }
    \\
    \leq & 1 + \frac{3 \gamma^2 \mu }{\eta} \cdot \frac{1 - \gamma \mu}{ 9 - 6 \gamma \mu - \gamma^2\mu^2}
    \tag{since $\gamma - \eta \leq \gamma$}
    \\
    \leq & 1 + \frac{\gamma^2 \mu}{3 \eta},
    \label{eq:fedacii:stab:bound:11}
  \end{align}
  where the last inequality is due to $\frac{1 - \gamma \mu}{ 9 - 6 \gamma \mu - \gamma^2\mu^2} \leq \frac{1}{9}$ since $\gamma \mu \leq 1$.

  \paragraph{Bound for $\|\mathcal{B}_{12}\|$.} Similarly we have
  \begin{equation}
    \|\mathcal{B}_{12}(\mu, \gamma, \eta, H)\| \leq \frac{3 \gamma^2 \mu (1 - \gamma\mu)}{(9 - (6 + \gamma \mu)\gamma \mu) \eta}(1-\eta \mu)
    \leq \frac{3 \gamma^2 \mu}{\eta} \cdot \frac{1 - \gamma \mu}{9 - (6 + \gamma \mu)\gamma \mu}
    \leq \frac{\gamma^2 \mu}{3 \eta},
    \label{eq:fedacii:stab:bound:12}
  \end{equation}
  where the last inequality is due to $\frac{1 - \gamma \mu}{ 9 - 6 \gamma \mu - \gamma^2\mu^2} \leq \frac{1}{9}$ since $\gamma \mu \leq 1$.

  \paragraph{Bound for $\|\mathcal{B}_{21}\|$.} Since $\gamma \geq \eta$, we have $(\gamma - \eta) \mu \left( 3 \gamma + 6 \eta - \gamma \mu (3 \gamma + 4\eta) \right) \geq 0$.
  Note that
  \begin{align}
     & (\gamma - \eta) \left( 3 \gamma + 6 \eta - \gamma \mu (3 \gamma + 4\eta) \right)
     \\
     = &  3 \gamma^2 + 3 \gamma \eta - 6 \eta^2 - \gamma \mu (3 \gamma^2 + \gamma \eta - 4 \eta^ 2)
     \\
     = & 4 \gamma^2 - 3\gamma^3 \mu - (\gamma^2 - 3 \gamma \eta + 6 \eta^2 + \gamma^2 \mu \eta - 4\eta^2 \gamma \mu  ),
  \end{align}
  and
  \begin{align}
        & \gamma^2 - 3 \gamma \eta + 6 \eta^2 + \gamma^2 \mu \eta - 4\eta^2 \gamma \mu 
    \\
    \geq & \gamma^2 - 3 \gamma \eta + 6 \eta^2 - 3\eta^2 \gamma \mu  \tag{since $\eta \leq \gamma$}
    \\
    \geq & \gamma^2 - 3 \gamma \eta + 3 \eta^2 \tag{since $\gamma \mu \leq 1$}
    \\
    \geq & 0. \tag{AM-GM inequality}
  \end{align}
  Consequently,
  \begin{equation}
    (\gamma - \eta) \mu \left( 3 \gamma + 6 \eta - \gamma \mu (3 \gamma + 4\eta) \right) \leq 4 \gamma^2 \mu - 3 \gamma^3 \mu^2.
    \label{eq:fedacii:stab:bound:tmp}
  \end{equation}
  It follows that
  \begin{align}
         & \|\mathcal{B}_{21}(\mu, \gamma, \eta)\| =  \frac{\mu (\gamma - \eta)  \left( 3 \gamma + 6 \eta - \gamma \mu (3 \gamma + 4\eta) \right)}{(9 - (6 + \gamma \mu)\gamma \mu) \eta }
    \\
    \leq    & \frac{4 \gamma^2 \mu - 3 \gamma^3 \mu^2}{(9 - (6 + \gamma \mu)\gamma \mu) \eta }  \tag{by \cref{eq:fedacii:stab:bound:tmp}}
    \\
    =   &  \frac{\gamma^2 \mu}{\eta} \cdot \frac{4 - 3 \gamma \mu}{9 - 6 \gamma \mu - \gamma^2 \mu^2} \leq  \frac{2\gamma^2 \mu}{3\eta}.
    \label{eq:fedacii:stab:bound:21}
  \end{align}
  where the last inequality is due to $\frac{4 - 3 \gamma \mu}{9 - 6 \gamma \mu - \gamma^2 \mu^2} \leq \frac{2}{3}$ since $\gamma \mu \leq 1$.

  \paragraph{Bound for $\mathcal{B}_{22}$.} Since $\gamma > \eta$ and $\gamma^2 \mu \leq \eta$, we have $3\eta - \gamma \mu(\gamma + \eta) \geq 3 \eta - 2 \gamma^2 \mu \geq \eta$. Thus $\mathcal{B}_{22} \succeq 0$, which implies
  \begin{equation}
    \|\mathcal{B}_{22}(\mu, \gamma, \eta)\|
    = \frac{3 (1 - \gamma \mu)\left( 3 \eta - \gamma \mu (\gamma + \eta)  \right)}{(9 - (6 + \gamma \mu)\gamma \mu) \eta}
    = 1 + \frac{\gamma \mu \left( -6 \eta - 3 \gamma + \gamma \mu (3 \gamma + 4 \eta ) \right) }{(9 - (6 + \gamma \mu)\gamma \mu) \eta }
    \leq 1.
    \label{eq:fedacii:stab:bound:22}
  \end{equation}

  The operator norm of block matrix $\mathcal{B}$ can be bounded via its blocks via \cref{helper:blocknorm} as
    \begin{align}
         & \mathcal{B}(\mu, \gamma, \eta, H)
    \\
    \leq & \max \left\{\| \mathcal{B}_{11}(\mu, \gamma, \eta, H) \|, \| \mathcal{B}_{22}(\mu, \gamma, \eta) \| \right\}
    +
    \max \left\{\| \mathcal{B}_{12}(\mu, \gamma, \eta, H) \|, \| \mathcal{B}_{21}(\mu, \gamma, \eta) \| \right\} \tag{\cref{helper:blocknorm}}
    \\
    \leq & \max \left\{ 1 + \frac{\gamma^2 \mu}{3\eta}, 1 \right\} + \max \left\{ \frac{\gamma^2 \mu}{3\eta}, \frac{2\gamma^2 \mu}{3\eta} \right\}
    \leq 1 + \frac{\gamma^2 \mu}{\eta}.
    \tag{\cref{eq:fedacii:stab:bound:11,eq:fedacii:stab:bound:12,eq:fedacii:stab:bound:21,eq:fedacii:stab:bound:22}}
  \end{align}
  \paragraph{Case II: $\gamma = \eta$.} In this case we have
  \begin{align}
    \|\mathcal{B}_{11}(\mu, \gamma, \eta, H)\| &
    \leq 1 - \eta \mu,
    \\
    \|\mathcal{B}_{12}(\mu, \gamma, \eta, H)\| &
    \leq \frac{3 \eta \mu - 6 \eta^2 \mu^2 + 3 \eta^3 \mu^3}{9 - 6 \eta \mu - \eta^2 \mu^2},
    \\
    \|\mathcal{B}_{21}(\mu, \gamma, \eta)\|    &
    = 0,
    \\
    \|\mathcal{B}_{22}(\mu, \gamma, \eta)\|    & = \frac{9- 15 \eta \mu + 6 \eta^2 \mu^2}{9 - 6 \eta \mu - \eta^2 \mu^2} = 1  - \frac{9 \eta \mu - 7 \eta^2 \mu^2}{9 - 6 \eta \mu - \eta^2 \mu^2}.
  \end{align}

 Similarly the operator norm of block matrix $\mathcal{B}$ can be bounded via its blocks via \cref{helper:blocknorm} as
   \begin{align}
         & \mathcal{B}(\mu, \gamma, \eta, H)
    \\
    \leq & \max \left\{\| \mathcal{B}_{11}(\mu, \gamma, \eta, H) \|, \| \mathcal{B}_{22}(\mu, \gamma, \eta) \| \right\}
    +
    \max \left\{\| \mathcal{B}_{12}(\mu, \gamma, \eta, H) \|, \| \mathcal{B}_{21}(\mu, \gamma, \eta) \| \right\} \tag{\cref{helper:blocknorm}}
    \\
    \leq & \max \left\{ 1 - \eta \mu + \frac{3 \eta \mu - 6 \eta^2 \mu^2 + 3 \eta^3 \mu^3}{9 - 6 \eta \mu - \eta^2 \mu^2}, \frac{9- 15 \eta \mu + 6 \eta^2 \mu^2}{9 - 6 \eta \mu - \eta^2 \mu^2}
    + \frac{3 \eta \mu - 6 \eta^2 \mu^2 + 3 \eta^3 \mu^3}{9 - 6 \eta \mu - \eta^2 \mu^2}\right\}
    \\
    \leq & \max \left\{ 1 - \frac{6 \eta \mu - 4 \eta^3 \mu^3}{9 - 6 \eta \mu - \eta^2 \mu^2} , 1 - \frac{6 \eta \mu - \eta^2\mu^2 - 3 \eta^3 \mu^3}{9 - 6 \eta \mu - \eta^2 \mu^2} \right\} \leq 1.
  \end{align}
  Summarizing the above two cases completes the proof of \cref{fedacii:stab:bound}.
\end{proof}

\subsubsection{Proof of \cref{fedacii:stab:2}}
\label{sec:fedacii:stab:2}
In this section we apply  \cref{fedacii:stab:1,fedacii:stab:bound} to establish \cref{fedacii:stab:2}.
\begin{proof}[Proof of \cref{fedacii:stab:2}]
  If $t+1$ is a synchronized step, then the bound trivially holds since $\Delta_{t+1}^{\mathrm{ag}} = \Delta_{t+1} = 0$ due to synchronization.

  From now on assume $t+1$ is not a synchronized step, for which \cref{fedacii:stab:1} is applicable. Multiplying $\mathcal{X}(\gamma, \eta)^{-1}$ to the left on both sides of \cref{fedacii:stab:1} gives (we omit the details since the reasoning is the same as in the proof of \cref{fedaci:stab:2}.
  \begin{align}
    \mathcal{X}(\gamma, \eta)^{-1} \begin{bmatrix}
      \Delta_{t+1}^{\mathrm{ag}}
      \\
      \Delta_{t+1}
    \end{bmatrix}
    =
    \mathcal{X}(\gamma, \eta)^{-1} \mathcal{A}(\mu, \gamma, \eta, H_t) \mathcal{X}(\gamma, \eta)^{-1} \left( \mathcal{X}(\gamma, \eta) \begin{bmatrix}
        \Delta_{t}^{\mathrm{ag}}
        \\
        \Delta_{t}
      \end{bmatrix}  \right)
    -
    \begin{bmatrix}   \gamma I \\ 0 \end{bmatrix} \Delta_t^{\varepsilon}.
    \label{eq:fedacii:stab:2:1}
  \end{align}
  Before we proceed, we introduce a few more notations to simplify the discussion. 
  Denote the shortcut $\mathcal{B}_t := \mathcal{X}(\gamma, \eta)^{-1} \mathcal{A}(\mu, \gamma, \eta, H_t) \mathcal{X}(\gamma, \eta)$,  $\mathcal{X} = \mathcal{X}(\gamma, \eta)$,
  $\tilde{\Delta}_t =  \mathcal{X}^{-1}\begin{bmatrix} \Delta_{t}^{\mathrm{ag}} \\  \Delta_{t} \end{bmatrix}$, and $\tilde{\Delta}_t^\varepsilon =  \begin{bmatrix} \gamma I \\ 0 \end{bmatrix} \Delta_t^{\varepsilon}$.
  Then \cref{eq:fedacii:stab:2:1} becomes $\tilde{\Delta}_{t+1} = \mathcal{B}_t \tilde{\Delta}_{t} - \tilde{\Delta}_{t}^{\varepsilon}$. Thus
  \begin{align}
         & \expt \left[ \| \tilde{\Delta}_{t+1} \|^4 |\mathcal{F}_t \right] = \expt \left[ \|  \mathcal{B}_t \tilde{\Delta}_{t} - \tilde{\Delta}_t^\varepsilon  \|^4 |\mathcal{F}_t \right]
    \tag{by \cref{fedacii:stab:1}}
    \\
    =    & \expt \left[  \left(  \| \mathcal{B}_t \tilde{\Delta}_t \|^2
      +
      \|\tilde{\Delta}_t^\varepsilon \|^2
      -
      2 \langle \mathcal{B}_t \tilde{\Delta}_t , \tilde{\Delta}_t^\varepsilon  \rangle \right)^2   \right]
    \\
    =    & \| \mathcal{B}_t \tilde{\Delta}_t \|^4
    + \expt \left[ \| \tilde{\Delta}_{t}^\varepsilon \|^4 |\mathcal{F}_t \right]
    + 4 \expt \left[     \langle \mathcal{B}_t \tilde{\Delta}_t , \tilde{\Delta}_t^\varepsilon  \rangle^2 |\mathcal{F}_t     \right]
    + 2 \| \mathcal{B}_t \tilde{\Delta}_t \|^2 \expt \left[ \| \tilde{\Delta}_{t}^\varepsilon \|^2 |\mathcal{F}_t \right]
    \\
         &
    - 4  \| \mathcal{B}_t \tilde{\Delta}_t \|^2  \expt \left[\langle \mathcal{B}_t \tilde{\Delta}_t , \tilde{\Delta}_t^\varepsilon  \rangle |\mathcal{F}_t     \right]
    - 4 \expt \left[  \| \tilde{\Delta}_{t}^\varepsilon \|^2   \langle \mathcal{B}_t \tilde{\Delta}_t , \tilde{\Delta}_t^\varepsilon  \rangle |\mathcal{F}_t     \right]
    \\
    =    & \| \mathcal{B}_t \tilde{\Delta}_t \|^4 + \expt \left[ \| \tilde{\Delta}_{t}^\varepsilon \|^4 |\mathcal{F}_t \right]
    + 4 \expt \left[     \langle \mathcal{B}_t \tilde{\Delta}_t , \tilde{\Delta}_t^\varepsilon  \rangle^2 |\mathcal{F}_t     \right]
    + 2 \| \mathcal{B}_t \tilde{\Delta}_t \|^2 \expt \left[ \| \tilde{\Delta}_{t}^\varepsilon \|^2 |\mathcal{F}_t \right]
    \\
         &
    - 4 \expt \left[  \| \tilde{\Delta}_{t}^\varepsilon \|^2   \langle \mathcal{B}_t \tilde{\Delta}_t , \tilde{\Delta}_t^\varepsilon  \rangle |\mathcal{F}_t     \right]
    \tag{by independence and $\expt [\tilde{\Delta}_t^\varepsilon |\mathcal{F}_t] = 0$}
    \\
    \leq & \| \mathcal{B}_t \tilde{\Delta}_t \|^4 + \expt \left[ \| \tilde{\Delta}_{t}^\varepsilon \|^4 |\mathcal{F}_t \right]
    + 6 \| \mathcal{B}_t \tilde{\Delta}_t \|^2 \expt \left[ \| \tilde{\Delta}_{t}^\varepsilon \|^2 |\mathcal{F}_t \right]
    + 4 \| \mathcal{B}_t \tilde{\Delta}_t \| \expt \left[  \| \tilde{\Delta}_{t}^\varepsilon \|^3   |\mathcal{F}_t     \right]
    \tag{Cauchy-Schwarz inequality}
    \\
    \leq & \| \mathcal{B}_t \tilde{\Delta}_t \|^4 + 5 \expt \left[ \| \tilde{\Delta}_{t}^\varepsilon \|^4 |\mathcal{F}_t \right]
    + 7 \| \mathcal{B}_t \tilde{\Delta}_t \|^2 \expt \left[ \| \tilde{\Delta}_{t}^\varepsilon \|^2 |\mathcal{F}_t \right]
    \tag{AM-GM inequality}
    \\
    \leq & \| \mathcal{B}_t \tilde{\Delta}_t \|^4
    + 40 \gamma^4  \sigma^4
    + 14 \gamma^2 \sigma^2\| \mathcal{B}_t \tilde{\Delta}_t \|^2
    \tag{bounded \nth{4} central moment via \cref{helper:diff:4th}}
    \\
    \leq & \left( \|\mathcal{B}_t \tilde{\Delta}_t\|^2 + 7 \gamma^2 \sigma^2 \right)^2 \leq  \left( \|\mathcal{B}_t\|^2 \|\tilde{\Delta}_t\|^2 + 7 \gamma^2 \sigma^2 \right)^2.
  \end{align}
  Applying \cref{fedacii:stab:bound},
  \begin{equation}
    \sqrt{\expt \left[ \| \tilde{\Delta}_{t+1} \|^4 |\mathcal{F}_t \right]}
    \leq
    7 \gamma^2 \sigma^2 +  \| \tilde{\Delta}_{t}\|^2 \cdot  
    \begin{cases}
      \left(1 + \frac{\gamma^2 \mu}{\eta}\right)^2 & \text{if~} \gamma \in \left(\eta, \sqrt{\frac{\eta}{\mu}}\right], \\
      1                              & \text{if~} \gamma =  \eta.
    \end{cases}
  \end{equation}
  Resetting the notations completes the proof.
\end{proof}

\subsubsection{Proof of \cref{fedacii:stab:3}}
\label{sec:fedacii:stab:3}
In this section we will prove \cref{fedacii:stab:3} in two steps via the following two claims. For both two claims $\mathcal{X}$ stands for the matrix-valued functions defined in \cref{eq:fedacii:X:def}.

\begin{claim} \label{fedacii:stab:3:1}
  In the same setting of \cref{fedacii:stab:main}, the following inequality holds (for all possible $t$)
  \begin{equation}
    \left\|  \nabla F(\overline{w_t^{\mathrm{md}}}) -  \frac{1}{M} \sum_{m=1}^M \nabla F (w_t^{\mathrm{md}, m})  \right\|^2 
    \leq
    \frac{Q^2}{4} \left\| \mathcal{X}(\gamma, \eta)^\intercal
  \begin{bmatrix} \frac{9 - 9 \gamma \mu + 2 \gamma^2 \mu^2}{9 - 6 \gamma \mu - \gamma^2 \mu^2}  I \\ \frac{3 \gamma \mu - 3\gamma^2 \mu^2}{9 - 6 \gamma \mu - \gamma^2 \mu^2} I \end{bmatrix}
     \right\|^4
     \left\| \mathcal{X}(\gamma, \eta)^{-1} \begin{bmatrix}
      \Delta_{t}^{\mathrm{ag}}
      \\
      \Delta_{t}
    \end{bmatrix}  \right\|^4.
  \end{equation}
\end{claim}
\begin{claim}
  \label{fedacii:stab:left}
  Assume $\mu > 0$, $\gamma \in [\eta,\sqrt{\frac{\eta}{\mu}}]$, then
  \(
    \left\| \mathcal{X}(\gamma, \eta)^\intercal
  \begin{bmatrix} \frac{9 - 9 \gamma \mu + 2 \gamma^2 \mu^2}{9 - 6 \gamma \mu - \gamma^2 \mu^2}  I \\ \frac{3 \gamma \mu - 3\gamma^2 \mu^2}{9 - 6 \gamma \mu - \gamma^2 \mu^2} I \end{bmatrix} \right\| \leq
    \frac{\sqrt{17} \eta }{3\gamma}.
  \)
\end{claim}

\begin{proof}[Proof of \cref{fedacii:stab:3}]
  Follow trivially with \cref{fedaci:stab:3:1,fedacii:stab:left} as
  \begin{align}
      \left\|  \nabla F(\overline{w_t^{\mathrm{md}}}) -  \frac{1}{M} \sum_{m=1}^M \nabla F (w_t^{\mathrm{md}, m})  \right\|^2 
      & \leq
      \frac{Q^2}{4} \left( \frac{\sqrt{17} \eta}{3 \gamma} \right)^4 \left\| \mathcal{X}(\gamma, \eta)^{-1} \begin{bmatrix}
        \Delta_{t}^{\mathrm{ag}}
        \\
        \Delta_{t}
      \end{bmatrix}  \right\|^4
      \\
      & = 
      \frac{289 \eta^4 Q^2}{324 \gamma^4} \left\| \mathcal{X}(\gamma, \eta)^{-1} \begin{bmatrix}
        \Delta_{t}^{\mathrm{ag}}
        \\
        \Delta_{t}
      \end{bmatrix}  \right\|^4.
  \end{align}
\end{proof}

Now we finish the proof of these two claims.

\begin{proof}[Proof of \cref{fedacii:stab:3:1}]
  Helper \cref{helper:3rd:Lip} shows that $\left\|  \nabla F(\overline{w_t^{\mathrm{md}}}) -  \frac{1}{M} \sum_{m=1}^M \nabla F (w_t^{\mathrm{md}, m})  \right\|^2 $ can be bounded by $\nth{4}$-moment of difference:
  \begin{align}
    &  \left\|  \nabla F(\overline{w_t^{\mathrm{md}}}) -  \frac{1}{M} \sum_{m=1}^M \nabla F (w_t^{\mathrm{md}, m})  \right\|^2 
    \leq
     \frac{Q^2}{4} \cdot \frac{1}{M} \sum_{m=1}^M  \|w_t^{\mathrm{md},m} - \overline{w_t^{\mathrm{md}}}\|^4
    \tag{\cref{helper:3rd:Lip}}
\\
\leq & \frac{Q^2}{4} \| \Delta_t^{\mathrm{md}} \|^4 
\tag{convexity of $\|\cdot\|^4$}
\\
=  & 
\frac{Q^2}{4}
 \left\|
 \begin{bmatrix} (1 - \beta^{-1}) I \\ \beta^{-1} I \end{bmatrix}^\intercal
 \begin{bmatrix} \Delta_{t}^{\mathrm{ag}} \\ \Delta_{t} \end{bmatrix}
 \right\|^4 
\tag{definition of ``md''}
\\
\leq & \frac{Q^2}{4} \left\| \mathcal{X}(\gamma, \eta)^\intercal \begin{bmatrix} (1 - \beta^{-1}) I \\ \beta^{-1} I \end{bmatrix} \right\|^4
\cdot \left\| \mathcal{X}(\gamma, \eta)^{-1}
 \begin{bmatrix} \Delta_{t}^{\mathrm{ag}} \\ \Delta_{t} \end{bmatrix}
 \right\|^4.
\tag{sub-multiplicativity}
\\
= & \frac{Q^2}{4} \left\| \begin{bmatrix} \frac{9 - 9 \gamma \mu + 2 \gamma^2 \mu^2}{9 - 6 \gamma \mu - \gamma^2 \mu^2}  I \\ \frac{3 \gamma \mu - 3\gamma^2 \mu^2}{9 - 6 \gamma \mu - \gamma^2 \mu^2} I \end{bmatrix} \right\|^4
\cdot \left\| \mathcal{X}(\gamma, \eta)^{-1}
 \begin{bmatrix} \Delta_{t}^{\mathrm{ag}} \\ \Delta_{t} \end{bmatrix}
 \right\|^4.
\end{align}
\end{proof}

\begin{proof}[Proof of \cref{fedacii:stab:left}]
  Direct calculation shows that
  \begin{equation}
    \mathcal{X}(\gamma, \eta)^\intercal
    \begin{bmatrix} \frac{9 - 9 \gamma \mu + 2 \gamma^2 \mu^2}{9 - 6 \gamma \mu - \gamma^2 \mu^2}  I \\ \frac{3 \gamma \mu - 3\gamma^2 \mu^2}{9 - 6 \gamma \mu - \gamma^2 \mu^2} I \end{bmatrix}
    =
    \begin{bmatrix}  \frac{3 \gamma^2 \mu (1-\gamma \mu) + \eta (3 - \gamma \mu)(3 - 2 \gamma \mu)}{\gamma (9 - 6 \gamma \mu - \gamma^2\mu^2)}  I \\ \frac{3 \gamma^2 \mu ( 1- \gamma \mu)}{\gamma (9 - 6 \gamma \mu - \gamma^2\mu^2)}  I \end{bmatrix}.
  \end{equation}
  Since $\gamma^2 \mu \leq \eta$ and $\gamma \mu \leq 1$, we have
  \begin{equation}
    0 \leq
    \frac{3 \gamma^2 \mu (1-\gamma \mu) + \eta (3 - \gamma \mu)(3 - 2 \gamma \mu)}{\gamma (9 - 6 \gamma \mu - \gamma^2\mu^2)}
    \leq
    \frac{\eta}{\gamma} \cdot \frac{12 - 12 \gamma \mu + 2 \gamma^2 \mu^2}{9 - 6 \gamma \mu - \gamma^2\mu^2}
    \leq
    \frac{4\eta}{3\gamma},
  \end{equation}
  and
  \begin{equation}
    0 \leq \frac{3 \gamma^2 \mu (1-\gamma \mu)}{\gamma (9 - 6 \gamma \mu - \gamma^2\mu^2)}
    \leq \frac{\eta}{\gamma} \cdot \frac{3  (1-\gamma \mu)}{9 - 6 \gamma \mu - \gamma^2\mu^2}
    \leq \frac{\eta}{3\gamma}.
  \end{equation}
  Consequently,
  \begin{equation}
    \left\|     \begin{bmatrix}  \frac{3 \gamma^2 \mu (1-\gamma \mu) + \eta (3 - \gamma \mu)(3 - 2 \gamma \mu)}{\gamma (9 - 6 \gamma \mu - \gamma^2\mu^2)}  I \\ \frac{3 \gamma^2 \mu ( 1- \gamma \mu)}{\gamma (9 - 6 \gamma \mu - \gamma^2\mu^2)}  I \end{bmatrix} \right\|
    \leq
    \sqrt{\left( \frac{4 \eta}{3\gamma} \right)^2 + \left( \frac{\eta}{3\gamma} \right)^2}
    \leq
    \frac{\sqrt{17} \eta }{3 \gamma}.
  \end{equation}
\end{proof}

\subsection{Convergence of \fedacii under \cref{asm1}: Complete version of \cref{fedac:a1}(b)}
\label{sec:fedacii:a1}
\subsubsection{Main theorem and lemma}
In this subsection we establish the convergence of \fedacii under \cref{asm1}. 
We will provide a complete, non-asymptotic version of \cref{fedac:a1}(b) and provide the proof. 
\begin{theorem}[Convergence of \fedacii under \cref{asm1}, complete version of \cref{fedac:a1}(b)]
  \label{fedacii:a1:full}
  Let $F$ be $\mu > 0$ strongly convex, and assume \cref{asm1}, then for 
  \begin{equation}
    \eta = \min \left\{ \frac{1}{L}, \frac{9K}{\mu T^2} \log^2 \left( \euler
     + \min \left\{ \frac{\mu M T \Phi_0}{\sigma^2} + \frac{\mu^3 T^4 \Phi_0}{L^2K^3 \sigma^2}\right\} \right)\right\},
  \end{equation}
  \fedacii yields
  \begin{align}
    \expt[\Phi_T] \leq & \min \left\{ \exp \left( - \frac{\mu T}{3 L} \right), \exp \left( - \frac{\mu^{\frac{1}{2}} T}{3 L^{\frac{1}{2}} K^{\frac{1}{2}}} \right)\right\} \Phi_0 
    \\
    & 
    + \frac{4 \sigma^2}{\mu M T} \log \left( \euler + \frac{\mu MT \Phi_0}{\sigma^2} \right)
    + \frac{8101 L^2 K^3 \sigma^2}{\mu^3 T^4} \log^4\left( \euler + \frac{\mu^3 T^4 \Phi_0}{L^2 K^3 \sigma^2} \right),
  \end{align}
  where $\Phi_t$ is the ``centralized'' potential function defined in \cref{eq:centralied:potential}.
\end{theorem} 
\begin{remark}
  The simplified version \cref{fedac:a1}(b) in the main body can be obtained by replacing $K$ with $T/R$ and upper bound $\Phi_0$ by $LD_0^2$.
\end{remark}
Note that most of the results established towards \cref{fedacii:a2:full} can be recycled as long as it does not assume \cref{asm2}. 
In particular, we will reuse the perturbed iterate analysis \cref{fedacii:conv:main}, and provide an alternative version of discrepancy overhead bounds, as shown in \cref{fedacii:a1:stab:main}. 
The only difference is that now we use $L$-smoothness to bound the discrepancy term. 
\begin{lemma}[Discrepancy overhead bounds]
  \label{fedacii:a1:stab:main}
  Let $F$ be $\mu > 0$-strongly convex, and assume \cref{asm1}, then for $\alpha = \frac{3}{2 \gamma \mu} - \frac{1}{2}$, $\beta =\frac{2 \alpha^2 - 1}{\alpha - 1}$, $\gamma \in [\eta, \sqrt{ \frac{\eta}{\mu}}]$, $\eta \in (0, \frac{1}{L}]$, \fedac satisfies (for all $t$)
  \begin{align}
     & \expt \left[ \left\|  \nabla F(\overline{w_t^{\mathrm{md}}}) -  \frac{1}{M} \sum_{m=1}^M \nabla F (w_t^{\mathrm{md}, m})  \right\|^2 \right]
    \leq  \begin{cases}
      4 \eta^2 L^2 K \sigma^2 \left(1 + \frac{\gamma^2\mu}{\eta}\right)^{2K}
       & \text{if~} \gamma \in \left(\eta, \sqrt{\frac{\eta}{\mu}} \right],
      \\
      4 \eta^2 L^2 K \sigma^2
       &
      \text{if~} \gamma = \eta.
    \end{cases}
  \end{align}
\end{lemma}
The proof of \cref{fedacii:a1:stab:main} is deferred to \cref{sec:fedacii:a1:stab:main}.

Now plug in the choice of $\gamma = \max \left\{ \sqrt{\frac{\eta}{\mu K}}, \eta\right\}$ to \cref{fedacii:conv:main,fedacii:a1:stab:main}, which leads to the following lemma.
\begin{lemma}[Convergence of \fedacii for general $\eta$ under \cref{asm1}]
  \label{fedacii:a1:general:eta}
    Let $F$ be $\mu > 0$-strongly convex, and assume \cref{asm1}, then for any $\eta \in (0, \frac{1}{L}]$, \fedacii yields
    \begin{equation}
      \expt[\Phi_T]
      \leq \exp \left(  - \frac{1}{3} \max \left\{ \eta \mu, \sqrt{\frac{\eta \mu}{K}}\right\}T \right) \Phi_0
      + \frac{\eta^{\frac{1}{2}} \sigma^2}{\mu^{\frac{1}{2}} M K^{\frac{1}{2}} }
      + \frac{100 \eta^2 L^2 K \sigma^2}{\mu}.
      \label{eq:fedacii:a1:general:eta}
    \end{equation}
\end{lemma}
\begin{proof}[Proof of \cref{fedacii:a1:general:eta}]
  Applying \cref{fedacii:conv:main} yields
  \begin{align}
    \expt[\Phi_T] \leq & \exp\left( - \frac{1}{3} \max \left\{ \eta \mu, \sqrt{\frac{\eta \mu}{K}}\right\}T \right) \Phi_0
    +
    \min\left\{ \frac{3\eta L \sigma^2}{2 \mu M}, \frac{3 \eta^{\frac{3}{2}} L K^{\frac{1}{2}} \sigma^2}{2 \mu^{\frac{1}{2}} M} \right\}
    \\
                       & + \max\left\{ \frac{\eta \sigma^2}{2M}, \frac{\eta^{\frac{1}{2}} \sigma^2}{2 \mu^{\frac{1}{2}} M K^{\frac{1}{2}} }          \right\}
    + \frac{3}{\mu} \max_{0 \leq t < T}  \expt \left[ \left\|  \nabla F(\overline{w_t^{\mathrm{md}}}) -  \frac{1}{M} \sum_{m=1}^M \nabla F (w_t^{\mathrm{md}, m})  \right\|^2 \right].
    \label{eq:fedaci:proof:1:mod}
  \end{align}
  Applying \cref{fedacii:a1:stab:main} yields (for all $t$)
  \begin{align}
         & \frac{3}{\mu} \expt \left[ \left\|  \nabla F(\overline{w_t^{\mathrm{md}}}) -  \frac{1}{M} \sum_{m=1}^M \nabla F (w_t^{\mathrm{md}, m})  \right\|^2 \right]
    \leq
    \begin{cases}
      12 \mu^{-1} \eta^2 L^2 K \sigma^2  \left(1 + \frac{1}{K}\right)^{2K}
       & \text{if~} \gamma = \sqrt{\frac{\eta}{\mu K}},
      \\
      12 \mu^{-1} \eta^2 L^2 K \sigma^2
       &
      \text{if~} \gamma = \eta
    \end{cases}
    \\
    \leq & 12 \euler^{2} \mu^{-1} \eta^2 L^2 K \sigma^2.
    \label{eq:fedaci:proof:2:mod}
  \end{align}
  Note that
  \begin{align}
         & \min\left\{ \frac{3\eta L \sigma^2}{2 \mu M}, \frac{3 \eta^{\frac{3}{2}} L K^{\frac{1}{2}} \sigma^2}{2 \mu^{\frac{1}{2}} M} \right\}
    + \max\left\{ \frac{\eta \sigma^2}{2M}, \frac{\eta^{\frac{1}{2}} \sigma^2}{2 \mu^{\frac{1}{2}} M K^{\frac{1}{2}} }          \right\}
    \\
    \leq & \frac{3 \eta^{\frac{3}{2}} L K^{\frac{1}{2}} \sigma^2}{2 \mu^{\frac{1}{2}} M} + \frac{\eta \sigma^2}{2M} + \frac{\eta^{\frac{1}{2}} \sigma^2}{2 \mu^{\frac{1}{2}} M K^{\frac{1}{2}} }
    \\
    \leq &
    \frac{7 \eta^{\frac{3}{2}} L K^{\frac{1}{2}} \sigma^2}{4 \mu^{\frac{1}{2}} M} + \frac{3 \eta^{\frac{1}{2}} \sigma^2}{4 \mu^{\frac{1}{2}} M K^{\frac{1}{2}} }.
    \tag{by AM-GM inequality, and $\mu \leq L$}
  \end{align}
  By Young's inequality,
  \begin{align}
    \frac{7 \eta^{\frac{3}{2}} L K^{\frac{1}{2}} \sigma^2}{4 \mu^{\frac{1}{2}} M} 
    \leq & \left( \frac{3}{4} \frac{\eta^{\frac{1}{2}} \sigma^2}{\mu^{\frac{1}{2}} M K^{\frac{1}{2}} } \right)^{\frac{1}{3}}
    \left( 3 \cdot \frac{\eta^2 L^\frac{3}{2} K \sigma^2}{\mu^{\frac{1}{2}} M} \right)^{\frac{2}{3}}
    \tag{since $\frac{7}{4} \leq \left( \frac{3}{4} \right)^{\frac{1}{3}} (3)^{\frac{2}{3}} $ }
    \\
    \leq & \frac{1}{4} \cdot \frac{\eta^{\frac{1}{2}} \sigma^2}{\mu^{\frac{1}{2}} M K^{\frac{1}{2}} } + 2 \cdot \frac{\eta^2 L^\frac{3}{2} K \sigma^2}{\mu^{\frac{1}{2}} M}
    \tag{by Young's inequality}
    \\
    \leq & \frac{\eta^{\frac{1}{2}} \sigma^2}{4 \mu^{\frac{1}{2}} M K^{\frac{1}{2}} } + \frac{2\eta^2 L^2 K \sigma^2}{\mu}.
    \tag{since $L \geq \mu$ and $M \geq 1$}
  \end{align}
  Combining the above inequalities gives
  \begin{equation}
    \expt[\Phi_T]
    \leq \exp \left(  - \frac{1}{3} \max \left\{ \eta \mu, \sqrt{\frac{\eta \mu}{K}}\right\}T \right) \Phi_0
    + \frac{\eta^{\frac{1}{2}} \sigma^2}{\mu^{\frac{1}{2}} M K^{\frac{1}{2}} }
    + \frac{(12 \euler^2 + 2) \eta^2 L^2 K \sigma^2}{\mu}.
  \end{equation}
  The proof then follows by the estimate $12 \euler^2 + 2 < 100$.
\end{proof}

\cref{fedacii:a1:full} then follows by plugging in the appropriate $\eta$ to \cref{fedacii:a1:general:eta}.
\begin{proof}[Proof of \cref{fedacii:a1:full}]
  To simplify the notation, we denote the decreasing term in \cref{eq:fedacii:a1:general:eta} in \cref{fedacii:a1:general:eta} as $\varphi_{\downarrow}(\eta)$ and the increasing term as $\varphi_{\uparrow}(\eta)$, namely
  \begin{align}
    \varphi_{\downarrow}(\eta) := \exp \left(  - \frac{1}{3} \max \left\{ \eta \mu, \sqrt{\frac{\eta \mu}{K}}\right\}T \right) \Phi_0,
    \quad
    \varphi_{\uparrow}(\eta) := \frac{\eta^{\frac{1}{2}} \sigma^2}{\mu^{\frac{1}{2}} M K^{\frac{1}{2}} }
    + \frac{100 \eta^2 L^2 K \sigma^2}{\mu}.
  \end{align}
  Let
  \begin{equation}
    \eta_0 := \frac{9K}{\mu T^2} \log^2 \left( \euler
     + \min \left\{ \frac{\mu M T \Phi_0}{\sigma^2} + \frac{\mu^3 T^4 \Phi_0}{L^2K^3 \sigma^2}\right\} \right)
     ,
     \quad
     \text{then }
     \eta = \min \left\{ \frac{1}{L}, \eta_0 \right\}.
  \end{equation}
  Therefore
  \begin{equation}
    \varphi_{\downarrow}(\eta)
    \leq
    \min \left\{ \exp \left( - \frac{\mu T}{3 L} \right), \exp \left( - \frac{\mu^{\frac{1}{2}} T}{3 L^{\frac{1}{2}} K^{\frac{1}{2}}} \right)\right\} \Phi_0
    +
    \frac{\sigma^2}{\mu MT} + \frac{L^2 K^3 \sigma^2}{\mu^3 T^4}.
  \end{equation}
  and
  \begin{align}
    \varphi_{\uparrow}(\eta)  \leq \varphi_{\uparrow}(\eta_0) 
    \leq
    & 
    \frac{3 \sigma^2}{\mu M T} \log \left( \euler + \frac{\mu MT \Phi_0}{\sigma^2} \right)
    + \frac{8100 L^2 K^3 \sigma^2}{\mu^3 T^4} \log^4\left( \euler + \frac{\mu^3 T^4 \Phi_0}{L^2 K^3 \sigma^2} \right).
    \label{eq:fedacii:a1:3}
  \end{align}
  Consequently,
  \begin{align}
    \expt[\Phi_T] \leq & \varphi_{\downarrow} \left(\frac{1}{L} \right) + \varphi_{\downarrow}(\eta_0) + \varphi_{\uparrow}(\eta_0) \leq \min \left\{ \exp \left( - \frac{\mu T}{3 L} \right), \exp \left( - \frac{\mu^{\frac{1}{2}} T}{3 L^{\frac{1}{2}} K^{\frac{1}{2}}} \right)\right\} \Phi_0 
    \\
    & 
    + \frac{4 \sigma^2}{\mu M T} \log \left( \euler + \frac{\mu MT \Phi_0}{\sigma^2} \right)
    + \frac{8101 L^2 K^3 \sigma^2}{\mu^3 T^4} \log^4\left( \euler + \frac{\mu^3 T^4 \Phi_0}{L^2 K^3 \sigma^2} \right).
  \end{align}
\end{proof}

\subsubsection{Proof of \cref{fedacii:a1:stab:main}}
\label{sec:fedacii:a1:stab:main}
We first introduce the supporting propositions for \cref{fedacii:a1:stab:main}. We omit most of the proof details since the analysis is largely shared.

The following proposition is parallel to \cref{fedacii:stab:2}, where the difference is that the present proposition analyzes the \nth{2}-order stability instead of \nth{4}-order.
\begin{proposition}
  \label{fedacii:a1:stab:2}
  In the same setting of \cref{fedacii:a1:stab:main}, the following inequality holds (for all possible $t$) 
  
    \begin{equation}
        \expt \left[ \left\| \mathcal{X}(\gamma, \eta)^{-1} \begin{bmatrix}
          \Delta_{t+1}^{\mathrm{ag}}
          \\
          \Delta_{t+1}
        \end{bmatrix}  \right\|^2 \middle| \mathcal{F}_t \right]
      \leq
      2 \gamma^2 \sigma^2 +
      \left\| \mathcal{X}(\gamma, \eta)^{-1} \begin{bmatrix}
        \Delta_{t}^{\mathrm{ag}}
        \\
        \Delta_{t}
      \end{bmatrix}  \right\|^2
      \cdot
      \begin{cases}
        \left(1 + \frac{\gamma^2 \mu}{\eta} \right)^2 & \text{if~} \gamma \in \left(\eta, \sqrt{\frac{\eta}{\mu}}\right], \\
        1                                              & \text{if~} \gamma =  \eta,
      \end{cases}
    \end{equation}
    
    where $\mathcal{X}$ is the matrix-valued function defined in \cref{eq:fedacii:X:def}.
\end{proposition}
\begin{proof}[Proof of \cref{fedacii:a1:stab:2}]
  Apply the uniform norm bound \cref{fedacii:stab:bound}, and the rest of the analysis is the same as \cref{fedaci:stab:2}.
\end{proof}

The following proposition is parallel to \cref{fedacii:stab:3}, where the difference is that the present proposition uses $L$-(\nth{2}-order)-smoothness to bound the LHS quantity.
\begin{proposition} 
  \label{fedacii:a1:stab:3}
  In the same setting of \cref{fedacii:a1:stab:main}, the following inequality holds (for all possible $t$)
  \begin{equation}
    \left\|  \nabla F(\overline{w_t^{\mathrm{md}}}) -  \frac{1}{M} \sum_{m=1}^M \nabla F (w_t^{\mathrm{md}, m})  \right\|^2 
    \leq
    \frac{17 \eta^2 L^2}{9 \gamma^2} \left\| \mathcal{X}(\gamma, \eta)^{-1} \begin{bmatrix}
      \Delta_{t}^{\mathrm{ag}}
      \\
      \Delta_{t}
    \end{bmatrix}  \right\|^2,
  \end{equation}
  where $\mathcal{X}$ is the matrix-valued function defined in \cref{eq:fedacii:X:def}.
\end{proposition}
\begin{proof}[Proof of \cref{fedacii:a1:stab:3}]
  By $L$-smoothness (\cref{asm1}(b)), 
  \begin{equation}
    \left\|  \nabla F(\overline{w_t^{\mathrm{md}}}) -  \frac{1}{M} \sum_{m=1}^M \nabla F (w_t^{\mathrm{md}, m})  \right\|^2 
    \leq
    L^2 \|\Delta_t^{\mathrm{md}}\|^2.
  \end{equation}
  By definition of ``md'', sub-multiplicativity, and \cref{fedacii:stab:left},
  \begin{equation}
    \|\Delta_t^{\mathrm{md}}\|^2
    =
    \left\|
    \mathcal{X}(\gamma, \eta)^{\intercal}
    \begin{bmatrix} \frac{9 - 9 \gamma \mu + 2 \gamma^2 \mu^2}{9 - 6 \gamma \mu - \gamma^2 \mu^2}  I \\ \frac{3 \gamma \mu - 3\gamma^2 \mu^2}{9 - 6 \gamma \mu - \gamma^2 \mu^2} I \end{bmatrix}
     \right\|^2
     \left\| \mathcal{X}(\gamma, \eta)^{-1} \begin{bmatrix}
      \Delta_{t}^{\mathrm{ag}}
      \\
      \Delta_{t}
    \end{bmatrix}  \right\|^2
    \leq
    \frac{17 \eta^2}{9 \gamma^2} \left\| \mathcal{X}(\gamma, \eta)^{-1} \begin{bmatrix}
      \Delta_{t}^{\mathrm{ag}}
      \\
      \Delta_{t}
    \end{bmatrix}  \right\|^2.
  \end{equation}
\end{proof}

\cref{fedacii:a1:stab:main} then follows by telescoping \cref{fedacii:a1:stab:2} and plugging in \cref{fedacii:a1:stab:3}.
\begin{proof}[Proof of \cref{fedacii:a1:stab:main}]
  Let $t_0$ be the latest synchronized step prior to $t$, then telescoping \cref{fedacii:a1:stab:2} from $t_0$ to $t$  (note that $\Delta_{t_0} = \Delta_{t_0} = 0$)
  \begin{equation}
    \expt \left[ \left\| \mathcal{X}(\gamma, \eta)^{-1}
      \begin{bmatrix} \Delta_{t}^{\mathrm{ag}} \\ \Delta_{t} \end{bmatrix}
      \right\|^2 \middle| \mathcal{F}_{t_0} \right]
    \leq
    2 \gamma^2 \sigma^2 K \cdot \begin{cases}
      \left(1 + \frac{\gamma^2 \mu}{\eta}\right)^{2K} & \text{if~} \gamma \in \left(\eta, \sqrt{\frac{\eta}{\mu}} \right],
      \\
      1                              & \text{if~} \gamma =  \eta.
    \end{cases}
    \label{eq:fedaci:mod:tmp}
  \end{equation}
  Thus, by \cref{fedacii:a1:stab:3},
  \begin{equation}
    \expt \left[ \left\|  \nabla F(\overline{w_t^{\mathrm{md}}}) -  \frac{1}{M} \sum_{m=1}^M \nabla F (w_t^{\mathrm{md}, m})  \right\|^2  \middle| \right]
    \leq
    \frac{34}{9} \eta^2 \sigma^2 K \cdot \begin{cases}
      \left(1 + \frac{\gamma^2 \mu}{\eta}\right)^{2K} & \text{if~} \gamma \in \left(\eta, \sqrt{\frac{\eta}{\mu}} \right],
      \\
      1                              & \text{if~} \gamma =  \eta.
    \end{cases}
  \end{equation}
  The \cref{fedacii:a1:stab:main} then follows by bounding $\frac{34}{9}$ with $4$.
\end{proof}
\section{Analysis of \fedavg under \cref{asm2}}
\label{sec:fedavg}
In this section we study the convergence of \fedavg under \cref{asm2}. 
We provide a complete, non-asymptotic version of \cref{fedavg:a2} and provide the proof. 
We formally define \fedavg in \cref{alg:fedavg} for reference. 

Formally we use $\mathcal{F}_t$ to denote the $\sigma$-algebra generated by $\{w_\tau^m\}_{\tau \leq t, m \in [M]}$. Since \fedavg is Markovian, conditioning on $\mathcal{F}_t$ is equivalent to conditioning on $\{w_t^m\}_{m \in [M]}$.

\begin{algorithm}
  \caption{Federated Averaging (a.k.a. Local SGD, Parallel SGD)}
  \begin{algorithmic}[1]
      \label{alg:fedavg}
      \Procedure{\fedavg }{$\eta$} 
      \State Initialize $= w_0^{m} = w_0$ for all $m \in [M]$
      \For{$t = 0, \ldots, T-1$}
      \For{every worker $m\in [M]$ {\bf in parallel}}
      \State $g_t^m \gets \nabla f(w_t^{m} ; \xi_t^m)$ 
      \Comment{Query gradient at $w_t^{m}$}
      \State $v_{t+1}^{m} \gets w_t^{m} - \eta \cdot g_t^m$
      \Comment{Compute next iterate candidate $v_{t+1}^{m}$}
      \If{sync}
      \State $w_{t+1}^m \gets \frac{1}{M} \sum_{m=1}^M v_{t+1}^m$
      \Comment{Average and broadcast}
      \Else
      \State $w_{t+1}^m \gets  v_{t+1}^m$
      \Comment{Candidates assigned to be the next iterates }
      \EndIf
      \EndFor
      \EndFor
      \EndProcedure
  \end{algorithmic}
\end{algorithm}

\subsection{Main theorem and lemma: Complete version of \cref{fedavg:a2}}
\begin{theorem}
  \label{fedavg:a2:full}
  Let $F$ be $\mu > 0$-strongly convex, and assume \cref{asm2}, then for 
  \begin{equation}
    \eta := \min \left\{ \frac{1}{4L}, \frac{2}{\mu T} \log \left( \euler + \min \left\{ \frac{\mu^2 M T^2 D_0^2}{\sigma^2}, \frac{\mu^6 T^5 D_0^2}{Q^2 K^2 \sigma^4}\right\} \right)\right\},
  \end{equation}
  \fedavg yields
  \begin{align}
    &  \expt \left[ F \left( \sum_{t=0}^{T-1} \frac{\rho_t}{S_T} \overline{w_t} \right)  \right] - F^*
    + \frac{\mu}{2} \expt [\|\overline{w_T} - w^*\|^2]
    \\
    \leq &
    \exp \left( -\frac{\mu T}{8 L} \right) 4L D_0^2
    +
    \frac{3 \sigma^2}{ \mu MT } \log \left( \euler + \frac{\mu^2 M T^2 D_0^2}{\sigma^2}  \right)
    +
    \frac{3073 Q^2 K^2 \sigma^4}{\mu^5 T^4} \log^4 \left( \euler + \frac{\mu^6 T^5 D_0^2}{Q^2 K^2 \sigma^4}  \right).
  \end{align}
  where $\rho_t := (1 - \frac{1}{2} \eta \mu)^{T-t-1}$, $S_T := \sum_{t=0}^{T-1}\rho_t$, and $D_0 = \|\overline{w_0}-w^*\|$.
\end{theorem}
The proof of \cref{fedavg:a2:full} is based on the following two lemmas regarding the convergence and \nth{4}-order stability of \fedavg. The averaging technique applied here is similar to \citep{Stich-arXiv19}.
\begin{lemma}[Perturbed iterate analysis for \fedavg under \cref{asm2}]
  \label{fedavg:a2:conv:main}
  Let $F$ be $\mu > 0$-strongly convex, and assume \cref{asm2}, then for $\eta \in (0, \frac{1}{4L}]$, \fedavg satisfies
  \begin{align}
         & \expt \left[ F \left( \sum_{t=0}^{T-1} \frac{\rho_t}{S_T} \overline{w_t} \right)  \right] - F^*
    + \frac{\mu}{2} \expt [\|\overline{w_T} - w^*\|^2]
    \\
    \leq & \frac{1}{\eta} \exp \left( - \frac{1}{2} \eta \mu T \right) D_0^2
    +
    \frac{1}{M} \eta \sigma^2 + \frac{Q^2}{\mu} \left( \max_{0 \leq t < T}  \frac{1}{M} \sum_{m=1}^M \expt \left[\|\overline{w_t} - w_t^m\|^4 \right] \right).
  \end{align}
  where $\rho_t, S_T$ are defined in the statement of \cref{fedavg:a2:full}.
\end{lemma}
The proof of \cref{fedavg:a2:conv:main} is deferred to \cref{sec:fedavg:a2:conv}.
\begin{lemma}[\nth{4}-order discrepancy overhead bound for \fedavg]
  \label{fedavg:a2:stab:main}
  In the same settings of \cref{fedavg:a2:conv:main}, \fedavg satisfies (for any $t$)
  \begin{align}
    \expt \left[ \frac{1}{M} \sum_{m=1}^M \| \overline{w}_t - w_t^m \|^4 \right]
    \leq
    192 \eta^4 K^2 \sigma^4.
  \end{align}
\end{lemma}
The proof of \cref{fedavg:a2:stab:main} is deferred to \cref{sec:fedavg:a2:stab}.

Combining \cref{fedavg:a2:conv:main,fedavg:a2:stab:main} gives
\begin{lemma}[Convergence of \fedavg under \cref{asm2} for general $\eta$]
  \label{fedavg:a2:general:eta}
  In the same settings of \cref{fedavg:a2:conv:main}, 
  \fedavg yields
  \begin{equation}
      \expt \left[ F \left( \sum_{t=0}^{T-1} \frac{\rho_t}{S_T} \overline{w_t} \right)  \right] - F^*
 + \frac{\mu}{2} \expt [\|\overline{w_T} - w^*\|^2]
    \leq
 \frac{1}{\eta} \exp \left( - \frac{1}{2} \eta \mu T \right) D_0^2
 +
 \frac{1}{M} \eta \sigma^2 + \frac{192 \eta^4 Q^2 K^2 \sigma^4 }{\mu}.
 \label{eq:fedavg:a2:general:eta}
  \end{equation}
\end{lemma}
\begin{proof}[Proof of \cref{fedavg:a2:general:eta}]
  Immediate from \cref{fedavg:a2:conv:main,fedavg:a2:stab:main}.
\end{proof}
\cref{fedavg:a2:full} then follows by plugging an appropriate $\eta$ to \cref{fedavg:a2:general:eta}. 
\begin{proof}[Proof of \cref{fedavg:a2:full}]
  To simplify the notation, denote the terms on the RHS of \cref{eq:fedavg:a2:general:eta} as
  \begin{equation}
    \varphi_{\downarrow}(\eta) := \frac{1}{\eta} \exp \left( - \frac{1}{2} \eta \mu T \right) D_0^2,
    \qquad
    \varphi_{\uparrow}(\eta) := 
    \frac{1}{M} \eta \sigma^2 + \frac{192 \eta^4 Q^2 K^2 \sigma^4 }{\mu}.
  \end{equation}
  Let
  \begin{equation}
    \eta_0 := \frac{2}{\mu T} \log \left( \euler + \min \left\{ \frac{\mu^2 M T^2 D_0^2}{\sigma^2}, \frac{\mu^6 T^5 D_0^2}{Q^2 K^2 \sigma^4}\right\} \right),
    \qquad
    \text{then }
    \eta = \min \left\{ \frac{1}{4L}, \eta_0\right\}.
  \end{equation}
  Therefore $\varphi_{\downarrow}(\eta) \leq \varphi_{\downarrow}(\frac{1}{4L}) + \varphi_{\downarrow}(\eta_0)$, where
  \begin{equation}
    \varphi_{\downarrow} \left( \frac{1}{4L} \right) = \exp \left( -\frac{\mu T}{8 L} \right) 4L D_0^2,
    \label{eq:fedavg:a2:proof:1}
  \end{equation} 
  and
  \begin{equation}
    \varphi_{\downarrow}(\eta_0) \leq \frac{\mu T}{2} D_0^2 \cdot \left( \min \left\{ \frac{\mu^2 M T^2 D_0^2}{\sigma^2}, \frac{\mu^6 T^5 D_0^2}{Q^2 K^2 \sigma^4}\right\}  \right)^{-1}
    \leq
    \frac{\sigma^2}{2 \mu MT} + \frac{Q^2 K^2 \sigma^4}{2 \mu^5 T^4}.
    \label{eq:fedavg:a2:proof:2}
  \end{equation}
  On the other hand
  \begin{align}
    \varphi_{\uparrow}(\eta) \leq \varphi_{\uparrow}(\eta_0)
    \leq
    \frac{2 \sigma^2}{\mu MT} \log \left( \euler + \frac{\mu^2 M T^2 D_0^2}{\sigma^2}  \right)
    +
    \frac{3072 Q^2 K^2 \sigma^4}{\mu^5 T^4} \log^4 \left( \euler + \frac{\mu^6 T^5 D_0^2}{Q^2 K^2 \sigma^4}  \right).
    \label{eq:fedavg:a2:proof:3}
  \end{align}
  Combining \cref{fedavg:a2:general:eta,eq:fedavg:a2:proof:1,eq:fedavg:a2:proof:2,eq:fedavg:a2:proof:3} gives
  \begin{align}
    &  \expt \left[ F \left( \sum_{t=0}^{T-1} \frac{\rho_t}{S_T} \overline{w_t} \right)  \right] - F^*
    + \frac{\mu}{2} \expt [\|\overline{w_T} - w^*\|^2]
    \\
    \leq &
    \exp \left( -\frac{\mu T}{8 L} \right) 4L D_0^2
    +
    \frac{3 \sigma^2}{ \mu MT } \log \left( \euler + \frac{\mu^2 M T^2 D_0^2}{\sigma^2}  \right)
    +
    \frac{3073 Q^2 K^2 \sigma^4}{\mu^5 T^4} \log^4 \left( \euler + \frac{\mu^6 T^5 D_0^2}{Q^2 K^2 \sigma^4}  \right).
  \end{align}
\end{proof}

\subsection{Perturbed iterative analysis for \fedavg: Proof of \cref{fedavg:a2:conv:main}}
\label{sec:fedavg:a2:conv}
We first state and proof the following proposition on one-step analysis.
\begin{proposition}
  \label{fedavg:a2:conv:onestep}
  Under the same assumption of \cref{fedavg:a2:conv:main}, for all $t$, the following inequality holds
  \begin{equation}
    \expt \left[ \|\overline{w_{t+1}} - w^*\|^2  |\mathcal{F}_t \right]
    \leq
    \left( 1 - \frac{1}{2}\eta \mu  \right)  \|\overline{w_t} - w^* \|^2 -  \eta (F(\overline{w_t} ) - F^*)
    +  \frac{\eta Q^2}{\mu M}  \sum_{m=1}^M \|\overline{w_t} - w_t^m\|^4  + \frac{\eta^2\sigma^2}{M}.
  \end{equation}
\end{proposition}

\begin{proof}[Proof of \cref{fedavg:a2:conv:onestep}]
  By definition of the $\fedavg$ procedure (see \cref{alg:fedavg}), for all $m \in [M]$, $v_{t+1}^m = w_t^m - \eta \nabla f(w_t^m; \xi_t^m)$. Taking average over $m = 1,\ldots, M$ gives
  \begin{equation}
    \overline{w_{t+1}} - w^* = w_t - \eta \cdot \frac{1}{M} \sum_{m=1}^M \nabla f(w_t^m; \xi_t^m) - w^*.
  \end{equation}
  Taking conditional expectation, by bounded variance \cref{asm1}(c),
  \begin{equation}
    \expt \left[ \|\overline{w_{t+1}} - w^*\|^2  |\mathcal{F}_t \right]
    =
    \left\| w_t - \eta \cdot \frac{1}{M} \sum_{m=1}^M \nabla F(w_t^m) - w^*  \right\|^2 + \frac{1}{M} \eta^2 \sigma^2.
    \label{eq:fedavg:a2:conv:onestep:0}
  \end{equation}
  Now we analyze the $\left\| w_t - \eta \cdot \frac{1}{M} \sum_{m=1}^M \nabla F(w_t^m) - w^*  \right\|^2$ term as follows
  \begin{align}
    & \left\| w_t - \eta \cdot \frac{1}{M} \sum_{m=1}^M \nabla F(w_t^m) - w^*  \right\|^2 
    \\
    =    &
    \left\| \overline{w_t} - \eta \cdot \nabla F (\overline{w_t}) - w^* + \eta \left( \nabla F (\overline{w_t}) -   \frac{1}{M} \sum_{m=1}^M \nabla F(w_t^m)  \right) \right\|^2 
    \\
    \leq & \left( 1 + \frac{1}{2} \eta \mu  \right) \left\| \overline{w_t} - \eta \nabla F (\overline{w_t}) - w^* \right\|^2
    + \eta^2 \left( 1 + \frac{2}{\eta \mu} \right) \left\|  \nabla F (\overline{w_t}) -   \frac{1}{M} \sum_{m=1}^M \nabla F(w_t^m)   \right\|^2 
    \tag{apply \cref{helper:unbalanced:ineq} with $\zeta = \frac{1}{2}\eta \mu$}
    \\
    \leq & \left( 1 + \frac{1}{2} \eta \mu  \right) \left\| \overline{w_t} - \eta \nabla F (\overline{w_t}) - w^* \right\|^2
    + \eta^2 \left( 1 + \frac{2}{\eta \mu} \right) \frac{Q^2}{4M} \sum_{m=1}^M \|\overline{w_t} - w_t^m\|^4 
    \tag{by \cref{helper:3rd:Lip}}
    \\
    \leq & \left( 1 + \frac{1}{2} \eta \mu  \right) \left\| \overline{w_t} - \eta \nabla F (\overline{w_t}) - w^* \right\|^2
    + \frac{\eta Q^2}{\mu M} \sum_{m=1}^M \|\overline{w_t} - w_t^m\|^4.
    \label{eq:fedavg:a2:conv:onestep:1}
  \end{align}
  where the last inequality is due to $1 + \frac{2}{\eta \mu} \leq \frac{4}{\eta\mu}$ since $\eta \mu \leq \eta L \leq \frac{1}{4}$.

  The first term of the RHS of \cref{eq:fedavg:a2:conv:onestep:1} is bounded as
  \begin{align}
         & \left\| \overline{w_t} - \eta \nabla F (\overline{w_t}) - w^* \right\|^2
    \\
    =    & \left\|  \overline{w_t} - w^* \right\|^2 - 2 \eta \left\langle \nabla F (\overline{w_t}),   \overline{w_t} - w^* \right\rangle + \eta^2 \|\nabla F(\overline{w_t})\|^2
    \tag{expansion of squared norm}
    \\
    \leq & \left\|  \overline{w_t} - w^* \right\|^2 - \eta \left( \mu \|\overline{w_t} - w^* \|^2 - 2 (F(\overline{w_t} ) - F^*) \right) + \eta^2 \cdot (2 L (F(\overline{w_t} ) - F^*))
    \tag{$\mu$-strongly convexity and $L$-smoothness by \cref{asm1}}
    \\
    =    &
    (1 - \eta \mu )  \|\overline{w_t} - w^* \|^2 - 2\eta (1 - \eta L) (F(\overline{w_t} ) - F^*)
    \\
    \leq & (1 - \eta \mu )  \|\overline{w_t} - w^* \|^2 - \eta (F(\overline{w_t} ) - F^*).
    \tag{since $\eta \leq \frac{1}{2L}$}
  \end{align}
  Multiplying $(1 + \frac{1}{2} \eta \mu)$ on both sides gives (note that $(1 + \frac{1}{2} \eta \mu)(1 - \eta \mu) \leq (1 - \frac{1}{2} \eta \mu)$)
  \begin{align}
    & \left( 1 + \frac{1}{2} \eta \mu  \right) \left\| \overline{w_t} - \eta_t \nabla F (\overline{w_t}) - w^* \right\|^2
    \\
    \leq &
    \left( 1 + \frac{1}{2} \eta \mu  \right)\left( 1 - \eta \mu \right)  \|\overline{w_t} - w^* \|^2 -  \eta \left( 1 + \frac{1}{2} \eta \mu  \right) \left( F (\overline{w_t}) - F^* \right)
    \\
    \leq & \left( 1 - \frac{1}{2} \eta \mu \right) \|\overline{w_t} - w^* \|^2 -  \eta  \left( F (\overline{w_t}) - F^* \right).
    \label{eq:fedavg:a2:conv:onestep:2}
  \end{align}
  Combining \cref{eq:fedavg:a2:conv:onestep:0,eq:fedavg:a2:conv:onestep:1,eq:fedavg:a2:conv:onestep:2}
  completes the proof of \cref{fedavg:a2:conv:onestep}.
\end{proof}
With \cref{fedavg:a2:conv:onestep} at hand we are ready to prove \cref{fedavg:a2:conv:main}. The telescoping techniques applied here are similar to \citep{Stich-arXiv19}.
\begin{proof}[Proof of \cref{fedavg:a2:conv:main}]
  Telescoping \cref{fedavg:a2:conv:onestep} yields
  \begin{align}
         & \expt \left[\|\overline{w_T}-w^*\|^2 \right]
    + \eta \sum_{t=0}^{T-1} \left( 1 - \frac{1}{2} \eta \mu \right)^{T-t-1}\left( \expt[F (\overline{w_t})] - F^* \right)
    \\
    \leq &
    \left( 1 - \frac{1}{2} \eta \mu \right)^T \| \overline{w_0} - w^*\|^2
    +\sum_{t=0}^{T-1} \left(1 - \frac{1}{2} \eta \mu \right)^{T-t-1}
    \left(  \frac{1}{M} \eta^2 \sigma^2 + \frac{\eta Q^2}{\mu M} \sum_{m=1}^M \expt \left[ \|\overline{w_t} - w_t^m\|^4 \right]  \right)
    \\
    \leq &
    \left( 1 - \frac{1}{2} \eta \mu \right)^T \| \overline{w_0} - w^*\|^2
    + S_T
    \left(  \frac{1}{M} \eta^2 \sigma^2 + \frac{\eta Q^2}{\mu} \max_{0 \leq t < T} \frac{1}{M} \sum_{m=1}^M \expt \left[\|\overline{w_t} - w_t^m\|^4 \right] \right).
  \end{align}
  Multiplying $\frac{1}{\eta S_T}$ on both sides and rearranging,
  \begin{align}
         & \sum_{t=0}^{T-1} \frac{\rho_t}{S_T} (\expt[F(\overline{w_t})] - F^*)
    + \frac{1}{\eta S_T} \expt [\|\overline{w_T} - w^*\|^2]
    \\
    \leq &
    \frac{(1 - \frac{1}{2}\eta \mu)^T}{\eta S_T} \|\overline{w_0} - w^*\|^2
    +
    \frac{1}{M} \eta \sigma^2 + \frac{Q^2}{\mu} \left( \max_{0 \leq t < T}  \frac{1}{M} \sum_{m=1}^M \expt \left[\|\overline{w_t} - w_t^m\|^4 \right] \right).
    \label{eq:fedavg:a2:conv:main:1}
  \end{align}
  Note that $S_T := \sum_{t=0}^{T-1} \rho_t = \frac{1 - (1 - \frac{1}{2} \eta \mu)^T}{\frac{1}{2} \eta \mu}$, we have
  \begin{equation}
    \frac{1}{\eta S_T} = \frac{\mu}{2\left( 1 - (1 - \frac{1}{2} \eta \mu)^T \right)} \geq \frac{\mu}{2},
    \label{eq:fedavg:a2:conv:main:2}
  \end{equation}
  and
  \begin{equation}
    \frac{(1 - \frac{1}{2}\eta \mu)^T}{\eta S_T} = \frac{\mu (1 - \frac{1}{2} \eta \mu)^T}{2 \left( 1 - (1 - \frac{1}{2} \eta \mu)^{T} \right)}
    \leq \frac{\mu(1 - \frac{1}{2} \eta \mu)^T}{\eta \mu}
    \leq \frac{1}{\eta} \exp \left( - \frac{1}{2} \eta \mu T \right).
    \label{eq:fedavg:a2:conv:main:3}
  \end{equation}
  Also by convexity
  \begin{equation}
    \sum_{t=0}^{T-1} \frac{\rho_t}{S_T} (\expt[F(\overline{w_t})] - F^*) \geq
    \expt \left[ F \left( \sum_{t=0}^{T-1} \frac{\rho_t}{S_T} \overline{w_t} \right) \right] - F^*.
    \label{eq:fedavg:a2:conv:main:4}
  \end{equation}
  Plugging \cref{eq:fedavg:a2:conv:main:2,eq:fedavg:a2:conv:main:3,eq:fedavg:a2:conv:main:4} to \cref{eq:fedavg:a2:conv:main:1} gives
  \begin{align}
         & \expt \left[ F \left( \sum_{t=0}^{T-1} \frac{\rho_t}{S_T} \overline{w_t} \right)  \right] - F^*
    + \frac{\mu}{2} \expt [\|\overline{w_T} - w^*\|^2]
    \\
    \leq & \frac{1}{\eta} \exp \left( - \frac{1}{2} \eta \mu T \right)\|\overline{w_0} - w^*\|^2
    +
    \frac{1}{M} \eta \sigma^2 + \frac{Q^2}{\mu} \left( \max_{0 \leq t < T}  \frac{1}{M} \sum_{m=1}^M \expt \left[\|\overline{w_t} - w_t^m\|^4 \right] \right).
  \end{align}
\end{proof}

\subsection{Discrepancy overhead bound for \fedavg: Proof of \cref{fedavg:a2:stab:main}}
\label{sec:fedavg:a2:stab}
In this subsection we will prove \cref{fedavg:a2:stab:main} regarding the 4th order stability of \fedavg.
We introduce a few more notations to simplify the discussions. 
Let $m_1, m_2 \in [M]$ be two arbitrary distinct workers. For any timestep $t$, let $\Delta_t := w_t^{m_1} - w_t^{m_2}$, and $\Delta_t^\varepsilon := \varepsilon_t^{m_1} - \varepsilon_t^{m_2}$ where $\varepsilon_t^{m} = \nabla f(w_t^m; \xi_t^m) - \nabla F(w_t^m)$ be the bias of the gradient oracle of the $m$-th worker evaluated at $w_t$. Let $\Delta_t^\nabla := \nabla F(w_t^{m_1}) - \nabla F(w_t^{m_2})$.

We first state and prove the following proposition on one-step 4th-order stability. The proof is analogous to the \nth{4}-order convergence analysis of \fedavg in \citep{Dieuleveut.Patel-NeurIPS19}.
\begin{proposition}
  \label{fedavg:a2:stab:1}
  In the same setting of \cref{fedavg:a2:stab:main}, for all $t$,
  \begin{equation}
    \sqrt{\expt \|\Delta_{t+1}\|^4}
    \leq
    \sqrt{\expt \|\Delta_t\|^4} + \sqrt{192} \eta^2 \sigma^2.
  \end{equation}
\end{proposition}

\begin{proof}[Proof of \cref{fedavg:a2:stab:1}]
  If $t+1$ is a synchronized step, then the result follows trivially. We assume from now on that $t+1$ is not a synchronized step, then
  \begin{align}
         & \expt [\|\Delta_{t+1}\|^4 |\mathcal{F}_t] = \expt \left[ \|\Delta_t - \eta (\Delta_t^{\nabla} + \Delta_t^\varepsilon) \|^4 |\mathcal{F}_t \right]
    \\
    =    & \expt \left[ \left( \|\Delta_t\|^2 - 2 \eta \langle \Delta_t, \Delta_t^\nabla + \Delta_t^\varepsilon \rangle + \eta^2 \|\Delta_t^\nabla + \Delta_t^\varepsilon\|^2 \right)^2 \middle| \mathcal{F}_t \right]
    \\
    =    & \expt \|\Delta_t\|^4 - 4\eta \|\Delta_t\|^2 \langle \Delta_t, \Delta_t^\nabla \rangle
    + 4 \eta^2 \expt \left[ \langle \Delta_t, \Delta_t^\nabla + \Delta_t^\varepsilon \rangle^2 |\mathcal{F}_t \right]
    + 2 \eta^2 \|\Delta_t\|^2  \expt \left[ \| \Delta_t^\nabla + \Delta_t^\varepsilon \|^2 |\mathcal{F}_t \right]
    \\
         & \quad - 4 \eta^3 \expt \left[ \langle \Delta_t, \Delta_t^\nabla + \Delta_t^\varepsilon \rangle  \cdot \| \Delta_t^\nabla + \Delta_t^\varepsilon \|^2 |\mathcal{F}_t  \right]
    + \eta^4 \expt \left[ \| \Delta_t^\nabla + \Delta_t^\varepsilon \|^4 |\mathcal{F}_t \right]
    \\
    \leq & \expt \|\Delta_t\|^4 - 4\eta \|\Delta_t\|^2 \langle \Delta_t, \Delta_t^\nabla \rangle + 6 \eta^2 \|\Delta_t\|^2  \expt \left[ \| \Delta_t^\nabla + \Delta_t^\varepsilon \|^2 |\mathcal{F}_t \right]
    \\
         & \quad + 4 \eta^3 \|\Delta_t\| \expt  \left[ \| \Delta_t^\nabla + \Delta_t^\varepsilon \|^3 |\mathcal{F}_t \right]
    + \eta^4 \expt \left[ \| \Delta_t^\nabla + \Delta_t^\varepsilon \|^4 |\mathcal{F}_t \right] \tag{Cauchy-Schwarz inequality}
    \\
    \leq & \expt \|\Delta_t\|^4 - 4\eta \|\Delta_t\|^2 \langle \Delta_t, \Delta_t^\nabla \rangle
    + 8 \eta^2  \|\Delta_t\|^2  \expt \left[ \| \Delta_t^\nabla + \Delta_t^\varepsilon \|^2 |\mathcal{F}_t \right]
    + 3 \eta^4 \expt \left[ \| \Delta_t^\nabla + \Delta_t^\varepsilon \|^4 |\mathcal{F}_t \right],
    \label{eq:fedavg:a2:stab:prop:1}
  \end{align}
  where the last inequality is due to
  \begin{equation}
    4 \eta^3 \|\Delta_t\| \expt  \left[ \| \Delta_t^\nabla + \Delta_t^\varepsilon \|^3 |\mathcal{F}_t \right]
    \leq 2 \eta^2 \|\Delta_t\|^2  \expt  \left[ \| \Delta_t^\nabla + \Delta_t^\varepsilon \|^2 |\mathcal{F}_t \right] +
    2 \eta^4 \expt  \left[ \| \Delta_t^\nabla + \Delta_t^\varepsilon \|^4 |\mathcal{F}_t \right]
  \end{equation}
  by AM-GM inequality.

  Note that by $L$-smoothness and convexity, we have the following inequality by standard convex analysis (\cf, Theorem 2.1.5 of \citep{Nesterov-18}),
  \begin{align}
    \|\Delta_t^\nabla\|^2
    =
    \| \nabla F(w_t^{m_1}) - \nabla F(w_t^{m_2})\|^2
    \leq
    L \left\langle w_t^{m_1} - w_t^{m_2},  \nabla F(w_t^{m_1}) - \nabla F(w_t^{m_2})  \right\rangle
    =
    L \langle \Delta_t, \Delta_t^\nabla \rangle.
    \label{eq:fedavg:a2:stab:prop:2}
  \end{align}
  Consequently
  \begin{equation}
    \expt \left[ \| \Delta_t^\nabla + \Delta_t^\varepsilon \|^2 |\mathcal{F}_t \right] = \| \Delta_t^\nabla\|^2 +  \expt \left[ \|\Delta_t^\varepsilon\|^2 |\mathcal{F}_t \right]
    \leq \| \Delta_t^\nabla\|^2 + 2\sigma^2
    \leq L \langle \Delta_t, \Delta_t^\nabla \rangle + 2\sigma^2.
  \end{equation}
  Similarly
  \begin{align}
         & \expt \left[ \| \Delta_t^\nabla + \Delta_t^\varepsilon \|^4 |\mathcal{F}_t \right]
    \leq
    8 \|\Delta_t^\nabla\|^4 + 8 \expt \left[ \|\Delta_t^\varepsilon\|^4 |\mathcal{F}_t \right]
    \tag{AM-GM inequality}
    \\
    \leq &
    8 \|\Delta_t^\nabla\|^4  + 64 \sigma^4
    \tag{by \cref{helper:diff:4th}}
    \\
    \leq & 8L^2 \|\Delta_t^2\|^2 \|\Delta_t^\nabla\|^2 + 64 \sigma^4
    \tag{by $L$-smoothness}
    \\
    \leq & 8L^3 \|\Delta_t^2\|^2 \langle \Delta_t, \Delta_t^\nabla \rangle + 64 \sigma^4.
    \tag{by \cref{eq:fedavg:a2:stab:prop:2}}
  \end{align}
  Plugging the above two bounds to \cref{eq:fedavg:a2:stab:prop:1} gives
  \begin{align}
    \expt [\|\Delta_{t+1}\|^4 |\mathcal{F}_t]
    \leq
    \|\Delta_t\|^4 - 4\eta (1 - 2 \eta L - 6 \eta^3 L^3) \|\Delta_t\|^2 \langle \Delta_t, \Delta_t^\nabla \rangle + 16 \eta^2 \|\Delta_t\|^2 \sigma^2 + 192 \eta^4 \sigma^4.
    \label{eq:fedavg:a2:stab:prop:3}
  \end{align}
  Since $\eta L \leq \frac{1}{4}$ we have $(1 - 2 \eta L - 6 \eta^3 L^3) > 0$. By convexity $\langle \Delta_t, \Delta_t^\nabla \rangle \geq 0$. Hence the second term on the RHS of \cref{eq:fedavg:a2:stab:prop:3} is non-positive. We conclude that
  \begin{align}
    \expt [\|\Delta_{t+1}\|^4 |\mathcal{F}_t]
    \leq
    \|\Delta_t\|^4 + 16 \eta^2  \sigma^2 \|\Delta_t\|^2 + 192 \eta^4 \sigma^4.
  \end{align}
  Taking expectation gives
  \begin{align}
         & \expt [\|\Delta_{t+1}\|^4] \leq \expt [ \|\Delta_t\|^4 ] + 16 \eta^2 \sigma^2 \expt[\|\Delta_t\|^2] + 192 \eta^4 \sigma^4
    \\
    \leq & \expt [ \|\Delta_t\|^4 ] + 16 \eta^2 \sigma^2 \sqrt{\expt[\|\Delta_t\|^4]} + 192 \eta^4 \sigma^4 = \left( \sqrt{\expt \|\Delta_t\|^4} + \sqrt{192} \eta^2 \sigma^2 \right)^2.
  \end{align}
  Taking square root on both sides completes the proof.
\end{proof}
With \cref{fedavg:a2:stab:1} at hand we are ready to prove \cref{fedavg:a2:stab:main}.
\begin{proof}[Proof of \cref{fedavg:a2:stab:main}]
  Let $t_0$ be the latest synchronized prior to $t$, then telescoping \cref{fedavg:a2:stab:1} yields (note that $\Delta_{t_0} = 0$)
  \begin{equation}
    \sqrt{\expt \|\Delta_{t}\|^4} \leq \sqrt{192} \eta^2 \sigma^2 (t - t_0) \leq \sqrt{192} \eta^2 K \sigma^2,
  \end{equation}
  where the last inequality is because $K$ is the synchronization gap. Thus
  \begin{align}
    \frac{1}{M} \sum_{m=1}^M \expt \left[ \left\| \overline{w_t} - w_t^m \right\|^4  \right] \leq \expt [\|\Delta_t\|^4] \leq 192 \eta^4 K^2 \sigma^4,
  \end{align}
  where the first ``$\leq$'' is due to Jensen's inequality.
\end{proof}

\section{Analysis of \fedac for general convex objectives}
\label{sec:gcvx}
\subsection{Main theorems}
In this section we study the convergence of \fedac for general convex ($\mu = 0$) objectives. Let $F$ be a general convex function, the main idea is to apply \fedac to the $\ell_2$-augmented $\tilde{F}_{\lambda}(w)$ defined as
\begin{equation}
    \tilde{F}_{\lambda}(w) := F(w) + \frac{1}{2} \lambda \|w - w_0\|^2,
    \label{eq:aug}
\end{equation}
where $w_0$ is the initial guess. Let $w_{\lambda}^*$ be the optimum of $\tilde{F}_{\lambda}(w)$ and define $\tilde{F}_{\lambda}^* := \tilde{F}_{\lambda}(w_{\lambda}^*)$.

One can verify that if $F$ satisfies \cref{asm1} with general convexity ($\mu = 0$) and $L$-smoothness, then $\tilde{F}_{\lambda}$ satisfies \cref{asm1} with smoothness $L+\lambda$ and strong-convexity $\lambda$ (variance does not change). If $F$ satisfies \cref{asm2}, then $\tilde{F}_{\lambda}$ also satisfies \cref{asm2} with the same $Q$-\nth{3}-order-smoothness (\nth{4}-order central moment does not change).
 
Now we state the convergence theorems. Note that the bounds in \cref{tab:conv:rate} can be obtained by replacing $K = T/R$. Recall $\|D_0 := \|w_0 - w^*\|$.
\begin{theorem}[Convergence of \fedaci for general convex objective, under \cref{asm1}]
    \label{fedaci:a1:gcvx}
    Assume \cref{asm1} where $F$ is general convex. Then for any $T \geq 24$,\footnote{We assume this constant lower bound for technical simplification.} applying \fedaci to $\tilde{F}_{\lambda}$ \eqref{eq:aug} with 
    \begin{equation}
        \lambda
        =
        \max \left\{\frac{\sigma}{M^{\frac{1}{2}}T^{\frac{1}{2}} D_0},
        \frac{L^{\frac{1}{3}} K^{\frac{2}{3}} \sigma^{\frac{2}{3}}}{T D_0^{\frac{2}{3}}},  
        \frac{2LK}{T^2} \log^2 \left( \euler^2 + \frac{T^2}{K} \right) \right\},
    \end{equation}
    and hyperparameter
    \begin{equation}
        \eta = \min \left\{ \frac{1}{L + \lambda},
        \frac{K}{\lambda T^2} \log^2  \left( \euler + \min \left\{ \frac{\lambda L M T D_0^2}{\sigma^2}, \frac{\lambda^2 T^3 D_0^2}{K^2 \sigma^2} \right\}\right),
        \frac{L^{\frac{1}{3}} K^{\frac{1}{3}}  D_0^{\frac{2}{3}}}{\lambda^{\frac{2}{3}}T \sigma^{\frac{2}{3}}},
        \frac{L^{\frac{1}{4}} K^{\frac{1}{4}}  D_0^{\frac{1}{2}}}{\lambda^{\frac{3}{4}} T \sigma^{\frac{1}{2}}}
        \right\}
    \end{equation}
    yields
    \begin{align}
        \expt \left[F(\overline{w_T^{\mathrm{ag}}}) - F^* \right] 
       \leq &
       \frac{2L K D_0^2}{T^2} \log^{2} \left( \euler^2 + \frac{T^2}{K} \right)
       +
       \frac{2\sigma D_0}{M^{\frac{1}{2}} T^{\frac{1}{2}}}
       \log^2 \left( \euler^2 + \frac{L M^{\frac{1}{2}} T^{\frac{1}{2}} D_0}{\sigma} \right)
       \\
       & \qquad
       +
       \frac{1005 L^{\frac{1}{3}} K^{\frac{2}{3}} \sigma^{\frac{2}{3}} D_0^{\frac{4}{3}}}{T}
       \log^4 \left( \euler^4 +  \frac{L^{\frac{2}{3}} T D_0^{\frac{2}{3}} }{K^{\frac{2}{3}} \sigma^{\frac{2}{3}}} \right).
   \end{align}
\end{theorem}
The proof of \cref{fedaci:a1:gcvx} is deferred to \cref{sec:fedaci:a1:gcvx}.

\begin{theorem}[Convergence of \fedacii for general convex objective, under \cref{asm1}]
    \label{fedacii:a1:gcvx}
    Assume \cref{asm2} where $F$ is general convex. Then for any $T \geq 10^3$, applying \fedacii to $\tilde{F}_{\lambda}$ \eqref{eq:aug} with 
    \begin{equation}
        \lambda =
        \max \left\{ \frac{\sigma}{M^{\frac{1}{2}} T^{\frac{1}{2}} D_0}, 
        \frac{L^{\frac{1}{2}} K^{\frac{3}{4}} \sigma^{\frac{1}{2}} }{T D_0^{\frac{1}{2}} },
        \frac{18 LK}{T^2} \log^2 \left( \euler^2 + \frac{T^2}{K} \right)         \right\},
    \end{equation}
    and hyperparameter
    \begin{equation}
        \eta = 
        \min \left\{ 
        \frac{1}{L+\lambda}
            ,
        \frac{9K}{\lambda T^2} 
        \log^2 \left( \euler + 
        \min \left\{ \frac{\lambda LMT D_0^2}{\sigma^2}, \frac{\lambda^3 T^4 D_0^2}{LK^3 \sigma^2} \right\} \right),
        \frac{L^{\frac{1}{3}} D_0^{\frac{2}{3}} }{\lambda^{\frac{2}{3}} T^{\frac{2}{3}} \sigma^{\frac{2}{3}}}
        \right\}
    \end{equation}
    yields
    \begin{align}
        \expt \left[F(\overline{w_T^{\mathrm{ag}}}) - F^* \right] 
        \leq &
        \frac{10LKD_0^2}{T^2}  \log^2 \left( \euler^2 + \frac{T^2}{K} \right)
        +
        \frac{5 \sigma D_0}{M^{\frac{1}{2}} T^{\frac{1}{2}}} \log \left( \euler + \frac{L M^{\frac{1}{2}} T^{\frac{1}{2}} D_0}{\sigma} \right)
        \\
        & \qquad +
        \frac{16411 L^{\frac{1}{2}} K^{\frac{3}{4}} \sigma^{\frac{1}{2}}  D_0^{\frac{3}{2}}}{T}
        \log^4 \left( \euler^4 + 
        \frac{L^{\frac{1}{2}} T D_0^{\frac{1}{2}}}{K^{\frac{3}{4}} \sigma^{\frac{1}{2}}} \right).
    \end{align}
\end{theorem}
The proof of \cref{fedacii:a1:gcvx} is deferred to \cref{sec:fedacii:a1:gcvx}.

\begin{theorem}[Convergence of \fedacii for general convex objective, under \cref{asm2}]
    \label{fedacii:a2:gcvx}
    Assume \cref{asm2} where $F$ is general convex. Then for any $T \geq 10^3$, applying \fedacii to $\tilde{F}_{\lambda}$ \eqref{eq:aug} with 
    \begin{equation}
        \lambda =
        \max \left\{
        \frac{\sigma}{M^{\frac{1}{2}} T^{\frac{1}{2}} D_0},
        \frac{L^{\frac{1}{3}} K^{\frac{2}{3}} \sigma^{\frac{2}{3}} }{M^{\frac{1}{3}} T D_0^{\frac{2}{3}} },
        \frac{Q^{\frac{1}{3}} K \sigma^{\frac{2}{3}}}{ T^{\frac{4}{3}} D_0^{\frac{1}{3}}},
        \frac{18 LK}{T^2} \log^2 \left( \euler^2 + \frac{T^2}{K} \right)
        \right\},
    \end{equation}
    and hyperparameter
    \begin{equation}
        \eta = 
        \min \left\{ 
        \frac{1}{L+\lambda}
            ,
         \frac{9K}{\lambda T^2} 
            \log^2 \left( \euler + 
            \min \left\{ \frac{\lambda LMT D_0^2}{\sigma^2},
            \frac{\lambda^2 M T^3 D_0^2}{K^2 \sigma^2}, 
            \frac{\lambda^5 L T^8 D_0^2}{Q^2 K^6 \sigma^4} \right\} \right)
        ,
        \frac{ L^{\frac{1}{3}} K^{\frac{1}{3}} M^{\frac{1}{3}} D_0^{\frac{2}{3}}}{\lambda^{\frac{2}{3}}T \sigma^{\frac{2}{3}}}
        \right\}
    \end{equation}
    yields
    \begin{align}
        & \expt \left[F(\overline{w_T^{\mathrm{ag}}}) - F^* \right] 
        \leq 
        \frac{10LKD_0^2}{T^2}  \log^2 \left( \euler^2 + \frac{T^2}{K} \right)
        +
        \frac{5 \sigma D_0}{M^{\frac{1}{2}} T^{\frac{1}{2}}} \log \left( \euler + \frac{L M^{\frac{1}{2}} T^{\frac{1}{2}} D_0}{\sigma} \right)
        \\
        & +
        \frac{139 L^{\frac{1}{3}} K^{\frac{2}{3}}  \sigma^{\frac{2}{3}} D_0^{\frac{4}{3}}}{M^{\frac{1}{3}}T}
        \log^3 \left( \euler^3 + 
        \frac{L^{\frac{2}{3}} M^{\frac{1}{3}} T D_0^{\frac{2}{3}}}{  K^{\frac{2}{3}} \sigma^{\frac{2}{3}}} \right)
        +
        \frac{\euler^{19} Q^{\frac{1}{3}} K  \sigma^{\frac{2}{3}} D_0^{\frac{5}{3}} }{T^{\frac{4}{3}}}
        \log^8 \left( \euler^8 +  \frac{L T^{\frac{4}{3}} D_0^{\frac{1}{3}}}{Q^{\frac{1}{3}} K  \sigma^{\frac{2}{3}} }  \right).
    \end{align}
\end{theorem}
The proof of \cref{fedacii:a2:gcvx} is deferred to \cref{sec:fedacii:a2:gcvx}.

For comparison, we also establish the convergence of \fedavg for general convex objective under \cref{asm2}.
\begin{theorem}[Convergence of \fedavg for general convex objective, under \cref{asm2}]
    \label{fedavg:a2:gcvx}
    Assume \cref{asm2} where $F$ is general convex, then for any $T \geq 100$,  applying \fedavg to $\tilde{F}_{\lambda}$ \eqref{eq:aug} with
    \begin{equation}
        \lambda := 
        \max \left\{        \frac{\sigma}{M^{\frac{1}{2}} T^{\frac{1}{2}} D_0},
        \frac{Q^{\frac{1}{3}} K^{\frac{1}{3}} \sigma^{\frac{2}{3}}}{ T^{\frac{2}{3}} D_0^{\frac{1}{3}}},
        \frac{16L }{T} \log (\euler + T)
        \right\},
    \end{equation}
    and hyperparameter $\eta$
    \begin{equation}
        \eta := \min \left\{ \frac{1}{4(L+\lambda)}, \frac{2}{\lambda T} \log \left( \euler + \min \left\{ \frac{\lambda^2 M T^2 D_0^2}{\sigma^2}, \frac{\lambda^6 T^5 D_0^2}{Q^2 K^2 \sigma^4}\right\} \right)\right\}
    \end{equation}
   yields
    
    \begin{equation}
        \expt \left[ F \left( \sum_{t=0}^{T-1} \frac{\rho_t}{S_T} \overline{w_t} \right)   - F^* \right]
        \leq \frac{50 L D_0^2}{T} \log (\euler + T)
        +
        \frac{6 \sigma D_0}{ M^{\frac{1}{2}} T^{\frac{1}{2}}}
        \log 
        \left( \euler^2 +  T \right)
        +
        \frac{3076 Q^{\frac{1}{3}}  K^{\frac{1}{3}}  \sigma^{\frac{2}{3}} D_0^{\frac{5}{3}}} { T^{\frac{2}{3}}}
        \log^4 \left( \euler^5 + T \right)
    \end{equation}
    
    where $\rho_t := (1 - \frac{1}{2} \eta \lambda)^{T-t-1}$, $S_T := \sum_{t=0}^{T-1}\rho_t$.
\end{theorem}
The proof of \cref{fedavg:a2:gcvx} is deferred to \cref{sec:fedavg:a2:gcvx}.

\subsection{Proof of \cref{fedaci:a1:gcvx} on \fedaci for general-convex objectives under \cref{asm1}}
\label{sec:fedaci:a1:gcvx}
We first introduce the supporting lemmas for \cref{fedaci:a1:gcvx}.
\begin{lemma}
    \label{fedaci:gcvx:1}
    Assume \cref{asm1} where $F$ is general convex, then for any $\lambda > 0$, for any $\eta \leq \frac{1}{L+\lambda}$, applying \fedaci to $\tilde{F}_{\lambda}$ gives
     \begin{align}
        \expt \left[F(\overline{w_T^{\mathrm{ag}}}) - F^* \right]
        \leq &  
        \frac{1}{2}\lambda  D_0^2
        +
        \frac{1}{2} LD_0^2 \exp \left(  -  \sqrt{\frac{\eta \lambda}{K}}T \right) 
        + 
        \frac{\eta^{\frac{1}{2}} \sigma^2}{2 \lambda^{\frac{1}{2}} M K^{\frac{1}{2}}}
        + 
        \frac{\eta \sigma^2}{2M}
        \\
        &  
        + 
        \frac{390 \eta^{\frac{3}{2}} L K^{\frac{1}{2}} \sigma^2} {\lambda^{\frac{1}{2}}}
        +
        7 \eta^2 L K \sigma^2
        + 
        390 \eta^{\frac{3}{2}} \lambda^{\frac{1}{2}} K^{\frac{1}{2}} \sigma^2
        + 
        7 \eta^2 \lambda K \sigma^2.
        \label{eq:fedaci:gcvx:1}
      \end{align}
\end{lemma}
The proof of \cref{fedaci:gcvx:1} is deferred to \cref{sec:fedaci:gcvx:1}. 
Now we plug in $\eta$.
\begin{lemma}
    \label{fedaci:gcvx:2}
    Assume \cref{asm1} where $F$ is general convex, then for any $\lambda > 0$, for
    \begin{equation}
        \eta = \min \left\{ \frac{1}{L + \lambda},
        \frac{K}{\lambda T^2}
        \log^2  \left( \euler + \min \left\{ \frac{\lambda L M T D_0^2}{\sigma^2}, \frac{\lambda^2 T^3 D_0^2}{K^2 \sigma^2} \right\}\right),
        \frac{L^{\frac{1}{3}} K^{\frac{1}{3}}  D_0^{\frac{2}{3}}}{\lambda^{\frac{2}{3}}T \sigma^{\frac{2}{3}}},
        \frac{L^{\frac{1}{4}} K^{\frac{1}{4}}  D_0^{\frac{1}{2}}}{\lambda^{\frac{3}{4}} T \sigma^{\frac{1}{2}}}
        \right\},
    \end{equation}
    applying \fedaci to $\tilde{F}_{\lambda}$ gives
    \begin{align}
        \expt \left[F(\overline{w_T^{\mathrm{ag}}}) - F^* \right] 
        \leq 
        &
        \frac{1}{2} \lambda D_0^2
        +
        \frac{3\sigma^2}{2\lambda MT}  \log^2 \left( \euler^2 + \frac{\lambda LMT  D_0^2}{\sigma^2} \right)
        \\
        & 
        +
        \frac{592 LK^2 \sigma^2}{\lambda^2 T^3} \log^4 \left( \euler^4 + \frac{\lambda^2 T^3 D_0^2}{K^2 \sigma^2} \right)
        \\
        & +  
        \frac{412 L^{\frac{1}{2}} K \sigma D_0 }{\lambda^{\frac{1}{2}}  T^{\frac{3}{2}} }
        + 
        \frac{1}{2} L D_0^2 \exp \left(  -  \sqrt{\frac{1}{(1 + L/\lambda) K}}T \right).
        \label{eq:fedaci:gcvx:2}
    \end{align}
\end{lemma}

\begin{proof}[Proof of \cref{fedaci:gcvx:2}]
    To simplify the notation, we name the terms of RHS of \cref{eq:fedaci:gcvx:1} as
    \begin{alignat}{2}
        & \varphi_0 (\eta) := \frac{1}{2} LD_0^2 \exp \left(  -  \sqrt{\frac{\eta \lambda}{K}}T \right),
        &&
        \\
        &
        \varphi_1 (\eta) := \frac{\eta^{\frac{1}{2}} \sigma^2}{2 \lambda^{\frac{1}{2}} M K^{\frac{1}{2}}},
        \quad
        &&
        \varphi_2(\eta) :=\frac{\eta \sigma^2}{2M},
        \\
        & \varphi_3 (\eta) := 
        \frac{390 \eta^{\frac{3}{2}} L K^{\frac{1}{2}} \sigma^2} {\lambda^{\frac{1}{2}}},
        \quad
        &&
        \varphi_4 (\eta) :=
        7 \eta^2 L K \sigma^2,
        \\
        &
        \varphi_5 (\eta) :=
        390 \eta^{\frac{3}{2}} \lambda^{\frac{1}{2}} K^{\frac{1}{2}} \sigma^2,
        \quad
        &&
        \varphi_6 (\eta) :=
        7 \eta^2 \lambda K \sigma^2.
    \end{alignat}
Define
\begin{equation}
    \eta_1 := \frac{K}{ \lambda T^2}
    \log^2  \left( \euler^2 + \min \left\{ \frac{\lambda L M T D_0^2}{\sigma^2}, \frac{\lambda^2 T^3 D_0^2}{K^2 \sigma^2} \right\}\right)
    ,
    \quad
    \eta_2 := \frac{ L^{\frac{1}{3}} K^{\frac{1}{3}} D_0^{\frac{2}{3}}}{\lambda^{\frac{2}{3}}T \sigma^{\frac{2}{3}}},
    \quad
    \eta_3 := \frac{ L^{\frac{1}{4}} K^{\frac{1}{4}}  D_0^{\frac{1}{2}}}{\lambda^{\frac{3}{4}} T \sigma^{\frac{1}{2}}}.
\end{equation}
then $\eta = \min \left\{ \eta_1, \eta_2, \eta_3, \frac{1}{L + \lambda}\right\}$. 
Now we bound $\varphi_1(\eta), \ldots, \varphi_6(\eta)$ term by term.
\begin{align}
    \varphi_1(\eta) 
    & \leq
    \varphi_1(\eta_1)
    \leq
    \frac{\sigma^2}{2\lambda MT} \log \left( \euler + \frac{\lambda LMT  D_0^2}{\sigma^2} \right),
    \\
    \varphi_2(\eta) 
    & \leq
    \varphi_2(\eta_1) 
    \leq
    \frac{K \sigma^2}{2 \lambda M  T^2} \log^2 \left( \euler + \frac{\lambda LMT  D_0^2}{\sigma^2}  \right)
    \leq
    \frac{\sigma^2}{2\lambda MT} \log^2 \left( \euler + \frac{\lambda LMT  D_0^2}{\sigma^2} \right),
    \tag{since $K \leq T$}
    \\
    \varphi_3(\eta) 
    & \leq
    \varphi_3(\eta_1)
    \leq
    \frac{390LK^2 \sigma^2}{\lambda^2 T^3} \log^3 \left( \euler + \frac{\lambda^2 T^3 D_0^2}{K^2 \sigma^2} \right),
    \\
    \varphi_4(\eta) 
    & \leq
    \varphi_4(\eta_1)
    \leq
    \frac{7 LK^3 \sigma^2}{\lambda^2 T^4}
    \log^4 \left( \euler + \frac{\lambda^2 T^3 D_0^2}{K^2 \sigma^2} \right)
    \leq
    \frac{7LK^2 \sigma^2}{\lambda^2 T^3} \log^4 \left( \euler + \frac{\lambda^2 T^3 D_0^2}{K^2 \sigma^2} \right),
    \tag{since $K \leq T$}
    \\
    \varphi_5(\eta) 
    & \leq
    \varphi_5(\eta_2) 
    =
    \frac{390 L^{\frac{1}{2}} K   D_0 \sigma}{\lambda^{\frac{1}{2}} T^{\frac{3}{2}}},
    \\
    \varphi_6(\eta) 
    & \leq
    \varphi_6(\eta_3)
    \leq
    7 \eta_3^2 \lambda K \sigma^2 = \frac{7 L^{\frac{1}{2}} K^{\frac{3}{2}}  D_0 \sigma}{\lambda^{\frac{1}{2}} T^2 }
    \leq
    \frac{7 L^{\frac{1}{2}} K   D_0 \sigma}{\lambda^{\frac{1}{2}} T^{\frac{3}{2}}}.
    \tag{since $K \leq T$}
\end{align}
In summary
\begin{equation}
    \sum_{i=1}^6 \varphi_i(\eta) \leq \frac{\sigma^2}{\lambda MT}\log^2 \left( \euler^2 + \frac{\lambda LMT  D_0^2}{\sigma^2} \right)
    +
    \frac{397LK^2 \sigma^2}{\lambda^2 T^3} \log^4 \left( \euler^4 + \frac{\lambda^2 T^3 D_0^2}{K^2 \sigma^2} \right)
    +
    \frac{397 L^{\frac{1}{2}} K   D_0 \sigma}{\lambda^{\frac{1}{2}} T^{\frac{3}{2}}}.
    \label{eq:fedaci:gcvx:2:1}
\end{equation}
On the other hand $\varphi_0(\eta) \leq \varphi_0(\eta_1) + \varphi_0(\eta_2) + \varphi_0(\eta_3) + \varphi_0(\frac{1}{L + \lambda})$, where
\begin{align}
    \varphi_0(\eta_1) 
    & =
    \frac{1}{2} L D_0^2  \left( \euler^2 + \min \left\{ \frac{\lambda L M T D_0^2}{\sigma^2}, \frac{\lambda^2 T^3 D_0^2}{K^2 \sigma^2} \right\}\right)^{-1}
     \leq  \frac{\sigma^2}{2\lambda M T} + \frac{195 L K^2 \sigma^2}{\lambda^2 T^3},
    \\
    \varphi_0(\eta_2) 
    & \leq 
    \frac{3!}{2}  L D_0^2 \left( \sqrt{\frac{\eta_2 \lambda}{K}}T \right)^{-3}
    =
     \frac{3 L K^{\frac{3}{2}} D_0^2}{\eta_2^{\frac{3}{2}} \lambda^{\frac{3}{2}} T^3}
     =
     \frac{3 L^{\frac{1}{2}} K D_0 \sigma }{\lambda^{\frac{1}{2}}    T^{\frac{3}{2}} },
    \\
    \varphi_0(\eta_3)
    & \leq
    \frac{4!}{2}  L D_0^2 \left( \sqrt{\frac{\eta_3 \lambda}{K}}T \right)^{-4}
    =
    \frac{12 L K^2 D_0^2}{\eta_3^{2} \lambda^{2} T^4}
    =
    \frac{12 L^{\frac{1}{2}} K^{\frac{3}{2}}  \sigma D_0}{\lambda^{\frac{1}{2}}  T^{2}}
    \leq
    \frac{12 L^{\frac{1}{2}} K D_0 \sigma }{\lambda^{\frac{1}{2}}    T^{\frac{3}{2}} }.
\end{align}
In summary
\begin{equation}
    \varphi_0(\eta)
    \leq
    \frac{1}{2} L D_0^2 \exp \left(  -  \sqrt{\frac{\lambda}{(L + \lambda) K}}T \right)
    +
    \frac{\sigma^2}{2 \lambda M T} 
    + \frac{195 L K^2  \sigma^2}{\lambda^2 T^3} 
    + 
    \frac{15 L^{\frac{1}{2}} K   D_0 \sigma}{\lambda^{\frac{1}{2}}  T^{\frac{3}{2}} }.
    \label{eq:fedaci:gcvx:2:2}
\end{equation}
Combining \cref{fedaci:gcvx:1,eq:fedaci:gcvx:2:1,eq:fedaci:gcvx:2:2} gives
\begin{align}
    & \expt \left[F(\overline{w_T^{\mathrm{ag}}}) - F^* \right] 
    \leq \sum_{i=0}^6 \varphi_i(\eta) + \frac{1}{2} \lambda D_0^2
    \\
    \leq &
    \frac{1}{2} \lambda D_0^2
    +
    \frac{3\sigma^2}{2\lambda MT}  \log^2 \left( \euler^2 + \frac{\lambda LMT  D_0^2}{\sigma^2} \right)
    +
    \frac{592 LK^2 \sigma^2}{\lambda^2 T^3} \log^4 \left( \euler^4 + \frac{\lambda^2 T^3 D_0^2}{K^2 \sigma^2} \right)
    \\
    & +  
    \frac{412 L^{\frac{1}{2}} K \sigma D_0 }{\lambda^{\frac{1}{2}}  T^{\frac{3}{2}} }
    + 
    \frac{1}{2} L D_0^2 \exp \left(  -  \sqrt{\frac{1}{(1 + L/\lambda) K}}T \right).
\end{align}
\end{proof}
The main \cref{fedaci:a1:gcvx} then follows by plugging in the appropriate $\eta$.
\begin{proof}[Proof of \cref{fedaci:a1:gcvx}]
To simplify the notation, we name the terms on the RHS of \cref{eq:fedaci:gcvx:2} as
\begin{alignat}{2}
    & 
    \psi_0(\lambda) := \frac{1}{2} \lambda D_0^2, 
    \quad
    &&
    \psi_1(\lambda) := \frac{3\sigma^2}{2\lambda MT}  \log^2 \left( \euler^2 + \frac{\lambda LMT  D_0^2}{\sigma^2} \right),
    \\
    &
    \psi_2(\lambda) :=  \frac{592 LK^2 \sigma^2}{\lambda^2 T^3} \log^4 \left( \euler^4 + \frac{\lambda^2 T^3 D_0^2}{K^2 \sigma^2} \right),
    \quad
    &&
    \psi_3(\lambda) := \frac{412 L^{\frac{1}{2}} K   D_0 \sigma}{\lambda^{\frac{1}{2}}  T^{\frac{3}{2}} },
    \\
    &
    \psi_4(\lambda) := \frac{1}{2} L D_0^2 \exp \left(  -  \sqrt{\frac{1}{(1 + L/\lambda) K}}T \right).
    &&
\end{alignat}
Let
\begin{equation}
    \lambda_1 := \frac{\sigma}{M^{\frac{1}{2}}T^{\frac{1}{2}} D_0},
    \quad
    \lambda_2 := \frac{L^{\frac{1}{3}} K^{\frac{2}{3}} \sigma^{\frac{2}{3}}}{T D_0^{\frac{2}{3}}},
    \quad
    \lambda_3 := \frac{2KL}{T^2} \log^2 \left( \euler^2 + \frac{T^2}{K} \right),
\end{equation}
then
\(
    \lambda := \max \left\{ \lambda_1, \lambda_2, \lambda_3 \right\}.
\)
By helper \cref{helper:inv:times:log}, $\psi_1$ and $\psi_2$ are monotonically decreasing w.r.t $\lambda$ for $\lambda > 0$. $\psi_3$ is trivially decreasing. Thus
\begin{align}
    \psi_1(\lambda)
    & \leq
    \psi_1(\lambda_1) 
    \leq
    \frac{3\sigma D_0}{2M^{\frac{1}{2}} T^{\frac{1}{2}}}
    \log^2 \left( \euler^2 + \frac{L M^{\frac{1}{2}} T^{\frac{1}{2}} D_0}{\sigma} \right),
    \label{eq:fedaci:a1:gcvx:proof:1}
    \\
    \psi_2(\lambda)
    & \leq 
    \psi_2(\lambda_2)
    \leq
    \frac{592 L^{\frac{1}{3}} K^{\frac{2}{3}} \sigma^{\frac{2}{3}} D_0^{\frac{4}{3}}}{T}
    \log^4 \left( \euler^4 +  \frac{L^{\frac{2}{3}} T D_0^{\frac{2}{3}} }{K^{\frac{2}{3}} \sigma^{\frac{2}{3}}} \right),
    \label{eq:fedaci:a1:gcvx:proof:2}
    \\
    \psi_3(\lambda)
    & \leq 
    \psi_3(\lambda_2)
    =
    \frac{412 L^{\frac{1}{3}} K^{\frac{2}{3}} \sigma^{\frac{2}{3}} D_0^{\frac{4}{3}}}{T}.
    \label{eq:fedaci:a1:gcvx:proof:3}
\end{align}
Now we analyze $\psi_4(\lambda_3)$. Note first that
\(
        \frac{\lambda_3}{L} = \frac{2K}{T^2} \log^{2} \left( \euler^2 + \frac{T^2}{K} \right).
\)
Since $T \geq 24$ we have $\frac{T^2}{K} \geq 24$. By helper \cref{helper:inv:times:log}, $x^{-1} \log^{2} (\euler^2 + x)$ is monotonically decreasing over $(0, +\infty)$, thus
\begin{equation}
    \frac{\lambda_3}{L} = \frac{2K}{T^2} \log^{2} \left( \euler^2 + \frac{T^2}{K} \right)
    \leq
    \frac{1}{12} \log^2 (\euler^2 + 24) < 1.
\end{equation}
Hence
\begin{equation}
    1 + \frac{L}{\lambda_3}
    \leq
    \frac{2L}{\lambda_3}
    =
    \frac{T^2}{K} \log^{-2}\left( \euler^2 + \frac{T^2}{K} \right).
\end{equation}
We conclude that
\begin{align}
    \psi_4(\lambda)
    & \leq
    \psi_4(\lambda_3) 
    =
    \frac{1}{2} L D_0^2 \exp \left(  -  \sqrt{\frac{1}{(1 + L/\lambda_3) K}}T \right)
    \leq
    \frac{1}{2} L D_0^2 \left( \euler^2 + \frac{T^2}{K} \right)^{-1}
    \leq
    \frac{LKD_0^2}{2T^2}.
    \label{eq:fedaci:a1:gcvx:proof:4}
\end{align}
Finally note that
\begin{align}
    \psi_0(\lambda) & \leq \frac{1}{2} \lambda_1 D_0^2 + \frac{1}{2} \lambda_2 D_0^2 + \frac{1}{2} \lambda_3 D_0^2 
    = \frac{\sigma D_0}{2 M^{\frac{1}{2}} T^{\frac{1}{2}} } + \frac{L^{\frac{1}{3}} K^{\frac{2}{3}} \sigma^{\frac{2}{3}} D_0^{\frac{4}{3}} }{2 T} + \frac{L K D_0^2}{T^2} \log^{2} \left( \euler^2 + \frac{T^2}{K} \right).
    \label{eq:fedaci:a1:gcvx:proof:5}
\end{align}
Combining \cref{fedaci:gcvx:2,eq:fedaci:a1:gcvx:proof:1,eq:fedaci:a1:gcvx:proof:2,eq:fedaci:a1:gcvx:proof:3,eq:fedaci:a1:gcvx:proof:4,eq:fedaci:a1:gcvx:proof:5} gives
\begin{align}
     & \expt \left[F(\overline{w_T^{\mathrm{ag}}}) - F^* \right] 
    \leq \sum_{i=0}^4 \psi_i(\lambda) 
    \\
    \leq & 
    \frac{2L K D_0^2}{T^2} \log^{2} \left( \euler^2 + \frac{T^2}{K} \right)
    +
    \frac{2\sigma D_0}{M^{\frac{1}{2}} T^{\frac{1}{2}}}
    \log^2 \left( \euler^2 + \frac{L M^{\frac{1}{2}} T^{\frac{1}{2}} D_0}{\sigma} \right)
    \\
    & \qquad
    +
    \frac{1005 L^{\frac{1}{3}} K^{\frac{2}{3}} \sigma^{\frac{2}{3}} D_0^{\frac{4}{3}}}{T}
    \log^4 \left( \euler^4 +  \frac{L^{\frac{2}{3}} T D_0^{\frac{2}{3}} }{K^{\frac{2}{3}} \sigma^{\frac{2}{3}}} \right).
\end{align}
\end{proof}

\subsubsection{Proof of \cref{fedaci:gcvx:1}}
\label{sec:fedaci:gcvx:1}
We first introduce a supporting proposition for \cref{fedaci:gcvx:1}.
\begin{proposition}
    \label{fedaci:gcvx:0}
    Assume $F$ is general convex and $L$-smooth, and let $\Psi_t$ be the decentralized potential \cref{eq:fedaci:potential} for $\tilde{F}_{\lambda}$, namely 
    \begin{equation}
        \Psi_t := \frac{1}{M} \sum_{m=1}^M \left( \tilde{F}_{\lambda}(w_t^{\mathrm{ag},m}) - \tilde{F}_{\lambda}^*  \right)  + \frac{1}{2} \lambda \|\overline{w_T} - w_{\lambda}^*\|^2.
    \end{equation}
    Then
    \begin{equation}
        \Psi_T \geq F(\overline{w_T^{\mathrm{ag}}}) - F^* - \frac{1}{2} \lambda D_0^2
        ,
        \qquad
        \Psi_0 \leq \frac{1}{2} L \|w_0 - w^*\|^2.
    \end{equation}
\end{proposition}

\begin{proof}[Proof of \cref{fedaci:gcvx:0}]
    Since $w_{\lambda}^*$ optimizes $\tilde{F}_{\lambda}(w)$ we have $\tilde{F}_{\lambda}(w_{\lambda}^*) \leq \tilde{F}_{\lambda}(w^*)$ (recall $w^*$ is defined as the optimum of the un-augmented objective $F$), and thus
    \begin{equation}
        \tilde{F}_{\lambda}^* = F(w_{\lambda}^*) + \frac{1}{2}\lambda \|w_{\lambda}^* - w_0\|^2
        \leq
        F(w^*) + \frac{1}{2} \lambda \|w^* - w_0\|^2.
        \label{eq:fedaci:gcvx:0:1}
    \end{equation}
    Consequently, $\Psi_T$ is lower bounded as
    \begin{align}
        \Psi_T & = \frac{1}{M} \sum_{m=1}^M \left( \tilde{F}_{\lambda}(w_T^{\mathrm{ag},m}) - \tilde{F}_{\lambda}^*  \right)
        + \frac{1}{2} \lambda \|\overline{w_T} - w_{\lambda}^*\|^2
        \geq
        \frac{1}{M} \sum_{m=1}^M \left( \tilde{F}_{\lambda}(w_T^{\mathrm{ag},m}) - \tilde{F}_{\lambda}^*  \right)
        \\
        & = \frac{1}{M} \sum_{m=1}^M \left[ \left( F(w_T^{\mathrm{ag},m}) + \frac{1}{2}\lambda  \|w_T^{\mathrm{ag},m} - w_0\|^2 \right)
        - 
        \tilde{F}_{\lambda}^*
         \right]
        \\
        & \geq \frac{1}{M} \sum_{m=1}^M \left[  F(w_T^{\mathrm{ag},m}) - F^*  + \frac{1}{2}\lambda  \left( \|w_T^{\mathrm{ag},m}  - w_0\|^2  - \|w^* - w_0\|^2 \right) \right]
        \tag{by \cref{eq:fedaci:gcvx:0:1}}
        \\
        & \geq \frac{1}{M} \sum_{m=1}^M \left( F(w_T^{\mathrm{ag},m}) - F^* \right) - \frac{1}{2}\lambda \|w^* - w_0\|^2
        \\
        & \geq F(\overline{w_T^{\mathrm{ag}}}) - F^* - \frac{1}{2} \lambda \|w^* - w_0\|^2
        \tag{by convexity}
        \\
        &   = F(\overline{w_T^{\mathrm{ag}}}) - F^* - \frac{1}{2} \lambda D_0^2.
    \end{align}
    The initial potential $\Psi_0$ is upper bounded as
    \begin{align}
        \Psi_0 & = \tilde{F}_{\lambda}(w_0) - \tilde{F}_{\lambda}^* 
        + \frac{1}{2} \lambda \|w_{\lambda}^* - w_0\|^2
        \\
        & =   F(w_0) - \left( F(w_{\lambda}^*) + \frac{1}{2}\lambda  \|w_{\lambda}^* - w_0\|^2 \right) 
        + \frac{1}{2} \lambda \|w_{\lambda}^* - w_0\|^2
        \tag{by definition of $\tilde{F}_{\lambda}$ \eqref{eq:aug}}
        \\
        & =  F(w_0) - F(w_{\lambda}^*) \leq F(w_0) - F^* \tag{by optimality $F(w_{\lambda}^*) \geq F^*$}
        \\
        & \leq \frac{1}{2} L \|w_0 - w^*\|^2 = \frac{1}{2} L D_0^2.
        \tag{by $L$-smoothness of $F$}
     \end{align}
\end{proof}

\cref{fedaci:gcvx:1} then follows by applying \cref{fedaci:general:eta} and \cref{fedaci:gcvx:0}.
\begin{proof}[Proof of \cref{fedaci:gcvx:1}]
By \cref{fedaci:general:eta} on the convergence of \fedaci, for any $\eta \in (0, \frac{1}{L + \lambda}$),
\begin{equation}
    \expt \left[ \Psi_T \right]
    \leq
     \exp \left(  -  \sqrt{\frac{\eta \lambda}{K}}T \right) \Psi_0 
    + \frac{\eta^{\frac{1}{2}} \sigma^2}{2 \lambda^{\frac{1}{2}} M K^{\frac{1}{2}} }
    + \frac{\eta \sigma^2}{2M}
    + \frac{390 \eta^{\frac{3}{2}} (L + \lambda) K^{\frac{1}{2}} \sigma^2} {\lambda^{\frac{1}{2}}}
    + 7 \eta^2 (L + \lambda) K \sigma^2.
\end{equation}
Applying \cref{fedaci:gcvx:0} gives
    \begin{align}
    \expt \left[F(\overline{w_T^{\mathrm{ag}}}) - F^* \right]
    \leq &  
    \frac{1}{2} LD_0^2 \exp \left(  -  \sqrt{\frac{\eta \lambda}{K}}T \right) 
    + 
    \frac{1}{2}\lambda  D_0^2
    + 
    \frac{\eta^{\frac{1}{2}} \sigma^2}{2 \lambda^{\frac{1}{2}} M K^{\frac{1}{2}}}
    + 
    \frac{\eta \sigma^2}{2M}
    \\
    &  
    + 
    \frac{390 \eta^{\frac{3}{2}} L K^{\frac{1}{2}} \sigma^2} {\lambda^{\frac{1}{2}}}
    +
    7 \eta^2 L K \sigma^2
    + 
    390 \eta^{\frac{3}{2}} \lambda^{\frac{1}{2}} K^{\frac{1}{2}} \sigma^2
    + 
    7 \eta^2 \lambda K \sigma^2.
  \end{align}
\end{proof}

\subsection{Proof of \cref{fedacii:a1:gcvx} on \fedacii for general-convex objectives under \cref{asm1}}
\label{sec:fedacii:a1:gcvx}
We omit some technical details since the proof is similar to \cref{fedaci:a1:gcvx}. 
We first introduce the supporting lemma for \cref{fedacii:a1:gcvx}.
\begin{lemma}
    \label{fedacii:a1:gcvx:1}
    Assume \cref{asm1} where $F$ is general convex, then for any $\lambda > 0$, for any $\eta \leq \frac{1}{L+\lambda}$, applying \fedacii to $\tilde{F}_{\lambda}$ gives
     \begin{equation}
        \expt \left[F(\overline{w_T^{\mathrm{ag}}}) - F^* \right]
        \leq 
        \frac{1}{2}\lambda  D_0^2
        +
        \frac{1}{2} LD_0^2 \exp \left(  -  \sqrt{\frac{\eta \lambda T^2}{9K}} \right) 
        + 
        \frac{\eta^{\frac{1}{2}} \sigma^2}{\lambda^{\frac{1}{2}} M K^{\frac{1}{2}}}
        +
        \frac{200 \eta^2 L^2 K \sigma^2}{\lambda}
        +
        200 \eta^2 \lambda K \sigma^2.
        \label{eq:fedacii:a1:gcvx:1}
      \end{equation}
\end{lemma}
The proof of \cref{fedacii:a1:gcvx:1} is deferred to \cref{sec:fedacii:a1:gcvx:1}.
\begin{lemma}
    \label{fedacii:a1:gcvx:2}
    Assume \cref{asm1} where $F$ is general convex, then for any $\lambda > 0$, for
    \begin{equation}
        \eta = 
        \min \left\{ 
        \frac{1}{L+\lambda}
            ,
        \frac{9K}{\lambda T^2} 
        \log^2 \left( \euler + 
        \min \left\{ \frac{\lambda LMT D_0^2}{\sigma^2}, \frac{\lambda^3 T^4 D_0^2}{LK^3 \sigma^2} \right\} \right),
        \frac{L^{\frac{1}{3}} D_0^{\frac{2}{3}} }{\lambda^{\frac{2}{3}} T^{\frac{2}{3}} \sigma^{\frac{2}{3}}},
        \right\}
    \end{equation}
    applying \fedacii to $\tilde{F}_{\lambda}$ gives
    \begin{align}
        \expt \left[F(\overline{w_T^{\mathrm{ag}}}) - F^* \right]
        \leq &
        \frac{1}{2}\lambda  D_0^2
        +
        \frac{1}{2} LD_0^2 \exp \left(  -  \sqrt{\frac{T^2}{9(1 + L/\lambda)K}} \right) 
        +
        \frac{209 L^{\frac{2}{3}} K D_0^{\frac{4}{3}} \sigma^{\frac{2}{3}}}{\lambda^{\frac{1}{3}} T^{\frac{4}{3}}}
        \\
        & + 
        \frac{4 \sigma^2}{\lambda M T} \log \left( \euler + 
         \frac{\lambda LMT D_0^2}{\sigma^2} \right)
        + \frac{16201 L^2 K^3 \sigma^2}{\lambda^3 T^4}
        \log^4 \left( \euler^4 + 
       \frac{\lambda^3 T^4 D_0^2}{LK^3 \sigma^2} \right).
       \label{eq:fedacii:a1:gcvx:2}
    \end{align}
\end{lemma}
\begin{proof}[Proof of \cref{fedacii:a1:gcvx:2}]
    To simplify the notation, define the terms on the RHS of \cref{eq:fedacii:a1:gcvx:1} as
    \begin{alignat}{2}
        & \varphi_0(\eta) := 
        \frac{1}{2} LD_0^2 \exp \left(  -  \sqrt{\frac{\eta \lambda T^2}{9K}} \right),
        \quad
        &&
        \varphi_1(\eta) :=
        \frac{\eta^{\frac{1}{2}} \sigma^2}{\lambda^{\frac{1}{2}} M K^{\frac{1}{2}}},
        \\
        &
        \varphi_2(\eta) :=
        \frac{200 \eta^2 L^2 K \sigma^2}{\lambda},
        \quad
        &&
        \varphi_3(\eta) :=
        200 \eta^2 \lambda K \sigma^2.
    \end{alignat}
    Define
    \begin{equation}
        \eta_1 :=  \frac{9K}{\lambda T^2} 
        \log^2 \left( \euler + 
        \min \left\{ \frac{\lambda LMT D_0^2}{\sigma^2}, \frac{\lambda^3 T^4 D_0^2}{LK^3 \sigma^2} \right\} \right),
        \qquad
        \eta_2 := 
        \frac{L^{\frac{1}{3}} D_0^{\frac{2}{3}} }{\lambda^{\frac{2}{3}} T^{\frac{2}{3}} \sigma^{\frac{2}{3}}},
    \end{equation}
    Then $\eta = \min \left\{ \eta_1, \eta_2\right\}$.
    Since $\varphi_1, \varphi_2, \varphi_3$ are increasing we have
    \begin{align}
        \varphi_1(\eta) \leq \varphi_1 (\eta_1)
        & \leq
        \frac{3 \sigma^2}{\lambda M T} \log \left( \euler + 
         \frac{\lambda LMT D_0^2}{\sigma^2} \right),
        \\
        \varphi_2(\eta) \leq \varphi_2 (\eta_1)
        & \leq
        \frac{16200 L^2 K^3 \sigma^2}{\lambda^3 T^4}
        \log^4 \left( \euler + 
       \frac{\lambda^3 T^4 D_0^2}{LK^3 \sigma^2} \right),
       \\
       \varphi_3(\eta) \leq \varphi_3 (\eta_2)
       & \leq
       \frac{200 L^{\frac{2}{3}} K D_0^{\frac{4}{3}} \sigma^{\frac{2}{3}}}{\lambda^{\frac{1}{3}} T^{\frac{4}{3}}}.
    \end{align}
    On the other hand, since $\varphi_0$ is decreasing we have $\varphi_0(\eta) \leq \varphi_0(\eta_1) + \varphi_0(\eta_2) + \varphi_0(\frac{1}{L + \lambda})$, where
    \begin{align}
        \varphi_0(\eta_1) 
        & \leq
        \frac{\sigma^2}{2 \lambda M T} + \frac{L^2 K^3 \sigma^2}{2 \lambda^3 T^4},
        \\
        \varphi_0(\eta_2)
        & \leq
        \frac{2!}{2} LD_0^2 \left( \sqrt{\frac{\eta_2 \lambda T^2}{9K}}  \right)^{-2}
        =
        \frac{9KL D_0^2}{\eta_2 \lambda T^2}
        =
       \frac{9 L^{\frac{2}{3}} K D_0^{\frac{4}{3}} \sigma^{\frac{2}{3}}}{\lambda^{\frac{1}{3}} T^{\frac{4}{3}}}.
    \end{align}
    Combining the above bounds completes the proof.
\end{proof}
\cref{fedacii:a1:gcvx} then follows by plugging in an appropriate $\lambda$.
\begin{proof}[Proof of \cref{fedacii:a1:gcvx}]
    To simplify the notation, define the terms on the RHS of \cref{eq:fedacii:a1:gcvx:2} as
    \begin{alignat}{2}
        & \psi_0(\lambda):=
        \frac{1}{2}\lambda  D_0^2,
        \quad
        &&
        \psi_1(\lambda):=
        \frac{1}{2} LD_0^2 \exp \left(  -  \sqrt{\frac{T^2}{9(1 + L/\lambda)K}} \right),
        \\
        &
        \psi_2(\lambda):=
        \frac{209 L^{\frac{2}{3}} K D_0^{\frac{4}{3}} \sigma^{\frac{2}{3}}}{\lambda^{\frac{1}{3}} T^{\frac{4}{3}}},
        \quad
        && \psi_3(\lambda):=
        \frac{4 \sigma^2}{\lambda M T} \log \left( \euler + 
         \frac{\lambda LMT D_0^2}{\sigma^2} \right),
        \\
        &
        \psi_4(\lambda):=
       \frac{16201 L^2 K^3 \sigma^2}{\lambda^3 T^4}
        \log^4 \left( \euler^4 + 
       \frac{\lambda^3 T^4 D_0^2}{LK^3 \sigma^2} \right).
    \end{alignat}
    Define
    \begin{align}
        \lambda_1 := \frac{\sigma}{M^{\frac{1}{2}} T^{\frac{1}{2}} D_0},
        \quad
        \lambda_2 := \frac{L^{\frac{1}{2}} K^{\frac{3}{4}} \sigma^{\frac{1}{2}} }{D_0^{\frac{1}{2}} T},
        \quad
        \lambda_3 := \frac{18 LK}{T^2} \log^2 \left( \euler^2 + \frac{T^2}{K} \right).
    \end{align}
    Then $\lambda = \max \left\{ \lambda_1, \lambda_2, \lambda_3 \right\}$.
    By helper \cref{helper:inv:times:log} $\psi_3$, $\psi_4$ are decreasing; $\psi_2$ is trivially decreasing, thus
    \begin{align}
        \psi_2(\lambda)
        &
        \leq \psi_2(\lambda_2) = \frac{209 L^{\frac{1}{2}} K^{\frac{3}{4}} D_0^{\frac{3}{2}} \sigma^{\frac{1}{2}}}{T},
        \\
        \psi_3(\lambda)
        &
        \leq \psi_3(\lambda_1) 
        =
        \frac{4 \sigma D_0}{M^{\frac{1}{2}} T^{\frac{1}{2}}} \log \left( \euler + \frac{L M^{\frac{1}{2}} T^{\frac{1}{2}} D_0}{\sigma} \right),
        \\
        \psi_4 (\lambda)
        & \leq
        \psi_4(\lambda_2)
        =
        \frac{16201 L^{\frac{1}{2}} K^{\frac{3}{4}} D_0^{\frac{3}{2}} \sigma^{\frac{1}{2}}}{T}
        \log^4 \left( \euler^4 + 
        \frac{L^{\frac{1}{2}} T D_0^{\frac{1}{2}}}{K^{\frac{3}{4}} \sigma^{\frac{1}{2}}} \right).
    \end{align}
    For $\psi_1(\lambda)$ since $T \geq 1000$ we have $\frac{T^2}{K} \geq 1000$, thus
    \begin{equation}
        \frac{\lambda_3}{L} = \frac{18 K}{T^2} \log^2 \left( \euler^2 + \frac{T^2}{K} \right) 
        \leq
        \frac{18}{1000}\log^2 \left( \euler^2 + 1000 \right) < 1.
    \end{equation}
    Thus $1 + \frac{L}{\lambda_3} \leq \frac{2L}{\lambda_3}$, and therefore
    \begin{equation}
        \psi_1(\lambda) \leq \psi_1(\lambda_3) 
        =
        \frac{1}{2} L D_0^2 \left( \euler^2 + \frac{T^2}{K} \right)^{-1} \leq \frac{L K D_0^2}{2T^2}.
    \end{equation}
    Finally
    \begin{align}
        \psi_0(\lambda)  \leq \sum_{i=1}^3 \psi_0(\lambda_i)
        \leq \frac{\sigma D_0}{2 M^{\frac{1}{2}} T^{\frac{1}{2}}} 
        +
        \frac{L^{\frac{1}{2}} K^{\frac{3}{4}} D_0^{\frac{3}{2}} \sigma^{\frac{1}{2}}}{2T}
        +
        \frac{9LKD_0^2}{T^2}  \log^2 \left( \euler^2 + \frac{T^2}{K} \right).
    \end{align}
    Consequently,
    \begin{align}
        \sum_{i=0}^4 \psi(\lambda)
        \leq &
        \frac{10LKD_0^2}{T^2}  \log^2 \left( \euler^2 + \frac{T^2}{K} \right)
        +
        \frac{5 \sigma D_0}{M^{\frac{1}{2}} T^{\frac{1}{2}}} \log \left( \euler + \frac{L M^{\frac{1}{2}} T^{\frac{1}{2}} D_0}{\sigma} \right)
        \\
        & \qquad +
        \frac{16411 L^{\frac{1}{2}} K^{\frac{3}{4}} D_0^{\frac{3}{2}} \sigma^{\frac{1}{2}}}{T}
        \log^4 \left( \euler^4 + 
        \frac{L^{\frac{1}{2}} T D_0^{\frac{1}{2}}}{K^{\frac{3}{4}} \sigma^{\frac{1}{2}}} \right),
    \end{align}
    completing the proof.
\end{proof}

\subsubsection{Proof of \cref{fedacii:a1:gcvx:1}}
\label{sec:fedacii:a1:gcvx:1}
\cref{fedacii:a1:gcvx:1} is parallel to \cref{fedaci:gcvx:1} where the main difference is the following supporting proposition.
\begin{proposition}
    \label{fedacii:gcvx:0}
    Assume $F$ is general convex and $L$-smooth, and let $\Phi_t$ be the centralized potential \cref{eq:centralied:potential} for $\tilde{F}_{\lambda}$ (with strong convexity estimate $\mu = \lambda$), namely 
    \begin{equation}
        \Phi_t := \left( \tilde{F}_{\lambda}(\overline{w_t^{\mathrm{ag}}}) - \tilde{F}_{\lambda}^*  \right)  + \frac{1}{6} \lambda \|\overline{w_T} - w_{\lambda}^*\|^2.
    \end{equation}
    Then
    \begin{equation}
        \Phi_T \geq F(\overline{w_T^{\mathrm{ag}}}) - F^* - \frac{1}{2} \lambda D_0^2
        ,
        \qquad
        \Phi_0 \leq \frac{1}{2} L \|w_0 - w^*\|^2.
    \end{equation}
\end{proposition}
\begin{proof}[Proof of \cref{fedacii:gcvx:0}]
    The proof is almost identical to \cref{fedaci:gcvx:0}.
\end{proof}
\begin{proof}[Proof of \cref{fedacii:a1:gcvx:1}]
    Follows by applying \cref{fedacii:a1:general:eta} and plugging in the bound of \cref{fedacii:gcvx:0}. The rest of proof is the same as \cref{fedaci:gcvx:1} which we omit the details.
\end{proof}

\subsection{Proof of \cref{fedacii:a2:gcvx} on \fedacii for general-convex objectives under \cref{asm2}}
\label{sec:fedacii:a2:gcvx}
We omit some of the proof details since the proof is similar to \cref{fedaci:a1:gcvx}. 
We first introduce the supporting lemma for \cref{fedacii:a2:gcvx}.
\begin{lemma}
    \label{fedacii:a2:gcvx:1}
    Assume \cref{asm2} where $F$ is general convex, then for any $\lambda > 0$, for any $\eta \leq \frac{1}{L+\lambda}$, applying \fedacii to $\tilde{F}_{\lambda}$ gives
    \begin{align}
        \expt \left[F(\overline{w_T^{\mathrm{ag}}}) - F^* \right]
        & \leq 
        \frac{1}{2}\lambda  D_0^2
        +
        \frac{1}{2} LD_0^2 \exp \left(  -  \sqrt{\frac{\eta \lambda T^2}{9K}} \right) 
        \\
        & 
        + 
        \frac{\eta^{\frac{1}{2}} \sigma^2}{\lambda^{\frac{1}{2}} M K^{\frac{1}{2}}}
        +
        \frac{2 \eta^{\frac{3}{2}} L K^{\frac{1}{2}} \sigma^2}{\lambda^{\frac{1}{2}}M }
        +
        \frac{2 \eta^{\frac{3}{2}} \lambda^{\frac{1}{2}} K^{\frac{1}{2}} \sigma^2}{M}
        +
         \frac{\euler^9 \eta^4 Q^2 K^2 \sigma^4}{\lambda}.
        \label{eq:fedacii:a2:gcvx:1}
    \end{align}
\end{lemma}
\begin{proof}[Proof of \cref{fedacii:a2:gcvx:1}]
    Follows by \cref{fedacii:a2:general:eta,fedacii:gcvx:0}. The proof is similar to \cref{fedaci:gcvx:1} so we omit the details.
\end{proof}
\begin{lemma}
    \label{fedacii:a2:gcvx:2}
    Assume \cref{asm2} where $F$ is general convex, then for any $\lambda > 0$, for
    \begin{equation}
        \eta = 
        \min \left\{ 
        \frac{1}{L+\lambda}
            ,
         \frac{9K}{\lambda T^2} 
            \log^2 \left( \euler + 
            \min \left\{ \frac{\lambda LMT D_0^2}{\sigma^2},
            \frac{\lambda^2 M T^3 D_0^2}{K^2 \sigma^2}, 
            \frac{\lambda^5 L T^8 D_0^2}{Q^2 K^6 \sigma^4} \right\} \right)
        ,
        \frac{ L^{\frac{1}{3}} K^{\frac{1}{3}} M^{\frac{1}{3}} D_0^{\frac{2}{3}}}{\lambda^{\frac{2}{3}}T \sigma^{\frac{2}{3}}}
        \right\},
    \end{equation}
    applying \fedacii to $\tilde{F}_{\lambda}$ gives
    \begin{align}
        & \expt \left[F(\overline{w_T^{\mathrm{ag}}}) - F^* \right]
        \leq
        \frac{1}{2}\lambda  D_0^2
        +
        \frac{1}{2} LD_0^2 \exp \left(  -  \sqrt{\frac{T^2}{9(1 + L/\lambda)K}} \right) 
        +
        \frac{4 \sigma^2}{\lambda M T} \log \left( \euler + 
         \frac{\lambda LMT D_0^2}{\sigma^2} \right)
        \\
        & 
         + 
         \frac{55 L K^2 \sigma^2}{\lambda^2 M T^3}
         \log^3 \left( \euler^3 + 
         \frac{\lambda^2 M T^3 D_0^2}{K^2 \sigma^2} \right)
         +
         \frac{83 L^{\frac{1}{2}} KD_0 \sigma}{\lambda^{\frac{1}{2}} M^{\frac{1}{2}} T^{\frac{3}{2}}}
         +
         \frac{\euler^{18} Q^2 K^6 \sigma^4}{\lambda^5 T^8}  \log^8 \left( \euler^8 +  \frac{\lambda^5 L T^8 D_0^2}{Q^2 K^6 \sigma^4} \right).
         \label{eq:fedacii:a2:gcvx:2}
    \end{align}
\end{lemma}
\begin{proof}[Proof of \cref{fedacii:a2:gcvx:2}]
    To simplify the notation, define the terms on the RHS of \cref{eq:fedacii:a2:gcvx:1} as
    \begin{alignat}{3}
        & \varphi_0(\eta) := 
        \frac{1}{2} LD_0^2 \exp \left(  -  \sqrt{\frac{\eta \lambda T^2}{9K}} \right),
        \quad
        &&
        \varphi_1(\eta) :=
        \frac{\eta^{\frac{1}{2}} \sigma^2}{\lambda^{\frac{1}{2}} M K^{\frac{1}{2}}},
        \quad
        &&
        \varphi_2(\eta) :=
        \frac{2 \eta^{\frac{3}{2}} L K^{\frac{1}{2}} \sigma^2}{\lambda^{\frac{1}{2}}M },
        \\
        & 
        \varphi_3(\eta) := 
        \frac{2 \eta^{\frac{3}{2}} \lambda^{\frac{1}{2}} K^{\frac{1}{2}} \sigma^2}{M},
        \quad
        &&
        \varphi_4(\eta) :=
         \frac{\euler^9 \eta^4 Q^2 K^2 \sigma^4}{\lambda}.
        &&
    \end{alignat}
    Define
    \begin{equation}
        \eta_1 :=  \frac{9K}{\lambda T^2} 
        \log^2 \left( \euler + 
        \min \left\{ \frac{\lambda LMT D_0^2}{\sigma^2},
        \frac{\lambda^2 M T^3 D_0^2}{K^2 \sigma^2}, 
        \frac{\lambda^5 L T^8 D_0^2}{Q^2 K^6 \sigma^4} \right\} \right),
        \quad
        \eta_2 := 
        \frac{ L^{\frac{1}{3}} K^{\frac{1}{3}} M^{\frac{1}{3}} D_0^{\frac{2}{3}}}{\lambda^{\frac{2}{3}}T \sigma^{\frac{2}{3}}}.
    \end{equation}
    Then $\eta = \min \left\{ \eta_1, \eta_2\right\}$.
    Since $\varphi_1, \ldots, \varphi_4$ are increasing we have
    \begin{align}
        \varphi_1(\eta) \leq \varphi_1 (\eta_1)
        & \leq
        \frac{3 \sigma^2}{\lambda M T} \log \left( \euler + 
         \frac{\lambda LMT D_0^2}{\sigma^2} \right),
        \\
        \varphi_2(\eta) \leq \varphi_2 (\eta_1)
        & \leq
        \frac{54 L K^2 \sigma^2}{\lambda^2 M T^3}
        \log^3 \left( \euler + 
        \frac{\lambda^2 M T^3 D_0^2}{K^2 \sigma^2} \right),
       \\
       \varphi_3(\eta) \leq \varphi_3 (\eta_2)
       & =
       \frac{2 L^{\frac{1}{2}} KD_0 \sigma}{\lambda^{\frac{1}{2}} M^{\frac{1}{2}} T^{\frac{3}{2}}},
       \\
       \varphi_4(\eta) \leq \varphi_4 (\eta_1)
       & \leq
       \frac{9^4 \euler^9 Q^2 K^6 \sigma^4}{\lambda^5 T^8}  \log^8 \left( \euler +  \frac{\lambda^5 L T^8 D_0^2}{Q^2 K^6 \sigma^4} \right).
    \end{align}
    On the other hand $\varphi_0(\eta) \leq \varphi_0(\eta_1) + \varphi_0(\eta_2) + \varphi_0(\frac{1}{L + \lambda})$, where
    \begin{align}
        \varphi_0(\eta_1) 
        & \leq
        \frac{\sigma^2}{2 \lambda M T} + \frac{L K^2 \sigma^2}{2 \lambda^2 M T^3} + \frac{Q^2 K^6 \sigma^4}{2 \lambda^5 T^8},
        \\
        \varphi_0(\eta_2)
        & \leq
        \frac{3!}{2} LD_0^2 \left( \sqrt{\frac{\eta_2 \lambda T^2}{9K}}  \right)^{-3}
        =
        \frac{81 L K^{\frac{3}{2}} D_0^2}{\eta_2^{\frac{3}{2}} \lambda^{\frac{3}{2}} T^3}
        =
        \frac{81 L^{\frac{1}{2}} KD_0 \sigma}{\lambda^{\frac{1}{2}} M^{\frac{1}{2}} T^{\frac{3}{2}}}.
    \end{align}
    Combining the above bounds completes the proof.
\end{proof}

\cref{fedacii:a2:gcvx} then follows by plugging in an appropriate $\lambda$.
\begin{proof}[Proof of \cref{fedacii:a2:gcvx}]
    To simplify the notation, define the terms on the RHS of \cref{eq:fedacii:a2:gcvx:2} as
    \begin{alignat}{2}
        & \psi_0(\lambda):=
        \frac{1}{2}\lambda  D_0^2
        ,
        &&
        \psi_1(\lambda):=
        \frac{1}{2} LD_0^2 \exp \left(  -  \sqrt{\frac{T^2}{9(1 + L/\lambda)K}} \right),
        \\
        & \psi_2(\lambda):=
        \frac{4 \sigma^2}{\lambda M T} \log \left( \euler + 
         \frac{\lambda LMT D_0^2}{\sigma^2} \right),
        &&
        \psi_3(\lambda):=
        \frac{55 L K^2 \sigma^2}{\lambda^2 M T^3}
         \log^3 \left( \euler^3 + 
         \frac{\lambda^2 M T^3 D_0^2}{K^2 \sigma^2} \right),
        \\
        & \psi_4(\lambda):= \frac{83 L^{\frac{1}{2}} KD_0 \sigma}{\lambda^{\frac{1}{2}} M^{\frac{1}{2}} T^{\frac{3}{2}}},
        &&
        \psi_5(\lambda):=
        \frac{\euler^{18} Q^2 K^6 \sigma^4}{\lambda^5 T^8}  \log^8 \left( \euler^8 +  \frac{\lambda^5 L T^8 D_0^2}{Q^2 K^6 \sigma^4} \right).
    \end{alignat}
    Define
    \begin{align}
        \lambda_1 := \frac{\sigma}{M^{\frac{1}{2}} T^{\frac{1}{2}} D_0},
        \quad
        \lambda_2 := \frac{L^{\frac{1}{3}} K^{\frac{2}{3}} \sigma^{\frac{2}{3}} }{M^{\frac{1}{3}} T D_0^{\frac{2}{3}} },
        \quad
        \lambda_3 := \frac{Q^{\frac{1}{3}} K \sigma^{\frac{2}{3}}}{D_0^{\frac{1}{3}} T^{\frac{4}{3}}},
        \quad
        \lambda_4 := \frac{18 LK}{T^2} \log^2 \left( \euler^2 + \frac{T^2}{K} \right).
    \end{align}
    Then $\lambda = \max \left\{ \lambda_1, \lambda_2, \lambda_3 \right\}$.
    By \cref{helper:inv:times:log}, $\psi_2$, $\psi_3$, $\psi_5$ are increasing. $\psi_4$ is trivially decreasing, thus
    \begin{align}
        \psi_2(\lambda)
        &
        \leq \psi_2(\lambda_1) 
        =
        \frac{4 \sigma D_0}{M^{\frac{1}{2}} T^{\frac{1}{2}}} \log \left( \euler + \frac{L M^{\frac{1}{2}} T^{\frac{1}{2}} D_0}{\sigma} \right),
        \\
        \psi_3 (\lambda)
        & \leq
        \psi_3(\lambda_2)
        =
        \frac{55 L^{\frac{1}{3}} K^{\frac{2}{3}} D_0^{\frac{4}{3}} \sigma^{\frac{2}{3}}}{M^{\frac{1}{3}}T}
        \log^3 \left( \euler^3 + 
        \frac{L^{\frac{2}{3}} M^{\frac{1}{3}} T D_0^{\frac{2}{3}}}{  K^{\frac{2}{3}} \sigma^{\frac{2}{3}}} \right),
        \\
        \psi_4(\lambda) 
        & \leq
        \psi_4(\lambda_2)
        =
        \frac{83 L^{\frac{1}{3}} K^{\frac{2}{3}} D_0^{\frac{4}{3}} \sigma^{\frac{2}{3}}}{M^{\frac{1}{3}}T},
        \\
        \psi_5(\lambda)
        & \leq 
        \psi_5(\lambda_3)
        = 
        \frac{\euler^{18} Q^{\frac{1}{3}} K D_0^{\frac{5}{3}} \sigma^{\frac{2}{3}} }{T^{\frac{4}{3}}}
        \log^8 \left( \euler^8 +  \frac{L T^{\frac{4}{3}} D_0^{\frac{1}{3}}}{Q^{\frac{1}{3}} K  \sigma^{\frac{2}{3}} }  \right).
    \end{align}    
    For $\psi_1(\lambda)$ since $T \geq 1000$ we have $\frac{T^2}{K} \geq 1000$, thus
    \begin{equation}
        \frac{\lambda_3}{L} = \frac{18 K}{T^2} \log^2 \left( \euler^2 + \frac{T^2}{K} \right) 
        \leq
        \frac{18}{1000}\log^2 \left( \euler^2 + 1000 \right) < 1.
    \end{equation}
    Thus $1 + \frac{L}{\lambda_3} \leq \frac{2L}{\lambda_3}$, and therefore
    \begin{equation}
        \psi_1(\lambda) \leq \psi_1(\lambda_3) 
        =
        \frac{1}{2} L D_0^2 \left( \euler^2 + \frac{T^2}{K} \right)^{-1} \leq \frac{L K D_0^2}{2T^2}.
    \end{equation}
    Finally
    \begin{align}
        \psi_0(\lambda)  \leq \sum_{i=1}^4 \psi_0(\lambda_i)
        \leq \frac{\sigma D_0}{2 M^{\frac{1}{2}} T^{\frac{1}{2}}} 
        +
        \frac{L^{\frac{1}{3}} K^{\frac{2}{3}} D_0^{\frac{4}{3}} \sigma^{\frac{2}{3}}}{2 M^{\frac{1}{3}}T}
        +
        \frac{Q^{\frac{1}{3}} K D_0^{\frac{5}{3}} \sigma^{\frac{2}{3}} }{2T^{\frac{4}{3}}}
        +
        \frac{9LKD_0^2}{T^2}  \log^2 \left( \euler^2 + \frac{T^2}{K} \right).
    \end{align}
    Consequently,
    \begin{align}
        & \sum_{i=0}^4 \psi(\lambda)
        \leq
        \frac{10LKD_0^2}{T^2}  \log^2 \left( \euler^2 + \frac{T^2}{K} \right)
        +
        \frac{5 \sigma D_0}{M^{\frac{1}{2}} T^{\frac{1}{2}}} \log \left( \euler + \frac{L M^{\frac{1}{2}} T^{\frac{1}{2}} D_0}{\sigma} \right)
        \\
        & +
        \frac{139 L^{\frac{1}{3}} K^{\frac{2}{3}}  \sigma^{\frac{2}{3}} D_0^{\frac{4}{3}}}{M^{\frac{1}{3}}T}
        \log^3 \left( \euler^3 + 
        \frac{L^{\frac{2}{3}} M^{\frac{1}{3}} T D_0^{\frac{2}{3}}}{  K^{\frac{2}{3}} \sigma^{\frac{2}{3}}} \right)
        +
        \frac{\euler^{19} Q^{\frac{1}{3}} K  \sigma^{\frac{2}{3}} D_0^{\frac{5}{3}} }{T^{\frac{4}{3}}}
        \log^8 \left( \euler^8 +  \frac{L T^{\frac{4}{3}} D_0^{\frac{1}{3}}}{Q^{\frac{1}{3}} K  \sigma^{\frac{2}{3}} }  \right).
    \end{align}
\end{proof}

\subsection{Proof of \cref{fedavg:a2:gcvx} on \fedavg for general-convex objectives under \cref{asm2}}
\label{sec:fedavg:a2:gcvx}
We omit some of the proof details since the proof is similar to \cref{fedaci:a1:gcvx}. 
We first introduce the supporting lemma for \cref{fedavg:a2:gcvx}.
\begin{lemma}
    \label{fedavg:a2:gcvx:1}
    Assume \cref{asm2} where $F$ is general convex, then for any $\lambda > 0$, 
    for 
    \begin{equation}
        \eta := \min \left\{ \frac{1}{4(L+\lambda)}, \frac{2}{\lambda T} \log \left( \euler + \min \left\{ \frac{\lambda^2 M T^2 D_0^2}{\sigma^2}, \frac{\lambda^6 T^5 D_0^2}{Q^2 K^2 \sigma^4}\right\} \right)\right\},
    \end{equation}
    applying \fedavg to $\tilde{F}_{\lambda}$ gives
    \begin{align}
        & \expt \left[ F \left( \sum_{t=0}^{T-1} \frac{\rho_t}{S_T} \overline{w_t} \right)   - F^* \right]
        \leq 
        3 \lambda  D_0^2
        +
        2 L D_0^2 \exp \left(  -  \frac{\lambda T}{8(L + \lambda)} \right) 
        \\
        & \qquad
        +
        \frac{3 \sigma^2}{ \lambda MT } \log \left( \euler^2 + \frac{\lambda^2 M T^2 D_0^2}{\sigma^2}  \right)
        +
        \frac{3073 Q^2 K^2 \sigma^4}{\lambda^5 T^4} \log^4 \left( \euler^5 + \frac{\lambda^6 T^5 D_0^2}{Q^2 K^2 \sigma^4}  \right),
        \label{eq:fedavg:a2:gcvx:1}
    \end{align}
    where $\rho_t := (1 - \frac{1}{2} \eta \lambda)^{T-t-1}$, $S_T := \sum_{t=0}^{T-1}\rho_t$, and $D_0 = \|\overline{w_0}-w^*\|$.
\end{lemma}
\begin{proof}[Proof of \cref{fedavg:a2:gcvx:1}]
    Apply \cref{fedavg:a2:full}. The rest of analysis is similar to \cref{fedaci:gcvx:1,fedaci:gcvx:2}.
\end{proof}
\begin{proof}[Proof of  \cref{fedavg:a2:gcvx}]
    To simplify the notation, define the RHS of \cref{eq:fedavg:a2:gcvx:1} as
    \begin{alignat}{2}
        & \psi_0(\lambda) := 3 \lambda  D_0^2
        ,
        && \psi_1(\lambda) :=  2 L D_0^2 \exp \left(  -  \frac{T}{8(1 +(L/\lambda))} \right),
        \\
        & \psi_2(\lambda) :=  \frac{3 \sigma^2}{ \lambda MT } \log \left( \euler^2 + \frac{\lambda^2 M T^2 D_0^2}{\sigma^2}  \right)
        ,
        && \psi_3(\lambda) := \frac{3073 Q^2 K^2 \sigma^4}{\lambda^5 T^4} \log^4 \left( \euler^5 + \frac{\lambda^6 T^5 D_0^2}{Q^2 K^2 \sigma^4}  \right).
    \end{alignat}
    Define
    \begin{equation}
        \lambda_1 := \frac{\sigma}{M^{\frac{1}{2}} T^{\frac{1}{2}} D_0}, 
        \quad
        \lambda_2 := \frac{Q^{\frac{1}{3}} K^{\frac{1}{3}} \sigma^{\frac{2}{3}}}{ T^{\frac{2}{3}} D_0^{\frac{1}{3}}},
        \quad
        \lambda_3 := \frac{16L }{T} \log (\euler + T).
    \end{equation}
    Then $\lambda = \max \left\{ \lambda_1, \lambda_2, \lambda_3 \right\}$. We have (by helper \cref{helper:inv:times:log} $\psi_2, \psi_3$ are decreasing)
    \begin{align}
        \psi_2(\lambda) & \leq \psi_2 (\lambda_1)
        \leq
        \frac{3 \sigma D_0}{ M^{\frac{1}{2}} T^{\frac{1}{2}}}
        \log 
        \left( \euler^2 +  T \right),
        \\
        \psi_3(\lambda) & \leq \psi_3 (\lambda_2)
        \leq
        \frac{3073 Q^{\frac{1}{3}}  K^{\frac{1}{3}}  \sigma^{\frac{2}{3}} D_0^{\frac{5}{3}}} { T^{\frac{2}{3}}}
        \log^4 \left( \euler^5 + T \right).
    \end{align}
    Since $T \geq 100$ we have (by helper \cref{helper:inv:times:log}, $x^{-1} \log (\euler + x)$ is decreasing)
    \begin{equation}
        \frac{\lambda_3}{L} 
        =
        \frac{16}{T} \log (\euler + T)
        \leq \frac{16}{100} \log(\euler + 100) < 1,
    \end{equation}
    and thus
    \begin{equation}
        \psi_1(\lambda) \leq \psi_1(\lambda_3)
        \leq
        2LD_0^2 \exp \left( - \frac{T}{16 (L /\lambda_3)} \right)
        =
        2LD_0^2 (\euler +  T)^{-1}
        \leq
        \frac{2L D_0^2}{T}.
    \end{equation}
    Finally
    \begin{equation}
        \psi_0(\lambda) 
        \leq
        \sum_{i=1}^3 \psi_0(\lambda_i)
        =
        \frac{3 \sigma D_0}{M^{\frac{1}{2}} T^{\frac{1}{2}}}
        +
        \frac{3 Q^{\frac{1}{3}}  K^{\frac{1}{3}}  \sigma^{\frac{2}{3}} D_0^{\frac{5}{3}}} { T^{\frac{2}{3}}}
        +
        \frac{48 L D_0^2}{T} \log (\euler + T).
    \end{equation}
    Accordingly
    \begin{align}
        \sum_{i=0}^{3} \psi_i(\lambda)
        \leq
        \frac{50 L D_0^2}{T} \log (\euler + T)
        +
        \frac{6 \sigma D_0}{ M^{\frac{1}{2}} T^{\frac{1}{2}}}
        \log 
        \left( \euler^2 +  T \right)
        +
        \frac{3076 Q^{\frac{1}{3}}  K^{\frac{1}{3}}  \sigma^{\frac{2}{3}} D_0^{\frac{5}{3}}} { T^{\frac{2}{3}}}
        \log^4 \left( \euler^5 + T \right).
    \end{align}
\end{proof}
\section{Initial-value instability of standard accelerated gradient descent}
\label{sec:instability}
\subsection{Main theorem and lemmas}
In this section we show that standard accelerated gradient descent \citep{Nesterov-18} may not be initial-value stable even for strongly convex and smooth objectives in the sense that the initial infinitesimal difference may grow exponentially fast. 
This provides an evidence on the necessity of acceleration-stability tradeoff.

We formally define the standard deterministic AGD in \cref{algo:agd} for $L$-smooth and $\mu$-strongly-convex objective $F$  \citep{Nesterov-18}.
\begin{algorithm}
    \caption{Nesterov's Accelerated Gradient Descent Method (\agd)}
    \begin{algorithmic}[1]
        \label{algo:agd}
        \Procedure{\agd}{$w_0^{\mathrm{ag}}, w_0, L, \mu$} 
        \State $\kappa \gets L/\mu$
        \For{$t = 0, \ldots, T-1$}
        \State $w_t^{\mathrm{md}} \gets \frac{1}{\sqrt{\kappa} + 1} w_t + \frac{\sqrt{\kappa}}{\sqrt{\kappa} + 1} w_t^{\mathrm{ag}} $
        \State $w_{t+1}^{\mathrm{ag}} \gets w_t^{\mathrm{md}} - \frac{1}{L} \nabla F(w_t^{\mathrm{md}})$
        \State $w_{t+1} \gets \left( 1 - \frac{1}{\sqrt{\kappa}} \right)w_t + \frac{1}{\sqrt{\kappa}} w_t^{\mathrm{md}} - \sqrt{\frac{1}{L \mu}} \nabla F(w_t^{\mathrm{md}})$
        \EndFor
        \EndProcedure
    \end{algorithmic}
\end{algorithm}

Now we introduce the formal theorem on the initial-value instability.
\begin{theorem}[Initial-value instability of deterministic standard \agd, complete version of \cref{instability:sketch}]
    \label{instability:full}
    For any $L, \mu > 0$ such that $\nicefrac{L}{\mu} \geq 25$, and for any $K \geq 1$, there exists a 1D objective $F$ that is $L$-smooth and $\mu$-strongly-convex, and an $\varepsilon_0 > 0$, such that for any positive $\varepsilon < \varepsilon_0$, there exists $w_0, u_0, w_0^{\mathrm{ag}}, u_0^{\mathrm{ag}}$ such that $|w_0 - u_0| \leq \varepsilon$, $|w_0^{\mathrm{ag}} - u_0^{\mathrm{ag}}| \leq \varepsilon$, but
    the sequence $\{w_t^{\mathrm{ag}}, w_t^{\mathrm{md}}, w_t\}_{t=0}^{3K}$ output by $\agd(w_0^{\mathrm{ag}}, w_0, L, \mu)$ and sequence
    $\{u_t^{\mathrm{ag}}, u_t^{\mathrm{md}}, u_t\}_{t=0}^{3K}$ output by $\agd(u_0^{\mathrm{ag}}, u_0, L, \mu)$ satisfies
    \begin{equation}
        |w_{3K} - u_{3K}| \geq \frac{1}{2} \varepsilon (1.02)^K, 
        \qquad
        |w^{\mathrm{ag}}_{3K} - u^{\mathrm{ag}}_{3K}| \geq \varepsilon (1.02)^K.
    \end{equation}
\end{theorem} 

We first introduce the supporting lemmas for \cref{instability:sketch}. \cref{instability:1} shows the existence of an objective $F$ and a trajectory of \agd on $F$ such that $F''(w_t^{\mathrm{md}})=L$ (including also the neighborhood) once every three steps and $F''(w_t^{\mathrm{md}})=\mu$ otherwise. The proof of \cref{instability:1} is deferred to \cref{sec:instability:1}.
\begin{lemma}
    \label{instability:1}
    For any $L > \mu > 0$, and for any $K \geq 1$, there exists a 1D objective $F$ that is $L$-smooth and $\mu$-strongly convex, a neighborhood bound $\delta > 0$, and initial points $w_0$ and $w_0^{\mathrm{ag}}$ such that the sequence $\{w_t^{\mathrm{ag}}, w_t^{\mathrm{md}}, w_t\}_{t=0}^{3K-1}$ output by $\agd(w_0^{\mathrm{ag}}, w_0, L, \mu)$ satisfies for any $t = 0, \ldots, 3K-1$,
    \begin{align}
        \text{ if }  t \mathrm{~mod~} 3 \neq 1, \text{ then } F''(w) \equiv \mu \text{, for all } w \in [w_t^{\mathrm{md}}-\delta, w_t^{\mathrm{md}} + \delta],
        \\
        \text{ if }  t \mathrm{~mod~} 3 = 1, \text{ then } F''(w) \equiv L \text{, for all } w \in [w_t^{\mathrm{md}}-\delta, w_t^{\mathrm{md}} + \delta].
    \end{align}
\end{lemma}

The following \cref{instability:2} analyzes the growth of the difference of two instances of \agd. The proof is very similar to the analysis of \fedac.
\begin{lemma}
    \label{instability:2}
    Let $F$ be a $L$-smooth and $\mu>0$-strongly convex 1D function.
    Let $(w_{t+1}^{\mathrm{ag}}, w_{t+1})$, $(u_{t+1}^{\mathrm{ag}}, u_{t+1})$ be generated by applying one step of \agd on $F$ with hyperparameter $(L, \mu)$ from $(w_t^{\mathrm{ag}}, w_t)$ and $(u_t^{\mathrm{ag}}, u_t)$, respectively. Then there exists a $\zeta_t$ within the interval between $w_t^{\mathrm{md}}$ and $u_t^{\mathrm{md}}$, such that
    \begin{equation}
        \begin{bmatrix}
            w_{t+1}^{\mathrm{ag}} - u_{t+1}^{\mathrm{ag}}
            \\
            w_{t+1} - u_{t+1}
        \end{bmatrix}
        =
        \begin{bmatrix}
            \frac{\sqrt{\kappa}}{\sqrt{\kappa} + 1} \left(1- \frac{1}{L} F''(\zeta_t) \right) 
            &
            \frac{1}{\sqrt{\kappa} + 1} \left(1- \frac{1}{L} F''(\zeta_t) \right) 
            \\
            \frac{1}{\sqrt{\kappa} + 1} \left(1- \frac{1}{\mu} F''(\zeta_t) \right) 
            &
            \frac{\sqrt{\kappa}}{\sqrt{\kappa} + 1} \left(1- \frac{1}{L} F''(\zeta_t) \right) 
        \end{bmatrix}
        \begin{bmatrix}
            w_{t}^{\mathrm{ag}} - u_{t}^{\mathrm{ag}}
            \\
            w_{t} - u_{t}
        \end{bmatrix}.
    \end{equation}
\end{lemma}
\begin{proof}[Proof of \cref{instability:2}]
    This is a special case of \cref{fedac:general:stab} with no noise.
\end{proof}

With \cref{instability:1,instability:2} at hand we are ready to prove \cref{instability:full}. The proof follows by constructing an auxiliary trajectory for around the one given by \cref{instability:1}.
\begin{proof}[Proof of \cref{instability:full}]
    First apply  \cref{instability:1}. 
    Let $F$ be the objective, $(w_0^{\mathrm{ag}}, w_0)$ be the initial point and $\delta$ be the neighborhood bound given by  \cref{instability:1}.
    Since $\{w_t^{\mathrm{ag}}, w_t^{\mathrm{md}}, w_t\}_{t=0}^{3K-1}$ is a continuous function with respect to the initial point $(w_0^{\mathrm{ag}}, w_0)$, there exists a $\varepsilon_0$ such that for any $(v_0^{\mathrm{ag}}, v_0)$ such that $|v_0^{\mathrm{ag}} - w_0^{\mathrm{ag}}| \leq \varepsilon_0$ and $|v_0 - w_0| \leq \varepsilon_0$, trajectory $\{v_t^{\mathrm{ag}}, v_t^{\mathrm{md}}, v_t\}_{t=0}^{3K}$ output by \agd $(v_0^{\mathrm{ag}}, v_0, L, \mu)$ satisfies $\max_{0 \leq t < 3K} |v_t^{\mathrm{md}} - w_t^{\mathrm{md}}| \leq \delta$. 
    
    Thus, by \cref{instability:2}, for any $t = 0, \ldots, 3K -1$,
    \begin{alignat}{2}
        \begin{bmatrix}
            w_{t+1}^{\mathrm{ag}} - v_{t+1}^{\mathrm{ag}}
            \\
            w_{t+1} - v_{t+1}
        \end{bmatrix}
        & =
        \begin{bmatrix}
            1 - \frac{1}{\sqrt{\kappa}}
            &
            \frac{1}{\kappa}(\sqrt{\kappa} - 1) 
            \\
            0
            &
            1 - \frac{1}{\sqrt{\kappa}}
        \end{bmatrix}
        \begin{bmatrix}
            w_{t}^{\mathrm{ag}} - v_{t}^{\mathrm{ag}}
            \\
            w_{t} - v_{t}
        \end{bmatrix},
        \quad
        &&
        \text{if } t \mathrm{~mod~} 3 \neq 1;
        \\
        \begin{bmatrix}
            w_{t+1}^{\mathrm{ag}} - v_{t+1}^{\mathrm{ag}}
            \\
            w_{t+1} - v_{t+1}
        \end{bmatrix}
        & =
        \begin{bmatrix}
            0
            &
            0
            \\
            1 - \sqrt{\kappa}
            &
            0
        \end{bmatrix}
        \begin{bmatrix}
            w_{t}^{\mathrm{ag}} - v_{t}^{\mathrm{ag}}
            \\
            w_{t} - v_{t}
        \end{bmatrix},
        \quad
        &&
        \text{if } t \mathrm{~mod~} 3 = 1.
    \end{alignat}
    Hence for any $k = 0, \ldots, K-1$,
    \begin{align}
        \begin{bmatrix}
            w_{3(k+1)}^{\mathrm{ag}} - v_{3(k+1)}^{\mathrm{ag}}
            \\
            w_{3(k+1)} - v_{3(k+1)}
        \end{bmatrix}
        & =
        -
        \begin{bmatrix}
            \frac{1}{\kappa^{\frac{3}{2}}}(\sqrt{\kappa} - 1)^3 
            & 
            \frac{1}{\kappa^2} (\sqrt{\kappa} - 1)^3 
            \\
            \frac{1}{\kappa}  (\sqrt{\kappa} - 1)^3 
            &
            \frac{1}{\kappa^\frac{3}{2}} (\sqrt{\kappa} - 1)^3 
        \end{bmatrix}
        \begin{bmatrix}
            w_{3k}^{\mathrm{ag}} - v_{3k}^{\mathrm{ag}}
            \\
            w_{3k} - v_{3k}
        \end{bmatrix}
        \\
        & = 
        - 2 \left( 1 - \frac{1}{\sqrt{\kappa}} \right)^3
        \begin{bmatrix}
            \frac{1}{2} & \frac{1}{2\sqrt{\kappa}}
            \\
            \frac{1}{2}\sqrt{\kappa} & \frac{1}{2}
        \end{bmatrix}
        \begin{bmatrix}
            w_{3k}^{\mathrm{ag}} - v_{3k}^{\mathrm{ag}}
            \\
            w_{3k} - v_{3k}
        \end{bmatrix}.
    \end{align}
    Note that
    \(
        \begin{bmatrix}
            \frac{1}{2} & \frac{1}{2\sqrt{\kappa}}
            \\
            \frac{1}{2}\sqrt{\kappa} & \frac{1}{2}
        \end{bmatrix}
    \)
    is idempotent, \ie,
    \(
        \begin{bmatrix}
            \frac{1}{2} & \frac{1}{2\sqrt{\kappa}}
            \\
            \frac{1}{2}\sqrt{\kappa} & \frac{1}{2}
        \end{bmatrix}^{K}
        =
        \begin{bmatrix}
            \frac{1}{2} & \frac{1}{2\sqrt{\kappa}}
            \\
            \frac{1}{2}\sqrt{\kappa} & \frac{1}{2}
        \end{bmatrix}.
    \)
    Thus
    \begin{equation}
        \begin{bmatrix}
            w_{3K}^{\mathrm{ag}} - v_{3K}^{\mathrm{ag}}
            \\
            w_{3K} - v_{3K}
        \end{bmatrix}
        =
        \left( -2 \left( 1 - \frac{1}{\sqrt{\kappa}} \right)^3 \right)^K
        \begin{bmatrix}
            \frac{1}{2} & \frac{1}{2\sqrt{\kappa}}
            \\
            \frac{1}{2}\sqrt{\kappa} & \frac{1}{2}
        \end{bmatrix}
        \begin{bmatrix}
            w_{0}^{\mathrm{ag}} - v_{0}^{\mathrm{ag}}
            \\
            w_{0} - v_{0}
        \end{bmatrix}.
    \end{equation}
    Thus for any given $\varepsilon \leq \varepsilon_0$, put $u_0^{\mathrm{ag}} = w_0^{\mathrm{ag}} - \varepsilon$, and $u_0 = w_0 - \varepsilon$, we have
    \begin{align}
        \begin{bmatrix}
            w_{3K}^{\mathrm{ag}} - u_{3K}^{\mathrm{ag}}
            \\
            w_{3K} - u_{3K}
        \end{bmatrix}
        =
        \frac{1}{2} \varepsilon \left( -2 \left( 1 - \frac{1}{\sqrt{\kappa}} \right)^3 \right)^K
        \begin{bmatrix}
            1 + \frac{1}{\sqrt{\kappa}}
            \\
            \sqrt{\kappa} + 1
        \end{bmatrix}.
    \end{align}
    For $\kappa \geq 25$ we have $ \left|2 \left( 1 - \frac{1}{\sqrt{\kappa}} \right)^3 \right| > 1.02 $. Therefore
    \begin{equation}
        | w_{3K}^{\mathrm{ag}} - u_{3K}^{\mathrm{ag}}| 
        \geq
        \frac{1}{2} (1.02)^K \cdot \varepsilon,
        \quad
        | w_{3K} - u_{3K}| 
        \geq
        (1.02)^K \cdot \varepsilon,
    \end{equation}
    completing the proof.
\end{proof}
As a sanity check, the proof framework above for instability does not apply to the convergence of \agd. For instability, we only need to locally change the curvature to ``separate'' two instances. This trick does not break the convergence proof where the progress depends on the global curvature. We refer readers to \citet{Lessard.Recht.ea-SIOPT16} for the relative discussion.

\subsection{Proof of \cref{instability:1}}
\label{sec:instability:1}
In this section we prove \cref{instability:1} on the existence of objective $F$ and the trajectory with specific curvature at certain intervals. The high-level rationale is that \cref{instability:1} only specifies local curvatures of $F$, and therefore we can modify an objective at certain local points to make \cref{instability:1} satisfied.
Here we provide a constructive approach by incrementally updating $F$. 

We inductively prove the following claim.
    \begin{claim}
        \label{instability:induction}
        For any $k = 0, \ldots, K$, there exists a function $H_k$ valued in $[\mu, L]$, 
        a neighborhood bound $\delta_k > 0$,  and a pair of initial points $(w_0^{\mathrm{ag}}, w_0)$, such that for objective
        $F_k(w) := \int_0^{w} \int_0^{y} H_k(x) \diff x \diff y$,  the sequence output by \agd($w_0^{\mathrm{ag}}, w_0, L, \mu$) on $F_k$ satisfies $|w_{t_1}^{\mathrm{md}} - w_{t_2}^{\mathrm{md}}| \geq 2\delta_k$ if $t_1 \neq t_2$, and  for any $t = 0, \ldots, 3K - 1$,
        \begin{align}
            & \text{ if }  t \mathrm{~mod~} 3 \neq 1 \text{ or } t \geq 3k, \text{ then } F''(w) \equiv H_k(w) \equiv \mu \text{ for all } w \in [w_t^{\mathrm{md}}-\delta_k, w_t^{\mathrm{md}} + \delta_k];
            \label{eq:curvature:1}
            \\
            & \text{ if }  t \mathrm{~mod~} 3 = 1 \text{ and } t < 3k, \text{ then } F''(w) \equiv H_k(w) \equiv L \text{ for all } w \in [w_t^{\mathrm{md}}-\delta_k, w_t^{\mathrm{md}} + \delta_k].
            \label{eq:curvature:2}
        \end{align}
    \end{claim}
    To simplify the notation, we refer to \cref{eq:curvature:1,eq:curvature:2} as ``curvature conditions'' and denote $\mathcal{U}(x;r) := \{y: |y - x| < r\}$, and $\bar{\mathcal{U}}(x;r) := \{y: |y - x| \leq r\}$.
    \begin{proof}[Inductive proof of \cref{instability:induction}]
        For $k = 0$, we can put $H_0(w) \equiv \mu$ (then $F_k(w) = \frac{1}{2}\mu w^2$) and select any arbitrary initial points $(w_0^{\mathrm{ag}}, w_0)$ as long as $w_{t_1}^{\mathrm{md}} \neq w_{t_2}^{\mathrm{md}}$ for $t_1 \neq t_2$, which is trivially possible.
        
        Suppose \cref{instability:induction} holds for $k$,     
        now we construct $H_{k+1}$ and $\delta_{k+1}$.  
        Let $\{{w}_{t,k}^{\mathrm{ag}}, {w}_{t,k}^{\mathrm{md}}, {w}_{t,k}\}_{t=0}^{3K-1}$ be the trajectory output by \agd ($w_0^{\mathrm{ag}}, w_0, L, \mu$) on $F_k$.
        For some positive $\varepsilon_{k} < \frac{1}{2}\delta_k$ to be determined, consider
        \begin{equation}
            \tilde{H}_{k+1}(w) = H_{k}(w) + (L - \mu) \mathbf{1} \left[ w \in \bar{\mathcal{U}}(w_{3k+1, k}^{\mathrm{md}}; \varepsilon_k) \right],
            \quad
            \tilde{F}_{k+1}(w) = \int_0^w \int_0^y \tilde{H}_{k+1}(x) \diff x \diff y.
        \end{equation}
        Let $\{\tilde{w}_{t,k+1}^{\mathrm{ag}}, \tilde{w}_{t,k+1}^{\mathrm{md}}, \tilde{w}_{t,k+1}\}_{t=0}^{3K-1}$
        be  the trajectory output by \agd($w_0^{\mathrm{ag}}, w_0, L, \mu$) on $\tilde{F}_{k+1}$. 
        Since the trajectory is continuous with respect to $\varepsilon_k$, 
        there exists a $\bar{\varepsilon} < \frac{1}{2} \delta_k$ such that for any $\varepsilon_k < \bar{\varepsilon}$ (which we assume from now on),  it is the case that $|\tilde{w}_{t, k+1}^{\mathrm{md}} - w_{t,k}^{\mathrm{md}}| \leq \frac{1}{2}\delta_k$ for all $t \leq 3k+1$.
        Then let
        \begin{equation}
            {H}_{k+1}(w) = H_{k}(w) + (L - \mu) \mathbf{1} \left[ w \in  \bar{\mathcal{U}}(\tilde{w}_{3k+1, k+1}^{\mathrm{md}}; \varepsilon_k) \right],
            \quad
            {F}_{k+1}(w) = \int_0^w \int_0^y {H}_{k+1}(x) \diff x \diff y.
        \end{equation}   
        and let $\{{w}_{t,k+1}^{\mathrm{ag}}, {w}_{t,k+1}^{\mathrm{md}}, {w}_{t,k+1}\}_{t=0}^{3K-1}$
        be  the trajectory output by \agd($w_0^{\mathrm{ag}}, w_0, L, \mu$) on ${F}_{k+1}$. 

        Consequently,
        \begin{enumerate}
            \item[(a)] By construction of $H_{k+1}$ and $\tilde{H}_{k+1}$, we have $H_{k+1}(w) = \tilde{H}_{k+1}(w) = H_k(w)$ and $\nabla F_{k+1}(w) = \nabla \tilde{F}_{k+1}(w)$ for all $w \notin \bar{U}(w_{3k+1,k}^{\mathrm{md}}; \delta_k)$. 
            \item[(b)] Since $\tilde{w}_{t, k+1}^{\mathrm{md}} \notin \bar{U}(w_{3k+1,k}^{\mathrm{md}}; \delta_k)$, by (a), we can inductively show that $\tilde{w}_{t, k+1}^{\mathrm{md}} = w_{t, k+1}^{\mathrm{md}}$ for $t < 3k+1$, namely the trajectories for $F_{k+1}$ and $\tilde{F}_{k+1}$ are identical up to timestep $t < 3k+1$. 
            \item[(c)] Since $|\tilde{w}_{t, k+1}^{\mathrm{md}} - w_{t,k}^{\mathrm{md}}| \leq \frac{1}{2}\delta_k$, by (b), we further have $|w_{t, k+1}^{\mathrm{md}} - w_{t, k}^{\mathrm{md}}| \leq \frac{1}{2} \delta_k$ for $t < 3k + 1$. 
            Thus, by (a), the curvature conditions will be satisfied for $w_{t, k+1}^{\mathrm{md}}$ and $H_{k+1}$ up to $t < 3k+1$ and any neighborhood bound $\delta_{k+1} < \frac{1}{2} \delta_k$ since $H_{k+1} \equiv H_k$ for $w \notin \bar{U}(w_{3k+1,k}^{\mathrm{md}}; \delta_k)$. 
            \item[(d)] By (b), we have $w_{3k+1, k+1}^{\mathrm{md}} = \tilde{w}_{3k+1, k+1}^{\mathrm{md}}$ since all previous gradients evaluated are identical for $F_{k+1}$ and $\tilde{F}_{k+1}$. Thus, by construction of $H_{k+1}$ the curvature conditions hold for $w_{3k+1,k+1}^{\mathrm{md}}$ and $H_{k+1}$.
            \item[(e)] Similarly, for sufficiently small $\varepsilon_k$, we have $|w_{t, k+1}^{\mathrm{md}} - w_{t,k}^{\mathrm{md}}| \leq \frac{1}{2} \delta_k$ for $t > 3k+1$, and the curvature conditions also hold for $t > 3k+1$.
        \end{enumerate}
        Summarizing (c), (d), and (e) completes the induction.
    \end{proof}
\begin{proof}[Proof of \cref{instability:1}]
    Follows by applying \cref{instability:induction}.
\end{proof}

\section{Helper Lemmas}
\label{sec:helper}
In this section we include some generic helper lemmas. Most of the results are standard and we provide the proof for completeness.
\begin{lemma}
    \label{helper:blocknorm}
    Let $A = \begin{bmatrix}
            A_{11} & A_{12}
            \\
            A_{21} & A_{22}
        \end{bmatrix}$ be an arbitrary $2d \times 2d$ block matrix, where $A_{11}, A_{12}, A_{21}, A_{22}$ are $d \times d$ matrix blocks. Then the operator norm of $A$ is bounded by
    \begin{equation}
        \|A\| \leq \max \left\{ \|A_{11}\|, \|A_{22}\| \right\} + \left\{ \|A_{12}\|, \|A_{21}\| \right\}.
    \end{equation}
\end{lemma}
\begin{proof}[Proof of \cref{helper:blocknorm}]
    Let $A_{ij} = U_{ij} \Sigma_{ij} V_{ij}^T$ be the SVD decomposition of matrix $A_{ij}$, for $i=1,2$, and $j=1,2$. Then
    \begin{equation}
        \begin{bmatrix}
            A_{11} & \\ & A_{22}
        \end{bmatrix}
        =
        \begin{bmatrix}
            U_{11} \Sigma_{11} V_{11}^\intercal & \\ & U_{22} \Sigma_{22} V_{22}^\intercal
        \end{bmatrix}
        =
        \begin{bmatrix}
            U_{11} & \\ & U_{22}
        \end{bmatrix}
        \begin{bmatrix}
            \Sigma_{11} & \\ & \Sigma_{22}
        \end{bmatrix}
        \begin{bmatrix}
            V_{11} & \\ & V_{22}
        \end{bmatrix}^\intercal,
    \end{equation}
    thus
    \begin{equation}
        \left\| \begin{bmatrix}
            A_{11} & \\ & A_{22}
        \end{bmatrix} \right\|
        =
        \left\| \begin{bmatrix}
            \Sigma_{11} & \\ & \Sigma_{22}
        \end{bmatrix} \right\|
        =
        \max\left\{ \|\Sigma_{11}\|, \|\Sigma_{22}\|\right\}
        =
        \max\left\{ \|A_{11}\|, \|A_{22}\|\right\}.
    \end{equation}
    Similarly
    \begin{equation}
        \begin{bmatrix}
             & A_{12} \\ A_{21} &
        \end{bmatrix}
        =
        \begin{bmatrix}
             & U_{12} \Sigma_{12} V_{12}^\intercal \\ U_{21} \Sigma_{21} V_{21}^\intercal &
        \end{bmatrix}
        =
        \begin{bmatrix}
             & U_{12} \\ U_{21} &
        \end{bmatrix}
        \begin{bmatrix}
            \Sigma_{21} & \\  & \Sigma_{12}
        \end{bmatrix}
        \begin{bmatrix}
            V_{21} & \\ & V_{12}
        \end{bmatrix}^\intercal,
    \end{equation}
    thus
    \begin{equation}
        \left\| \begin{bmatrix}
             & A_{12} \\ A_{21 }&
        \end{bmatrix} \right\|
        =
        \left\| \begin{bmatrix}
            \Sigma_{21} & \\ & \Sigma_{12}
        \end{bmatrix} \right\|
        =
        \max\left\{ \|\Sigma_{12}\|, \|\Sigma_{21}\|\right\}
        =
        \max\left\{ \|A_{12}\|, \|A_{21}\|\right\}.
    \end{equation}
    Consequently, by the subadditivity of the operator norm,
    \begin{equation}
        \|A\| \leq  \left\| \begin{bmatrix}
            A_{11} & \\ & A_{22}
        \end{bmatrix} \right\| + \left\| \begin{bmatrix}
             & A_{12} \\ A_{21 }&
        \end{bmatrix} \right\|
        \leq
        \max\left\{ \|A_{11}\|, \|A_{22}\|\right\} + \max\left\{ \|A_{12}\|, \|A_{21}\|\right\}.
    \end{equation}
\end{proof}

\begin{lemma}
    \label{helper:unbalanced:ineq}
    Let $x, y \in \reals^d$, then for any $\zeta > 0$, the following inequality holds
    \begin{equation}
        \|x + y\|^2 \leq (1 + \zeta) \|x\|^2 + (1 + \zeta^{-1}) \|y\|^2.
    \end{equation}
\end{lemma}
\begin{proof}[Proof of \cref{helper:unbalanced:ineq}]
    First note that $\|x + y\|^2 = \|x\|^2 + \|y\|^2 + 2 \langle x, y \rangle$, then the proof follows by $2\langle x, y \rangle \leq \zeta \|x\|^2 + \zeta^{-1} \|y\|^2$ due to Cauchy-Schwartz inequality.
\end{proof}
\begin{lemma}
    \label{helper:3rd:Lip}
    Let $F$: $\reals^d \to \reals$ be an arbitrary twice-continuous-differentiable function that is $Q$-3rd-order-smooth. Then for any $w^1, \ldots, w^M \in \reals^d$, the following inequality holds
    \begin{equation}
        \left\|  \nabla F (\overline{w}) - \frac{1}{M} \sum_{m=1}^M \nabla F(w^m)  \right\|^2
        \leq
        \frac{Q^2}{4M}  \sum_{m=1}^M\left\| w^m - \overline{w} \right\|^4,
    \end{equation}
    where $\overline{w} := \frac{1}{M} \sum_{m=1}^M w^m$.
\end{lemma}
\begin{proof}[Proof of \cref{helper:3rd:Lip}]
    \begin{align}
             & \left\|  \frac{1}{M} \sum_{m=1}^M \nabla F(w^m) - \nabla F (\overline{w})    \right\|^2
        \\
        =    & \left\|  \frac{1}{M} \sum_{m=1}^M \left( \nabla F(w^m) - \nabla F (\overline{w}) - \nabla^2 F (\overline{w})(w^m - \overline{w}) \right)   \right\|^2
        \tag{since $\frac{1}{M} \sum_{m=1}^M w^m - \overline{w} = 0$}
        \\
        \leq & \frac{1}{M} \sum_{m=1}^M\left\|   \nabla F(w^m) - \nabla F (\overline{w}) - \nabla^2 F (\overline{w})(w^m - \overline{w}) \right\|^2
        \tag{Jensen's inequality}
        \\
        \leq & \frac{Q^2}{4M}  \sum_{m=1}^M\left\| w^m - \overline{w} \right\|^4.
        \tag{$Q$-3rd-order-smoothness}
    \end{align}
\end{proof}

\begin{lemma}
    \label{helper:diff:4th}
    Let $X$ and $Y$ be two i.i.d. $\reals^d$-valued random vectors, and assume $\expt X = 0$, $\expt \|X\|^4 \leq \sigma^4$. Then
    \begin{equation}
        \expt \|X+Y\|^2 \leq 2 \sigma^2, \quad \expt \|X+Y\|^3 \leq 4 \sigma^3, \quad \expt \|X+Y\|^4 \leq 8 \sigma^4.
    \end{equation}
\end{lemma}
\begin{proof}[Proof of \cref{helper:diff:4th}]
    The first inequality is due to $ \expt \|X+Y\|^2 = \expt \|X\|^2 +  \expt \|Y\|^2 = 2\sigma^2$ where $\expt \|X\|^2 \leq \sigma^2$ follows by applying H\"older's inequality to the assumption $\expt \|X\|^4 \leq \sigma^4$.

    The \nth{4} moment is bounded as
    \begin{align}
             & \expt \|X+Y\|^4 = \expt \left[ \|X\|^2 + \|Y\|^2 + 2 \langle X,  Y \rangle \right]^2
        \\
        =    & \expt \left[ \|X\|^4 + \|Y\|^4 + 2 \|X\|^2 \|Y\|^2 + 4 \langle X,  Y \rangle^2 + 4 \|X\|^2 \langle X,  Y \rangle  + 4\|Y\|^2 \langle X,  Y \rangle  \right]
        \\
        =    & \expt \left[ \|X\|^4 + \|Y\|^4 + 2 \|X\|^2 \|Y\|^2 + 4 \langle X,  Y \rangle^2 \right] \tag{by independence and mean-zero assumption}
        \\
        \leq & \expt \left[ 4\|X\|^4 + 4\|Y\|^4 \right] \tag{Cauchy-Schwarz inequality} \leq 8 \sigma^4.
    \end{align}

    The \nth{3} moment is bounded via Cauchy-Schwarz inequality since
    \begin{equation}
        \expt \|X+Y\|^3 \leq \sqrt{\expt \|X+Y\|^2 \expt \|X+Y\|^4} \leq 4 \sigma^3.
    \end{equation}
\end{proof}

\begin{lemma}
    \label{helper:inv:times:log}
    Let $\varphi(x) := \frac{1}{x^q} \log^p (a + b x)$, where $a, p, q \geq 1$, $b > 0$ are constants. Then suppose $a \geq \exp (p/q)$, it is the case that $\varphi(x)$ is monotonically decreasing over $(0, +\infty)$.
\end{lemma}
\begin{proof}[Proof of \cref{helper:inv:times:log}]
    Without loss of generality assume $b = 1$, otherwise we put $\psi(x) = \varphi(x/b)$ then $\psi$ has the same form (up to constants) with $b = 1$. Taking derivative for $\varphi(x) = x^{-q} \log^p (a + x)$ gives
    \begin{align}
        \varphi'(x)
        & =
        \frac{px^{-q} \log^{p -1} (a+x) }{a + x} - q x^{-q-1} \log^p (a+x)
        \\
        & =
        \frac{x^{-q-1}\log^{p-1}(a+x)}{a+ x} \left( p x - q(a + x) \log(a+x) \right).
    \end{align}
    Since $a \geq 1$ and $x > 0$ we always have $ \frac{x^{-q-1}\log^{p-1}(a+x)}{a+ x} \geq 0$. Suppose $a \geq \exp(p/q)$ then
    \begin{equation}
        px - q(a+x) \log (a+x)
        <
        px - qx \log (a)
        \leq
        px - qx \cdot \frac{p}{q}
        \leq
        0.
    \end{equation}
    Hence $\varphi'(x) < 0$ and thus $\varphi(x)$ is monotonically decreasing.
\end{proof}

\end{appendices}

\end{document}